\documentclass[10pt]{report}

\usepackage[margin=1in]{geometry}
\usepackage{setspace}
\usepackage{graphicx}
\usepackage{amsmath, amssymb}

\usepackage[numbers]{natbib}

\usepackage{hyperref}
\usepackage{url}
\hypersetup{
    colorlinks=true,
    linkcolor=blue,
    citecolor=blue,
    urlcolor=blue,
}

\usepackage{titlesec}
\titleformat{\chapter}[display]
  {\bfseries\centering\LARGE}
  {\MakeUppercase{\chaptertitlename}~\thechapter} 
  {0pt}
  {\MakeUppercase} 
\titleformat{\section}{\bfseries\large}{\thesection}{1em}{}

\usepackage{caption}
\usepackage{subcaption}
\usepackage{booktabs}
\usepackage{enumitem}

\usepackage{graphicx}
\usepackage[labelfont=bf,textfont=it]{caption}
\usepackage{tikz}
\usepackage{pgfplots}
\pgfplotsset{compat=1.17}
\usetikzlibrary{patterns} 

\usepackage{amsmath, bm, amssymb, amsfonts, amsthm, mathtools}

\usepackage{cleveref}
\usepackage{graphicx}
\allowdisplaybreaks

\usepackage{etoolbox}
\patchcmd{\thebibliography}{\section*{\refname}}{}{}{}

\usepackage{xspace}
\usepackage{todonotes}

\newtheorem{proposition}{Proposition}
\newtheorem{assumption}{Assumption}
\newtheorem{theorem}{Theorem}
\newtheorem{corollary}{Corollary}
\newtheorem{lemma}{Lemma}
\newtheorem{definition}{Definition}
\newtheorem{remark}{Remark}

\crefname{assumption}{Assumption}{Assumptions}

\usepackage{adjustbox}
\usepackage{xcolor}
\usepackage{algorithm}
\usepackage{algorithmic}

\newcommand{\norm}[1]{\left\lVert#1\right\rVert_2}
\newcommand\inner[2]{\left\langle #1, #2 \right\rangle}

\renewcommand\d[0]{\boldsymbol{d}}

\renewcommand\r[0]{\boldsymbol{r}}

\newcommand\x[0]{\boldsymbol{x}}
\newcommand\y[0]{\boldsymbol{y}}

\newcommand{\vbb}{\mathbf{v}}

\renewcommand\aa[0]{\mathbb{A}}
\newcommand\bb[0]{\mathbb{B}}
\newcommand\cc[0]{\mathbb{C}}

\newcommand\ee[0]{\mathbb{E}}

\newcommand\ii[0]{\mathbb{I}}

\newcommand\nn[0]{\mathbb{N}}

\newcommand\rr[0]{\mathbb{R}}
\renewcommand\ss[0]{\mathbb{S}}

\newcommand\zz[0]{\mathbb{Z}}

\newcommand\aaa[0]{\mathcal{A}}
\newcommand\bbb[0]{\mathcal{B}}

\newcommand\ddd[0]{\mathcal{D}}
\newcommand\eee[0]{\mathcal{E}}
\newcommand\fff[0]{\mathcal{F}}
\renewcommand\ggg[0]{\mathcal{G}}
\newcommand\hhh[0]{\mathcal{H}}
\newcommand\iii[0]{\mathcal{I}}

\newcommand\nnn[0]{\mathcal{N}}
\newcommand\ooo[0]{\mathcal{O}}
\newcommand\ppp[0]{\mathcal{P}}

\newcommand\rrr[0]{\mathcal{R}}

\newcommand\xxx[0]{\mathcal{X}}

\newcommand\zzz[0]{\mathcal{Z}}

\newcommand\rb[1]{\left(#1\right)}
\renewcommand\sb[1]{\left[#1\right]}
\newcommand\cb[1]{\left\{#1\right\}}
\newcommand\supp[1]{\text{supp}\left(#1\right)}

\usepackage[makeroom]{cancel}

\usepackage{lipsum}

\usepackage[most]{tcolorbox}

\usepackage{tikz}
\usetikzlibrary{positioning, arrows.meta, shapes, calc, patterns, decorations.pathreplacing}

\definecolor{convexblue}{RGB}{88, 138, 190}
\definecolor{thirdgreen}{RGB}{124, 180, 115}
\definecolor{quadorange}{RGB}{255, 140, 90}

\usepackage{multirow}

\newcommand{\la}{\langle}

\usepackage[algo2e,ruled]{algorithm2e} 

\begin{document}

\doublespacing

\newcommand{\thesisTitle}{What Makes Local Updates Effective: The Role of Data Heterogeneity and Smoothness}
\newcommand{\yourName}{\textbf{Kumar Kshitij Patel}}
\newcommand{\yourDept}{TOYOTA TECHNOLOGICAL INSTITUTE AT CHICAGO}
\newcommand{\yourSchool}{Chicago, Illinois}
\newcommand{\monthYear}{May, 2025}
\newcommand{\advisor}{Nathan Srebro, Lingxiao Wang}
\newcommand{\committee}{Nathan Srebro, Lingxiao Wang, Avrim Blum, Suhas Diggavi}


\begin{titlepage}
\begin{center}
\begin{tikzpicture}[remember picture,overlay]
      \node[opacity=0.1,inner sep=0pt,anchor=south west] at (current page.south west)
      {\includegraphics[width=\paperwidth,height=\paperheight]{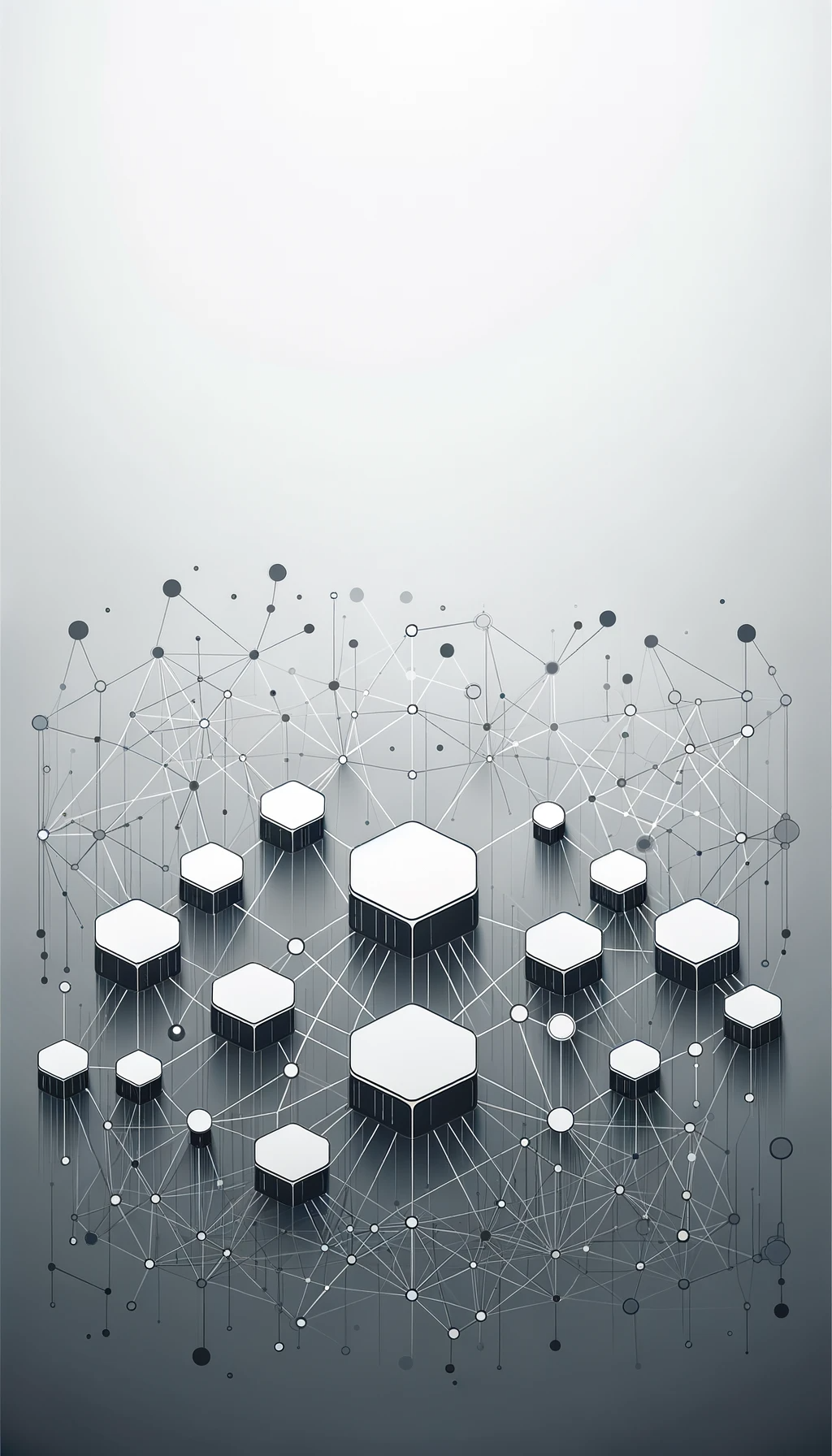}};
    \end{tikzpicture}
\begin{doublespacing}
\LARGE
\vspace{\baselineskip}
Ph.D.~Thesis\\
\vspace{\baselineskip}

\textbf{\huge{\thesisTitle}}\\

\vspace{2\baselineskip}
By\\
\yourName\\
\vspace{3\baselineskip}
Committee: \committee\\
\vspace{\baselineskip}
\LARGE
\yourDept\\
\yourSchool\\
\monthYear

\vfill

\end{doublespacing}

\end{center}
\end{titlepage}

\pagenumbering{roman}
\tableofcontents
\clearpage

\pagenumbering{arabic}

\chapter{Introduction}

Humans have an exceptional ability to quickly learn new tasks by recognizing patterns and integrating prior knowledge through shared representations in the pre-frontal cortex~\cite{tomasello1997primate,kanwisher2010functional,botvinick2015reinforcement,collins2013cognitive,duncan2010multiple,botvinick2014computational,sorscher2022neural}. From the early days of artificial intelligence (AI), researchers have sought to replicate this multi-task learning capability~\cite{samuel1959some,minsky1961steps,newell1972human,schmidhuber1987evolutionary,thrun1995learning,caruana1997multitask,baxter2000model,ruder2017overview,finn2017model}, yet most recent breakthroughs have come from improving learning performance on single tasks by scaling up models, datasets, and computation~\cite{brown2020language,sutton2019bitter,kaplan2020scaling}. We are now at a critical juncture where further scaling faces multiple challenges. Acquiring high-quality data is increasingly complex due to privacy regulations and intellectual property concerns~\cite{gdpr,villalobosposition,appel2023generative}. Worse still, collecting large centralized datasets in fields like healthcare and drug discovery remains infeasible altogether~\cite{ching2018opportunities,clark2021machinelearning,prakash2023chemicalspace}. Moreover, the resources needed to operate at scale---amidst diminishing returns---have concentrated power within a few companies, slowing down the pace of research~\cite{bommasani2021opportunities,bender2021dangers,ahmed2020democratization,zuboff2019age,besiroglu2024compute,thompson2020computational,powerai,gholami2024ai}. Compounding this, current AI systems frequently exhibit bias, poor robustness, and struggle to generalize across distribution shifts~\cite{bender2021dangers,selbst2019fairness,szegedy2013intriguing,hendrycks2019benchmarking,taori2020measuring}.

Federated learning (FL)~\citep{mcmahan2016federated, mcmahan2016communication, kairouz2019advances} has emerged as a transformative approach for mitigating several of the challenges above by enabling decentralized training that preserves data privacy, complies with regulatory constraints, and effectively handles diverse data distributions. FL successfully combines the advantages of multi-task learning with secure, collaborative training. Its practical impact is evident across various domains: in healthcare, facilitating drug discovery and medical imaging analyses~\citep{chen2020fl, li2019privacy, kaissis2021end}, notably in collaborative efforts during the COVID-19 pandemic~\citep{dayan2021federated, powell2019nvidia, roth2020federated}; in mobile technologies, enhancing smart assistants and predictive keyboards~\citep{mcmahan_ramage_2017, apple, paulik2021federated}; and in finance, improving detection and prevention of financial crimes~\citep{shiffman_zarate_deshpande_yeluri_peiravi_2021}.

Research interest in FL has also surged dramatically---from a handful of studies in 2016~\cite{konevcny2016federated, mcmahan_ramage_2017} to over 3000 publications in 2020 alone~\citep{kairouz2019advances, wang2021field}. Despite its rapid adoption, however, fundamental questions remain unresolved, ranging from theoretical inquiries about FL’s core objectives to practical concerns such as incentivizing the participation of rational agents. Recent vulnerabilities and attacks highlight critical security and privacy risks inherent in FL systems~\citep{liu2022threats}. Furthermore, competing methodologies, such as data escrows and centralized foundation models, have emerged as viable alternatives~\cite{angel2024federated, xia2023data}. This landscape highlights the pressing need for thorough theoretical and practical examinations of federated learning algorithms.

Central to federated learning is \textbf{Local SGD} (a.k.a. Federated Averaging), arguably its most popular optimization algorithm~\citep{mcmahan_ramage_2017}. Local SGD involves each machine performing $K$ local stochastic updates before averaging the resulting models with other machines at each communication round (see update~\eqref{eq:local_updates}). Despite its simplicity, Local SGD consistently outperforms alternative first-order methods, including mini-batch SGD~\citep{charles2021large, wang2022unreasonable, dekel2012optimal, lin2018don, woodworth2020local}. While numerous SGD variants have been proposed for FL~\citep{kairouz2019advances, wang2021field}, most rely fundamentally on local updates.

This robust empirical performance has inspired a significant body of theoretical work aimed at explaining the benefits of local updates~\citep{mcdonald2009efficient, zinkevich2010parallelized, zhang2016parallel, stich2018local, dieuleveut2019communication, khaled2020tighter, koloskova2020unified, woodworth2020local, karimireddy2020scaffold, woodworth2020minibatch, yuan2020federated, mishchenko2022proxskip, woodworth2021min, glasgow2022sharp, wang2022unreasonable}.

This Ph.D.\ thesis contributes to this growing literature by developing a unified and rigorous theoretical account of local update algorithms---especially Local SGD---under realistic models of data heterogeneity. It provides both a conceptual framework for reasoning about heterogeneity in federated settings and a set of technical results---including novel upper and lower bounds for non-asymptotic convergence rates---that characterize the strengths and limitations of local updates across a range of problem settings.

While Local SGD has been widely studied, many existing analyses rely on restrictive assumptions. This thesis relaxes such assumptions by introducing a framework that incorporates higher-order notions of heterogeneity and smoothness. The resulting theory clarifies trade-offs between optimization accuracy, communication efficiency, and algorithmic bias. It systematically characterizes the convergence behavior of Local SGD across convex, non-convex, and online settings---each with distinct challenges---while integrating insights from fixed-point theory, minimax optimization complexity, and algorithmic design.

\section{Contributions of this Thesis}

The main contributions are summarized below:

\begin{itemize}
    \item \textbf{A heterogeneity-aware min-max complexity framework.}  
    This thesis introduces a unified framework for analyzing federated optimization that disentangles different sources of heterogeneity using first- and second-order measures. These refined assumptions clarify not only when local updates succeed, but also \emph{why}---offering more interpretable alternatives to commonly used bounded-gradient or dissimilarity conditions. The framework also supports a finer-grained complexity theory based on min-max optimality, used throughout the thesis.

    \item \textbf{The central role of second-order heterogeneity.}  
    Across both convex and non-convex settings, our results show that small second-order heterogeneity is necessary and sufficient for local update algorithms to outperform centralized or mini-batch methods. This theme unifies lower and upper bounds and illustrates how higher-order structure enables provable gains from local computation---even when first-order assumptions are insufficient.

    \item \textbf{New insights into the limits of local updates.}  
    In regimes where local updates fail to yield benefits, we identify and prove the min-max optimality of classical algorithms such as mini-batch SGD. These results clarify the boundary between problem settings where local computation is provably effective and those where centralized algorithms remain optimal.

    \item \textbf{Extending fixed-point and implicit bias perspectives.}  
    We revisit the fixed-point behavior of Local SGD under quadratic objectives and connect it to heterogeneity-aware bias and conditioning. These results uncover an implicit regularization effect induced by local updates and situate the fixed-point behavior within broader discussions on optimization geometry and bias.

    \item \textbf{Third-order smoothness improves convergence under heterogeneity.}  
    We show that the known benefits of third-order smoothness in homogeneous settings extend to heterogeneous regimes. Our consensus-error-based analysis framework captures these effects, enabling tighter finite-time bounds by simultaneously leveraging smoothness and heterogeneity structure.

    \item \textbf{A theory of federated online optimization.}  
    Moving beyond static data distributions, we develop a theory for sequential decision-making in federated environments. Our results delineate the limits of collaboration under full-information feedback and show that local updates can provably improve regret under bandit feedback in high-dimensional or low-heterogeneity regimes.
\end{itemize}

Collectively, these contributions offer a principled, granular understanding of local update algorithms across diverse optimization settings. Grounded in practical models of heterogeneity, the results advance both theory and practice toward the reliable, efficient deployment of federated systems.

\section{An Outline of the Results in this Thesis}

The technical portion of this thesis begins in \Cref{ch:setting}, which introduces the formal setup and structural assumptions used throughout. This includes the oracle and communication models, first- and second-order heterogeneity measures, and the notion of min-max optimality that anchors our analyses. While this chapter primarily serves as groundwork, it also develops conceptual insights, such as the limitations of first-order heterogeneity assumptions and the importance of alignment between the geometries of different machines (\Cref{prop:barB_vs_B}, \Cref{sec:is_first_enough}).

\Cref{ch:baselines_and_lower_bounds} establishes new lower bounds for distributed algorithms under varying heterogeneity assumptions, based on our work in~\cite{patel2024limits,patel2025revisiting}. Key results include a tight lower bound for Local SGD under bounded first-order heterogeneity (\Cref{thm:new_LSGD_lower_bound}); an algorithm-independent lower bound that identifies mini-batch SGD as min-max optimal in this regime (\Cref{thm:AIlb_zeta0}); and a lower bound demonstrating that small second-order heterogeneity can enable Local SGD to outperform centralized methods (\Cref{thm:new_LSGD_lower_bound_with_tau}). These results clarify the precise conditions under which local updates offer provable gains, and set expectations for the upper bounds derived in later chapters.

The next two chapters provide complementary perspectives on the convergence of Local SGD.

\Cref{ch:fixed_point} analyzes the algorithm’s limiting behavior under quadratic objectives by characterizing its fixed point, drawing on results from~\cite{patel2024limits,patel2025revisiting}. We derive new closed-form expressions for the fixed-point discrepancy (\Cref{lem:fixed_disc_UB}) and establish a finite-time convergence bound (\Cref{thm:conv_with_fixed_pt_pers}) that quantifies trade-offs between optimization error, variance, and heterogeneity-induced bias. While prior works (e.g., \cite{charles2020outsized,malinovskiy2020local}) have explored fixed-point analyses, our main contribution is integrating this perspective with a heterogeneity-aware view. We also extend the analysis to general convex objectives in \Cref{sec:fixed_derivation_convex}, showing that the fixed point corresponds to a reweighted least-squares solution, revealing an implicit regularization effect induced by local updates, reconciling optimization theory with other existing results~\cite{gu2023and}.

\Cref{ch:upper_bounds} introduces a new finite-time convergence analysis based on consensus error, improving on previous results through tighter recursions and relaxed assumptions~\cite{patel2025revisiting}. We first obtain sharper bounds under third-order smoothness and second-order heterogeneity (\Cref{thm:UB_convex_w_Q,thm:UB_Sconvex_w_Q_zeta,thm:UB_Sconvex_func_w_Q_zeta}), then derive results under relaxed first-order heterogeneity using coupled recursions (\Cref{thm:UB_Sconvex_wo_Q,thm:UB_function_wo_Q,thm:UB_Sconvex_quad,thm:UB_function_quad}), and finally combine these to give the most general result in \Cref{thm:UB_Sconvex_w_Q}. This requires delicate control over higher-order terms, including fourth-moment bounds on consensus error. Because the methods in this chapter form a core technical contribution of the thesis, \Cref{app:chap5} includes a self-contained tutorial on consensus error–based analyses.

\Cref{ch:non_convex} generalizes these insights to the non-convex setting by analyzing \textsc{CE-LSGD}, a new communication efficient, variance-reduced variant of Local SGD introduced in~\cite{patel2022towards}. We prove that \textsc{CE-LSGD} is minimax optimal under deterministic oracles and nearly optimal under stochastic ones (\Cref{thm:alg1,thm:alg1_full_main}). We also present new lower bounds for both zero-respecting distributed algorithms and centralized methods (\Cref{thm:lb,thm:lb_cent}), showing that small second-order heterogeneity is again essential for local updates to offer improvements. This chapter also presents an alternative perspective frequently used in distributed optimization: analyzing the trade-off between oracle complexity and communication cost in high-accuracy regimes, with implications for overparameterized deep learning.

Finally, \Cref{ch:online} develops a theory of federated online optimization to address dynamic environments where data evolves over time, based on~\cite{patel2023federated}. We first show that collaboration offers no benefit under full-gradient feedback (\Cref{thm:first_lip,thm:first_smth}), connecting stochastic and adversarial models via a unified minimax regret framework. We then propose two new bandit algorithms, \textsc{FedPOSGD} and \textsc{FedOSGD}, and prove they achieve strictly better regret under zeroth-order feedback (\Cref{thm:favlb,thm:bd_grad_first_stoch,thm:smooth_first_stoch}), particularly in high-dimensional or low-heterogeneity regimes. These results extend the reach of local update algorithms to real-world, sequential learning settings.

Taken together, the results in this thesis establish a unified and principled theory of local update methods across convex, non-convex, and online optimization, grounded in realistic models of data heterogeneity. The thesis contributes new lower bounds, improved upper bounds, and new algorithms, all of which have implications for both the theory and practice of federated optimization. 

\section{Relevant References for this Thesis}

All of the results in this thesis are derived from four papers for which I am the primary author:

\begin{enumerate}
    \item \citet{patel2022towards}, NeurIPS'22 (with Lingxiao Wang, Blake Woodworth, Brian Bullins, Nathan Srebro);
    \item \citet{patel2023federated}, ICML'23 (with Lingxiao Wang, Aadirupa Saha, Nathan Srebro);
    \item \citet{patel2024limits}, COLT'25 (with Margalit Glasgow, Ali Zindari, Lingxiao Wang, Sebastian U.\ Stich, Ziheng Cheng, Nirmit Joshi, Nathan Srebro);
    \item \citet{patel2025revisiting}, under review (with Ali Zindari, Lingxiao Wang, Sebastian U.\ Stich).
\end{enumerate}

The observations in \Cref{ch:setting}---in particular \Cref{prop:barB_vs_B} and \Cref{sec:is_first_enough}---the lower bounds in \Cref{ch:baselines_and_lower_bounds} (\Cref{thm:new_LSGD_lower_bound,thm:AIlb_zeta0,thm:new_LSGD_lower_bound_with_tau}), the fixed-point analysis in \Cref{ch:fixed_point} (\Cref{thm:conv_with_fixed_pt_pers}), and the consensus-error framework developed in \Cref{ch:upper_bounds} (\Cref{thm:UB_Sconvex_w_Q_zeta,thm:UB_Sconvex_func_w_Q_zeta,thm:UB_convex_w_Q,thm:UB_Sconvex_wo_Q,thm:UB_function_wo_Q,thm:UB_Sconvex_quad,thm:UB_function_quad,thm:UB_Sconvex_w_Q}) are drawn from the last two papers~\cite{patel2024limits,patel2025revisiting}. 

The analysis of the non-convex setting in \Cref{ch:non_convex}---including the development of \textsc{CE-LSGD} and associated upper and lower bounds (\Cref{thm:alg1,thm:alg1_full_main,thm:lb,thm:lb_cent})---comes from the first paper~\cite{patel2022towards}, which also contains additional results on partial participation and higher-order algorithms that fall outside the scope of this thesis. 

Finally, all algorithms and theoretical results in the online setting presented in \Cref{ch:online} (\Cref{thm:first_lip,thm:first_smth,thm:favlb,thm:bd_grad_first_stoch,thm:smooth_first_stoch}) are drawn from the second paper~\cite{patel2023federated}.

\paragraph{Beyond this Thesis.} I have also co-authored earlier works on Local SGD~\cite{woodworth2020local,woodworth2020minibatch,lin2018don,dieuleveut2019communication}, which, while influential to my thinking, are not included here. Several other works---e.g.,~\cite{stich2018local,karimireddy2020scaffold,woodworth2018graph,koloskova2020unified,glasgow2022sharp,yuan2020federated,wang2022unreasonable,charles2020outsized,murata2021bias,karimireddy2020mime,khaled2020tighter,arjevani2015communication,arjevani2019lower,shamir2014communication,shamir2014dane,duchi2015optimal,shamir2017optimal}---have also provided key intellectual foundations for the analyses in this thesis.

The study of local updates remains a rapidly evolving area, increasingly relevant to challenges around data ownership, privacy, and communication. I hope the tools and results presented here contribute meaningfully to this growing field.

\chapter{Formal Setting and Some Preliminary Observations}\label{ch:setting}

This chapter lays the formal foundation for the rest of the thesis by introducing the distributed optimization framework, the oracle model, the key structural assumptions commonly used in the analysis of local update algorithms, and the notion of min-max complexity (defined in \eqref{eq:min_max_optimality}) that will serve as a focal point in the following three chapters. While much of the material functions as technical setup, we interleave critical observations---such as \Cref{prop:barB_vs_B}---that illuminate the practical and theoretical implications of these assumptions. In particular, we examine various notions of data heterogeneity, ranging from first- and second-order heterogeneity to divergences in local optima across machines, and we clarify how these different forms of heterogeneity affect algorithmic behavior. Our discussion highlights both the strengths and limitations of standard assumptions in federated learning---for example, in \Cref{sec:is_first_enough}---and motivates the introduction of more nuanced complexity measures. Together, these components not only provide a formal foundation but also a conceptual framework that informs the lower bounds, analyses, and algorithms developed in the subsequent chapters.

\subsection*{Outline and Important References}
This chapter introduces the formal setup and assumptions that underlie the rest of the thesis. Most of the definitions and concepts presented are standard in the distributed optimization literature~\citep{stich2018local,karimireddy2020scaffold,karimireddy2020mime,karimireddy2019error,khaled2020tighter,woodworth2020local,woodworth2020minibatch,koloskova2020unified}. \Cref{sec:ch2.1} presents the consensus optimization formulation and the intermittent communication model, which were formalized in the graph oracle framework by~\citet{woodworth2018graph}. \Cref{sec:ch2.2,sec:ch2.3,sec:ch2.4} introduce standard assumptions from convex optimization, such as convexity, strong convexity, and various notions of smoothness.

In \Cref{sec:ch2.5}, we turn to data heterogeneity assumptions. While many of the assumptions are drawn from prior work, this section includes novel commentary on their relationships, limitations, and practical interpretations. These insights originate from our papers~\citep{patel2024limits,patel2025revisiting}, co-authored with Margalit Glasgow, Ali Zindari, Lingxiao Wang, Sebastian U.\ Stich, Ziheng Cheng, Nirmit Joshi, and Nathan Srebro. \Cref{sec:ch2.6} discusses the notion of zero-respecting algorithms, introduced in the optimization lower bounds literature~\citep{arjevani2015communication,arjevani2019lower} and extended to the federated setting in our prior work~\citep{patel2022towards}. Finally, \Cref{sec:min_max} revisits the classical concept of min-max optimality~\citep{nesterov2018lectures} in the context of our distributed learning setup and serves as a foundation for the convergence results in subsequent chapters.

\section{Federated Optimization with Intermittent Communication}\label{sec:ch2.1}
Federated learning is a form of distributed optimization characterized by training data that is dispersed across numerous devices---potentially millions---instead of being centralized in a single location. Each device typically holds data drawn from distinct, heterogeneous distributions, creating significant variability across the network. Additionally, privacy considerations are paramount, as devices often prefer not to share their raw data directly. Compounding these constraints, communication among devices is infrequent and limited, necessitating algorithms that efficiently aggregate learning updates under these challenging conditions. Our aim in this section is to formally define a mathematical model which can incorporate these characteristics of FL.

The most commonly studied optimization objective assuming $M$ machines is the following,
\begin{align}\label{prob:scalar}
    \min_{x\in \rr^d}\rb{F(x):=\frac{1}{M}\sum_{m\in[M]}F_m(x)}\enspace,
\end{align}
where $F_m := \ee_{z_m \sim \ddd_m}[f(x;z_m)]$ is a stochastic objective on machine $m$, defined using a loss function $f(\cdot;z\in \zzz)\in \fff$ and a data distribution $\ddd_m\in \Delta(\zzz)$. Problem \eqref{prob:scalar} is ubiquitous in machine learning---from training in a data center on multiple GPUs \citep{krizhevsky2012imagenet}, to decentralized training on millions of devices \citep{mcmahan2016federated, mcmahan2016communication}. Notably, objective \eqref{prob:scalar} aims to find a single consensus model for the $M$ different objectives. When data heterogeneity is very high, solving this problem might only be the first step, followed by personalization on each machine.

We also need to define a communication model for optimizing Problem \eqref{prob:scalar}. Perhaps the simplest, most basic, and most important distributed setting is that of intermittent communication (IC) \citep{woodworth2018graph}, where $M$ machines work in parallel over $R$ communication rounds to optimize objective \eqref{prob:scalar}, and during each round of communication, each machine may sequentially compute $K$ oracle calls (such as stochastic gradients). See Figure \ref{fig:IC} for an illustration of the IC framework. 

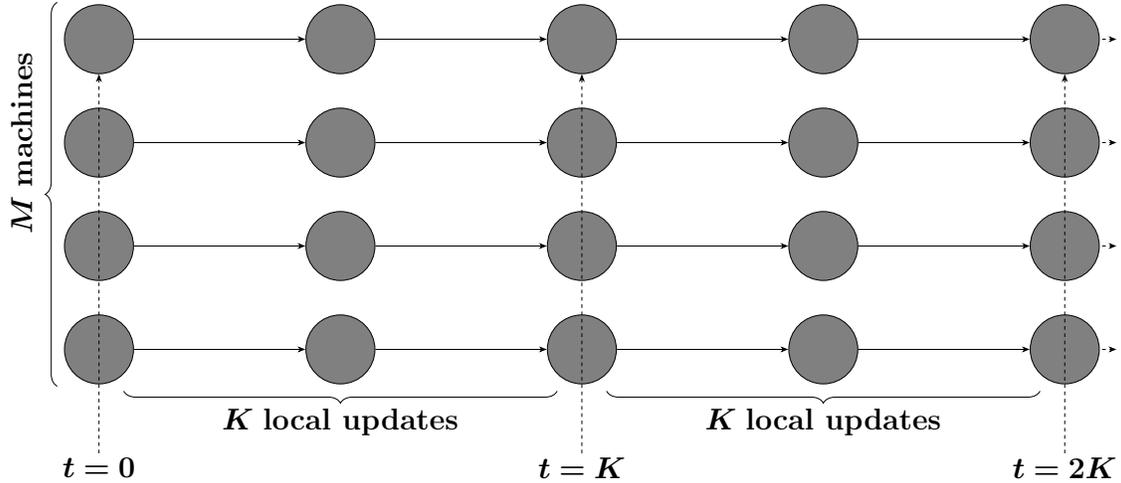
\begin{figure}
    \centering
    \resizebox{0.9\linewidth}{!}{

\begin{tikzpicture}[
    node distance=2cm and 3.5cm,
    line/.style={-Stealth, thick},
    node/.style={circle, draw, thick, fill=gray, minimum size=2cm},
    dashedline/.style={dashed, thick},
    timeline/.style={dashed, thick, -Stealth},
    continueline/.style={dashed, thick, -Stealth, shorten <=2pt},
    brace/.style={decorate, decoration={brace, amplitude=10pt, raise=4pt}, thick}
]

\foreach \x in {1,...,5} {
    \foreach \y in {1,...,4} {
        \node[node] (N-\x-\y) at (\x*7, -\y*3) {};
    }
}

\foreach \x [remember=\x as \lastx (initially 1)] in {2,...,5} {
    \foreach \y in {1,...,4} {
        \draw[line] (N-\lastx-\y) -- (N-\x-\y);
    }
}


\draw[brace] ([yshift=-10pt, xshift=-10pt]N-1-4.south west) -- ([yshift=10pt, xshift=-10pt]N-1-1.north west);
\node[left=0.8cm of N-1-2, anchor=south, rotate=90, align=center] {\Huge $\bm{M}$ \textbf{machines}}; 



\node[below=of N-1-4] (t0) {\Huge $\bm{t=0}$};
\draw[timeline] (t0) -- (N-1-1);

\draw[brace, decoration={mirror}] ([yshift=-10pt]N-1-4.south east) -- ([yshift=-10pt]N-3-4.south west) node[midway, below=15pt, align=center] {\Huge$\bm{K}$ \textbf{local updates}};

\node[below=of N-3-4] (t2) {\Huge $\bm{t=K}$};
\draw[timeline] (t2) -- (N-3-1);

\draw[brace, decoration={mirror}] ([yshift=-10pt]N-3-4.south east) -- ([yshift=-10pt]N-5-4.south west) node[midway, below=15pt, align=center] {\Huge$\bm{K}$ \textbf{local updates}};

\node[below=of N-5-4] (t4) {\Huge $\bm{t=2K}$};
\draw[timeline] (t4) -- (N-5-1);

\foreach \y in {1,...,4} {
    \draw[continueline] (N-5-\y) -- ++(1.5, 0);
}


\end{tikzpicture}
}
    \caption{Illustration of the intermittent communication setting.}
    \label{fig:IC}
\end{figure}

Having fixed our communication model, an instance of problem \eqref{prob:scalar} can be characterized by the client distributions $\{\ddd_m\in \Delta(\zzz)\}_{m\in[M]}$ and a differentiable loss function $f(\cdot;z\in\zzz):\rr^d\to \rr$ and assume it belongs to some function class $f\in\fff$. With this we can denote the set of all problem instances by $\ppp\in \Delta(\zzz)^{\otimes M}\times \fff$. In the rest of this section, we will define different restrictions on $\fff$ and distributions $\cb{\ddd_1, \dots, \ddd_M}$, which would lead to interesting sub-problem classes of $\ppp$.

\section{Restrictions on the Function Class \texorpdfstring{$\fff$}{TEXT} and Data Distributions \texorpdfstring{$\cb{\ddd_m}_{m\in[M]}$}{TEXT}}\label{sec:ch2.2}
Throughout this thesis, we assume that the loss function $f(\cdot; z)$ is \textbf{differentiable} for all $z \in \mathcal{Z}$, and that for each fixed $x \in \mathbb{R}^d$, the map $z \mapsto f(x; z)$ is \textbf{measurable}. Most of our analyses can also be extended to sub-differentiable functions using standard techniques from the optimization literature. We omit discussing these generalizations here, to keep the focus on the novel analysis techniques we develop. We will now state the regularity assumptions we make for different machines' objectives $F_m$'s, while noting that these assumptions on the machines together also imply the regularity of the average objective $F$.


Most of our analyses will assume that the objective function on each machine is convex.
\begin{assumption}[Convexity]\label{ass:convex}
    For all machines $m \in [M]$, the function $F_m(\cdot)$ satisfies,
    \begin{align*}
        F_m(x) + \inner{\nabla F_m(x)}{y-x} &\leq F_m(y)\enspace, && \forall\ x,\ y\ \in \rr^d\enspace.
    \end{align*}
    When $F_m(\cdot)$ is twice differentiable, this is equivalent to assuming $0 \preceq \nabla^2 F_m(\cdot)$.
\end{assumption}
 
\begin{remark}[Restrictions on $f$ and $\ddd_m$]
The convexity of $F_m$ depends on the choice of the loss function $f$ and the data distribution $\mathcal{D}_m$. In particular, if the loss $f(\cdot; z)$ is convex for all $z \in \operatorname{supp}(\mathcal{D}_m)$, then convexity is preserved due to linearity of expectation. On the other hand, it is possible for $f(\cdot; z)$ to be non-convex for some $z \in \operatorname{supp}(\mathcal{D}_m)$, while $F_m$ remains convex due to averaging effects. We directly assume the convexity of $F_m$ to abstract away these intricacies.
\end{remark}
We will also assume that the objective functions are strongly convex for some of our analyses.
\begin{assumption}[Strong Convexity]\label{ass:strongly_convex}
    For all machines $m \in [M]$, the function $F_m(\cdot)$ satisfies,
    \begin{align*}
        F_m(x) + \inner{\nabla F_m(x)}{y-x} + \frac{\mu}{2}\norm{x-y}^2 &\leq F_m(y)\enspace, && \forall\ x,\ y\ \in \rr^d\enspace.
    \end{align*}
    When $F_m(\cdot)$ is twice differentiable, this is equivalent to assuming $\mu\cdot I_d \preceq \nabla^2 F_m(\cdot)$.
\end{assumption}
\begin{remark}
Strong convexity implies that each function $F_m$ admits a unique minimizer, which we denote by $x_m^\star$. Moreover, strong convexity guarantees that the function $F_m$ grows quadratically away from its minimum:
\[
F_m(x) \geq F_m(x_m^\star) + \frac{\mu}{2} \|x - x_m^\star\|^2, \quad \forall x \in \mathbb{R}^d \enspace.
\]
This inequality is useful for bounding distances to the optimum in terms of sub-optimality. One way to guarantee strong convexity is to assume that the loss function is strongly convex. This can be ensured, for instance, by adding a regularizer $\frac{\mu}{2}\norm{\cdot}^2$ to the loss function. When $\ee_{z\sim\ddd_m}\sb{\norm{\nabla f(\cdot; z)}}<\infty$, we can interchange expectation and the gradients, then to ensure that $F_m$'s are also $\mu$-strongly convex.  
\end{remark}
Our analyses will also rely on a smoothness assumption on the objective function. In particular, we consider Lipschitzness of the function, the gradients, and the Hessian.
\begin{assumption}[Bounded Gradients]\label{ass:bounded_gradients}
    For all machines $m \in [M]$, the function $\nabla F_m(\cdot)$ satisfies for some scalar $G>0$,
    \begin{align*}
        \norm{\nabla F_m(\cdot)} \leq G\enspace,
    \end{align*}
    which is equivalent to assuming Lipschitzness of the function, i.e.,
    \begin{align*}
        \norm{F_m(x) - F_m(y)} &\leq G\norm{x-y}\enspace, &&\forall x,\ y\ \in \rr^d\enspace.
    \end{align*}
\end{assumption}
We will only need this assumption when we analyze the online adversarial setting. For analyzing problem \eqref{prob:scalar}, we will only rely on the following assumption about the Lipschitzness of the gradients.
\begin{assumption}[Second-order Smoothness]\label{ass:smooth_second}
    For all machines $m \in [M]$, the function $F_m(\cdot)$ satisfies for some scalar $H>0$,
    \begin{align*}
         F_m(y) &\leq F_m(x) + \inner{\nabla F_m(x)}{y-x} + \frac{H}{2}\norm{x-y}^2 \enspace, && \forall\ x,\ y\ \in \rr^d\enspace,
    \end{align*}
    which is equivalent to assuming Lipschitzness of the gradients, i.e.,
    \begin{align*}
        \norm{\nabla F_m(x) - \nabla F_m(y)} &\leq H\norm{x-y}\enspace, &&\forall x,\ y\ \in \rr^d\enspace.
    \end{align*}
    When $F_m(\cdot)$ is twice differentiable these conditions are equivalent to assuming that $\nabla^2 F_m(\cdot)\preceq H\cdot I_d$ or that the spectral norm of $\nabla^2 F_m(\cdot)$ is bounded, i.e., $\norm{\nabla^2F_m(\cdot)} \leq H$.
\end{assumption}
\begin{remark}[Self Bounding Property]\label{rem:self_bounding}
Second-order smoothness also implies a useful bound on the gradient norm in terms of function sub-optimality. In particular, for any $x \in \mathbb{R}^d$, we have:
\[
\|\nabla F_m(x)\|^2 \leq 2H \big(F_m(x) - F_m(x_m^\star)\big) \enspace,
\]
where $x_m^\star := \arg\min_x F_m(x)$. This inequality can be derived by applying the smoothness condition with the choice
\[
y = x - \frac{1}{H} \nabla F_m(x) \enspace,
\]
which yields
\[
F_m\left(x - \frac{1}{H} \nabla F_m(x)\right) \leq F_m(x) - \frac{1}{2H} \|\nabla F_m(x)\|^2 \enspace.
\]
Since $F_m(x_m^\star) \leq F_m\left(x - \frac{1}{H} \nabla F_m(x)\right)$, rearranging gives the desired bound. This inequality is widely used in optimization to relate stationarity to function sub-optimality under second-order smoothness.
\end{remark}
We often refer to the above assumption as ``smoothness'' or second-order smoothness, where additional context is required to disambiguate it from the following assumption.

\begin{assumption}[Third-order Smoothness]\label{ass:smooth_third}
    For all machines $m \in [M]$, the function $F_m(\cdot)$ is twice continuously differentiable and satisfies for some scalar $Q>0$,
    \begin{align*}
        \norm{\nabla^2 F_m(x) - \nabla^2 F_m(y)} &\leq Q\norm{x-y}\enspace, &&\forall x,\ y\ \in \rr^d\enspace.
    \end{align*}
    When $F_m(\cdot)$ is thrice-differentiable, this is equivalent to assuming,
\begin{align*}
    \left|\frac{\partial^3 F_m(x)}{\partial x_i \partial x_j \partial x_k}\right| \leq Q, 
    \quad \forall x \in \mathbb{R}^d,\;\forall i,j,k\in[d]\enspace.
\end{align*} 
\end{assumption}

We will now consider two canonical loss functions used in machine learning to illustrate the above assumptions.

\subsection{Example 1. Regression with Square Loss}\label{subsec:square_loss}
Assume $\mathcal{Z} = \mathbb{R}^{d+1}$ and all data points are covariate-label pairs, $z = (a, y)$. Then we can define the square loss function as
\begin{align}\label{eg:square_loss}
    f^{\text{square}}\rb{x;(a,y)} &:= \frac{1}{2}\rb{\inner{a}{x}-y}^2\enspace, &&\forall\ a,\ x\ \in \mathbb{R}^{d},\ y\ \in\mathbb{R}\enspace.
\end{align}
Note that the square loss is twice differentiable, and
\begin{align*}
    \nabla_x f^{\text{square}}\rb{x;(a,y)} &= a\rb{\inner{a}{x}-y}\enspace,\\
    \nabla^2_x f^{\text{square}}\rb{x;(a,y)} &= aa^T\enspace,\\
    \nabla^3_x f^{\text{square}}\rb{x;(a,y)} &= 0\enspace,
\end{align*}
which implies that for any $a$, $0\preceq \nabla^2 f^{\text{square}}\rb{x;(a,y)} \preceq \norm{a}^2\cdot I_d$. Thus, the square loss always satisfies \Cref{ass:convex,ass:smooth_third} (with $Q=0$), and can be made to satisfy \Cref{ass:strongly_convex,ass:smooth_second} by imposing suitable conditions on the distribution of $a$. On the other hand, without assuming a bounded domain for $x$, i.e., an upper bound on $\norm{x}$, the square loss does not satisfy \Cref{ass:bounded_gradients}\footnote{While there are relaxations of \Cref{ass:bounded_gradients} that the square loss can satisfy, such as bounding gradients near the optimizer along with a quadratic upper bound elsewhere, we do not pursue those here. Variants like the Huber loss are designed to satisfy \Cref{ass:bounded_gradients} and are useful in robust regression.}.

\subsection{Example 2. Classification with Logistic Loss}
Assume $\mathcal{Z} = \mathbb{R}^{d+1}$ and all data points are covariate-label pairs, $z = (a, y)$ with binary labels. Then we can define the logistic loss function as
\begin{align}\label{eg:logistic_loss}
f^{\text{logistic}}\left(x;(a,y)\right) &:= \log\left(1 + \exp\left(y\langle a, x\rangle\right)\right), &&\forall\ a, x \in \mathbb{R}^{d},\ y \in \{-1,1\}\enspace.
\end{align}
The logistic loss is infinitely differentiable, and
\begin{align*}
\nabla_x f^{\text{logistic}}\left(x;(a,y)\right) &= \frac{-y a}{1 + \exp(y\langle a, x \rangle)}\enspace,\\
\nabla_x^2 f^{\text{logistic}}\left(x;(a,y)\right) &= \frac{\exp(y\langle a, x \rangle)}{\left(1 + \exp(y\langle a, x \rangle)\right)^2}aa^T\enspace,\\
\nabla_x^3 f^{\text{logistic}}\left(x;(a,y)\right) &= \frac{y\exp(y\langle a, x \rangle)\left(1 - \exp(y\langle a, x \rangle)\right)}{\left(1 + \exp(y\langle a, x \rangle)\right)^3}a\otimes a\otimes a\enspace,
\end{align*}
where $\otimes$ denotes the tensor (outer) product, generalizing the outer product of vectors to higher-order tensors. For vectors $u, v, w \in \mathbb{R}^d$, the tensor product $u \otimes v \otimes w$ results in a third-order tensor with entries $(u \otimes v \otimes w)_{ijk} = u_i v_j w_k$. This notation compactly expresses higher-order derivatives. In particular, logistic loss always satisfies \Cref{ass:convex} and can be made to satisfy \Cref{ass:strongly_convex} by imposing assumptions on the distribution of covariates $a$. The norm of its gradient is bounded by $\norm{a}$, and the spectral norm of its Hessian is bounded by $\norm{a}^2$. Since the third-order derivative tensor has operator norm at most $\norm{a}^3$, the third-order smoothness constant can also be bounded accordingly. Hence, if $\norm{a} \leq C$ with high probability, then \Cref{ass:bounded_gradients,ass:smooth_second,ass:smooth_third} are satisfied with $L \propto C$, $H \propto C^2$, and $Q \propto C^3$ respectively. This shows that the logistic loss is “smoother” than square loss for large values of $x$.

\begin{remark}
Given the above two loss functions, it is tempting to make regularity assumptions directly on the loss function $f$ instead of the objectives $F_m$'s. Indeed, in practice, most loss functions encountered in learning problems are not merely convex in expectation (over data) but are individually convex and smooth for each sample. Nonetheless, assuming strong convexity and smoothness only at the level of the expected objective can sometimes lead to tighter constants in theoretical bounds.

Finally, there is a growing body of work that leverages heterogeneity in regularity across machines---for example, using importance sampling strategies to weight gradient updates differently---but such techniques are beyond the scope of this thesis.
\end{remark}

\section{Restrictions on the Oracle Model}\label{sec:ch2.3}
The oracle framework is a very common abstraction in optimization literature \citep{nesterov2018lectures, nemirovski1994efficient, woodworth2018graph}, and an oracle call can be seen as a unit of information and/or computation. This is especially useful when providing lower-bound results. Specifically, recall that in the intermittent communication model, the machines communicate for $R$ communication rounds with $K$ time steps in between. We denote the total time horizon by $T:=KR$. Then at each time step $t\in[0,T-1]$, machine $m\in[M]$ queries its oracle $\ooo^m$ with a point $x_t^m\in\rr^d$ which returns an output $\ooo^m(x_t^m)$. The most common oracle studied in stochastic optimization is a stochastic first-order oracle which can be used to implement different distributed variants of SGD.

\begin{definition}[Stochastic First-order Oracle]\label{def:oracle_first}
For all $m\in[M]$, machine $m\in[M]$ is equipped with an oracle $\ooo_m:\rr^d\times \Delta(\zzz)\to\rr\times\rr^d$, such that for any $x\in\rr^d$ the oracle samples a \textit{random datum} $z\sim \ddd_m$ and outputs $\rb{s_z(x),\ g_z(x)}$ such that $\ee\sb{s_z(x)|x} = F_m(x)$ and $\ee\sb{g_z(x)|x} = \nabla F_m(x)$.
\end{definition}
We will assume that the stochastic gradients have bounded moments, which allows us to control the variability in stochastic gradients. 
\begin{assumption}[Bounded Fourth Moment of Stochastic Gradients]\label{ass:stoch_bounded_fourth_moment}
    For all $m \in [M]$ and $x \in \rr^d$, let $\rb{s_z(x),\ g_z(x)} = \ooo_m(x)$ then we assume
    \(
    \ee_{z \sim \ddd_m}[\norm{g_z(x) - \nabla F_m(x)}^4 \mid x] \leq \sigma_{4,m}^4\enspace
    \). We also denote $\bar\sigma_4^4 := \frac{1}{M}\sum_{m\in[M]}\sigma_{4,m}^4$. 
\end{assumption}
In some cases we only need the following weaker second moment bound.
\begin{assumption}[Bounded Second Moment of Stochastic Gradients]\label{ass:stoch_bounded_second_moment}
    For all $m \in [M]$ and $x \in \rr^d$, let $\rb{s_z(x),\ g_z(x)} = \ooo_m(x)$ then we assume
    \(
    \ee_{z \sim \ddd_m}[\norm{g_z(x) - \nabla F_m(x)}^2 \mid x] \leq \sigma_{2,m}^2\enspace.
    \)We also denote $\bar\sigma_2^2 := \frac{1}{M}\sum_{m\in[M]}\sigma_{2,m}^2$. 
\end{assumption}
\begin{remark}[Stochastic Gradients for Learning Problems]
For problems of the form \eqref{prob:scalar}, implementing a first-order oracle under either of the above moment assumptions is straightforward: it amounts to sampling a data point $z \sim \mathcal{D}_m$ and computing $(f(\cdot; z), \nabla f(\cdot; z))$. This is justified by our assumption that for each $z \in \mathcal{Z}$, the function $f(\cdot; z)$ is differentiable, and for each fixed $x \in \mathbb{R}^d$, the map $z \mapsto \nabla f(x; z)$ is measurable. Together with the bounded moment assumptions---namely, $\mathbb{E}_{z \sim \mathcal{D}_m}[\|\nabla f(x; z)\|] < \infty$ for all $x \in \mathbb{R}^d$ (which can be ensured by \Cref{ass:stoch_bounded_second_moment,ass:stoch_bounded_fourth_moment})---these regularity conditions are sufficient to justify the interchange of expectation and differentiation:
\[
\nabla F_m(x) = \nabla \mathbb{E}_{z \sim \mathcal{D}_m}[f(x; z)] = \mathbb{E}_{z \sim \mathcal{D}_m}[\nabla f(x; z)]\enspace.
\]
For discrete probability distributions $\mathcal{D}_m$, the interchange of expectation and differentiation follows directly from the linearity of expectation, since $F_m(x)$ reduces to a finite or countable sum over differentiable functions $f(x; z)$.
\end{remark}
\begin{remark}[Relaxing the Regularity Assumptions]
For most of our analyses, the regularity assumptions stated in the previous section can be relaxed to hold only for the expected objective function on each machine. For instance, instead of requiring \Cref{ass:strongly_convex,ass:smooth_second} to hold pointwise for every realization of $z \in \mathcal{Z}$, it suffices to assume that the expected objective satisfies $\mu \cdot I_d \preceq \nabla^2 F_m(\cdot) \preceq H \cdot I_d$, for all $m \in [M]$. The same applies to other regularity assumptions, such as smoothness or third-order bounds. 

This relaxation is possible because, as we will see later, our analysis primarily relies on \Cref{ass:stoch_bounded_second_moment,ass:stoch_bounded_fourth_moment} and uses conditional expectations to abstract away the stochasticity of individual updates. 
\end{remark}
\begin{remark}[Relaxing the Bounded Moments Assumption]
It is sometimes possible in the analysis of SGD algorithms to relax \Cref{ass:stoch_bounded_second_moment,ass:stoch_bounded_fourth_moment} so that these assumptions only need to hold at the optima of each machine, denoted by $S_m^\star = \arg\min_{x \in \mathbb{R}^d} F_m(x)$, rather than for all $x \in \mathbb{R}^d$. Some of our analyses can be extended to accommodate this weaker assumption, particularly when combined with \Cref{ass:smooth_second}. 

However, it is unclear how to extend this relaxation to all our results—especially in the strongly convex setting, where global moment bounds are crucial for bounding quantities such as the \emph{consensus error}. Moreover, relaxing the moment bounds to hold only near the optima is known to make optimization more complex, particularly in the context of accelerated~\cite{woodworth2021even}. For these reasons, we do not pursue this relaxation in our work.
\end{remark}

\section{Local SGD Algorithm and Notation}\label{sec:ch2.4}
For ease of notation in this section and throughout we will often omit the oracle calls, and assume the stochastic first-order oracle (c.f., \Cref{def:oracle_first}) is implemented by sampling a data point on each machine and computing the gradient on that point. With this in mind, we can write the update for Local SGD (initialized at $x_0$) for round $r\in[R]$ as follows,
\begin{equation}\label{eq:local_updates}
    \begin{aligned}
    x_{r,0}^m &= x_{r-1}, && \forall\ m\in[M]\\
    x_{r,k+1}^m &= x_{r,k}^m - \eta \nabla f(x_{r,k}; z_{r,k}^m),\ z_{r,k}^m\sim \ddd_m, && \forall\ m\in[M], k\in[0, K-1]\\
    \bar x_{r} &= x_{r-1} + \frac{\beta}{M}\sum_{m\in[M]}\rb{x_{r,K}^m-\bar x_{r-1}}.   
    \end{aligned}
\end{equation}
Above $x_{r,k}^m$ is the $k^{th}$ local model on machine $m$, leading up to the $r^{th}$ round of communication, while $x_r$ is the consensus model at the end of the $r^{th}$ communication\footnote{We will also often denote the stochastic gradient output of $\ooo^m(x_{r,k}^m)$ by $g_{r,k}^m$.}. For local SGD, $\eta$ is referred to as the inner step size, while $\beta$ is the outer step size. Setting $\beta=1$ recovers \textbf{\textit{``vanilla local SGD''}} with a single step size which has been analyzed in several earlier works \citep{stich2018local, dieuleveut2019communication, khaled2020tighter, woodworth2020local}. Vanilla local SGD is equivalent to averaging the machine's models after $K$ local updates.  An alternative indexing of time would also be useful for vanilla Local SGD, which we will pay the most attention to. At time step $t \in [0,T-1]$ each machine $m\in[M]$ samples $z_t^m \sim \ddd_m$ and performs the update:
\begin{equation}\label{eq:local_updates_continuous}
    \begin{aligned}
    x_{t+1}^m &:= x_t^m - \eta \nabla f(x_t^m; z_t^m), && \text{if}\quad t+1 \mod K \neq 0\enspace,\\
    x_{t+1}^m &:= \frac{1}{M}\sum_{n\in[M]}\rb{x_t^n - \eta \nabla f(x_t^n; z_t^n)} && \text{if}\quad t+1 \mod K = 0\enspace,
    \end{aligned}
\end{equation}
For the above indexing it would be useful to define $\delta(t) = t - (t\mod K)$, i.e., the last communication round on or before time step $t\in[0,T]$.

We will often compare Local SGD with Mini-batch SGD~\citep{dekel2012optimal}. Mini-batch SGD's updates (initialized at $x_0$) for round $r\in[R]$ are as follows,
\begin{equation}\label{eq:MB_SGD}
    \begin{aligned}
    g_{r,k}^m &= \nabla f(\bar x_{r-1}; z_{r,k}^m),\ z_{r,k}^m\sim \ddd_m, && \forall\ m\in[M], k\in[0, K-1]\\
    \bar x_{r} &= \bar x_{r-1} - \frac{\beta}{M}\sum_{m\in[M], k\in[0, K-1]}g_{r,k}^m.
    \end{aligned}
\end{equation}
 The main difference in the mini-batch update compared to vanilla local SGD is that its local gradient is computed at the same point for the entire communication round\footnote{If in the Local-SGD updates in \eqref{eq:local_updates}, we replace $x_{r,K}^m - x_{r-1}$ by $\sum_{k=0}^{K-1}\nabla f(x_{r,k}^m; z_{r,k}^m)$ then by setting $\eta=0$ we can recover the mini-batch update in \eqref{eq:MB_SGD}. Thus Local SGD with two step-sizes generalizes both vanilla Local SGD and mini-batch SGD.}. Due to this, mini-batch SGD is not impacted by data heterogeneity, as it optimizes $F$ without getting affected by the multi-task nature of problem \eqref{prob:scalar}\footnote{We will discuss lower bounds and optimization baselines further in \Cref{ch:baselines_and_lower_bounds}.}. However, this is also why local SGD can intuitively outperform mini-batch SGD: it has more effective updates than mini-batch SGD. For e.g., without noise, i.e., $\sigma=0$, mini-batch SGD keeps obtaining the gradient at the same point, thus making only $R$ updates instead of $KR$ updates of local SGD. 

Comparing vanilla Local SGD and mini-batch SGD, and showing that Local SGD can beat mini-batch SGD when the data distributions across the machines is similar, is one of the main thrusts in the optimization literature in Federated Learning. This leads us to the discussion of data heterogeneity assumptions in the next section. 

\section{Restrictions on Data Heterogeneity}\label{sec:ch2.5}
In order to discuss different notions of data-heterogeneity, it would be helpful to define following sets of optima.
\begin{definition}[Machines' and Average Objective's Optima]
    For all $m \in [M]$, define the set of optima as $S_m^\star := \arg\min_{x \in \rr^d} F_m(x)$. Similarly, define the set of optima for the average objective as $S^\star := \arg\min_{x \in \rr^d} F(x)$. 
\end{definition}
For optimization to be feasible, we must assume that at least some of the solutions in the target set are easily recoverable. One necessary condition for that to happen is that the optima are not arbitrarily large. 
\begin{assumption}[Bounded Optima]\label{ass:bounded_optima}
    For all machines $m\in[M]$, $\exists$ $x_m^\star \in S_m^\star$ such that $\norm{x_m^\star} \leq B_m$. We will define $\bar B := \frac{1}{M}\sum_{m\in[M]}B_m$. Similarly, $\exists$ $x^\star \in S^\star$ such that $\norm{x^\star} \leq B$.
\end{assumption}
It is worth noting that $B$ can be much larger than $\bar B$, especially as $M$ grows.
\begin{figure}[t]
    \centering
    \includegraphics[width=0.8\linewidth]{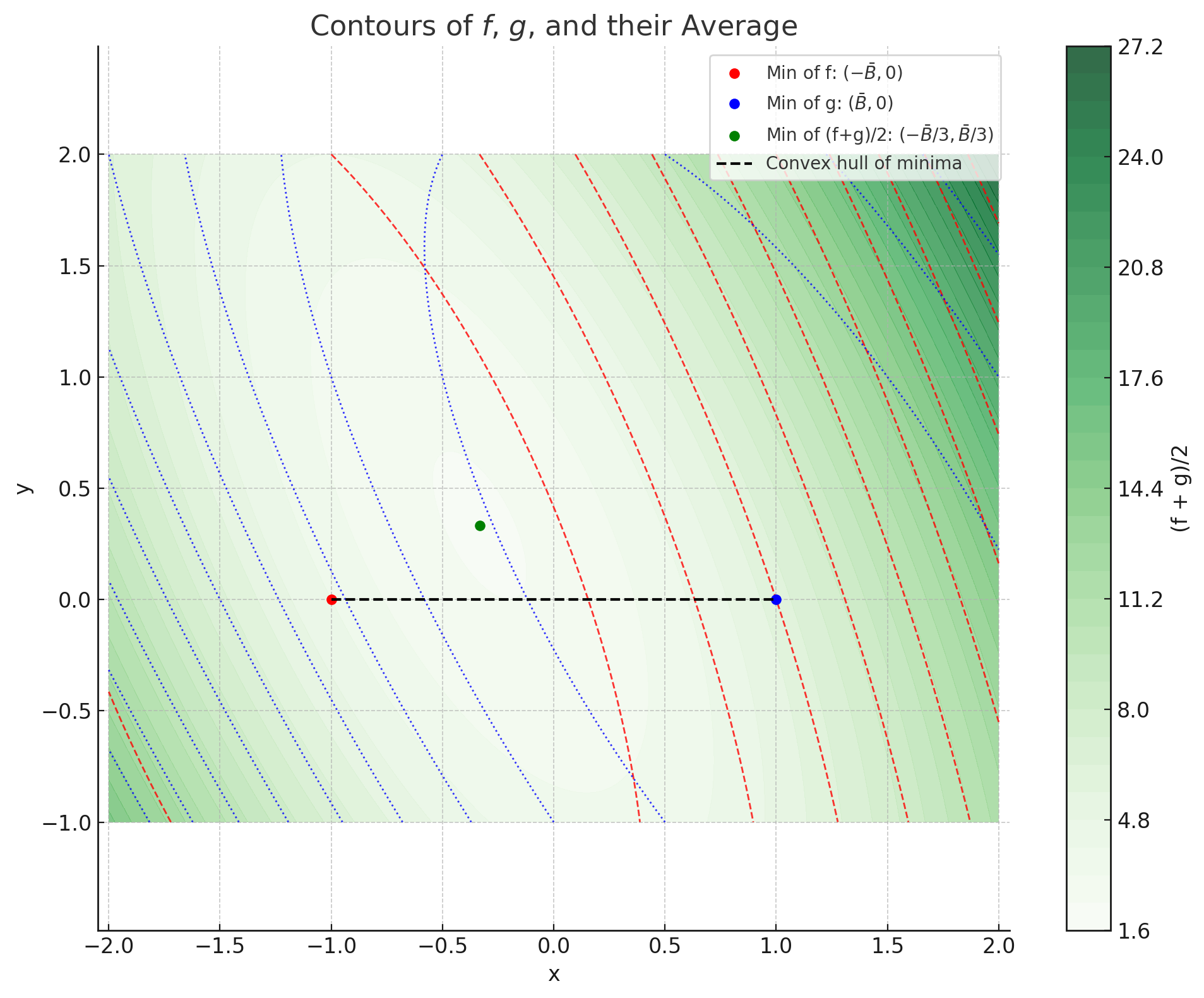}
    \caption{Illustration of the functions in \Cref{prop:barB_vs_B}.}
    \label{fig:prop_1}
\end{figure}
\begin{proposition}\label{prop:barB_vs_B}
    There exists a problem instance satisfying \Cref{ass:convex,ass:smooth_second,ass:smooth_third} such that $B \geq \frac{\sqrt{M}\bar B}{3}$. 
\end{proposition}
\begin{proof}
    Our construction uses square loss, which makes $F_m$'s quadratic functions. Note that assuming every data-point is $z=(\beta,y)$ we can denote the hessian of machine $m$'s objective as,
    \begin{align*}
        \nabla^2 F_m(\cdot) = \ee_{(\beta,y)}\sb{\beta\beta^T}\enspace.  
    \end{align*}
    The Hessian only depends on the co-variate distribution. In particular, note that by supporting the distribution $\ddd_m$ on $d$ different $\beta$'s and choosing the probabilities appropriately we can construct any positive semi-definite Hessian: this follows from the singular value decomposition of $\nabla^2 F_m(\cdot)$. This allows us to state the functions on each machine, without loss of generality, and ignore the actual distributions $\ddd_m$'s.

    Before we state the functions explicitly, we first consider the following two quadratic functions in two dimensions:
    \begin{align*}
        f(x,y) = 2\rb{x+\bar B}^2 + \rb{x+y+\bar B}^2\quad \text{and} \quad g(x,y) = \rb{x-\bar B}^2 + \rb{x+y-\bar B}^2\enspace.
    \end{align*}
Note that both these functions are strictly convex with optimizers at $\rb{-\bar B,0}$ and $\rb{\bar B,0}$ respectively. However, the optimizer of the average of these functions is given by $\rb{-\frac{\bar B}{3}, \frac{\bar B}{3}}$, which is notably not on the convex hull of the optimizers of the constituent functions. We illustrate these functions in \Cref{fig:prop_1}. To see this, note the gradients for these functions and the average function:
\begin{align*}
    \nabla f(x,y) &= \begin{bmatrix}
        4(x+\bar B) + 2(x + y + \bar B)\\
        2(x + y + \bar B)
    \end{bmatrix}\enspace;\ \nabla f(x,y) =0 \Rightarrow (x,y) = (-\bar B, 0)\enspace.\\
    \nabla g(x,y) &= \begin{bmatrix}
        2(x-\bar B) + 2(x + y - \bar B)\\
        2(x + y - \bar B)
    \end{bmatrix}\enspace;\ \nabla g(x,y)=0 \Rightarrow (x,y) = (\bar B, 0)\enspace.\\ 
    \nabla \rb{\frac{f+g}{2}}(x,y) &= \begin{bmatrix}
        6x + 2\bar B + 2(x + y)\\
        2(x + y)
    \end{bmatrix}\enspace;\ \rb{\frac{f+g}{2}}(x,y) =0 \Rightarrow (x,y) =\rb{-\frac{\bar B}{3}, \frac{\bar B}{3}} \enspace.
\end{align*}
Now we define $M$ different objectives on $d$ dimensions (assuming $M$, $d$ are even for simplicity) as follows: 
\begin{align*}
    F_1(x) &= f(x[1],x[2]) + \frac{1}{2}\norm{(0,0,x[3], \dots, x[M])}^2\enspace,\\
    F_2(x) &= g(x[1],x[2]) + \frac{1}{2}\norm{(0,0,x[3], \dots, x[M])}^2\enspace,\\
    F_3(x) &= f(x[3],x[4]) + \frac{1}{2}\norm{(x[1],x[2],0,0,x[4] \dots, x[M])}^2\enspace,\\
    F_4(x) &= g(x[3],x[4]) + \frac{1}{2}\norm{(x[1],x[2],0,0,x[4] \dots, x[M])}^2\enspace,\\
    &\vdots\\
    F_{M-1}(x) &= f(x[M-1],x[M]) + \frac{1}{2}\norm{(0, 0, \dots, x[M-1],x[M])}^2\enspace,\\
    F_M(x) &= g(x[M-1],x[M]) + \frac{1}{2}\norm{(0, 0, \dots, x[M-1],x[M])}^2\enspace.
\end{align*}
Due to the properties of $f,g$ that we discussed above note that the optimizer of each machine has a norm $\bar B$. However, due to the decoupling of dimensions across every other pair of odd and even machine, the optimizer of the average objective is given by $x^\star = \rb{-\frac{\bar B}{3}, \frac{\bar B}{3},-\frac{\bar B}{3}, \frac{\bar B}{3},\dots, -\frac{\bar B}{3}, \frac{\bar B}{3}}$. Thus in order to satisfy \Cref{ass:bounded_optima} we most choose $B\geq \frac{B\sqrt{M}}{3}$.

This proves the proposition.
\end{proof}
The proof of the above proposition relies on the key idea that, in higher dimensions, the optimizer of the average objective may not lie within the convex hull of the optima of individual objectives. This is an aspect of optimizing Problem \eqref{prob:scalar} that could make it much more complex than the individual optimization problems of converging to solution sets $S_m^\star$'s. To control for this behavior, we will first introduce the following assumption, which governs the maximum distance between the solution sets of each machine. 
\begin{assumption}[Discrepancy between Machines' Optima]\label{ass:zeta_star}
    For machines $m,n \in [M]$,  there exists $\zeta_{\star,m,n}= \zeta_{\star,n,m} \leq B_m + B_n$ such that, 
    \[
    \inf_{x_m^\star \in S_m^\star,\ x_n^\star \in S_n^\star} \norm{x_m^\star - x_n^\star} = \zeta_{\star,m,n} = \zeta_{\star,n,m}\enspace.
    \]
    We also denote $\zeta_{\star,m} = \frac{1}{M}\sum_{n\in[M]}\zeta_{\star,m,n}\leq 2\bar B$ for all $m\in[M]$, and $\zeta_\star^4 := \frac{1}{M^2}\sum_{m,n\in[M]}\zeta_{\star,m,n}^4$.
\end{assumption}
All machines share a common minimizer if and only if $\zeta_\star = 0$, and in that situation, solving Problem~\eqref{prob:scalar} recovers this global optimum. However, when machines do not share an optimizer, we must additionally assume that at least one minimizer of the average objective is approximately optimal for each machine. Without this condition, some clients may not benefit from collaboration.
\begin{assumption}[Discrepancy between Machines' and the Average Objective's Optima]\label{ass:phi_star}
    For each machine $m\in[M]$, there exists a $x^\star \in S^\star$ and a $\phi_{\star,m} \leq B + B_m$ such that 
    \[
    \inf_{x_m^\star \in S_m^\star} \norm{x_m^\star - x^\star} = \phi_{\star,m} \enspace.
    \]
    We also denote $\phi_\star^4 := \frac{1}{M}\sum_{m\in[M]}\phi_{\star, m}^4$.
\end{assumption}
\begin{remark}[Other First-order Heterogeneity Assumptions]\label{rem:other_first_order_ass}
Most existing first-order heterogeneity conditions are variants of \Cref{ass:phi_star}. Notably, using \Cref{ass:smooth_second} and choosing $x^\star$ and $x_m^\star$ to be the optima implied by \Cref{ass:phi_star} we can conclude that,
\begin{align*}
    \frac{1}{M}\sum_{m\in[M]}\norm{\nabla F_m(x^\star)}^2 &=  \frac{1}{M}\sum_{m\in[M]}\norm{\nabla F_m(x^\star) - \nabla F_m(x_m^\star)}^2\enspace,\\
    &\leq^{\text{(\Cref{ass:smooth_second})}} \frac{1}{M}\sum_{m\in[M]} H^2\norm{x^\star - x_m^\star}^2\enspace,\\
    &\leq \frac{1}{M}\sum_{m\in[M]}H^2\phi_{\star,m}^2 = H^2\phi_\star^2\enspace.
\end{align*}
Essentially, \Cref{ass:phi_star} implies that the clients are approximately simultaneously stationary at some $x^\star\in S^\star$. This notion of simultaneous stationarity was used in several existing works as a data heterogeneity assumption~\cite{woodworth2020minibatch,glasgow2022sharp,patel2022towards,patel2023still,patel2024limits}. But we believe that stating this in the form of \Cref{ass:phi_star}, makes the actual source of heterogeneity more transparent. The quantity $\phi_\star$ captures the notion of \emph{approximate simultaneous realizability} across clients and has also appeared in the literature on collaborative PAC learning and incentives for Federated Learning \citep{blum2017collaborative, nguyen2018improved, haghtalab2022demand, han2023effect}. All these assumptions are usually referred to as first-order heterogeneity assumptions at the optima.
\end{remark}
\begin{remark}[$\zeta_\star$ vs.\ $\phi_\star$]\label{rem:zeta_star_vs_phi_star}
    With \Cref{ass:strongly_convex}, all the machines and the average objective have a unique optimum. In particular then \Cref{ass:phi_star} implies \Cref{ass:zeta_star} with $\zeta_\star \leq 2\phi_\star$. However, the reverse is not true in general. Using a similar construction as in \Cref{prop:barB_vs_B} we note that there is a quadratic problem instance which satisfies \Cref{ass:zeta_star} with $\zeta_\star$ can be chosen to be $2\bar B$, yet satisfies \Cref{ass:phi_star} only with $\phi_\star \geq \frac{\sqrt{M}\bar B}{3}$ which can be much larger when $M$ is large. Essentially, the gap between $\phi_\star$ and $\zeta_\star$ also comes from the fact that in higher dimensions, averaging the machines' objectives can push the minimizer away from $S_m^\star$'s.  
\end{remark}
\subsection{Is a Small First-order Heterogeneity Enough?}\label{sec:is_first_enough}
One might conjecture that when $\phi_\star$ in \Cref{ass:phi_star} is small then local updates in \eqref{eq:local_updates} should be helpful, as we could more quickly recover the approximately simultaneously optimal solution $x^\star$. However, consider the following quadratic example\footnote{Recall we can construct any quadratic objectives using singular value decomposition and an appropriate data distribution as in the proof of \Cref{prop:barB_vs_B}.} on two machines and in two dimensions~\cite{patel2024limits},
\begin{align*}
    F_1(x) &:= \frac{1}{2}(x-x^\star)^T\begin{bmatrix}
        H & 0\\
        0 & 0
    \end{bmatrix}(x-x^\star)\enspace,\\
    F_2(x) &:= \frac{1}{2}(x-x^\star)^T\begin{bmatrix}
        0 & 0\\
        0 & H
    \end{bmatrix}(x-x^\star)\enspace.
\end{align*}
The above two objectives share an optimizer $x^\star$, and thus both notions of data heterogeneity in \Cref{ass:zeta_star,ass:phi_star} are zero. Each machine's objective also satisfies \Cref{ass:smooth_second}. Assuming we run local GD (i.e., we don't have stochastic gradients but exact gradients) on both machines initialized at $(0,0)$, then the iterate after $R$ rounds is given by (proved in \Cref{app:chap2}), 
\begin{align}\label{eq:motivation}
  \bar x_R = x^\star\rb{1 - \rb{1- \frac{\beta}{2}\rb{1 - (1-\eta H)^K}}^R}\enspace.  
\end{align}
Let us first consider vanilla local SGD, i.e., $\beta=1$. The above expression simplifies to
\begin{align*}
  \bar x_R = x^\star\rb{1 - \rb{\frac{1 + (1-\eta H)^K}{2}}^R}\enspace.  
\end{align*}
Thus, even if $K\to\infty$, $x_R$ \textbf{does not} converge to $x^\star$ for finite $R$. 
\begin{remark}[The Role of the Outer Step-size]
        For the above example, if we set $\beta=2$, $x_R$ will converge to $x^\star$ in a single communication round when $K\to\infty$! Thus, in this example, we can not rule out that \Cref{ass:zeta_star,ass:phi_star} capture all the essential data heterogeneity in our problem. In \Cref{ch:baselines_and_lower_bounds} we will show a construction that works for any outer step-size $\beta$ and also show that (an accelerated variant of) mini-batch SGD in \eqref{eq:MB_SGD} is ``optimal''.
\end{remark}

A closer inspection of the above example reveals that the two machines exhibit very different curvature profiles along different directions. This discrepancy causes their local updates to diverge between communication rounds. To account for this effect, it becomes essential to control the geometry induced by the second derivative of each local objective, which governs the behavior of first-order methods.

To formalize this idea, we introduce the following second-order heterogeneity assumption:

\begin{assumption}[Bounded Second-order Heterogeneity]\label{ass:tau}
There exists $\tau \leq 2H$ such that
\[
\sup_{m,n \in [M]} \sup_{x \in \mathbb{R}^d} \|\nabla^2 F_m(x) - \nabla^2 F_n(x)\| \leq \tau \enspace.
\]
\end{assumption}

Note that $\tau$ measures the second-order smoothness of the difference $F_m(\cdot) - F_n(\cdot)$ for any pair of machines $m,n \in [M]$. This observation will allow us to replace $H$ with the typically smaller quantity $\tau$ in several parts of our analysis, leading to tighter bounds that better capture the role of data heterogeneity. This second-order heterogeneity assumption has been used by several works, although mainly in the non-convex setting~\citep{karimireddy2020mime,murata2021bias,patel2022towards}. 

\begin{remark}[$\zeta_\star$ and $\tau$ vs.\ $\phi_\star$]\label{rem:zeta_star_tau_vs_phi_star}
    In some settings, the parameter $\phi_\star$ in \Cref{ass:phi_star} can be bounded using $\zeta_\star$ and $\tau$ from \Cref{ass:zeta_star,ass:tau}. For example, if each $F_m$ is a strongly convex quadratic function with Hessian $A_m$ and unique minimizer $x_m^\star$, then the global optimum satisfies
    \[
    x^\star = A^{-1} \cdot \frac{1}{M} \sum_{m \in [M]} A_m x_m^\star \enspace,
    \]
    where \( A := \frac{1}{M} \sum_{m \in [M]} A_m \). Using this and defining $\bar{x}^\star := \frac{1}{M} \sum_{m} x_m^\star$ we can derive~\citep{patel2025revisiting} for $m\in[M]$,
    \begin{align*}
    \norm{x_m^\star - x^\star} &\leq \norm{x_m^\star - \bar x^\star} + \norm{\bar x^\star - x^\star}\enspace,\\
    &= \norm{\frac{1}{M}\sum_{n\in[M]}\rb{x_m^\star - x_n^\star}} + \norm{\frac{1}{M}\sum_{n\in[M]}\rb{x_n^\star - A^{-1}A_nx_n^\star}}\enspace,\\
    &\leq \frac{1}{M}\sum_{n\in[M]}\norm{x_m^\star - x_n^\star} + \norm{\frac{1}{M}\sum_{n\in[M]}A^{-1}\rb{A - A_n}\rb{x_n^\star - \bar x^\star}}\enspace,\\
    &\leq^{\text{(\Cref{ass:zeta_star})}} \frac{1}{M}\sum_{n\in[M]}\zeta_{\star,m,n} +  \frac{1}{M}\sum_{n\in[M]}\norm{A^{-1}\rb{A-A_n}\rb{x_n^\star-\bar x^\star}}\enspace,\\
    &\leq \frac{1}{M}\sum_{n\in[M]}\zeta_{\star,m,n} + \frac{1}{M}\sum_{n\in[M]}\norm{A^{-1}}\norm{A-A_n}\norm{x_n^\star-\bar x^\star}\enspace,\\
    &\leq \zeta_{\star,m} + \frac{1}{M^2}\sum_{l,n\in[M]}\frac{\tau \zeta_{\star,l,n}}{\mu} = \zeta_{\star,m} + \frac{\tau\zeta_\star}{\mu}\enspace.
\end{align*}
   Averaging this over $m\in[M]$ implies that we can choose $\phi_\star = \zeta_\star (1 + \tau / \mu)$. Thus, when $\tau/\mu$ is small $\phi_\star$ and $\zeta_\star$ can be comparable, which is in contrast to what we discussed in the absence of \Cref{ass:tau} in \Cref{rem:zeta_star_vs_phi_star}. For general non-quadratic problems, however, it may not be possible to eliminate the dependence on $\phi_\star$.
\end{remark}
Finally, we state a first-order heterogeneity assumption which, while more restrictive and less transparent than the second-order assumptions discussed earlier, significantly simplifies the analysis of local update algorithms.

\begin{assumption}[Uniform Bounded First-order Heterogeneity]\label{ass:zeta}
Assume the objectives on each machine satisfy \Cref{ass:smooth_second}. Then there exists a constant $\zeta > 0$ such that
\[
\sup_{m,n \in [M]} \sup_{x \in \mathbb{R}^d} \|\nabla F_m(x) - \nabla F_n(x)\| \leq H \cdot \zeta \enspace.
\]
\end{assumption}

\Citet{woodworth2020minibatch} demonstrated that vanilla Local SGD can outperform mini-batch SGD under \Cref{ass:zeta}. In \Cref{ch:upper_bounds}, we also analyze Local SGD's convergence in terms of the heterogeneity parameter $\zeta$. However, as shown below, \Cref{ass:zeta} can be overly restrictive.

\begin{proposition}\label{prop:zeta_limitation}
Let $F_m(x) = \frac{1}{2}x^\top A_m x + b_m^\top x + c_m$ for all $m \in [M]$. If the objectives $\{F_m\}_{m \in [M]}$ satisfy \Cref{ass:zeta} for some finite $\zeta < \infty$, then for any two machines $m, n \in [M]$, it must hold that $A_m = A_n$.
\end{proposition}
The above proposition is proved in \Cref{app:chap2}. In essence, \Cref{ass:zeta} allows heterogeneity only in the linear terms $b_m$ and constants $c_m$, but not in the curvature (i.e., the Hessians) of the functions. This rigidity is precisely why many recent works---such as \citep{khaled2020tighter, karimireddy2020scaffold, koloskova2020unified}---have considered the relaxed first-order heterogeneity assumptions we discussed earlier.

In fact, if we know that the local SGD iterates lie in $\bb_2(D)$, the ball of diameter $D$ around origin, then using \Cref{ass:phi_star,ass:tau}, we can give the following upper bound which avoids \Cref{ass:zeta}. 
\begin{proposition}\label{prop:zeta_bound}
    If the objectives $\{F_m\}_{m \in [M]}$ satisfy \Cref{ass:smooth_second,ass:phi_star,ass:tau} then we have for all $m,n\in[M]$,
    \begin{align*}
        \sup_{x\in\bb_2(D)}\norm{\nabla F_m(x)-\nabla F_n(x)} \leq H\rb{\zeta_{\star,m} + \zeta_{\star,n}} + \tau\rb{D + \bar B}\enspace.
    \end{align*}
\end{proposition}
\begin{proof}
    Denote the function $G_{m,n} := F_m - F_n$ for all $m,n\in [M]$. Then note that using Taylor expansion, we can write for any $x\in\bb_2(D)$ and $\bar x^\star = \frac{1}{M}\sum_{m\in[M]}x_m^\star$ where $x_m^\star\in S_m^\star$,
    \begin{align*}
        \nabla G_{m,n}(x) - \nabla G_{m,n}(\bar x^\star) = \sb{\int_{0}^1\nabla^2 G_m(\bar x^\star + s(x-\bar x^\star))ds}(x-\bar x^\star)\enspace.
    \end{align*}
    Re-arranging, taking norms, and applying the triangle inequality implies,
    \begin{align*}
        \norm{\nabla G_{m,n}(x)}  &= \norm{\nabla G_{m,n}(\bar x^\star)  + \sb{\int_{0}^1\nabla^2 G_{m,n}(\bar x^\star + s(x-\bar x^\star))ds}(x-\bar x^\star)}\enspace,\\
        &\leq \norm{\nabla F_{m}(\bar x^\star) - \nabla F_{n}(\bar x^\star)} + \norm{\int_{0}^1\nabla^2 G_{m,n}(\bar x^\star + s(x-\bar x^\star))ds}\norm{x-\bar x^\star}\enspace,\\
        &\leq \norm{\nabla F_{m}(\bar x^\star)} + \norm{\nabla F_{n}(\bar x^\star)} + \sb{\int_{0}^1\norm{\nabla^2 G_{m,n}(\bar x^\star + s(x-\bar x^\star))}ds}\norm{x-\bar x^\star}\enspace,\\
        &\leq^{\text{(\Cref{ass:phi_star,ass:tau})}} \frac{H}{M}\sum_{l\in[M]}\rb{\zeta_{\star,m,l} + \zeta_{\star,n,l}} + \sb{\int_{0}^1\tau ds}\rb{\norm{x} + \norm{\bar x^\star}}\enspace,\\
        &= H\rb{\zeta_{\star,m} + \zeta_{\star,n}} + \tau\rb{D + \bar B} \enspace.
    \end{align*}
    This gives us the desired result.
\end{proof}
The above proposition implies that if we know our algorithm's iterates will be inside a ball $\bb_2(D)$, the smaller the second-order heterogeneity of our problem, the smaller the bound on its first-order heterogeneity. In the extreme case of $\tau=0$, we can replace $\zeta$ with other more reasonable heterogeneity assumptions. Finally, as a sanity check, in the homogeneous setting, i.e., when $\ddd_m$'s are all the same, the right-hand side is zero. Surprisingly despite the above connection, aside from our results in \Cref{ch:upper_bounds}~\citep{patel2024limits, patel2025revisiting}, existing literature has not been able to demonstrate a meaningful advantage of local updates (in the convex setting) without relying on \Cref{ass:zeta}.

\section{Distributed Zero-Respecting Algorithms}\label{sec:ch2.6}
To place our algorithms in a formal framework, which is also useful when discussing lower bounds and the optimality of our algorithms, we will examine the following class of algorithms, which generalizes the class of zero-respecting algorithms~\citep{carmon2020lower,arjevani2019lower} to the distributed setting. 
\begin{definition}[Distributed Zero-respecting Algorithms]\label{def:zero_respecting}
Consider $M$ machines in the intermittent communication setting, each endowed with an oracle $\ooo_m:\rr^d \times \Delta(\zzz)\to \rr\times\rr^d$ and a distribution $\ddd_m$ on $\zzz$. Let $I_{r,k}^m$ denote the input to the $k^{th}$ oracle call, leading up to the $r^{th}$ communication round on machine $m$. An optimization algorithm initialized at $0$ is distributed zero-respecting if:
\begin{enumerate}
    \item for all $\ r\in [R], k\in[K], m\in[M]$, $I_{r,k}^{m}$ is in
$$\cb{\bigcup_{l\in[k-1]}\supp{\ooo_{m}(I_{r,l}^m; \ddd_m)}}\cup \cb{\bigcup_{n\in[M], s\in[r-1], l \in [K]} \supp{\ooo_{n}(I_{s,l}^n; \ddd_n)}},$$
    \item for all $\ r\in [R], k\in[K], m\in[M]$, $I_{r,k}^{m}$ is a deterministic function (which is same across all the machines) of $$\cb{\bigcup_{l\in[k-1]}\ooo_{m}(I_{r,l}^m; \ddd_m)}\cup \cb{\bigcup_{n\in[M], s\in[r-1], l \in [K]} \ooo_{n}(I_{s,l}^n; \ddd_n)},$$
    \item at the $r^{th}$ communication round, the machines only communicate vectors in
$$\cb{\bigcup_{n\in[M], s\in[r], l \in [K]} \supp{\ooo_{n}(I_{s,l}^n; \ddd_n)}}.$$
\end{enumerate}
We denote this class of algorithms by $\boldsymbol{\aaa_{ZR}}$. Furthermore, if all the oracle inputs are the same between two communication rounds, i.e., $I_{r,k}^m = I_r\in \iii$ for all $m\in[M], k\in[K], r\in[R]$, then we say that the algorithm is centralized, and denote this class of algorithms by $\boldsymbol{\aaa_{ZR}^{cent}}\subset \aaa_{ZR}$. 
\end{definition}

This class encompasses a diverse range of distributed optimization algorithms, including Local SGD \citep{mcmahan2016communication} (c.f., \eqref{eq:local_updates}), as well as many distributed variance-reduced algorithms \citep{karimireddy2020mime, zhao2021fedpage, khanduri2021stem, patel2022towards}. Non-distributed zero-respecting algorithms are those whose iterates have components in directions about which the algorithm has no information, meaning that, in some sense, it is just ``wild guessing''. We have also defined the smaller class of centralized algorithms $\aaa_{ZR}^{cent}$. These algorithms query the oracles at the same point within each communication round and use the combined $MK$ oracle queries each round to get a \textit{``mini-batch''} estimate of the gradient. Thus, the class $\aaa_{ZR}^{cent}$ includes algorithms such as mini-batch SGD\citep{dekel2012optimal} (c.f., \eqref{eq:MB_SGD}), accelerated mini-batch SGD \citep{ghadimi2012optimal}, mini-batch SARAH \citep{nguyen2017sarah}, and mini-batch STORM \citep{cutkosky2019momentum}, but doesn't include local-update algorithms in $\aaa_{ZR}$ such as Local-SGD.

\section{Min-Max Optimality}\label{sec:min_max}

In the optimization literature, \textbf{min-max optimality} is a fundamental concept used to characterize the \emph{hardness} of optimizing a problem with a certain type of algorithm. In our setting, a problem instance \( P \) can be fully specified by a loss function \( f : \mathbb{R}^d \times \mathcal{Z} \to \mathbb{R} \) and a collection of \( M \) data distributions \( \mathcal{D}_1, \dots, \mathcal{D}_M \in \Delta(\mathcal{Z}) \). Alternatively, we may adopt the \emph{oracle framework}, representing each instance via objective functions \( \{F_m : \mathbb{R}^d \to \mathbb{R} \}_{m \in [M]} \) and corresponding oracles \( \{\mathcal{O}_m : \mathbb{R}^d \times \Delta(\mathcal{Z}) \to \mathbb{R} \times \mathbb{R}^d \}_{m \in [M]} \), as defined in \Cref{def:oracle_first}.

Let \(\mathcal{P}\) denote a \textbf{class of problems}, defined by restricting the loss function and data distributions to satisfy certain assumptions---for example, smoothness, convexity, bounded stochastic gradient moments, and bounded first-order heterogeneity (\Cref{ass:convex,ass:smooth_second,ass:stoch_bounded_second_moment,ass:zeta}). Let \(\mathcal{A}\) denote a \textbf{class of algorithms}, such as zero-respecting algorithms (\(\mathcal{A}_{ZR}\)) defined over this problem class. Assuming we are interested in the expected sub-optimality of the output, the \emph{min-max optimization error} for this problem-algorithm pair is given by
\begin{align}\label{eq:min_max_optimality}
    \rrr(\aaa,\ppp) := \min_{A \in \mathcal{A}} \max_{P \in \mathcal{P}} \left( \mathbb{E}[F(x^A)] - \min_{x^\star \in \mathbb{R}^d} F(x^\star) \right)\enspace,
\end{align}
where \(x^A\) denotes the (random) output of algorithm \(A\), which may depend on the data distributions and oracles across machines. Importantly, we treat different hyperparameter choices (e.g., different step sizes for Local SGD) as distinct algorithms within \(\mathcal{A}\).

The central goal in much of optimization theory is to characterize the quantity in \eqref{eq:min_max_optimality} up to numerical constants. To establish \textbf{upper bounds} on the min-max complexity, it suffices to construct an algorithm \(A \in \mathcal{A}\) that achieves the desired convergence rate for every \(P \in \mathcal{P}\). Conversely, to establish \textbf{lower bounds}, one must construct a problem instance \(P \in \mathcal{P}\) on which all algorithms in \(\mathcal{A}\) suffer at least a specific error.

In the serial (non-distributed) setting, min-max complexity is well-understood for several foundational problem classes, including smooth convex and strongly convex stochastic optimization~\citep{nesterov2018lectures,nemirovski1994efficient,ghadimi2012optimal}, as well as smooth non-convex optimization~\citep{arjevani2019lower,nguyen2017sarah}. In \Cref{ch:baselines_and_lower_bounds}, we extend this min-max perspective to distributed optimization to investigate how data heterogeneity impacts the effectiveness of local update algorithms. In particular, we often restrict our focus to the class \(\mathcal{A}^{\text{Local}}\), which comprises various instantiations of local update algorithms, such as Local SGD. This focus necessitates deriving algorithm-dependent upper and lower bounds for a fixed problem class---a direction that is relatively less explored in the serial optimization literature.

\chapter{Optimization Lower Bounds and Baselines}\label{ch:baselines_and_lower_bounds}
In this section, we study baseline distributed optimization algorithms and establish lower bounds on their convergence. These lower bounds will be instrumental both for identifying min-max optimality and for guiding the development of matching upper bounds. Our contributions are threefold:
\begin{enumerate}
    \item In \Cref{thm:new_LSGD_lower_bound}, we establish a tight lower bound for Local SGD under bounded first-order heterogeneity (\Cref{ass:phi_star}), fully characterizing its min-max convergence rate. This result shows that \Cref{ass:phi_star} alone is insufficient to demonstrate an advantage of local updates.
    
    \item Under the same assumption, we show in \Cref{thm:AIlb_zeta0} that the min-max optimal algorithm is (accelerated) mini-batch SGD (cf. updates \eqref{eq:MB_SGD}). Together, \Cref{thm:new_LSGD_lower_bound,thm:AIlb_zeta0} close a recent line of work and highlight the need for stronger heterogeneity assumptions to justify the benefits of local updates.
    
    \item We partially address this in \Cref{thm:new_LSGD_lower_bound_with_tau}, where we incorporate second-order heterogeneity (\Cref{ass:tau}) and show that when $\tau$ is small, Local SGD can outperform mini-batch SGD.
\end{enumerate}

Our technical approach combines classical lower-bound techniques with new analytical constructions~\citep{arjevani2015communication,arjevani2019lower,woodworth2020local,woodworth2020minibatch,glasgow2022sharp}. We discuss motivating examples for \Cref{thm:new_LSGD_lower_bound,thm:AIlb_zeta0} in the main text, and defer the more intricate construction behind \Cref{thm:new_LSGD_lower_bound_with_tau} to \Cref{app:chap3}.

\subsection*{Outline and Important References}

In \Cref{sec:ch3.1}, we begin with the homogeneous setting, where all clients have the same distribution $\mathcal{D}_m$. This setting isolates the effect of objective regularity from data heterogeneity and serves as an essential baseline. The upper and lower bounds in this section synthesize results from several prior works~\citep{dieuleveut2019communication,woodworth2020local,khaled2020tighter,yuan2020federated,glasgow2022sharp,woodworth2021min,woodworth2021minimax}, which collectively establish the min-max complexity of homogeneous distributed optimization with intermittent communication.

\Cref{sec:ch3.2} then considers the setting where client objectives satisfy~\Cref{ass:zeta_star,ass:phi_star}. The results here (\Cref{thm:new_LSGD_lower_bound,thm:AIlb_zeta0}) are from our COLT 2024 paper~\citep{patel2024limits}, co-authored with Margalit Glasgow, Ali Zindari, Lingxiao Wang, Sebastian U.\ Stich, Ziheng Cheng, Nirmit Joshi, and Nathan Srebro.

Finally, \Cref{sec:ch3.3} presents the lower bound (\Cref{thm:new_LSGD_lower_bound_with_tau}) under the additional assumption of bounded second-order heterogeneity (\Cref{ass:tau}), taken from our upcoming work~\citep{patel2025revisiting}, co-authored with Ali Zindari, Lingxiao Wang, Sebastian U.\ Stich.

\section{Homogeneous Problems and the Role of Third-order Smoothness}\label{sec:ch3.1}
We begin by first revisiting two baseline algorithms: mini-batch SGD (c.f., updates in \eqref{eq:MB_SGD}) and SGD on a single machine. Intuitively, when noise and heterogeneity are both low, one could expect SGD on a single machine to outperform both Local SGD and mini-batch SGD, both collaborative algorithms. This motivates us first to discuss the homogeneous setting, when $\ddd_m=\ddd$ for each $m\in[M]$.

In the homogeneous setting, \citet{dieuleveut2019communication,woodworth2020local} showed that when the problem instances are convex quadratics—i.e., they satisfy \Cref{ass:convex,ass:smooth_second,ass:smooth_third,ass:bounded_optima} with $Q=0$—and we are equipped with stochastic first-order oracles satisfying \Cref{ass:stoch_bounded_second_moment}, then Local SGD outputs an iterate $x^{L\textnormal{-}SGD}$\footnote{Throughout the thesis, we typically consider the output of Local SGD to be either the final averaged iterate $x_T$ or a weighted average of the iterates $x_0, \dots, x_T$. While the last iterate is more common in practice, weighted averages often simplify theoretical analysis. \citet{woodworth2020local} specifically used $x^{L\textnormal{-}SGD} = \frac{1}{T} \sum_{t=0}^{T-1} x_t$.} that satisfies the following convergence bound:\footnote{In convex optimization, it is standard to express convergence rates in terms of expected function value sub-optimality~\cite{nesterov2018lectures,nemirovski1994efficient}. These results can often be extended to high-probability bounds; see, e.g., \cite{liu2023high}. We focus only on expected guarantees in this thesis.}
\begin{align}\label{rate:lsgd_quad_hom}
    \ee\sb{F(x^{L\textnormal{-}SGD})} - F(x^\star) \leq c_1\cdot\rb{\frac{HB^2}{KR} + \frac{\sigma_2 B}{\sqrt{MKR}}}\enspace,
\end{align}
where $c_1$ is a numerical constant.

\begin{remark}[Extreme Communication Efficiency]\label{rem:extreme_communication_efficiency}
    The above convergence rate reveals a striking property: as $K$ tends to infinity, the upper bound tends to zero. That is, with sufficiently many local updates, even a single round of communication can suffice. This \emph{extreme communication efficiency} makes Local SGD particularly attractive in settings where communication is the primary bottleneck, such as cross-device federated learning~\cite{kairouz2019advances}. The special case where local updates are followed by only a single communication step is often referred to as \emph{One-shot Averaging}. While this behavior is specific to homogeneous quadratic objectives, it highlights an important tradeoff: large $R$ can often be replaced with large $K$, an idea that underpins much of communication-efficient distributed optimization.
\end{remark}

The convergence rate in \eqref{rate:lsgd_quad_hom} is strictly better than the corresponding bounds for both mini-batch SGD and single-machine SGD:
\begin{equation}\label{rate:mb_sm_sgd_quad_hom}
\begin{aligned}
    \ee\sb{F(x^{MB\textnormal{-}SGD})} - F(x^\star) &\leq c_2\cdot\rb{\frac{HB^2}{R} + \frac{\sigma_2 B}{\sqrt{MKR}}}\enspace,\\
    \ee\sb{F(x^{SM\textnormal{-}SGD})} - F(x^\star) &\leq c_3\cdot\rb{\frac{HB^2}{KR} + \frac{\sigma_2 B}{\sqrt{KR}}}\enspace.
\end{aligned}
\end{equation}

Moreover, \citet{woodworth2020local} showed that an accelerated variant of Local SGD, inspired by the acceleration technique of \citet{ghadimi2012optimal}, achieves the following improved rate:
\begin{align}\label{rate:acc_lsgd_quad_hom}
    \ee\sb{F(x^{Acc\textnormal{-}L\textnormal{-}SGD})} - F(x^\star) \leq c_4\cdot\rb{\frac{HB^2}{K^2R^2} + \frac{\sigma_2 B}{\sqrt{MKR}}}\enspace,
\end{align}
which improves upon the accelerated rates for both mini-batch and single-machine SGD:
\begin{equation}\label{rate:acc_mb_sm_sgd_quad_hom}
\begin{aligned}
    \ee\sb{F(x^{Acc-MB\textnormal{-}SGD})} - F(x^\star) &\leq c_5\cdot\rb{\frac{HB^2}{R^2} + \frac{\sigma_2 B}{\sqrt{MKR}}}\enspace,\\
    \ee\sb{F(x^{Acc-SM\textnormal{-}SGD})} - F(x^\star) &\leq c_6\cdot\rb{\frac{HB^2}{K^2R^2} + \frac{\sigma_2 B}{\sqrt{KR}}}\enspace.
\end{aligned}
\end{equation}

Importantly, \citet{woodworth2020local} also showed that the rate in \eqref{rate:acc_lsgd_quad_hom} is min-max optimal: there exists a quadratic homogeneous instance satisfying \Cref{ass:convex,ass:smooth_second,ass:smooth_third,ass:bounded_optima,ass:stoch_bounded_second_moment} on which no distributed zero-respecting algorithm (cf. \Cref{def:zero_respecting}) can perform better. This result thus characterizes the min-max complexity of convex quadratic stochastic optimization under distributed settings.

In a follow-up work, \citet{woodworth2021min} extended this analysis to general smooth convex functions. A key finding was that, for such functions, the min-max optimal convergence rate is achieved by the best of (accelerated) single-machine SGD and mini-batch SGD. That is, no algorithm can achieve a strictly better rate than the minimum of the two in \eqref{rate:acc_mb_sm_sgd_quad_hom}. Their construction of a non-quadratic hard instance revealed that Local SGD may perform worse in high-noise regimes for general convex functions. In particular, Local SGD outperforms mini-batch SGD only in the regime where single-machine SGD already does—thus limiting its practical benefit in those settings.

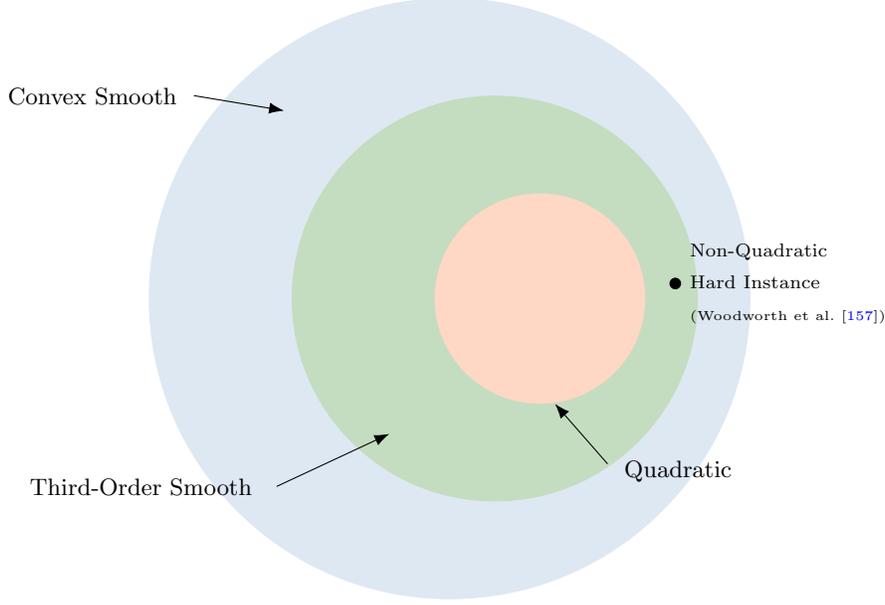
\begin{figure}
    \centering
    \begin{tikzpicture}

        \fill[convexblue!20] (-3,0) circle (4cm);
        
        \fill[thirdgreen!45] (-2.4,0) circle (2.7cm);
        
        \fill[quadorange!35] (-1.8,0) circle (1.4cm);
        
        \node[left] at (-6.5,2.7) {\small Convex Smooth};
        \draw[-{Latex[length=2mm]}] (-6.4,2.7) -- (-5.2,2.5);
        
        \node[left] at (-5.5,-2.5) {\small Third-Order Smooth};
        \draw[-{Latex[length=2mm]}] (-5.3,-2.5) -- (-3.8,-1.8);
        
        \node[right] at (-0.8,-2.3) {\small Quadratic};
        \draw[-{Latex[length=2mm]}] (-0.9,-2.2) -- (-1.6,-1.4);
        
        \filldraw[black] (0,0.2) circle (2pt);
        \node[align=left] at (1.5,0.2) {\scriptsize Non-Quadratic\\\scriptsize Hard Instance\\\tiny (\citet{woodworth2021min})};
    
    \end{tikzpicture}

    \caption{The class of convex and third-order smooth problems satisfying \Cref{ass:convex,ass:smooth_third} interpolates between the class of problems with quadratic objective functions and convex-smooth functions satisfying \Cref{ass:convex,ass:smooth_second}. Notably, the lower bound of \citet{woodworth2021min} is a non-quadratic instance.}
    \label{fig:third_smooth}
\end{figure}

This motivates the need to reconcile the favorable results for quadratic problems with the limitations in general convex settings. Here, \Cref{ass:smooth_third} proves useful, as it enables interpolation between smooth convex functions: from the very smooth case (with $Q=0$) to less smooth settings (with $Q \approx 2H$). Under this assumption, \citet{yuan2020federated} derived the following convergence guarantee for Local SGD:
\begin{align}\label{rate:lsgd_third_smooth_hom}
    \ee\sb{F(x^{L\textnormal{-}SGD})} - F(x^\star) = \tilde\ooo\rb{\frac{HB^2}{KR} + \frac{\sigma_2 B}{\sqrt{MKR}} + \frac{Q^{1/3}\sigma_4^{2/3}B^{5/3}}{K^{1/3}R^{2/3}} }\enspace,
\end{align}
where $\tilde\ooo$ hides constants and logarithmic factors in problem-dependent parameters.\footnote{Techniques exist to remove logarithmic factors (e.g.,~\cite{lacoste2012simpler, stich2019unified}), but since they are often dominated by polynomial terms, we hide them throughout this thesis.}

This guarantee can be strictly better than those in \eqref{rate:mb_sm_sgd_quad_hom} when $Q$ and $\sigma_4$ are small. \citet{yuan2020federated} also proposed an accelerated variant of Local SGD with the following rate:
\begin{align}\label{rate:acc_lsgd_third_smooth_hom}
    \ee\sb{F(x^{Acc\textnormal{-}L\textnormal{-}SGD})} - F(x^\star) = \tilde\ooo\rb{\frac{HB^2}{KR^2} + \frac{\sigma_2 B}{\sqrt{MKR}} + \frac{H^{1/3}\sigma_2^{2/3}B^{4/3}}{M^{1/3}K^{1/3}R} + \frac{Q^{1/3}\sigma_4^{2/3}B^{5/3}}{K^{1/3}R^{4/3}}}\enspace.
\end{align}

Whether this convergence rate is tight, and whether this accelerated variant is min-max optimal, remains an open question. Notably, \eqref{rate:acc_lsgd_third_smooth_hom} does not match the following lower bound derived by \citet{woodworth2021min} for any distributed zero-respecting algorithm in the homogeneous, third-order smooth setting:
\begin{align}\label{lb:third_smooth}
    \ee\sb{F(\hat x)} - F(x^\star) = \tilde\Omega\Bigg( \frac{HB^2}{K^2R^2} + \min \cb{\frac{\sigma B}{\sqrt{MKR}}, HB^2} + \min\cb{\frac{HB^2}{R^2}, \frac{\sqrt{Q\sigma}B^2}{K^{1/4}R^2}, \frac{\sigma B}{\sqrt{KR}}}\Bigg)\enspace.
\end{align}

While we do not delve further into accelerated local-update algorithms in this thesis, it will be helpful to compare our subsequent convergence bounds with the lower bound in \eqref{lb:third_smooth}.

\begin{remark}[Optimal Rates for Accelerated Mini-batch and Single-Machine SGD]\label{rem:mb_rate_is_tight}
The rates in \eqref{rate:acc_mb_sm_sgd_quad_hom} for accelerated mini-batch and single-machine SGD are tight, even for smooth convex quadratic functions. In the case of accelerated single-machine SGD, this follows from classical lower bounds in the serial setting~\citep{nesterov2018lectures,nemirovski1994efficient,ghadimi2012optimal}.

For accelerated mini-batch SGD, observe that in the homogeneous setting, where all machines share the same objective, the algorithm effectively uses $MK$ stochastic gradients per update over $R$ communication rounds. This is equivalent to performing $R$ accelerated updates with reduced stochastic noise variance $\sigma_2^2/(MK)$, as captured by \Cref{ass:stoch_bounded_second_moment}. Thus, the optimal rate in this setting matches that of accelerated single-machine SGD with appropriately reduced variance, which is precisely the rate stated in \eqref{rate:acc_mb_sm_sgd_quad_hom}.
\end{remark}

Three key takeaways emerge from analyzing the homogeneous setting (cf. \Cref{fig:third_smooth}):
\begin{itemize}
    \item \textbf{Extreme communication efficiency}, as seen in the convex quadratic setting \eqref{rate:acc_lsgd_quad_hom}, is a highly desirable property of local update algorithms. Ideally, in heterogeneous settings with large $K$, we aim for communication complexity to improve as data heterogeneity decreases.
    
    \item \textbf{Third-order smoothness} (cf. \Cref{ass:smooth_second}) plays a crucial role in establishing the effectiveness of Local SGD even in homogeneous scenarios. This highlights its potential as a structural property to exploit in the heterogeneous case as well---a central theme in our subsequent analysis.
    
    \item \textbf{Min-max optimality in smooth convex problems} is attained by the best of single-machine SGD and mini-batch SGD~\citep{woodworth2021min}. This is surprising given that Local SGD often outperforms both in practice~\citep{mcmahan_ramage_2017,charles2021large}. A likely reason for this discrepancy is that the homogeneous model is overly simplistic: in real-world applications, data across machines is typically similar but not identical. Capturing this mild heterogeneity is the first step we take in the next section.
\end{itemize}

\begin{remark}[Local SGD as a Quadratic Solver]
    The strong performance of Local SGD on quadratic objectives, as shown in the convergence rates \eqref{rate:lsgd_quad_hom} and \eqref{rate:acc_lsgd_quad_hom}, is highly encouraging. It naturally motivates a broader strategy: reduce the task of optimizing general convex objectives to a sequence of well-chosen quadratic subproblems. This reduction underlies many classical second-order methods, including Newton’s method~\citep{nesterov1994interior}, trust-region methods~\citep{nocedal2006numerical,carmon2020acceleration}, and cubic regularization~\citep{nesterov2006cubic}. It also inspires more recent approaches that go beyond second-order information~\citep{nesterov2019implementable,bullins2020highly}.

    In~\cite{bullins2021stochastic}, we leverage Local SGD to implement a distributed stochastic Newton method. When the objective is highly smooth---specifically, quasi self-concordant~\citep{bach2010self}---this method can provably outperform existing first-order distributed algorithms. However, we do not explore these results in this thesis, as our focus is on understanding the intrinsic value of local updates. The stochastic Newton method abstracts away the role of local updates, and thus lies beyond the scope of our present discussion.
\end{remark}

\section{Mild Heterogeneity: The Case of Shared Optimizers}\label{sec:ch3.2}
We begin our discussion by revisiting the simple problem instance introduced in \Cref{sec:is_first_enough}. Consider two machines, each optimizing a two-dimensional objective:
\begin{equation}\label{eq:simple_example_shared_optimum}
    \begin{aligned}
    F_1(x) &:= \frac{1}{2}(x - x^\star)^T \begin{bmatrix}
        H & 0 \\
        0 & 0
    \end{bmatrix} (x - x^\star)\enspace,\\
    F_2(x) &:= \frac{1}{2}(x - x^\star)^T \begin{bmatrix}
        0 & 0 \\
        0 & H
    \end{bmatrix} (x - x^\star)\enspace.
\end{aligned}
\end{equation}

While this setup does not fall under the homogeneous setting, it exhibits \emph{minimal heterogeneity} in the sense that both machines share the same optimizer $x^\star$. We already discussed in \Cref{sec:is_first_enough} that even this small amount of heterogeneity can preclude the extreme communication efficiency seen in the homogeneous quadratic setting for vanilla Local SGD.

This example is compelling, though, as it makes another point. Despite being convex and smooth, the minimal heterogeneity is enough to make single-machine SGD ineffective. Specifically, if one runs SGD on only one of the two machines and evaluates sub-optimality for the average objective, the error remains lower bounded by $HB^2$---regardless of how long SGD is run.

This motivates a re-interpretation of the min-max optimality result from \citet{woodworth2021min}, which identifies the best of mini-batch and single-machine SGD as optimal for smooth convex problems. In the homogeneous case, this result is intuitive: the sole benefit of collaboration is variance reduction through averaging stochastic gradients, a benefit that mini-batch SGD already captures. Thus, the lower bound essentially reflects a dichotomy between low and high stochastic gradient noise $\sigma_2$.

However, when the objectives are only connected via sharing a common optimizer while exhibiting other functional differences, this dichotomy breaks down. In such heterogeneous settings, collaboration can be beneficial even when the noise is low: to find a shared optimizer quickly. One might ask whether the appropriate baseline in this case is to run SGD independently on each machine. However, this strategy is also insufficient: in the above example, there is no mechanism in place to enforce consensus, and independent optimization may lead to significantly different local models. Averaging such models would not necessarily yield a meaningful or accurate solution.

This also highlights why \emph{one-shot averaging}, which works well in homogeneous quadratic problems, cannot be expected to succeed in the heterogeneous setting without additional assumptions. In the remainder of this section, we formalize these insights by deriving tight convergence guarantees for Local SGD under the assumption that the machines’ objectives share a common optimizer. We will also demonstrate that, under this minimal heterogeneity, accelerated mini-batch SGD is indeed optimal for this problem class.

We begin with the following new lower bound for Local SGD, which incorporates \Cref{ass:phi_star}. Note that when the machines share an optimizer, the quantity $\phi_\star$ in the assumption is zero. Thus, this lower bound remains valid even when the objectives do not share a common minimizer.

\begin{theorem}\label{thm:new_LSGD_lower_bound}
There exists a problem instance satisfying \Cref{ass:convex,ass:smooth_second,ass:bounded_optima,ass:stoch_bounded_second_moment,ass:phi_star}, such that for all $K \geq 2$, the final iterate $\bar{x}_R$ of Local SGD as defined in \eqref{eq:local_updates}, initialized at zero and using any step sizes $\eta,\ \beta \geq 0$, satisfies for a numerical constant $c_6$:
\begin{align*}
\ee\sb{F(\bar{x}_R)} - F(x^\star)
\;\geq\; c_6\cdot\rb{\frac{HB^2}{R}
+ \frac{(H\sigma_2^2 B^4)^{1/3}}{K^{1/3} R^{2/3}}
+ \frac{\sigma_2 B}{\sqrt{MKR}}
+ \frac{(H\phi_\star^2 B^4)^{1/3}}{R^{2/3}}}\enspace.
\end{align*}
\end{theorem}

\begin{remark}[No Extreme Communication Efficiency]
The lower bound in \Cref{thm:new_LSGD_lower_bound} rules out the possibility of extreme communication efficiency for Local SGD under bounded first-order heterogeneity, as captured by \Cref{ass:phi_star}. This is consistent with the behavior observed in the simple example from \eqref{eq:simple_example_shared_optimum}. Furthermore, since the hard instance used to derive the first term of the lower bound is quadratic, we cannot appeal to third-order smoothness (\Cref{ass:smooth_third}) to improve the rate. The best previously known lower bound in this setting, due to \citet{glasgow2022sharp}, vanished as $K \to \infty,\ \phi_\star\to 0$ and therefore did not preclude extreme communication efficiency when the machines had a shared optimum.
\end{remark}

\begin{remark}[Tightness under \Cref{ass:phi_star}]
\citet{koloskova2020unified} established a matching upper bound to \Cref{thm:new_LSGD_lower_bound}, which means our lower bound is tight and fully characterizes the min-max optimal convergence rate of Local SGD under \Cref{ass:phi_star}. This is noteworthy, as several prior works~\citep{woodworth2020local,woodworth2020minibatch,glasgow2022sharp} had speculated that Local SGD might achieve faster convergence under \Cref{ass:phi_star}.

Moreover, observe that the first two terms in the lower bound are identical to those that appear in the convergence rate for mini-batch SGD (cf. \eqref{rate:mb_sm_sgd_quad_hom}). Thus, under \Cref{ass:phi_star}, there is no provable separation between Local SGD and mini-batch SGD, implying that this assumption alone is insufficient to explain the empirical dominance of Local SGD. This suggests that additional structural assumptions on data heterogeneity are necessary to identify regimes in which Local SGD can provably outperform mini-batch SGD---thereby reconciling theory with empirical observations.
\end{remark}

While the main idea behind the proof of \Cref{thm:new_LSGD_lower_bound} resembles the simple example in \eqref{eq:simple_example_shared_optimum}, recall that in that case, setting $\beta = 2$ made Local SGD converge with a single round of communication. To obtain a lower bound that holds for arbitrary outer step sizes $\beta$, we modify the Hessian of the second objective. This change effectively reduces the local updates within a communication round to a single SGD step. With this structure in place, we invoke the following auxiliary result:

\begin{lemma}\label{lem:condition_number}
There exists a convex quadratic function $F(x)$ over $x \in \mathbb{R}^2$, which is $H$-smooth, $\mu$-strongly convex with condition number $\kappa = 12R$, and whose minimizer $x^\star$ satisfies $\norm{x^\star} \leq B$, such that the $R^{\text{th}}$ gradient descent iterate $\hat{x}_R$ (initialized at zero and using any step size $\eta > 0$) satisfies
\[
F(\hat{x}_R) - F(x^\star) \;\geq\; \frac{HB^2}{8R}\enspace.
\]
\end{lemma}

We prove \Cref{lem:condition_number,thm:new_LSGD_lower_bound} in \Cref{app:chap3}.

A natural question arises: if Local SGD is not optimal, what is the best algorithm for the class of minimally heterogeneous problems where all machines share a common optimizer? We have already ruled out single-machine SGD as a viable approach, and \Cref{thm:new_LSGD_lower_bound} establishes that Local SGD cannot strictly outperform mini-batch SGD under \Cref{ass:phi_star}. This raises the possibility that (accelerated) mini-batch SGD may, in fact, be min-max optimal in this setting. We confirm this intuition by proving the following lower bound for all distributed zero-respecting algorithms (cf. \Cref{def:zero_respecting}).

\begin{theorem}[Algorithm Independent Lower Bound]\label{thm:AIlb_zeta0}
There exists a problem instance satisfying \Cref{ass:convex,ass:smooth_second,ass:bounded_optima,ass:stoch_bounded_second_moment,ass:phi_star}, such that for all $K \geq 2$, the final iterate $\hat{x}$ of any distributed zero-respecting algorithm initialized at zero with $R$ rounds of communication and $K$ stochastic gradient computations per machine per round satisfies,
\begin{align}
 \mathbb{E}\left[F(\hat{x})\right] - F(x^{\star}) \geq c_7\cdot\rb{\frac{HB^2}{R^2} + \frac{\sigma B}{\sqrt{MKR}}}\enspace.
\end{align}
\end{theorem}
\begin{remark}[Min-max Optimality of Accelerated Mini-batch SGD]
    The lower bound above is matched by accelerated mini-batch SGD~\citep{ghadimi2012optimal}, establishing its min-max optimality under \Cref{ass:phi_star}. This resolves a line of work on the intermittent communication setting under this assumption~\citep{khaled2020tighter, karimireddy2020scaffold, koloskova2020unified, woodworth2020minibatch, glasgow2022sharp, wang2022unreasonable}. Notably, the convergence rate of mini-batch SGD is \emph{independent of data heterogeneity}~\citep{woodworth2020minibatch}, as it relies solely on variance-reduced stochastic gradients for the averaged objective $F$, without requiring alignment across local objectives. Another way to see this is by assigning the same objective to all machines: specifically, the quadratic objective for which the rate in \eqref{rate:acc_mb_sm_sgd_quad_hom} is known to be tight~\citep{nesterov2018lectures,nemirovski1994efficient}. This fully homogeneous construction satisfies any heterogeneity assumption. Similarly, we can show that the convergence rate for mini-batch SGD in \eqref{rate:mb_sm_sgd_quad_hom} is tight, under any notion of data heterogeneity, using \Cref{lem:condition_number} and a standard mean estimation lower bound.
\end{remark}
The proof of \Cref{thm:AIlb_zeta0} is technically interesting and somewhat different from the intuition in \eqref{eq:simple_example_shared_optimum} but to avoid digressing from the main narrative, we defer it to \Cref{app:chap3}.

Taken together, the observations in this section highlight the limitations of mini-batch SGD in exploiting benign problem structure, and motivate the analysis of Local SGD under stronger heterogeneity assumptions beyond \Cref{ass:phi_star}—where it has the potential to outperform mini-batch methods. To this end, the next section shows how \Cref{ass:tau} can circumvent the lower bound in \Cref{thm:new_LSGD_lower_bound}.

\section{Formalizing the Role of Second-order Heterogeneity}\label{sec:ch3.3}
Several recent works have established that \Cref{ass:tau} plays a central role in determining the communication complexity of distributed optimization. A prominent line of research has focused on distributed proximal-point methods. Early results in the quadratic setting showed that when $\tau = 0$, these methods achieve \emph{extreme communication efficiency}, requiring only a constant number of communication rounds~\citep{shamir2014dane}. More recent work extended these guarantees to the general case with $\tau > 0$~\citep{sun2022distributed,kovalev2022optimal,jiang2024fedred,jiang2024dane}. This naturally raises the question: can Local SGD—and more broadly, local-update algorithms—benefit from smaller second-order heterogeneity $\tau$?

To motivate this possibility, we revisit the example in \eqref{eq:simple_example_shared_optimum}, where the second-order heterogeneity is $H$, the worst possible value for $\tau$. To construct an instance satisfying both \Cref{ass:smooth_second,ass:tau}, we decouple $\tau$ from $H$ by introducing an additional dimension:
\begin{equation}\label{eq:simple_example_shared_optimum_w_tau}
    \begin{aligned}
    F_1(x) &:= \frac{1}{2}(x - x^\star)^T \begin{bmatrix}
        \tau & 0 & 0 \\
        0 & 0 & 0\\
        0 & 0 & H
    \end{bmatrix} (x - x^\star)\enspace,\\
    F_2(x) &:= \frac{1}{2}(x - x^\star)^T \begin{bmatrix}
        0 & 0 & 0\\
        0 & \tau & 0\\
        0 & 0 & H
    \end{bmatrix} (x - x^\star)\enspace.
    \end{aligned}
\end{equation}

This instance is $H$-smooth and $\tau$-second-order heterogeneous. In the shared (third) direction of high curvature, local updates are effective and enable extreme communication efficiency. In contrast, in the first two directions, local updates remain ineffective even as $K \to \infty$. The local SGD iterate after $R$ communication rounds is:
\begin{align*}
  \bar x_R = x^\star - \begin{bmatrix}
      x^\star[1]\rb{1- \frac{\beta}{2}\rb{1 - (1-\eta \tau)^K}}^R\\
      x^\star[2]\rb{1- \frac{\beta}{2}\rb{1 - (1-\eta \tau)^K}}^R\\
      x^\star[3]\rb{1- \beta\rb{1 - (1-\eta H)^K}}^R
  \end{bmatrix}\enspace, 
\end{align*}
which for $\beta = 1$ simplifies to:
\begin{align*}
  \bar x_R - x^\star = -\begin{bmatrix}
      x^\star[1]\rb{\frac{1 + (1-\eta\tau)^K}{2}}^R\\
      x^\star[2]\rb{\frac{1 + (1-\eta\tau)^K}{2}}^R\\
      x^\star[3](1-\eta H)^{KR}
  \end{bmatrix}\enspace.
\end{align*}

This yields the following sub-optimality:
\begin{align*}
    F(\bar x_R) - F(x^\star) &= \frac{\tau}{4} \left(\rb{x^\star[1]}^2 + \rb{x^\star[2]}^2\right) \cdot \rb{\frac{1 + (1-\eta\tau)^K}{2}}^{2R} + \frac{H\rb{x^\star[3]}^2}{2}(1-\eta H)^{2KR}\enspace,\\ 
    &\geq^{\text{(a)}} \frac{\tau}{4} \left(\rb{x^\star[1]}^2 + \rb{x^\star[2]}^2\right) \cdot\rb{1- \frac{\eta \tau K}{2}}^{2R}+ \frac{H\rb{x^\star[3]}^2}{2}(1-\eta H)^{2KR}\enspace,
\end{align*}
where in (a) we apply Bernoulli's inequality. When $\tau = 0$, the first term vanishes, while the second term decays with $K$, allowing extreme communication efficiency. Even for small $\tau > 0$, we observe improved communication efficiency as $\tau$ decreases. We now formalize this insight with the following lower bound:

\begin{theorem}\label{thm:new_LSGD_lower_bound_with_tau}
There exists a problem instance satisfying \Cref{ass:convex,ass:smooth_second,ass:bounded_optima,ass:stoch_bounded_second_moment,ass:phi_star,ass:tau}, such that for all $K \geq 2$, the final iterate $\bar{x}_R$ of Local SGD (as defined in \eqref{eq:local_updates}), initialized at zero and using any step sizes $\eta,\ \beta \geq 0$, satisfies for some constant $c_8$:
\begin{align*}
        F(x_{KR}) - F(x^\star)  &\geq c_8\cdot\Bigg( \frac{\tau B^2}{R} + \frac{H B^2}{K R} + \frac{\sigma_2 B}{\sqrt{M K R}} + \min\left\{ \frac{\sigma_2 B}{\sqrt{K R}}, \frac{H^{1/3} \sigma_2^{2/3} B^{4/3}}{K^{1/3} R^{2/3}} \right\} \\
        &\qquad\qquad\qquad\qquad+ \tau\cdot \min\left\{ \phi_\star^2, \frac{\phi_\star^{2/3} B^{4/3}}{R^{2/3}} \right\} \Bigg)\enspace.
    \end{align*}
\end{theorem}

The proof (in \Cref{app:chap3}) builds on the construction in \eqref{eq:simple_example_shared_optimum_w_tau} and introduces a rotation for the second machine as in the proof of \Cref{thm:new_LSGD_lower_bound}. It also utilizes \Cref{lem:condition_number} along with existing hard instances from~\citep{glasgow2022sharp,nemirovski1994efficient}.

\begin{remark}[$\tau$ and Communication Complexity]
When $\phi_\star$ is small and $K$ is large, the lower bound is dominated by the term $\frac{\tau B^2}{R}$, suggesting that the communication complexity of Local SGD scales as $\frac{\tau B^2}{\epsilon}$—mirroring results for non-convex optimization~\citep{murata2021bias,patel2022towards} and distributed proximal methods~\citep{shamir2014dane,sun2022distributed,kovalev2022optimal,jiang2024fedred,jiang2024dane}.
\end{remark}

\begin{remark}[Comparison to \Cref{thm:new_LSGD_lower_bound}]
\Cref{thm:new_LSGD_lower_bound} does not incorporate \Cref{ass:tau}. Setting $\tau = 0$ in its hard instance would also eliminate smoothness ($H = 0$), making the problem trivial. In contrast, the construction in \Cref{thm:new_LSGD_lower_bound_with_tau} introduces an extra dimension to decouple the effects of $\tau$ and $H$, analogous to the difference between \eqref{eq:simple_example_shared_optimum} and \eqref{eq:simple_example_shared_optimum_w_tau}.
\end{remark}

\subsection{Potential Future Improvements to the Lower Bound}

When $\tau = 0$, the lower bound loses all dependence on $\phi_\star$, reducing to the homogeneous case studied by \citet{glasgow2022sharp}. While that bound is tight for homogeneous problems, we do not expect first-order heterogeneity to become irrelevant when $\tau = 0$. We suspect that in the last term of \Cref{thm:new_LSGD_lower_bound_with_tau}, $\tau$ could be replaced by $H$. Furthermore, the current lower bound can not highlight the dependence on third-order smoothness $Q$. Resolving both these issues remains an open question. 

Furthermore, our bound does not depend on $\zeta_\star$: the difficulty is captured entirely by $\phi_\star$. In our hard instance, $\zeta_\star \approx \phi_\star$, and since $\phi_\star \geq \zeta_\star$ in general (see \Cref{rem:zeta_star_vs_phi_star,rem:zeta_star_tau_vs_phi_star}), we state the result in terms of $\phi_\star$. Deriving a lower bound that distinguishes between $\zeta_\star$ and $\phi_\star$---i.e., the proximity of local optima versus the recoverability of $S^\star$---remains open. 

Finally, all known quadratic lower bounds for Local SGD~\citep{woodworth2020minibatch,glasgow2022sharp,patel2024limits} assume bounded second moments. It is unknown whether tighter lower bounds can be derived by leveraging higher-order moments of the stochastic gradients to ``confuse'' the local updates. This, too, is an open direction.

This concludes our discussion of lower bounds. In the following two sections, we present upper bounds for Local SGD, guided by the insights developed in this section. Our goal is to derive guarantees that leverage both small third-order smoothness $Q$ and low second-order heterogeneity $\tau$, ideally recovering extreme communication efficiency in favorable regimes of data heterogeneity.

\chapter{On the Fixed Point Perspective for Local SGD}\label{ch:fixed_point}

In this chapter, we initiate our analysis of Local SGD in the heterogeneous setting, focusing on quadratic objectives. By initially setting aside the effects of third-order smoothness \( Q \), we aim to isolate and understand the role of data heterogeneity. The structure of quadratic functions allows us to exploit closed-form expressions for their gradients, which we use to study the limiting behavior of Local SGD as the number of communication rounds \( R \) becomes large. This fixed-point analysis provides key insights into how heterogeneity influences convergence and serves as a complementary perspective to the finite-time upper bounds we develop in the next chapter.

Our main contributions in this chapter are as follows:
\begin{enumerate}
    \item We characterize the limiting behavior of Local SGD in the strongly convex quadratic setting. In \Cref{prop:conv_to_fixed}, we show that as \( R \to \infty \), the iterates converge to a fixed point, highlighting Local SGD’s extreme communication efficiency. We then quantify the discrepancy between this fixed point and the global optimum in \Cref{lem:fixed_disc_UB}, providing a non-asymptotic upper bound that depends on the step-size \( \eta \), the number of local updates \( K \), and the spectral properties of the objective. Finally, in \Cref{thm:conv_with_fixed_pt_pers}, we combine these insights to obtain a finite-time convergence bound for Local SGD under data heterogeneity, capturing the trade-offs between optimization error, statistical variance, and heterogeneity-induced bias.

    \item In \Cref{prop:fixed_cvx}, we extend our analysis to the general convex (non-strongly convex) setting. We characterize the limiting behavior of Local SGD as the minimum-norm solution to a reweighted least-squares problem, where the reweighting reflects the interaction between local updates and the step-size. 
\end{enumerate}

Together, these results lay the foundation for our extension to non-quadratic convex objectives in the next chapter.

\subsection*{Outline and Relevant References}

The results in this chapter were first developed in~\citet{patel2024limits} and further refined in~\citet{patel2025revisiting}, in collaboration with Margalit Glasgow, Ali Zindari, Lingxiao Wang, Sebastian U.~Stich, Ziheng Cheng, Nirmit Joshi, and Nathan Srebro. While prior works~\citep{charles2020outsized,malinovskiy2020local} have analyzed the asymptotic behavior of Local SGD, our contribution introduces an explicit dependence on data heterogeneity through the parameters \( \tau \) and \( \zeta_\star \) (see \Cref{ass:tau,ass:zeta_star}), providing tighter and more interpretable guarantees.

\Cref{sec:ch4.1} presents our results for strongly convex quadratics, including convergence to a fixed point and non-asymptotic upper bounds on the fixed-point discrepancy. \Cref{sec:fixed_derivation_convex} extends this analysis to general convex quadratics, culminating in a novel fixed-point characterization as the minimum-norm solution to a reweighted least-squares problem. Finally, \Cref{sec:regularize} connects this geometric perspective to existing work on implicit regularization by local updates~\citep{barba2021implicit,limanalyzing,gu2023and}, showing how Local SGD's update structure induces a form of spectral filtering that biases learning toward directions of consensus across clients.

\section{Fixed-point Analysis for Local SGD on Strongly Convex Quadratics}\label{sec:ch4.1}

Several works have pointed out with varying levels of explicitness~\cite{malinovskiy2020local,charles2020outsized,patel2024limits} that the hardness of analysing Local SGD's convergence comes from a fixed-point discrepancy, i.e., Local SGD in the limit of large $R$ converges to a point different from any $x^\star\in S^\star$ whenever $K>1$. Our goal in this section is to write this fixed-point $x_\infty$ explicitly for strongly convex quadratic functions. Then we will: (i) show that Local SGD converges very quickly---with extreme communication efficiency---to this fixed point; and (ii) bound the fixed point discrepancy, i.e., $\norm{x_\infty-x^\star}$ in terms of \Cref{ass:zeta_star,ass:phi_star,ass:tau}. 

We will begin our analysis with the strongly convex setting where $x^\star$ and $x_\infty$ (if it exists) will be unique. In particular, we assume that the objective on each machine is quadratic satisfying \Cref{ass:strongly_convex,ass:smooth_second,ass:bounded_optima} and of the form,
\begin{align}
    F_m(x) &= \frac{1}{2}(x-x_m^\star)^TA_m(x-x_m^\star)\enspace, &&\forall\ m\in[M]\enspace,\label{eq:quad_prob_form}
\end{align}
where $0 \prec \mu\cdot I_d \preceq A_m \preceq H.\cdot I_d$ and $x_m^\star$ is the unique optimizer of machine $m$. 

\subsection{Deriving the Closed Form for the Fixed Point}
To first motivate what the fixed-point of Local SGD should be, we consider the noiseless setting---i.e., when our first-order oracles return exact gradients $\nabla F_m(\cdot)$. Assuming the local SGD algorithm converges, i.e., the hyper-parameters are set to achieve that and $R\to \infty$, we would like to calculate $x_\infty$. $x_\infty$ must satisfy the following fixed-point equation (cf., \eqref{eq:local_updates}),
\begin{align*}
    x_\infty = x_\infty + \frac{\beta}{M}\sum_{m\in[M]}\Delta^m(x_\infty) \equiv \sum_{m\in[M]}\Delta^m(x_\infty) =0\enspace,
\end{align*}
where $\Delta^m(x_\infty)$ is the update on machine $m$ for a communication round starting at the fixed point $x_\infty$. Note that the above equation does not depend on $\beta$. Unwinding the update, we get the following,
\begin{align*}
    &\sum_{m\in[M]}\Delta^m(x_\infty) =0\Leftrightarrow\sum_{m\in[M]}\rb{x_m^\star + \rb{I-\eta A_m}^K\rb{x_\infty-x_m^\star} - x_\infty} = 0\enspace,\\
    &\Leftrightarrow \sum_{m\in[M]}\rb{I-\rb{I-\eta A_m}^K}x_m^\star = \sum_{m\in[M]}\rb{I-\rb{I-\eta A_m}^K}x_\infty \enspace,\\
    &\Leftrightarrow x_\infty = \frac{1}{M}\sum_{m\in[M]}C^{-1}C_mx_m^\star,
\end{align*}
where $C_m := I- (I-\eta A_m)^K$, and $C := \frac{1}{M}\sum_{m\in[M]}C_m$ and we assume $\eta <1/H$ so that $C_m\succ 0$ for each $m\in[M]$. Note that $x_\infty(\eta, K)$ is a function of $\eta, K$ and is unaffected by the choice of $\beta$. Now we will show that even when we only have an inexact stochastic oralce as in \Cref{def:oracle_first} we converge to this fixed-point in expectation.

\subsection{Fast Convergence to the Fixed-point}
The following Lemma will show that Local SGD converges to $x_\infty$ derived above with extreme communication efficiency.

\begin{lemma}\label{lem:fast_conv_to_fixed_point}
    For quadratic problems of the form \eqref{eq:quad_prob_form} satisfying~\Cref{ass:strongly_convex,ass:smooth_second,ass:bounded_optima,ass:stoch_bounded_second_moment}, with $\eta <\frac{1}{H}$, and $\beta \leq \frac{1}{1-(1-\eta H)^K}$ the Local-SGD iterate $\bar x_{R}$ (with initialization $\bar x_0=0$) satisfies,
    \begin{align*}
        \ee\sb{\norm{\bar x_R - x_\infty}^2}\leq \rb{1-\beta\rb{1 - (1-\eta \mu)^K}}^{2R}\norm{x_\infty}^2 + \eta\beta\rb{1-(1-\beta\rb{1-(1-\eta \mu)^K})^{R}}\frac{\sigma^2}{\mu M}\enspace,
    \end{align*}
    where we define $x_\infty := \frac{1}{M}\sum_{m\in[M]}C^{-1}C_mx_m^\star$ for $C_m := I- (I-\eta A_m)^K$ and $C:= \frac{1}{M}\sum_{m\in[M]}C_m$. In particular, when $\beta=1$ we have,
    \begin{align*}
        &\ee\sb{\norm{\bar x_R - x_\infty}^2} \leq \rb{1-\eta\mu}^{2KR}\norm{x_\infty}^2 + \eta\rb{1-\rb{1-\eta\mu}^{KR}}\frac{\sigma^2}{\mu M}\enspace.
    \end{align*}
\end{lemma}
\begin{proof}
We note the following about the local-SGD updates between two communication rounds on machine $m\in[M]$,
\begin{align*}
    x_{r,K}^m - x_m^\star &= x_{r,K-1}^m - x_m^\star - \eta A_m(x_{r,K-1}^m - x_m^\star) + \eta \rb{A_m(x_{r,K-1}^m - x_m^\star) - \nabla f(x_{r,K-1}^m; z_{r,K-1}^m)}\enspace,\\
    &= \rb{I-\eta A_m}^K\rb{x_{r,0}^m - x_m^\star} + \eta \sum_{k=0}^{K-1}\rb{I-\eta A_m}^{K-1-k}\rb{A_m(x_{r,k}^m - x_m^\star) - \nabla f(x_{r,k}^m; z_{r,k}^m)}\enspace,\\
    &= \rb{I-\eta A_m}^K\rb{x_{r-1} - x_m^\star} + \eta \sum_{k=0}^{K-1}\rb{I-\eta A_m}^{K-1-k}\xi_{r,k}^m\enspace,
\end{align*}
where we denote by $\xi_{r,k}^m$ the stochastic noise on machine $m$ at local step $k$ leading up to round $r$. This implies the following
\begin{align*}
    x_{r,K}^m - x_{r-1} &= x_m^\star - x_{r-1} + \rb{I-\eta A_m}^K\rb{x_{r-1} - x_m^\star} + \eta \sum_{k=0}^{K-1}\rb{I-\eta A_m}^{K-1-k}\xi_{r,k}^m\enspace,\\
    &= -\rb{I - \rb{I-\eta A_m}^K}\rb{x_{r-1} - x_m^\star}  + \eta \sum_{k=0}^{K-1}\rb{I-\eta A_m}^{K-1-k}\xi_{r,k}^m\enspace,\\
    &= -C_m\rb{x_{r-1} - x_m^\star}+ \eta \sum_{k=0}^{K-1}\rb{I-\eta A_m}^{K-1-k}\xi_{r,k}^m\enspace,\\
    &= -C_mx_{r-1} + C_mx_m^\star + \eta \sum_{k=0}^{K-1}\rb{I-\eta A_m}^{K-1-k}\xi_{r,k}^m\enspace.
\end{align*}
This implies for the $r$-th synchronized model,
\begin{align*}
    x_r &= x_{r-1} + \frac{\beta}{M}\sum_{m\in[M]}\Bigg(-C_mx_{r-1} + C_mx_m^\star + \eta \sum_{k=0}^{K-1}\rb{I-\eta A_m}^{K-1-k}\xi_{r,k}^m\bigg)\enspace,\\
    &= \rb{I-\beta C}x_{r-1} + \frac{\beta}{M}\sum_{m\in[M]}C_mx_m^\star + \eta\beta\sum_{k=0}^{K-1}\rb{I-\eta A_m}^{K-1-k}\rb{\frac{1}{M}\sum_{m\in[M]}\xi_{r,k}^m}\enspace,\\
    &=  \rb{I-\beta C}\rb{x_{r-1} - x_{\infty}} + x_{\infty} - \beta C x_{\infty} + \frac{\beta}{M}\sum_{m\in[M]}C_mx_m^\star + \eta\beta\sum_{k=0}^{K-1}\rb{I-\eta A_m}^{K-1-k}\xi_{r,k}\enspace,\\
    &= \rb{I-\beta C}\rb{x_{r-1} - x_{\infty}} + x_{\infty} - \frac{\beta}{M}\sum_{m\in[M]}C_mx_m^\star+ \frac{\beta}{M}\sum_{m\in[M]}C_mx_m^\star + \eta\beta\sum_{k=0}^{K-1}\rb{I-\eta A_m}^{K-1-k}\xi_{r,k}\enspace.
\end{align*}
Simplifying and rearranging this, we get for $r=R$,
\begin{align*}
    x_R - x_\infty &= \rb{I-\beta C}\rb{x_{R-1} - x_{\infty}} + \eta\beta\sum_{k=0}^{K-1}\rb{I-\eta A_m}^{K-1-k}\xi_{R,k}\enspace,\\
    &= \rb{I-\beta C}^R\rb{x_{0} - x_{\infty}} + \sum_{r=0}^{R-1}\rb{I-\beta C}^{R-1-r}\rb{\eta\beta\sum_{k=0}^{K-1}\rb{I-\eta A_m}^{K-1-k}\xi_{r,k}}\enspace.
\end{align*}
Take the norm, squaring, taking the expectation, noting that the noise across the machines and local steps is independent, and using the tower rule of conditional expectation repeatedly, we get,
\begin{align*}
    &\ee\sb{\norm{x_r - x_\infty}^2}\\
    &\quad \leq \norm{I-\beta C}^{2R}\ee\sb{\norm{\rb{x_{0} - x_{\infty}}}^2} + \eta^2\beta^2\sum_{r=0}^{R-1}\norm{I-\beta C}^{2(R-1-r)}\rb{\sum_{k=0}^{K-1}\rb{I-\eta A_m}^{2(K-1-k)}\ee\sb{\norm{\xi_{r,k}}^2}}\enspace,\\
    &\quad \leq^{\text{(\Cref{ass:stoch_bounded_second_moment})}} \norm{I-\beta C}^{2R}\norm{x_\infty}^2 + \eta^2\beta^2\sum_{r=0}^{R-1}\norm{I-\beta C}^{2(R-1-r)}\rb{\sum_{k=0}^{K-1}(1-\eta\mu)^{K-1-k}\frac{\sigma_2^2}{M}}\enspace,\\
    &\quad \leq (1-\beta\lambda_{min}(C))^{2R}\norm{x_\infty}^2 + \eta^2\beta^2\sum_{r=0}^{R-1}(1-\beta\lambda_{min}(C))^{R-1-r}\rb{\frac{1-(1-\eta\mu)^{K}}{\eta\mu}\cdot\frac{\sigma_2^2}{M}}\enspace,
\end{align*}
We now need upper and lower bounds on the minimum eigenvalue of $C$. For this, note that for any $v\in\rr^d$ with $\norm{v}=1$, 
\begin{align*}
    v^TCv  &= v^T\rb{\frac{1}{M}\sum_{m\in[M]}\rb{I-(I-\eta A_m)^K}}v\enspace,\\
    &= \frac{1}{M}\sum_{m\in[M]}v^T\rb{I-(I-\eta A_m)^K}v = 1- \frac{1}{M}\sum_{m\in[M]}v^T(I-\eta A_m)^Kv\enspace.
\end{align*}
Using this calculation along with \Cref{ass:strongly_convex,ass:smooth_second} we note the following,
\begin{align*}
    v^TCv&\in 1 - \rb{(1-\eta \mu)^K, (1-\eta H)^K}\subseteq \rb{1 - (1-\eta \mu)^K, 1- (1-\eta H)^K}\enspace,
\end{align*}
which bounds the range of the eigenvalues of $C$. Plugging these bounds into the above inequality leads to,
\begin{align*}
    &\ee\sb{\norm{x_r - x_\infty}^2}\\ 
    &\quad \leq \rb{1-\beta\rb{1 - (1-\eta \mu)^K}}^{2R}\norm{x_\infty}^2 + \eta\beta \frac{1-(1-\beta\rb{1-(1-\eta \mu)^{K}})^{R}}{1 - (1-\eta \mu)^{K}}\cdot\frac{1-(1-\eta\mu)^{K}}{\mu}\cdot\frac{\sigma_2^2}{M}\enspace,\\
    &\quad \leq \rb{1-\beta\rb{1 - (1-\eta \mu)^K}}^{2R}\norm{x_\infty}^2 + \eta\beta\rb{1-(1-\beta\rb{1-(1-\eta \mu)^K})^{R}}\frac{\sigma_2^2}{\mu M}\enspace.
\end{align*}
Note that the range of $\lambda_{min}(C)$ is what suggests the upper bound on $\beta$ of $\frac{1}{1-(1-\eta H)^K}$.
\end{proof}

Next we will establish an upper bound on $\norm{x_\infty}$, which would allow us to provide the upper bound in terms of $B$ from \Cref{ass:bounded_optima}.

\begin{lemma}\label{lem:fixed_point_norm}
    For quadratic problems of the form \eqref{eq:quad_prob_form} satisfying~\Cref{ass:strongly_convex,ass:smooth_second,ass:bounded_optima,ass:stoch_bounded_second_moment}, with $\eta <\frac{1}{H}$, and $\beta \leq \frac{1}{1-(1-\eta H)^K}$ denoting $\kappa=\frac{H}{\mu}$ we have,
    \begin{align*}
        \norm{x_\infty} \leq \min \cb{\eta\tau K\kappa\zeta_\star + \bar B, \kappa \bar B}\enspace.
    \end{align*}
\end{lemma}
\begin{proof}
Recall the definition of $x_\infty$ and let $\bar x^\star = \frac{1}{M}\sum_{n\in[M]}x_n^\star$,
\begin{align*}
    \norm{x_\infty} &= \norm{C^{-1}\rb{\frac{1}{M}\sum_{m\in[M]}C_mx_m^\star}}\enspace,\\
    &= \norm{C^{-1}\rb{\frac{1}{M}\sum_{m\in[M]}\rb{C_m - C + C}\rb{x_m^\star- \bar x^\star + \bar x^\star}}}\enspace,\\
    &= \norm{C^{-1}\rb{\frac{1}{M}\sum_{m\in[M]}\rb{C_m - C}\rb{x_m^\star- \bar x^\star}} + \bar x^\star}\enspace,\\
    &\leq \frac{1}{M^2}\sum_{m,n\in[M]} \norm{C^{-1}(C_m-C_n)} + \frac{1}{M}\sum_{m\in[M]}\norm{x_m^\star}\enspace,\\
    &= \frac{1}{M^2}\sum_{m,n\in[M]} \norm{C^{-1}} \norm{(I-\eta A_n)^K - (I-\eta A_m)^K}\norm{x_m^\star - x_n^\star} + \frac{1}{M}\sum_{m\in[M]}\norm{x_m^\star}\enspace,\\
    &\leq^{\text{(\Cref{lemma: matrix telescope,ass:bounded_optima,ass:zeta_star,ass:tau})}} \frac{\eta \tau K \rb{1-(1-\eta H)^{K-1}}}{1-(1-\eta \mu)^K}\zeta_\star + \bar B\enspace,\\
    &\leq \eta \tau K\cdot\frac{1-(1-\eta H)^{K}}{1-(1-\eta \mu)^K}\zeta_\star + \bar B\enspace.
\end{align*}
Now we will show that the factor $g(K) = \frac{1-(1-\eta H)^{K}}{1-(1-\eta \mu)^K}$ can be upper bounded by $\kappa=g(1)$ for any choice of step-size $\eta$. To do this, we show that $g(K)$ is a non-increasing function in \Cref{lem:g_non_increasing}. Plugging this above gives us,
\begin{align*}
    \norm{x_\infty} &\leq \eta\tau K\kappa\zeta_\star + \bar B\enspace. 
\end{align*}
To get the alternative upper bound, note that in the very first step of the proof, we could have instead done the following, 
\begin{align*}
    \norm{x_\infty} &= \norm{C^{-1}\rb{\frac{1}{M}\sum_{m\in[M]}C_mx_m^\star}}\enspace,\\
    &\leq \norm{C^{-1}}\frac{1}{M}\sum_{m\in[M]}\norm{C_m}\norm{x_m^\star}\leq \frac{1-(1-\eta H)^{K}}{1-(1-\eta \mu)^K}\cdot \frac{1}{M}\sum_{m\in[M]}\norm{x_m^\star}\enspace,\\
    &\leq g(K)\cdot \bar B \leq g(1)\cdot \bar B = \kappa \cdot \bar B\enspace.
\end{align*}
This proves the lemma.
\end{proof}
\begin{remark}[Norm of $x^\star$]\label{rem:norm_of_x_star_sc_quad}
    Note that in the setting considered in this section assuming $A = \frac{1}{M}\sum_{m\in[M]}A_m$,
    \begin{align*}
        \norm{x^\star} &= \norm{\frac{1}{M}\sum_{m\in[M]}A^{-1}A_mx_m^\star}\enspace,\\
        &= \norm{\frac{1}{M}\sum_{m\in[M]}A^{-1}(A_m-A+A)(x_m^\star-\bar x^\star + \bar x^\star)}= \norm{\frac{1}{M}\sum_{m\in[M]}A^{-1}(A_m-A)(x_m^\star-\bar x^\star) +\bar x^\star}\enspace,\\
        &\leq \frac{1}{M}\sum_{m\in[M]}\norm{A^{-1}}\norm{A_m-A}\norm{x_m^\star-\bar x^\star} +\norm{\bar x^\star}\leq \frac{\tau\zeta_\star}{\mu} + \bar B\enspace.
    \end{align*}
    This means in the strongly convex quadratic setting, when $\tau$ is small, the norm of the fixed point as well as $x^\star$ are close to $\bar B$. Having said that the bound obtained on $x_\infty$ in the above lemma seems bigger than $\norm{x^\star}$ in general. 
\end{remark}
Combining the previous two lemmas and simplifying for $\beta=1$ we get the following convergence rate to the fixed point.
\begin{proposition}\label{prop:conv_to_fixed}
    For quadratic problems satisfying~\Cref{ass:strongly_convex,ass:smooth_second,ass:bounded_optima,ass:stoch_bounded_second_moment}, with $\eta <\frac{1}{H}$, and $\beta =1$ the Local-SGD iterate $\bar x_{R}$ (with initialization $\bar x_0=0$) satisfies,
    \begin{align*}
        &\ee\sb{\norm{\bar x_R - x_\infty}^2} \leq \min \cb{\frac{\tau^2 H^2\zeta_\star^2}{\mu^4R^2} + 2\bar B^2e^{-2\eta\mu KR}, \kappa^2 \bar B^2e^{-2\eta\mu KR}}+ \frac{\eta\sigma_2^2}{\mu M}\enspace.
    \end{align*}
    In particular, using step-size\footnote{This choice of step-size is standard for strongly convex optimization with SGD. For instance, see the unified analysis of SGD due to \citet{stich2019unified}.}, $\eta = \min\cb{\frac{1}{2H}, \frac{1}{\mu KR}\ln\rb{\frac{\bar B^2\mu^2 MKR}{\sigma_2^2}}}$ we can get,
    \begin{align*}
        &\ee\sb{\norm{\bar x_R - x_\infty}^2} = \tilde\ooo\rb{\min\cb{\frac{\tau^2 H^2\zeta_\star^2}{\mu^4R^2} + \bar B^2e^{-KR/\kappa} + \frac{\sigma_2^2}{\mu^2 MKR},\quad  \kappa^2 \bar B^2e^{-2KR/\kappa} + \frac{\sigma_2^2\kappa^2}{\mu^2 MKR} }}\enspace.
    \end{align*}
\end{proposition} 

\begin{remark}[Extreme Communication Efficiency for $x_\infty$]
    The above convergence rate shows that Local SGD converges very quickly to its fixed point. In particular, it is extremely communication-efficient, and in the limit, when $K$ tends to infinity, even with a single communication round, it converges to $x_\infty$. When $\tau=0$, we observe that the convergence rate is identical to that of ``dense mini-batch SGD,'' which communicates $KR$ times (cf. \Cref{rem:norm_of_x_star_sc_quad}). However, mini-batch SGD converges to $x^\star$ and not $x_\infty$, and $x_\infty$ could in general be far away from $x^\star$.
\end{remark}
\begin{proof}
    First, we combine the upper bound on $\norm{x_\infty}$ with \Cref{lem:fast_conv_to_fixed_point}. To get the first statement, we first note that the function $x^2e^{-2x}$ is upper bound by $1/2$ for all $x\geq 0$. This allows us to note that,
    \begin{align*}
        e^{-2\eta\mu KR}2\rb{\eta\tau K\kappa\zeta_\star}^2 = \rb{2e^{-2\eta\mu KR}\rb{\eta\mu KR}^2}\cdot \rb{\frac{\tau H\zeta_\star}{\mu^2 R}}^2 \leq \frac{\tau^2 H^2\zeta_\star^2}{\mu^4R^2}\enspace. 
    \end{align*}
    Using this, we can get the following upper bound,
    \begin{align*}
        \ee\sb{\norm{\bar x_R - x_\infty}^2} &\leq e^{-2\eta\mu KR}\cdot\min \cb{\eta\tau K\kappa\zeta_\star + \bar B, \kappa \bar B}^2 + \frac{\eta\sigma_2^2}{\mu M}\enspace,\\
        &\leq \min \cb{\frac{\tau^2 H^2\zeta_\star^2}{\mu^4R^2} + 2\bar B^2e^{-2\eta\mu KR}, \kappa^2 \bar B^2e^{-2\eta\mu KR}}+ \frac{\eta\sigma_2^2}{\mu M}\enspace.
    \end{align*}
    Now using the step-size $\eta = \min\cb{\frac{1}{2H}, \frac{1}{\mu KR}\ln\rb{\frac{\bar B^2\mu^2 MKR}{\sigma_2^2}}}$ we can get,
    \begin{align*}
         \ee\sb{\norm{\bar x_R - x_\infty}^2} &\leq \min \Bigg\{\frac{\tau^2 H^2\zeta_\star^2}{\mu^4R^2} + 2\max\cb{\bar B^2e^{-KR/\kappa} , \frac{\sigma_2^2}{\mu^2 MKR}},\kappa^2\max\cb{\bar B^2e^{-2KR/\kappa}, \frac{\sigma_2^2}{\mu^2 MKR}}\Bigg\}\\
         &\qquad\qquad\qquad + \frac{\sigma_2^2}{\mu^2 MKR}\ln\rb{\frac{\bar B^2\mu^2 MKR}{\sigma_2^2}}\enspace,\\
         &= \tilde\ooo\rb{\min\cb{\frac{\tau^2 H^2\zeta_\star^2}{\mu^4R^2} + \bar B^2e^{-KR/\kappa} + \frac{\sigma_2^2}{\mu^2 MKR},\quad  \kappa^2 \bar B^2e^{-2KR/\kappa} + \frac{\sigma_2^2\kappa^2}{\mu^2 MKR} }}\enspace,
    \end{align*}
    which proves the second statement of the lemma. 
\end{proof}

As a final ingredient, in the next section we will bound $\norm{x^\star-x_\infty}$, which would allow us to provide a convergence guarantee in terms of the expected distance to $x^\star$. 

\subsection{Upper-bounding Fixed-point Discrepancy for Strongly Convex Quadratics}
The following lemma proves an upper bound on the Fixed Point Discrepancy of Local SGD for strongly convex quadratics. 

\begin{lemma}\label{lem:fixed_disc_UB}
    For quadratic problems of the form \eqref{eq:quad_prob_form} satisfying~\Cref{ass:strongly_convex,ass:smooth_second,ass:bounded_optima,ass:stoch_bounded_second_moment}, with $\eta <\frac{1}{H}$, we have
    \begin{align*}
    \norm{x^\star - x_\infty} &\leq \frac{\zeta_\star\tau}{\mu}\cdot\min\cb{ \frac{(1-\eta H)^K - 1 + \eta H K + \eta \mu K\rb{1 - (1-\eta H)^{K-1}}}{1-(1-\eta\mu)^K}, 1 + \frac{\eta \mu K(1-\eta\mu)^{K-1}}{1-(1-\eta \mu)^K}}\enspace.
\end{align*}
\end{lemma}
\begin{remark}
    Note that the above bound goes to zero when $\tau$ or $\zeta_\star$ is zero. Thus the data heterogeneity notions in \Cref{ass:zeta_star,ass:tau} are intricately linked to fixed-point discrepancy. Furthermore, when $K=1$, there are no unsynchronized local updates, and Local SGD becomes mini-batch SGD; the above upper bound also becomes zero. Finally, when $\eta\to 0$ but $K$ is held constant, then note that using L'Hospital's rule (on the first term in the minimum), we get that the upper bound tends to zero. On the other hand, when $\eta = \Omega(1/K)$ and $K\to\infty$, then the upper bound goes to $\frac{2\zeta_\star\tau}{\mu}$ (as the second term in the minimum becomes active). We illustrate the effect of the step-size on the fixed-point discrepancy in \Cref{fig:fixed,fig:convergence}. 
\end{remark}
\begin{proof}
    Note the following,
\begin{align*}
    \textcolor{red}{\norm{x^\star - x_\infty}} &= \norm{\frac{1}{M^2}\sum_{m,n\in[M]}\rb{A^{-1}A_m - C^{-1}C_m}\rb{x_m^\star - x_n^\star}}\enspace,\\
    &\leq \frac{1}{M^2}\sum_{m,n\in[M]}\norm{A^{-1}A_m - C^{-1}C_m}\norm{x_m^\star - x_n^\star}\enspace,\\
    &\leq^{\text{(\Cref{ass:zeta_star})}} \textcolor{red}{\frac{1}{M}\sum_{m\in[M]}\norm{A^{-1}A_m - C^{-1}C_m}\zeta_{\star,m,n}}\enspace.
\end{align*}
Let us denote the following,
\begin{align*}
    C_m := I- (I-\eta A_m)^K =: \eta K A_m + R_m\quad \text{and} \quad R := \frac{1}{M}\sum_{m\in[M]}R_m\enspace.
\end{align*}
In particular, note that when $K=1$, then $R_m=0$, which implies that $R=0$. $R_m$nis essentially the first-order Binomial residual on machine $m\in[M]$. Using this notation, we have the following,
\begin{align*}
   \norm{A^{-1}A_m - C^{-1}C_m} &= \norm{C^{-1}\rb{CA^{-1}A_m - C_m}}= \norm{C^{-1}\rb{\rb{\eta K A + R}A^{-1}A_m - \eta K A_m - R_m}}\enspace,\\
    &=\norm{C^{-1}\rb{RA^{-1}A_m - R_m}} =\norm{C^{-1}\rb{RA^{-1}A_m - R + R - R_m}}\enspace,\\
    &\leq \norm{C^{-1}}\rb{\norm{R}\norm{A^{-1}A_m -I} + \norm{R-R_m}}\enspace,\\
    &\leq^{\text{(\Cref{ass:strongly_convex,ass:tau})}} \frac{1}{1-(1-\eta\mu)^K}\frac{1}{M}\sum_{n\in[M]}\rb{\frac{\tau}{\mu}\textcolor{blue}{\norm{R_n}} + \textcolor{cyan}{\norm{R_m - R_n}}}\enspace.
\end{align*}
Now it suffices to upper bound the two terms $\norm{R_m}$ and $\norm{R_m - R_n}$. As a sanity check, note that when $K=1$ and $\tau=0$, the upper bounds are still zero. For the first term, note the following using the diagonalization of the matrix $A_m = V_m\Sigma_mV_m^{-1}$,
\begin{align*}
    \textcolor{blue}{\norm{R_n}} &= \norm{I -\eta K A_n - (I-\eta A_n)^K}= \norm{I -\eta K \Sigma_n - (I-\eta \Sigma_n)^K}\enspace,\\
    &\leq \sup_{\lambda \in [\mu, H]} |1 - \eta K \lambda - (1-\eta\lambda)^K| = \textcolor{blue}{(1-\eta H)^K - 1 + \eta H K}\enspace, 
\end{align*}
where we use the fact that $\eta < \frac{1}{H}$ which implies that $\eta \lambda < 1$ in the above function, which in turn implies that $|1 - \eta K \lambda - (1-\eta\lambda)^K|$ is an increasing function in the range $\lambda \in [\mu, H]$. Now we need to bound the second term $\norm{R_m - R_n}$. Note that ideally we would like the upper bound to also vanish with $K=1$ and $\tau=0$. We cannot use the strategy from above because we do not know if the matrices $A_m$ and $A_n$ commute. Instead we will use the following property (see \Cref{lem:R_Lipschitz} in \Cref{app:chap4}),
\begin{align*}
    \norm{R(A_m) - R(A_n)} \leq \sup_{\lambda\in[\mu,H]}|R'(\lambda)|\norm{A_m-A_n}\enspace,
\end{align*}
where we define $R(\lambda):= 1 - (1-\eta\lambda)^K - \eta K\lambda$. Note the following,
\begin{align*}
    |R'(\lambda)| &= |\eta K(1-\eta\lambda)^{K-1} -\eta K| = |-\eta K\rb{1 - (1-\eta\lambda)^{K-1}}| = \eta K\cdot|1 - (1-\eta\lambda)^{K-1}|\enspace.
\end{align*}
Plugging this into the above bound gives us,
\begin{align*}
 \textcolor{cyan}{\norm{R_m-R_n}} &= \norm{R(A_m) - R(A_n)}\enspace,\\ 
 &\leq \sup_{\lambda\in[\mu,H]}\eta K\norm{A_m-A_n}\cdot|1 - (1-\eta\lambda)^{K-1}| \leq \textcolor{cyan}{\eta K\tau\rb{1 - (1-\eta H)^{K-1}}}\enspace.
\end{align*}
For a sanity check, note that when $K=1$ or $\tau=0$, this bound is zero. Plugging the \textcolor{blue}{blue} and \textcolor{cyan}{cyan} upper bounds back into \textcolor{red}{the original bound} on fixed-point discrepancy, we get,
\begin{align*}
    &\norm{x^\star - x_\infty}\\
    &\quad\leq \frac{1}{1-(1-\eta\mu)^K}\frac{1}{M^2}\sum_{m,n\in[M]}\rb{\frac{\tau}{\mu}\norm{R_n} + \norm{R_m - R_n}}\zeta_{\star,m,n}\enspace,\\
    &\quad\leq \frac{1}{1-(1-\eta\mu)^K}\frac{1}{M^2}\sum_{m,n\in[M]}\rb{\frac{\tau}{\mu}\rb{(1-\eta H)^K - 1 + \eta H K} + \eta K\tau\rb{1 - (1-\eta H)^{K-1}}}\zeta_{\star,m,n}\enspace,\\
    &\quad\leq \frac{\zeta_\star\tau}{\mu}\cdot \frac{(1-\eta H)^K - 1 + \eta H K + \eta \mu K\rb{1 - (1-\eta H)^{K-1}}}{1-(1-\eta\mu)^K}\enspace.
\end{align*}
This proves the bound in the lemma. 

Now for the second bound, we first bound the distance between $x^\star$ and $\bar x^\star = \frac{1}{M}\sum_{m\in[M]}x_m^\star$\footnote{Note that the hard instance used in the proof of \Cref{prop:barB_vs_B} has a constant second order heterogeneity lower bounded by $2$. This means the conclusion of the following inequalities, which implies that the gap between $\bar x^\star$ and $x^\star$ tends to zero as $\tau$ tends to zero, is not a contradiction. },
    \begin{align*}
        \norm{x^\star - \bar x^\star} &= \norm{\frac{1}{M}\sum_{m\in[M]}\rb{I-A^{-1}A_m}x_m^\star} = \norm{\frac{1}{M}\sum_{m\in[M]}A^{-1}\rb{A-A_m}(x_m^\star-\bar x^\star)}\enspace,\\
        &\leq \frac{1}{M}\sum_{m\in[M]} \norm{A^{-1}\rb{A-A_m}(x_m^\star-\bar x^\star)}\leq \frac{1}{M^2}\sum_{m,n\in[M]} \norm{A^{-1}}\norm{A-A_m}\norm{x_m^\star-x_n^\star}\enspace,\\
        &\leq \frac{1}{M}\sum_{m\in[M]} \frac{1}{\mu}\cdot\tau \cdot \zeta_{\star,m,n} = \frac{\tau \zeta_\star}{\mu}\enspace.
    \end{align*}
    We now bound the distance between $x_\infty(K>1, \eta)$ and $\bar x^\star$ similarly,
    \begin{align*}
        \norm{x_\infty - \bar x^\star} &= \norm{\frac{1}{M}\sum_{m\in[M]}\rb{I-C^{-1}C_m}x_m^\star} = \norm{\frac{1}{M}\sum_{m\in[M]}C^{-1}\rb{C-C_m}(x_m^\star-\bar x^\star)}\enspace,\\
        &\leq \frac{1}{M}\sum_{m\in[M]} \norm{C^{-1}\rb{C-C_m}(x_m^\star-\bar x^\star)}\leq \frac{1}{M^2}\sum_{m,n\in[M]} \norm{C^{-1}}\norm{C-C_m}\norm{x_m^\star-x_n^\star}\enspace,\\
        &\leq \frac{1}{M^2}\sum_{m,n\in[M]}\frac{1}{1-(1-\eta \mu)^K}\sup_{m,n\in[M]} \norm{(I-\eta A_m)^K-(I-\eta A_n)^K}\zeta_{\star,m,n}\enspace,\\
        &\leq^{\text{(\Cref{lemma: matrix telescope})}} \frac{1}{M^2}\sum_{m,n\in[M]}\frac{1}{1-(1-\eta \mu)^K}\sup_{m,n\in[M]} \norm{\eta(A_m-A_n)}K(1-\eta\mu)^{K-1}\zeta_{\star,m,n}\enspace,\\
        &\leq \frac{\tau \zeta_\star}{\mu}\cdot\frac{\eta \mu K(1-\eta\mu)^{K-1}}{1-(1-\eta \mu)^K}\enspace,
    \end{align*}
    which finishes the proof of the second bound in the lemma.
\end{proof}

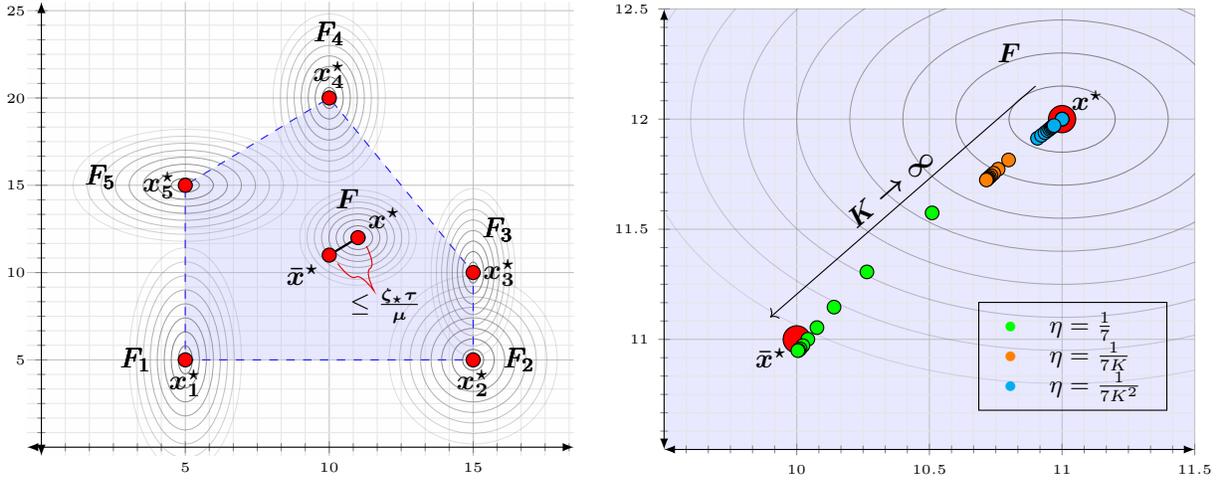
\begin{figure}[ht]
\centering
\begin{minipage}{0.48\textwidth}
\centering
\resizebox{\linewidth}{!}{%

\begin{tikzpicture}
\begin{axis}[
    xmin=0, xmax=18,
    ymin=0, ymax=25,
    grid=both,
    grid style={line width=.1pt, draw=gray!20},
    major grid style={line width=.2pt,draw=gray!50},
    axis lines=middle,
    minor tick num=5,
    enlargelimits={abs=0.5},
    axis line style={latex-latex},
    axis background/.style={fill=gray!1},
    ticklabel style={font=\tiny},
    xlabel style={at={(ticklabel* cs:1)},anchor=north west},
    ylabel style={at={(ticklabel* cs:1)},anchor=south west}
]

\addplot+[only marks, mark=*, mark size=2.5pt, mark options={fill=red, draw=black}]
    coordinates {
    (5,5) (15,5) (15,10) (10,20) (5,15) (10,11) (11,12)
    };
\node at (axis cs:5,5) [anchor=north] {$\bm{x_1^\star}$};
\node at (axis cs:15,5) [anchor=north] {$\bm{x_2^\star}$};
\node at (axis cs:15,10) [anchor=west] {$\bm{x_3^\star}$};
\node at (axis cs:10,20) [anchor=south] {$\bm{x_4^\star}$};
\node at (axis cs:5,15) [anchor=east] {$\bm{x_5^\star}$};

\draw[dashed, blue] (axis cs:5,5) -- (axis cs:15,5) -- (axis cs:15,10) -- (axis cs:10,20) -- (axis cs:5,15) -- cycle;
\filldraw[fill=blue!30, fill opacity=0.3, draw=none] (axis cs:5,5) -- (axis cs:15,5) -- (axis cs:15,10) -- (axis cs:10,20) -- (axis cs:5,15) -- cycle;

\node at (axis cs:10,11) [anchor=north east] {$\bm{\bar x^\star}$};
\node at (axis cs:11,12) [anchor=south west] {$\bm{x^\star}$};

\foreach \r in {1,...,8}{ 
    \pgfmathsetmacro{\step}{\r*2.5}
    \pgfmathsetmacro{\opacity}{max(0.2, 1 - \r / 10)}
    \edef\temp{\noexpand\draw[black!60, very thin, opacity=\opacity] (axis cs:5,5) ellipse ({\step pt} and {\step*2 pt});}
    \temp
    \node[label={[label distance=0cm]180:$F_1$}] at (axis cs:4.5,5) {};
    \edef\temp{\noexpand\draw[black!60, very thin, opacity=\opacity] (axis cs:15,5) ellipse ({\step*1.5 pt} and {\step*1.5 pt});}
    \temp
    \node[label={[label distance=-1.2cm]180:$F_2$}] at (axis cs:14.5,5) {};
    \edef\temp{\noexpand\draw[black!60, very thin, opacity=\opacity] (axis cs:15,10) ellipse ({\step*0.75 pt} and {\step*1.5 pt});}
    \temp
    \node[label={[label distance=0.2cm]45:$F_3$}] at (axis cs:14.25,10) {};
    \edef\temp{\noexpand\draw[black!60, very thin, opacity=\opacity] (axis cs:10,20) ellipse ({\step pt} and {\step*1.5 pt});}
    \temp
    \node[label={[label distance=0.3cm]90:$F_4$}] at (axis cs:10,20.5) {};
    \edef\temp{\noexpand\draw[black!60, very thin, opacity=\opacity] (axis cs:5,15) ellipse ({\step*2 pt} and {\step pt});}
    \temp
    \node[label={[label distance=-1.5cm]0:$F_5$}] at (axis cs:5,15.5) {};
}

\foreach \r in {1,...,8}{
    \pgfmathsetmacro{\stepX}{\r*sqrt(3)} 
    \pgfmathsetmacro{\stepY}{\r*sqrt(2)} 
    \pgfmathsetmacro{\opacity}{max(0.2, 1 - \r / 10)}
    \edef\temp{\noexpand\draw[black!60, very thin, opacity=\opacity] (axis cs:11,12) ellipse ({\stepX*1.5 pt} and {\stepY*1.5 pt});}
    \temp
}
\node[label={[label distance=-0.1cm]150:$\bm{F}$}] at (axis cs:11.5,13) {};

\draw [thick] (axis cs:11,12) -- (axis cs:10,11);
\draw [decorate,decoration={brace,amplitude=15pt},xshift=3pt,yshift=-3pt, color=red!90!black]
(axis cs:11,12) -- (axis cs:10,11) node [black,midway,xshift=12pt,yshift=-18pt] {\small$\leq\bm{\frac{\zeta_\star\tau}{\mu}}$};

\end{axis}
\end{tikzpicture}
}
\end{minipage}\hfill
\begin{minipage}{0.51\textwidth}
\centering
\resizebox{\linewidth}{!}{%

\begin{tikzpicture}
\begin{axis}[
    xmin=10, xmax=11,
    ymin=11, ymax=12,
    grid=both,
    grid style={line width=.1pt, draw=gray!20},
    major grid style={line width=.2pt,draw=gray!50},
    axis lines=middle,
    minor tick num=5,
    enlargelimits={abs=0.5},
    axis line style={latex-latex},
    axis background/.style={fill=blue!30, fill opacity=0.3},
    ticklabel style={font=\tiny},
    xlabel style={at={(ticklabel* cs:1)},anchor=north west},
    ylabel style={at={(ticklabel* cs:1)},anchor=south west}
]

\addplot+[only marks, mark=*, mark size=5pt, mark options={fill=red, draw=black}]
    coordinates {
    (10,11) (11,12)
    };
\node at (axis cs:10,11) [anchor=north east] {$\bm{\bar x^\star}$};
\node at (axis cs:11,12) [anchor=south west] {$\bm{x^\star}$};

\addplot+[only marks, mark=*, mark size=2.5pt, mark options={fill=green, draw=black}]
    coordinates {
    (11, 12)
    (10.51020408, 11.57407407)
    (10.26448363, 11.30620985)
    (10.14049587, 11.14633494)
    (10.07609909, 11.05334492)
    (10.04177659, 11.00042072)
    (10.0231446, 10.97136128)
    (10.01290431, 10.95654894)
    (10.00722831, 10.95023809)
    (10.00406319, 10.9489501)
    };

\addplot+[only marks, mark=*, mark size=2.5pt, mark options={fill=orange, draw=black}]
    coordinates {
    (11, 12)
    (10.79831933, 11.81451613)
    (10.75903614, 11.77268722)
    (10.74197075, 11.7540491)
    (10.73240831, 11.74348994)
    (10.72628847, 11.73669031)
    (10.72203465, 11.73194517)
    (10.71890591, 11.72844536)
    (10.71650785, 11.72575743)
    (10.71461121, 11.72362817)
    };

\addplot+[only marks, mark=*, mark size=2.5pt, mark options={fill=cyan, draw=black}]
    coordinates {
    (11, 12)
    (10.90733591, 11.91287879)
    (10.92232661, 11.9259796)
    (10.93551932, 11.93831929)
    (10.94530409, 11.94760228)
    (10.95263775, 11.95459758)
    (10.95828684, 11.95999987)
    (10.96275457, 11.96427829)
    (10.9663694, 11.96774267)
    (10.96935104, 11.97060156)
    };

\foreach \r in {1,...,8}{
    \pgfmathsetmacro{\stepX}{\r*0.2} 
    \pgfmathsetmacro{\stepY}{\r*0.15} 
    \pgfmathsetmacro{\opacity}{max(0.2, 1 - \r / 10)}
    \edef\temp{\noexpand\draw[black!60, very thin, opacity=\opacity] (axis cs:11,12) ellipse ({\stepX} and {\stepY});}
    \temp
}

\node at (axis cs:10.8, 12.3) {$\bm{F}$};

\draw[->] (10.9,12.15) -- (9.9,11.1) node[midway, above, sloped, pos=0.5] {$\bm{K\to\infty}$};

\end{axis}
\node[
    draw,
    at={(6.5,0.5)},
    anchor=south east,
    font=\small
    ] {
    \begin{tabular}{cl}
        \textcolor{green}{$\bullet$} & $\eta = \frac{1}{7}$ \\
        \textcolor{orange}{$\bullet$} & $\eta = \frac{1}{7K}$ \\
        \textcolor{cyan}{$\bullet$} & $\eta = \frac{1}{7K^2}$
    \end{tabular}
};

\end{tikzpicture}
}
\end{minipage}
\caption{Illustration of a two-dimensional optimization problem with $M=5$ machines, each with a $1$-strongly convex and $6$-smooth objective. On the left, we draw the contour lines for each machine's objective as well as for the average objective. We also indicate the two relevant solution concepts $\bar x^\star$ and $x^\star$. On the right, we zoom into the convex hull of the machines' optima, plotting the sequence of fixed points for local GD as a function of $\eta$ and increasing $K \in [10]$. We show the fixed points for three values of $\eta$, each demonstrating a different trend for $\lim_{K\to\infty}x_{\infty}(K, \eta, \beta)$.}
\label{fig:fixed}
\end{figure}

\begin{figure}
    \centering
    \includegraphics[width=0.7\textwidth]{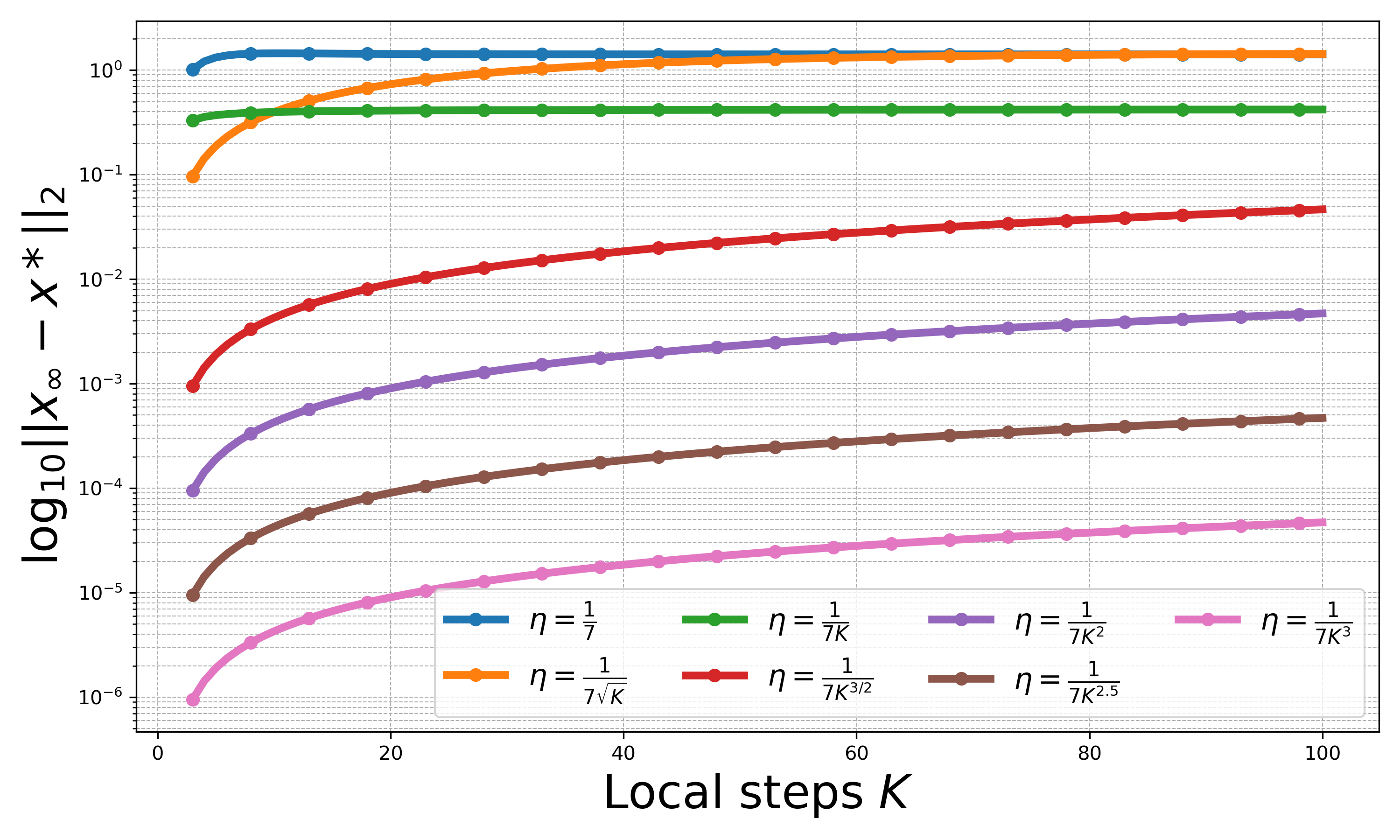}
    \caption{Illustration of the same distributed problem as Figure \ref{fig:fixed} to understand where the fixed point converges as $K$ grows. We consider $7$ different choices of $\eta$ (as a function of $K$) and plot $\log\norm{x_{\infty}(K,\eta,1)-x^\star}$ as a function of $K\in[100]$. We notice that for $\eta >\frac{1}{HK}$, the fixed point goes to $\bar x^\star$ as $K$ increases, while for $\eta<\frac{1}{HK}$, the fixed point gets progressively closer to $x^\star$. }
    \label{fig:convergence}
\end{figure}

\subsection{Towards a Convergence Guarantee using Fixed-point Discrepancy}

To derive a final convergence guarantee for Local SGD on strongly convex quadratic objectives, we combine the fixed-point characterization from \Cref{prop:conv_to_fixed} with the fixed-point discrepancy bound in \Cref{lem:fixed_disc_UB}. This yields the following convergence result, expressed in terms of the step-size $\eta$:

\begin{theorem}\label{thm:conv_with_fixed_pt_pers}
    For quadratic objectives satisfying \Cref{ass:strongly_convex,ass:smooth_second,ass:bounded_optima,ass:stoch_bounded_second_moment}, with step-size $\eta < \frac{1}{H}$ and momentum parameter $\beta = 1$, the Local-SGD iterate $\bar x_R$ (initialized with $\bar x_0 = 0$) satisfies:
    \begin{align*}
        &\ee\sb{\norm{\bar x_R - x^\star}^2} \leq c_9\cdot\Bigg(\min\cb{\frac{\tau^2 H^2\zeta_\star^2}{\mu^4R^2} + 2\bar B^2e^{-2\eta\mu KR}, \kappa^2 \bar B^2e^{-2\eta\mu KR}}+ \frac{\eta\sigma_2^2}{\mu M}\\
        &\qquad\quad + \frac{\zeta_\star^2\tau^2}{\mu^2}\cdot\min\cb{ \frac{(1-\eta H)^K - 1 + \eta H K + \eta \mu K\rb{1 - (1-\eta H)^{K-1}}}{1-(1-\eta\mu)^K}, 1 + \frac{\eta \mu K(1-\eta\mu)^{K-1}}{1-(1-\eta \mu)^K}}^2\Bigg)\enspace.
    \end{align*}
\end{theorem}

While this result highlights how convergence depends on the choice of $\eta$, tuning the step-size remains subtle, and we do not yet obtain a closed-form rate. Notably, selecting $\eta$ as in \Cref{prop:conv_to_fixed} ensures rapid convergence to the fixed point, but the upper bound on the fixed-point discrepancy from \Cref{lem:fixed_disc_UB} does not decay with increasing $R$ or $K$. This does not necessarily imply that the actual discrepancy remains large---instead, it suggests that our bound may be loose in this regime (see \Cref{fig:convergence} for empirical evidence). 

In the next section, we revisit the convergence analysis from a different perspective, directly bounding the Local SGD error without relying on fixed-point characterization. This alternative approach recovers the same key terms appearing in \Cref{prop:conv_to_fixed}, and crucially, extends to general (non-quadratic) smooth and strongly convex objectives. Before we proceed with this analysis, however, we will briefly consider what happens in the convex setting, while also highlighting the potential implicit regularization of Local SGD.

\section{Local SGD's Fixed Point for Convex Quadratics}\label{sec:fixed_derivation_convex}

While in the general convex setting, we cannot write an explicit formula for the fixed point $x_\infty$, we can characterize it as the minimum-norm solution of a certain least-squares problem, where the geometry for each machine is defined by the matrices $C_m$. 

\begin{proposition}[Fixed Point for Convex Quadratics]\label{prop:fixed_cvx}
    Assume we have a quadratic problem instance satisfying \Cref{ass:convex,ass:smooth_second,ass:bounded_optima,ass:stoch_bounded_second_moment} with $\sigma_2=0$, $\eta < 1/H$. Further define $C_m := I-(I-\eta A_m)^K$, $C:=\frac{1}{M}\sum_{m\in[M]}C_m$ and $c:=\frac{1}{M}\sum_{m\in[M]}C_mx_m^\star$ for some $x_m^\star\in S_m^\star$ for each $m\in[M]$. If $c\neq 0$ and $c \in \text{im}(C)=\text{ker}(C)^\perp$, then Local GD converges to the following solution in the limit of $R\to\infty$,
    \begin{align*}
        x_\infty = \arg\min\quad \norm{x}\enspace,\quad \text{s.t.}\quad x\in \min_{x\in\rr^d}\frac{1}{M}\sum_{m\in[M]}\|x-x_m^\star\|_{C_m}^2\enspace.
    \end{align*}
    If on the other hand $c\neq 0$ and $c\not\in \text{im}(C)$, the the iterates do not converge, but if we define the sequence $y_R = C\bar x_R$, then 
    \begin{align*}
        \lim_{R\to\infty}y_R = \sum_{i\in[l]}v_iv_i^Tc\rb{\lim_{R\to\infty}\rb{1-(1-\lambda_i)^R}}\enspace,
    \end{align*}
    where $C = \sum_{i\in[l]}\lambda_iv_iv_i^T$ is the eigen-value decomposition of $C$ for orthonormal vectors $\cb{v_1, \dots, v_l}$. If $c=0$, the iterates of Local-GD do not move from $\bar x_0=0$.
\end{proposition}
\begin{remark}
    When the objectives on each machine are strongly convex, then we always have $c\in \text{im}(C) = \rr^d$. In general when $\text{im}(C) = \rr^d$, we can guarantee convergence to a fixed point. An even weaker sufficient condition is to assume that $\cap_{m\in[M]}\text{ker}(A_m) = \cb{0}$, which guarantees that $\text{ker}(C) = \cb{0}$ and hence $\text{im}(C) = \rr^d$. We prove this last condition in \Cref{lem:fact_C_m_A_m}. The condition \( \bigcap_{m=1}^M \ker(A_m) = \{0\} \) is equivalent to the average Hessian $A$ being positive definite, i.e., \( A \succ 0 \), or in the global objective being strongly convex. This condition ensures that local curvature from different clients collectively constrains all directions and the machines are no simultaneously blind to some direction. 
\end{remark}
\begin{proof}
We recall that even in the convex setting (i.e., with $\mu=0$) we can write the following for the Local SGD iterate $\bar x_R$ in the noise-less setting with $\beta =1$ and initialization $\bar x_0 = 0$,
\begin{align*}
    \bar x_R &= \frac{1}{M}\sum_{m\in[M]}\rb{\rb{I-\eta A_m}^K\rb{\bar x_{R-1}-x_m^\star} + x_m^\star}\enspace,\\
    &= \frac{1}{M}\sum_{m\in[M]}\rb{\rb{I-\eta A_m}^K} \bar x_{R-1} +\frac{1}{M}\sum_{m\in[M]}\rb{I-\rb{I-\eta A_m}^K}x_m^\star\enspace,\\
    &= \rb{I-C}\bar x_{R-1} +  \frac{1}{M}\sum_{m\in[M]} C_m x_m^\star\enspace,\\
    &=^{(x_0=0)} \sum_{j=0}^{R-1}(I-C)^jc\enspace.
\end{align*}
Now let us assume an orthonormal basis for the span of $C$ is given by $\cb{v_1, \dots, v_l}$ where $l\leq d$. This allows us to write,
\begin{align*}
    C = \sum_{i\in[l]}\lambda_i v_iv_i^T\enspace,
\end{align*}
where $0< \lambda_i \leq 1-(1-\eta H)^K < 1$ as our step-size is $\eta < 1/H$. Let us extend this to an orthonormal basis for the entire vector space $\rr^d$ as $\cb{v_1, \dots, v_l, v_{l+1}, \dots, v_d}$ so that $v_{l+1}, \dots, v_d \in \text{ker}(C)$. This also implies for $j\in \zz_{\geq 0}$,
\begin{align*}
    (I-C)^j = \rb{\sum_{i\in[l]}(1-\lambda_i) v_iv_i^T + \sum_{i\in[l+1,d]}v_iv_i^T}^j = \sum_{i\in[l]}(1-\lambda_i)^j v_iv_i^T + \sum_{i\in[l+1,d]}v_iv_i^T\enspace.
\end{align*}
Now we will inspect how $\bar x_R$ evolves in each direction, $i\in[d]$. 

First let us consider $i\in[l]$,
\begin{align*}
    v_i^T\bar x_R = \sum_{j=0}^{R-1}v_i^T(I-C)^jc = \sum_{j=0}^{R-1}(1-\lambda_i)^jv_i^Tc = \frac{1 - (1-\lambda_i)^R}{\lambda_i}v_i^Tc\enspace.
\end{align*}
No matter how we pick $\eta,\ K$, this would converge as $R\to\infty$, to some quantity proportional to $v_i^Tc$. 

Now let us consider $i\in[l+1,d]$,
\begin{align*}
    v_i^T\bar x_R = \sum_{j=0}^{R-1}v_i^T(I-C)^jc = \sum_{j=0}^{R-1}v_i^Tc = Rv_i^Tc\enspace.
\end{align*}
Notably, this does not converge unless $v_i^Tc = 0$. 

In particular, the iterates of local GD converge iff $v_i^Tc = 0$ for all $i\in[l+1,d]$. Or in other words, $c\in \text{im}(C) = \text{ker}(A)^{\perp}$. First let us assume, this is true, then we can conclude that the local GD iterates only evolve in the sub-space $\text{im}(C)$. Where do they converge? Solving the fixed-point equation in the limit of large $R$ gives us,
\begin{align*}
    x_\infty = \rb{I-C}x_\infty+  c \quad \Rightarrow\quad Cx_\infty = c\enspace.
\end{align*}
Summarizing the two key findings so far we get that assuming $c\in \text{im}(C)$, $Cx_\infty = c$ and $x_\infty\not\in \text{ker}(C)$. This is equivalent to saying that $x_\infty$ is the minimum norm solution of the linear system $Cx=c$. In other words,
\begin{align*}
    x_\infty = \arg\min_{x\text{ s.t. }Cx=c}\norm{x}\enspace.
\end{align*}
Further, note that using the least square formulation we can write the solutions of the linear system $Cx=c$ as,
\begin{align*}
    \min_{x\in\rr^d}\frac{1}{M}\sum_{m\in[M]}\|x-x_m^\star\|_{C_m}^2\enspace. 
\end{align*}
This implies that $x_\infty$ (when it exists) is the solution of the following optimization problem,
\begin{align*}
    \min\quad \norm{x}\enspace,\quad \text{s.t.}\quad x\in \min_{x\in\rr^d}\frac{1}{M}\sum_{m\in[M]}\|x-x_m^\star\|_{C_m}^2\enspace.
\end{align*}

In the case, that $c\not\in \text{im}(A)$, there exists $i\in[l+1,d]$ such that $v_i^Tc\neq 0$. The iterates will explode in this direction, but still, notably, the sequence $C\bar x_R$ does converge, because
\begin{align*}
    \lim_{R\to\infty} C\bar x_R &=  \lim_{R\to\infty}\sum_{i\in[l]}\lambda_iv_iv_i^T\bar x_R =  \lim_{R\to\infty}\sum_{i\in[l]}v_iv_i^Tc \rb{1-(1-\lambda_i)^R}\enspace,
\end{align*}
No matter how we pick $\eta,\ K$, this limit exists. 
\end{proof}

\begin{figure}
    \centering
    \includegraphics[width=0.8\linewidth]{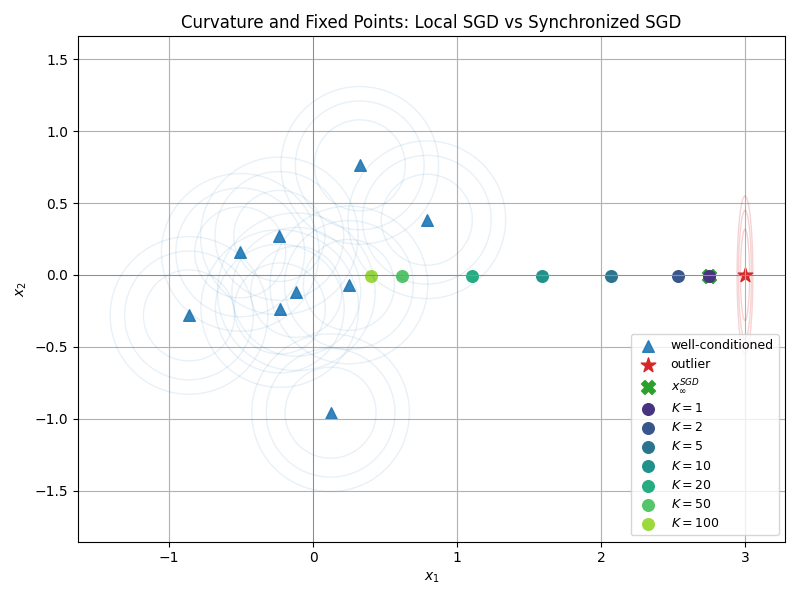}
    \caption{The effect of having an outlier with a sharp curvature on Local SGD's fixed point with progressively higher local update steps.}
    \label{fig:implcit}
\end{figure}

\section{Implicit Regularization due to Local Updates}\label{sec:regularize}
Several works have attempted to understand the effectiveness of Local-SGD from a different perspective, namely by arguing that the solution obtained by Local-SGD is somehow superior. In other words, these works have attempted to characterize the implicit regularization achieved through local update steps. On such notable work is due to Gu et al.~\cite{gu2023and}. For convex quadratic problems, the fixed-point perspective can also be used to understand the implicit regularization of Local SGD. Specifically, recall that under the assumption we discussed in the previous sub-section, i.e., $\cap_{m\in[M]}\ker(A_m) = \cb{0}$, we can also characterize the fixed-point of synchronized SGD as follows,
\begin{align*}
        x_\infty^{SGD} = \arg\min\quad \norm{x}\enspace,\quad \text{s.t.}\quad x\in \min_{x\in\rr^d}\frac{1}{M}\sum_{m\in[M]}\|x-x_m^\star\|_{A_m}^2\enspace.
\end{align*}
Thus the main difference with respect to Local SGD with $K>1$, is a different geometry on each machine defined by $A_m$ as opposed to $C_m$ of Local-SGD,
\begin{align*}
        x_\infty^{L-SGD} = \arg\min\quad \norm{x}\enspace,\quad \text{s.t.}\quad x\in \min_{x\in\rr^d}\frac{1}{M}\sum_{m\in[M]}\|x-x_m^\star\|_{C_m}^2\enspace.
\end{align*}
One natural question is: \textbf{when is the geometry endowed by Local-SGD better?} 

When $\eta$ is ``large enough,'' then for larger \( K \), the matrix polynomial \( C_m = I - (I - \eta A_m)^K \) increasingly flattens the influence of high-curvature (i.e., high-eigenvalue) directions in \( A_m \). In other words, Local SGD implicitly applies a spectral filter that down-weighs directions where the local objective is sharply curved. This has a regularization effect: machines with highly ill-conditioned losses or extremely sharp curvature (possibly due to overfitting, poor conditioning, or adversarial data) contribute less in those sensitive directions. Instead, Local SGD emphasizes agreement in directions where curvature is more moderate or shared across machines.

As a result, the fixed point \( x_\infty^{L\text{-}SGD} \) avoids overreacting to any single client's extreme curvature and instead biases the solution toward directions of consensus and smoothness. In this sense, Local SGD can be interpreted as interpolating between machine-specific optimization (via \( A_m \)) and a more uniform averaging of preferences (via \( C_m \)), particularly in settings with heterogeneous curvature. This implicit regularization may lead to better generalization in practice, especially when the global objective inherits pathological structure from just a few problematic machines.

In \Cref{fig:implcit} we simulate the effect of having an outlier with a sharp curvature, showing how progressively more local update steps regularize the geometry.

\subsection{Extension to Non-quadratics?}
The biggest issue with extending the above analysis to non-quadratics, is that it becomes hard to even write the expression for the fixed point in a closed form. As we will see in \Cref{ch:upper_bounds} it is much easier to use the usual consensus error based analysis in these settings.

\chapter{Local SGD Analyses using Consensus Error}\label{ch:upper_bounds}
In this chapter, we present one of the central contributions of this thesis: improved analyses of Local SGD through sharper bounds on the \emph{consensus error}. This quantity captures the cost of asynchrony across machines and plays a key role in distributed optimization. Unlike the previous section, our analysis does not rely on explicitly characterizing the fixed point of Local SGD, thereby avoiding the need to bound the fixed-point discrepancy. In particular, our contributions are as follows:

\begin{enumerate}
    \item In \Cref{thm:UB_Sconvex_w_Q_zeta,thm:UB_Sconvex_func_w_Q_zeta,thm:UB_convex_w_Q} we first provide new upper bounds for Local SGD under \Cref{ass:zeta}, that improve over existing bounds~\citep{woodworth2020minibatch} by incorporating the effect of second-order heterogeneity and third-order smoothness (\Cref{ass:tau,ass:smooth_third}). These analyses require us to prove new uniform bounds on the consensus error for Local SGD.
     
    \item In \Cref{thm:UB_Sconvex_wo_Q,thm:UB_function_wo_Q,thm:UB_Sconvex_quad,thm:UB_function_quad} we further improve our analyses, deriving coupled recursions between iterate sub-optimality and consensus errors to provide new upper bounds for Local SGD that avoids \Cref{ass:zeta}, and only relies on more relaxed \Cref{ass:zeta_star,ass:phi_star}. Our upper bounds reflect the qualitative behavior predicted by our lower bounds (cf. \Cref{thm:new_LSGD_lower_bound_with_tau}): convergence accelerates as second-order heterogeneity (\Cref{ass:tau}) diminishes. 

    \item Finally, in \Cref{thm:UB_Sconvex_w_Q} we avoid both \Cref{ass:zeta} and incorporate third-order smoothness (\Cref{ass:smooth_third}). This requires an even more careful treatment of higher-order terms, bounding the fourth moment of the consensus error, and handling several coupled recursions.

\end{enumerate}

Collectively, these contributions deepen our understanding of how various forms of heterogeneity influence the behavior of Local SGD and lay the groundwork for further theoretical and algorithmic developments in heterogeneous distributed optimization.

Our techniques extend naturally to other settings where consensus-like errors arise, including communication compression~\citep{stich2018sparsified,karimireddy2019error,gao2024econtrol}, quantization~\citep{alistarh2018convergence,magnusson2019quantization}, asynchronous updates~\citep{ye2018coding,stich2020error}, differential privacy~\citep{wei2020federated,wang2024consensus}, and Byzantine robustness~\citep{alistarh2018byzantine,karimireddy2021learning}. We will present proof sketches in this chapter and \Cref{app:chap5} provides a fully self-contained tutorial on our analyses for Local SGD. 

\subsection*{Outline and Relevant References}

The results in this chapter are based on two papers~\cite{patel2024limits,patel2025revisiting} co-authored with Margalit Glasgow, Ali Zindari, Lingxiao Wang, Sebastian U.\ Stich, Ziheng Cheng, Nirmit Joshi, and Nathan Srebro.

Consensus formation is a fundamental challenge in distributed optimization, particularly when solving Problem~\eqref{prob:scalar}, and has been extensively studied in the context of asynchronous algorithms, compression, and error tolerance. One of the first analyses of Local SGD that explicitly bounded the consensus error was provided by~\citet{stich2018local}. Later,~\citet{woodworth2020local} gave an improved upper bound also relying on bounding the consensus error, which was subsequently shown to be tight in the convex quadratic case via a lower bound from~\citet{glasgow2022sharp}.

The most directly related prior work is that of~\citet{woodworth2020minibatch}, who analyzed Local SGD under a first-order heterogeneity assumption (\Cref{ass:zeta}) and derived a convergence rate using a bound on the consensus error. We revisit and strengthen this result in \Cref{sec:ch5.1} by introducing a new coupled recursion between iterate sub-optimality and consensus error. Solving this recursion yields improved convergence rates that depend on the second-order heterogeneity constant $\tau$ (\Cref{ass:tau}). 

In \Cref{sec:third_order_smooth}, we extend the analysis to settings with third-order smoothness (\Cref{ass:smooth_third}), incorporating the constant $Q$ into the bounds. While \ citet {yuan2020federated} provided a related analysis in the homogeneous setting (which also utilizes \Cref{ass:smooth_third}), our work generalizes it to heterogeneous environments, where the analysis is technically more involved and requires solving four different coupled recursions. 

Finally, \Cref{sec:ch5.3} presents new upper bounds for Local SGD in the convex (non-strongly convex) setting. These results are derived via a convex-to-strongly-convex reduction that leverages the guarantees established in the earlier sections. \Cref{sec:ch5.4} presents simulations that decouple the effect of first and second-order heterogeneity, confirming the predictions of our theory.

\section{Strongly Convex Setting}\label{sec:ch5.1}
In this section, we first present our result on convergence in iterates in the strongly convex setting in \Cref{thm:UB_Sconvex_wo_Q}, and then on convergence in function values in \Cref{thm:UB_function_wo_Q}. The latter allows us to compare our upper bounds with the lower bound in \Cref{thm:new_LSGD_lower_bound}. We will focus here on the key ideas used to derive \Cref{thm:UB_Sconvex_wo_Q}; the proof of \Cref{thm:UB_function_wo_Q} is morally similar, and deferred to \Cref{app:together_2}.

Our analysis proceeds in three stages. We begin by introducing a standard one-step progress result in \Cref{lem:iterate_error_second_recursion_main_body}, which quantifies the improvement of Local SGD in terms of the \emph{consensus error}---a quantity that measures the deviation between local and global iterates and plays a central role in the analysis of many distributed optimization algorithms. We then identify the two main issues in the existing consensus error bounds: (i) they rely on restrictive assumptions~\cite{woodworth2020minibatch,patel2024limits}; and (ii) they do not characterize the effect of second-order heterogeneity. To address both these issues we establish a new upper bound on the consensus error in \Cref{lem:consensus_error_second_recursion_main_body}, that only depends on \Cref{ass:zeta_star,ass:phi_star,ass:tau}. Finally, we substitute this bound into the progress lemma and unroll the resulting recursion to obtain convergence guarantees for both strongly convex and general convex objectives. These results reveal how the convergence of Local SGD depends on the data heterogeneity parameters $\tau$, $\zeta_\star$, and $\phi_\star$, and highlight the algorithm’s communication efficiency in regimes of low data heterogeneity and large $K$.

\begin{lemma}[Canonical One-step Lemma]\label{lem:iterate_error_second_recursion_main_body}
    Assume that the problem instance satisfies \Cref{ass:strongly_convex,ass:smooth_second,ass:stoch_bounded_second_moment}. Then, for step-size $\eta < \frac{1}{H}$ and all $t \in [0,T-1]$, Local-SGD's iterates satisfy:
    \begin{align*}
        \ee\left[\norm{x_{t+1} - x^\star}^2\right] 
        &\leq \left(1 - \eta \mu\right)\ee\left[\norm{x_t - x^\star}^2\right] 
        + \frac{\eta H^2}{\mu} \cdot \textcolor{blue}{\frac{1}{M} \sum_{m \in [M]} \ee\left[\norm{x_t - x_t^m}^2\right]} 
        + \frac{\eta^2 \sigma_2^2}{M} \enspace.
    \end{align*}
\end{lemma}
The above lemma is standard in the analysis of Local SGD~\cite{stich2018local,dieuleveut2019communication,woodworth2020local,woodworth2020minibatch,yuan2020federated,glasgow2022sharp,patel2024limits}; we include a proof in \Cref{app:recursions} for completeness. The \textcolor{blue}{blue} term is the \emph{consensus error}, which vanishes when all clients communicate at every time step (i.e., in fully synchronous SGD). Early analyses of Local SGD~\cite{woodworth2020minibatch}, often controlled this term using the restrictive \Cref{ass:zeta}.

In particular, \citet{woodworth2020minibatch} showed that the consensus error can be bounded as:
\begin{align}\label{eq:old_consensus_error_UB}
\textcolor{blue}{\frac{1}{M} \sum_{m \in [M]} \ee\left[\norm{x_t - x_t^m}^2\right]} 
\leq 2 K^2 \eta^2 H^2 \zeta^2+  6 K \sigma_2^2 \eta^2 \enspace.    
\end{align}
Similar upper bounds have also appeared in other works. We include a proof of the above statement in \Cref{app:zeta_results} for completeness.  Substituting \eqref{eq:old_consensus_error_UB} into \Cref{lem:iterate_error_second_recursion_main_body} and unrolling the recursion yields a convergence rate. However, as we discussed in \Cref{ch:baselines_and_lower_bounds}, \Cref{ass:zeta} is very restrictive as it requires the gradient functions across clients to be pointwise similar, allowing only limited heterogeneity---essentially in the linear terms. Notably, ~\citet{wang2022unreasonable} criticized the uniform consensus error bound in \eqref{eq:old_consensus_error_UB}, arguing that contrary to practice it implies an overly conservative step-size $\eta = \ooo(1/K)$ to prevent consensus error from diverging as $K \to \infty$.

The following result relaxes the need for \Cref{ass:zeta} by providing a new upper bound on the consensus error that depends on $\zeta_\star$, $\tau$, and the expected iterate error at the most recent communication round---a quantity that decreases over time.

\begin{lemma}[A Coupled Recursion for Consensus Error]\label{lem:consensus_error_second_recursion_main_body}
    Assume that the problem instance satisfies \Cref{ass:strongly_convex,ass:smooth_second,ass:stoch_bounded_second_moment,ass:bounded_optima,ass:zeta_star,ass:phi_star,ass:tau}. Then, for step-size $\eta < \frac{1}{H}$ and all $t \in [0, T]$, Local-SGD's iterates satisfy:
    \begin{align*}
        \frac{1}{M} \sum_{m \in [M]} \ee\left[\norm{x_t - x_t^m}^2\right] 
        &\leq 2 \eta^2 H^2 K^2 \zeta_\star^2 + \frac{2 \eta^3 \tau^2 K^2 \sigma_2^2}{\mu} + 2 \eta^2 \sigma_2^2 K \ln(K) \\
        &\quad + 4 \eta^2 \tau^2 (t - \delta(t))^2 (1 - \eta \mu)^{2(t - 1 - \delta(t))} \left( \ee\left[\norm{x_{\delta(t)} - x^\star}^2\right] + \phi_\star^2 \right) \enspace,
    \end{align*}
    where $\delta(t) := t - (t \mod K)$ is the most recent communication round prior to or at time $t$.
\end{lemma}

We prove this result in \Cref{app:double_recursions}. Unlike the earlier bound in \eqref{eq:old_consensus_error_UB}, our upper bound improves with lower second-order heterogeneity. In the limit $\tau \to 0$, it effectively replaces $\zeta$ with $\zeta_\star$ in \eqref{eq:old_consensus_error_UB}, and can therefore be significantly smaller. While our bound does require setting $\eta = \mathcal{O}(1/K)$ to prevent blow-up as $K \to \infty$, we provide an alternative bound in \Cref{app:explode} that avoids this and addresses the concerns raised by Wang et al.~\cite{wang2022unreasonable}. That said, as we explain in \Cref{app:explode}, the regime $\eta = \mathcal{O}(1/K)$ is ultimately the most relevant for our analysis, making \Cref{lem:consensus_error_second_recursion_main_body} more useful.

Combining the coupled recursions in \Cref{lem:consensus_error_second_recursion_main_body,lem:iterate_error_second_recursion_main_body} leads to the following convergence guarantee:

\begin{theorem}[Informal, Iterate Error]\label{thm:UB_Sconvex_wo_Q}
    Assume a problem instance satisfies \Cref{ass:strongly_convex,ass:smooth_second,ass:stoch_bounded_second_moment,ass:bounded_optima,ass:zeta_star,ass:phi_star,ass:tau} and $R = \tilde\Omega\left(\frac{H \tau}{\mu^2}\right)$. Then, for a suitable step-size $\eta$, and $x_0 = 0$ Local SGD outputs $x_{KR}$ such that:
    \begin{align*}
        \ee\left[\norm{x_{KR} - x^\star}^2\right] 
        &= \tilde\ooo\left( e^{-\frac{\mu K R}{2H}} B^2 
        + \frac{\sigma_2^2}{\mu^2 M K R} 
        + \frac{\tau^2 H^2 \phi_\star^2}{\mu^4 R^2} 
        + \frac{H^4 \zeta_\star^2}{\mu^4 R^2} 
        + \frac{H^2 \tau^2 \sigma_2^2}{\mu^6 K R^3} 
        + \frac{H^2 \sigma_2^2}{\mu^4 K R^2} \right) \enspace.
    \end{align*}
\end{theorem}

For the complete theorem statement, the precise step-size choice, and the derivation of the bound, see \Cref{app:together_1}. As a baseline, we can compare the above rate to the convergence rate of mini-batch SGD in the intermittent communication setting (see e.g., \cite{woodworth2020minibatch}),
\begin{align*}
    \ee\left[\norm{x_{KR}^{MB-SGD} - x^\star}^2\right] &= \ooo\rb{e^{-\frac{\mu R}{2H} }B^2 + \frac{\sigma_2^2}{\mu^2 M K R}}\enspace. 
\end{align*}
It is well known that the convergence rate for mini-batch SGD above is tight and \textbf{can not} improve with lower data heterogeneity \cite{woodworth2020minibatch,patel2024limits} (also cf. \Cref{rem:mb_rate_is_tight}). As we discussed in \Cref{ch:baselines_and_lower_bounds} local SGD can not beat mini-batch SGD under just \Cref{ass:zeta_star,ass:phi_star}, leaving open the question of what happens when we additionally have \Cref{ass:tau}. \Cref{thm:UB_Sconvex_wo_Q} answers this question, showing that with a small $\tau$, Local SGD can converge much faster than mini-batch SGD. Notably, when $K \to \infty$, the communication complexity of Local SGD for target accuracy $\epsilon$\footnote{We say a solution $\hat x$ has $\epsilon$ target iterate sub-optimality if $\ee\sb{\norm{\hat x - x^\star}}\leq \epsilon$.} and large $K$ satisfies:
\begin{align}
R^{L-SGD}(\epsilon) = \tilde\ooo\left( \frac{H \tau}{\mu^2} 
+ \frac{\tau H \phi_\star}{\mu^2 \sqrt{\epsilon}} 
+ \frac{H^2 \zeta_\star}{\mu^2 \sqrt{\epsilon}} \right) \enspace.\label{eq:comm_compl_sconvex_wo_Q}   
\end{align}
The above communication complexity decreases with data heterogeneity, suggesting that Local SGD becomes increasingly communication-efficient when tasks are more aligned. In particular, the convergence rate smoothly interpolates to the behavior on homogeneous problems, for which our bound implies that a constant number of communication rounds suffice. On the other hand, with similar $K$, the communication complexity of mini-batch SGD is $\tilde\Omega(\kappa)$, and does not improve with a lower data heterogeneity.

We note that \Cref{thm:UB_Sconvex_wo_Q} is the first ever result to prove the domination of Local SGD over mini-batch SGD in settings of reasonable heterogeneity, i.e., these rates only depend on $\tau$, $\zeta_\star$, $\phi_\star$, and not on $\zeta$, while also showing a provable benefit of local update steps. 

Using a different progress lemma (\Cref{app:function_value}), we also derive a corresponding function-value convergence result based on the same consensus error bound in \Cref{lem:consensus_error_second_recursion_main_body}.
\begin{theorem}[Informal, Function Error with Strong Convexity]\label{thm:UB_function_wo_Q}
    Assume a problem instance satisfies \Cref{ass:strongly_convex,ass:smooth_second,ass:stoch_bounded_second_moment,ass:bounded_optima,ass:zeta_star,ass:phi_star,ass:tau}, $R= \tilde\Omega\rb{\frac{\tau\sqrt{\kappa}}{\mu}}$, and $KR = \Omega(\kappa)$. Then, for a suitable choice of step-size $\eta$, Local SGD initialized at $x_0 = 0$ outputs $\hat x$, a weighted combination of its iterates, satisfying,
    \begin{align*}
        \ee\sb{F(\hat x)} - F(x^\star) &= \tilde\ooo\rb{e^{-\frac{\mu KR}{2H}}\mu B^2 + \frac{\sigma_2^2}{\mu MKR} + \frac{\tau^2H \phi_\star^2}{\mu^2R^2} + \frac{H^3\zeta_\star^2}{\mu^2R^2}+ \frac{H \tau^2\sigma_2^2}{\mu^4KR^3}+ \frac{H \sigma_2^2}{\mu^2 KR^2}}\enspace.
    \end{align*}
\end{theorem}
The proof of the above theorem can be found in \Cref{app:together_2}\footnote{In the regime $\kappa >>1$ \Cref{thm:UB_function_wo_Q} is much better than just applying second-order smoothness to \Cref{thm:UB_Sconvex_wo_Q}.}. 

The above convergence rate also exhibits a desirable dependence on the data heterogeneity constants and outperforms mini-batch SGD.

Finally, it is worth noting that the hard instance in \Cref{thm:new_LSGD_lower_bound} is a quadratic function, whereas \Cref{thm:UB_function_wo_Q} applies to general strongly convex objectives. This raises the possibility that, by restricting attention to quadratic functions, we may be able to improve upon the upper bound in \Cref{thm:UB_function_wo_Q}. In the next section, we explore this direction by deriving tighter upper bounds in regimes where the third-order smoothness constant $Q$ from \Cref{ass:smooth_third} is small.

\section{Incorporating Third-order Smoothness}\label{sec:third_order_smooth}

We will begin by stating a modified one-step progress result in \Cref{lem:iterate_error_second_recursion_main_body_w_Q} that explicitly captures second-order heterogeneity and third-order smoothness. This directly recovers improved bounds for quadratic objectives by setting $Q = 0$ in \Cref{thm:UB_Sconvex_quad,thm:UB_function_quad}. To handle general third-order smooth functions, we combine this with new bounds on the fourth moment of the consensus error (\Cref{app:fourth_consensus}) and a corresponding fourth-moment progress lemma (\Cref{app:fourth_error_iterate}), resulting in \Cref{thm:UB_Sconvex_w_Q}. These results show that when $Q$ and $\tau$ are small, Local SGD can achieve significantly faster convergence, even under substantial first-order heterogeneity.

\begin{lemma}[Modified One-step Lemma]\label{lem:iterate_error_second_recursion_main_body_w_Q}
    Assume the problem instance satisfies \Cref{ass:strongly_convex,ass:smooth_second,ass:smooth_third,ass:stoch_bounded_second_moment,ass:tau}. Then, for step-size $\eta < \frac{1}{H}$ and all $t \in [0, T-1]$, the iterates of Local SGD satisfy:
    \begin{align*}
        \mathbb{E}\left[\norm{x_{t+1} - x^\star}^2\right] 
        &\leq \left(1 - \eta \mu\right)\mathbb{E}\left[\norm{x_t - x^\star}^2\right] 
        + \frac{\eta^2 \sigma_2^2}{M} + \frac{2\eta Q^2}{\mu} \cdot \textcolor{red}{\frac{1}{M} \sum_{m \in [M]} \mathbb{E}\left[\norm{x_t - x_t^m}^4\right]}\\
        &\quad  + \frac{2\eta \tau^2}{\mu} \cdot \textcolor{blue}{\frac{1}{M} \sum_{m \in [M]} \mathbb{E}\left[\norm{x_t - x_t^m}^2\right]} \enspace.
    \end{align*}
\end{lemma}

We prove \Cref{lem:iterate_error_second_recursion_main_body_w_Q} in \Cref{app:recursions}. Compared to \Cref{lem:iterate_error_second_recursion_main_body}, this recursion introduces an additional \textcolor{red}{fourth-moment} of the \textcolor{blue}{consensus error}, weighted by the third-order smoothness constant $Q$. While this fourth-moment term can dominate the second-moment term, the decomposition reveals how smoother problems (with small $Q$ and $\tau$) reduce the impact of delayed communication. In particular, when $Q = 0$---which is true when each $F_m$ is quadratic---we obtain significantly sharper bounds than in \Cref{thm:UB_Sconvex_wo_Q}.

\begin{theorem}[Informal, Iterate Error for Quadratics]\label{thm:UB_Sconvex_quad}
    Assume the problem instance is quadratic and satisfies \Cref{ass:strongly_convex,ass:smooth_second,ass:stoch_bounded_second_moment,ass:bounded_optima,ass:zeta_star,ass:phi_star,ass:tau}, $R = \tilde{\Omega}\left(\frac{\tau^2}{\mu^2}\right)$ and $K R = \tilde{\Omega}(1)$. Then, for a suitable choice of step-size $\eta$, Local SGD initialized at $x_0 = 0$ outputs $x_{KR}$ such that:
    \begin{align*}
        \mathbb{E}\left[\norm{x_{KR} - x^\star}^2\right] 
        &= \tilde{\mathcal{O}}\left( e^{-\frac{\mu K R}{2H}} B^2 
        + \frac{\sigma_2^2}{\mu^2 M K R} 
        + \frac{\tau^4 \phi_\star^2}{\mu^4 R^2} 
        + \frac{\tau^2 H^2 \zeta_\star^2}{\mu^4 R^2} 
        + \frac{\tau^4 \sigma_2^2}{\mu^6 K R^3} 
        + \frac{\tau^2 \sigma_2^2}{\mu^4 K R^2} \right) \enspace.
    \end{align*}
\end{theorem}
\begin{remark}[Comparison to the Fixed-Point Perspective]
    The first, second, and fourth terms in the convergence rate above closely resemble those appearing in the fixed-point analysis of \Cref{ch:fixed_point}, particularly in \Cref{prop:conv_to_fixed}. Although we do not establish the tightness of the overall upper bound in \Cref{thm:UB_Sconvex_quad}, we believe these three terms are individually tight---up to potential differences between $\bar B$ and $B$.
\end{remark}
\begin{remark}[Removing $\phi_\star$ Dependence]\label{rem:no_phi_star_UB_Sconvex_quad}
    Recall that due to \Cref{rem:zeta_star_tau_vs_phi_star} we can upper bound $\phi_\star$ by $\zeta_\star\rb{1+ \frac{\tau}{\mu}}$ for quadratic. This allows us to simplify the above convergence rate to 
    \begin{align*}
        \mathbb{E}\left[\norm{x_{KR} - x^\star}^2\right] 
        &= \tilde{\mathcal{O}}\left( e^{-\frac{\mu K R}{2H}} B^2 
        + \frac{\sigma_2^2}{\mu^2 M K R} 
        + \frac{\tau^6 \zeta_\star^2}{\mu^6 R^2} 
        + \frac{\tau^2 H^2 \zeta_\star^2}{\mu^4 R^2} 
        + \frac{\tau^4 \sigma_2^2}{\mu^6 K R^3} 
        + \frac{\tau^2 \sigma_2^2}{\mu^4 K R^2} \right) \enspace.
    \end{align*}
    The above simpler upper bound can be comparable to \Cref{thm:UB_Sconvex_quad} when $\tau<<H$. It would be useful to compare this result later to the convergence rate in terms of $\zeta$.  
\end{remark}
We prove this theorem in \Cref{app:together_1}. To understand the improvement over \Cref{thm:UB_Sconvex_wo_Q}, consider the implied communication complexity in the large $K$ regime:
\begin{align}
R(\epsilon) = \tilde{\mathcal{O}}\left( \frac{\tau^2}{\mu^2} 
+ \frac{\tau^2 \phi_\star}{\mu^2 \sqrt{\epsilon}} 
+ \frac{\tau H \zeta_\star}{\mu^2 \sqrt{\epsilon}} \right) \enspace, \label{eq:comm_compl_sconvex_quad}
\end{align}
which becomes constant when $\tau = 0$. In contrast, the bound in \eqref{eq:comm_compl_sconvex_wo_Q} still depends on $\zeta_\star$ even when $\tau = 0$. This highlights how low third-order smoothness ($Q$) and low second-order heterogeneity ($\tau$) improve Local SGD’s performance---especially in settings where first-order heterogeneity remains large. It is also worth noting that the convergence rates for mini-batch SGD do not improve with a lower third-order smoothness, as the hard instances for mini-batch SGD are all quadratic~\cite{nesterov2018lectures} (also cf. \Cref{rem:mb_rate_is_tight}).

Using a different modified progress lemma (see \Cref{app:function_value}), we also derive the following convergence rate in terms of function values. 

\begin{theorem}[Informal, Function Error for Quadratics]\label{thm:UB_function_quad}
    Assume the problem instance is quadratic and satisfies \Cref{ass:strongly_convex,ass:smooth_second,ass:stoch_bounded_second_moment,ass:bounded_optima,ass:zeta_star,ass:phi_star,ass:tau}, $R= \tilde\Omega\rb{\frac{\tau^2}{\mu^2}}$, and $KR = \Omega(\kappa)$. Then, for a suitable choice of step-size $\eta$, Local SGD initialized at $x_0 = 0$ outputs $\hat x$, a weighted combination of its iterates, satisfying,
    \begin{align*}
        \ee\sb{F(\hat x)} - F(x^\star) &= \tilde\ooo\rb{e^{-\frac{\mu KR}{2H}}\mu B^2 + \frac{\sigma_2^2}{\mu MKR} + \frac{\tau^4 \phi_\star^2}{\mu^3R^2} + \frac{\tau^2H^2\zeta_\star^2}{\mu^3R^2}+ \frac{\tau^4\sigma_2^2}{\mu^5KR^3}+ \frac{\tau^2\sigma_2^2}{\mu^3 KR^2}}\enspace.
    \end{align*}
\end{theorem}
The proof for the above theorem can be found in \Cref{app:together_2}. Compared to \Cref{thm:UB_Sconvex_wo_Q} we again see an improvement, as all but the first two terms in the convergence rate go to zero when $\tau=0$.
\begin{remark}[Removing $\phi_\star$ Dependence]\label{rem:no_phi_star_UB_function_quad}
    Recall that due to \Cref{rem:zeta_star_tau_vs_phi_star} we can upper bound $\phi_\star$ by $\zeta_\star\rb{1+ \frac{\tau}{\mu}}$ for quadratic. This allows us to simplify the above convergence rate to 
    \begin{align*}
        \ee\sb{F(\hat x)} - F(x^\star) &= \tilde\ooo\rb{e^{-\frac{\mu KR}{2H}}\mu B^2 + \frac{\sigma_2^2}{\mu MKR} + \frac{\tau^6 \zeta_\star^2}{\mu^5R^2} + \frac{\tau^2H^2\zeta_\star^2}{\mu^3R^2}+ \frac{\tau^4\sigma_2^2}{\mu^5KR^3}+ \frac{\tau^2\sigma_2^2}{\mu^3 KR^2}}\enspace.
    \end{align*}
    This rate suggests that the communication complexity for quadratic in the regime when $K$ is large is given by $\ooo\rb{\frac{\tau^3\zeta_\star}{\mu^{5/2}\epsilon^{1/2}} + \frac{\tau H \zeta_\star}{\mu^{3/2}\epsilon^{1/2}}}$ for target accuracy $\epsilon$. 
\end{remark}

Finally, we prove the following result for general third-order smooth functions.
\begin{theorem}[Informal, Iterate Error with $Q$]\label{thm:UB_Sconvex_w_Q}
    Assume a problem instance satisfies \Cref{ass:strongly_convex,ass:smooth_second,ass:smooth_third,ass:stoch_bounded_second_moment,ass:stoch_bounded_fourth_moment,ass:bounded_optima,ass:zeta_star,ass:phi_star,ass:tau}. Then, for a suitable choice of step-size $\eta$, Local SGD initialized at $x_0 = 0$ outputs $x_{KR}$ satisfying:
    \begin{align*}
        &\ee\sb{\norm{x_{KR}-x^\star}^2} + \frac{1}{B^2}\ee\sb{\norm{x_{KR}-x^\star}^4} = \tilde\ooo\Bigg(e^{-KR/\kappa}B^2 + \frac{\sigma_2^2}{\mu^2 MKR} + \frac{\sigma_4^4}{\mu^4K^3R^3 M^2B^2}\\ 
        &\quad+ \kappa'\rb{\frac{\tau^2\phi_\star^2}{\mu^2R^2}+ \frac{\tau^4\sigma_4^2}{\mu^6 KR^5B^2}\phi_\star^2 + \frac{\sigma_2^2\tau^2}{\mu^4KR^4B^2}\phi_\star^2 + \frac{\tau^4}{\mu^4B^2R^4}\phi_\star^4 + \frac{H^2\zeta_\star^2}{\mu^2 R^2} + \frac{\tau^2\sigma_2^2}{\mu^4KR^3}}\\
        &\quad + \kappa'\rb{\frac{\sigma_2^2\ln(K)}{\mu^2KR^2}+\frac{H^4\zeta_\star^4}{\mu^4R^3 B^2} +\frac{\tau^4\sigma_4^4}{\mu^8K^2R^5B^2} + \frac{\sigma_2^2H^2\zeta_\star^2}{\mu^4B^2R^4} + \frac{\tau^2\sigma_2^4}{\mu^6KR^5 B^2} + \frac{\sigma_4^4\ln(K)}{\mu^4KB^2R^4}}\Bigg)\enspace,
    \end{align*}
    where we assume $R = \tilde \Omega\rb{\frac{\tau \sqrt{\kappa'}}{\mu}}$ and define $\kappa':= 2 + \frac{4Q^2B^2}{\mu^2} + \frac{6H^4}{\mu^4}$.
\end{theorem}
We can see that the above convergence rate improves with smaller $\tau$ and $Q$, via the constant $\kappa'$, and the effect of a low third-order smoothness is most pronounced when $B/\mu^2$ is large relative to $\kappa^4$. To prove the above theorem, we first derive new fourth-moment bounds on the consensus error and one-step progress in \Cref{app:recursions,app:double_recursions}. Solving the resulting four coupled recursions directly is challenging, so we stack the iterate and consensus recursions into two vectors and apply matrix algebra, leading to a cleaner proof in \Cref{app:together_3}. A similar strategy was employed by \citet{yuan2020federated}, but in the much simpler homogeneous setting, where they did not need to address coupled recursions. A limitation of our analysis is that the final bound is expressed in terms of the norm of a stacked vector that includes both second and fourth-moment errors. Since bounding the fourth moment of the iterate error is not strictly necessary, this may have introduced extraneous terms in the upper bound. For instance, the upper bound does not recover the quadratic convergence rate in \Cref{thm:UB_Sconvex_quad} when $Q=0$. We therefore believe that \Cref{thm:UB_Sconvex_w_Q} could be further improved through a more refined analysis of the underlying matrix inequalities.

Before we end this section, it would be helpful to state the following convergence guarantee in terms of \Cref{ass:smooth_third,ass:zeta,ass:tau} to highlight what the best version of the \Cref{thm:UB_Sconvex_w_Q} might look like. 
\begin{theorem}[Informal, Iterate Error with $Q,\ \zeta,\ \tau$]\label{thm:UB_Sconvex_w_Q_zeta}
    Assume a problem instance satisfies \Cref{ass:strongly_convex,ass:smooth_second,ass:smooth_third,ass:stoch_bounded_second_moment,ass:stoch_bounded_fourth_moment,ass:bounded_optima,ass:zeta,ass:tau}. Then, for a suitable choice of step-size $\eta$, Local SGD initialized at $x_0 = 0$ outputs $x_{KR}$ satisfying:
    \begin{align*}
        \ee\sb{\norm{x_{KR}-x^\star}^2} &= \tilde\ooo\Bigg(e^{-KR/2\kappa} B^2 + \frac{Q^2  H^4 \zeta^4}{\mu^6R^4} + \frac{Q^2 \sigma_2^4}{\mu^6K^2R^4} +\frac{Q^2 \sigma_4^4}{\mu^6K^3R^4}+\frac{\tau^2  H^2 \zeta^2}{\mu^4R^2}+\frac{\tau^2 \sigma_2^2}{\mu^4 KR^2}+\frac{\sigma_2^2}{\mu^2 MKR}\Bigg)\enspace.
    \end{align*}
\end{theorem}
We prove the above theorem in \Cref{app:zeta_iterate}. Note that, unlike \Cref{thm:UB_convex_w_Q}, the above theorem recovers the homogeneous extreme communication efficiency when $Q,\ \tau=0$. 
\begin{remark}[Low Second-order Heterogeneity]
When only $\tau=0$ we recall that due to \Cref{prop:zeta_bound} we can effectively replace $\zeta$ by $\zeta_\star$. This means we can recover the following convergence guarantee in terms of $\zeta_\star,\ Q$,
\begin{align*}
        \ee\sb{\norm{x_{KR}-x^\star}^2} &= \tilde\ooo\Bigg(e^{-KR/2\kappa} B^2 + \frac{Q^2  H^4 \zeta_\star^4}{\mu^6R^4} + \frac{Q^2 \sigma_2^4}{\mu^6K^2R^4} +\frac{Q^2 \sigma_4^4}{\mu^6K^3R^4}+\frac{\sigma_2^2}{\mu^2 MKR}\Bigg)\enspace.
\end{align*}
It is worth noting that \Cref{thm:UB_convex_w_Q} has many other extra terms in the convergence rate in this regime. This makes us further suspect that  \Cref{thm:UB_convex_w_Q} can be improved further. 
\end{remark}
\begin{remark}[Low Third-order Smoothness]
When only $Q=0$ we can recover the following convergence guarantee in terms of $\zeta,\ \tau$,
\begin{align*}
        \ee\sb{\norm{x_{KR}-x^\star}^2} &= \tilde\ooo\Bigg(\textcolor{blue}{e^{-KR/2\kappa} B^2} + \frac{\tau^2  H^2 \zeta^2}{\mu^4R^2}+\textcolor{blue}{\frac{\tau^2 \sigma_2^2}{\mu^4 KR^2}}+\textcolor{blue}{\frac{\sigma_2^2}{\mu^2 MKR}}\Bigg)\enspace.
    \end{align*}
We can compare the convergence guarantee to \Cref{rem:no_phi_star_UB_Sconvex_quad} for quadratic functions. We note that the \textcolor{blue}{blue} terms in the rate above are shared with \Cref{rem:no_phi_star_UB_Sconvex_quad}. Furthermore, replacing $\zeta$ by $\zeta_\star$ matches the fourth term in \Cref{rem:no_phi_star_UB_Sconvex_quad}. Based on this, we conjecture that the convergence rate in \Cref{rem:no_phi_star_UB_Sconvex_quad} (and we suspect also in \Cref{thm:UB_Sconvex_quad}) is nearly tight. 
\end{remark}
\begin{remark}[High Kurtosis Distributions]
    Across all the upper bounds above we note that the terms with $\sigma_4$ usually decay much faster in $R$ or $K$ or both, which highlights the advantage of differentiating between the second and fourth moments of noise in \Cref{ass:stoch_bounded_second_moment,ass:stoch_bounded_fourth_moment}. In particular, in \Cref{thm:UB_Sconvex_w_Q_zeta} we notice that increasing the local updates $K$ reduces the fourth moment term much more, implying that local updates can be beneficial for fat-tailed distributions where $\sigma_4 >> \sigma_2$. To the best of our knowledge, this benefit of local updates has not been previously highlighted (cf. the rates due to \citet{yuan2020federated}). 
\end{remark}

\section{Convex Setting}\label{sec:ch5.3}

To obtain convergence guarantees in the convex setting under \Cref{ass:zeta_star,ass:phi_star,ass:tau}, one natural approach is to begin with \Cref{thm:UB_function_wo_Q} and apply a convex-to-strongly-convex reduction via regularization. However, this strategy imposes overly stringent constraints on the heterogeneity constants and the number of communication rounds, which cannot be simultaneously satisfied. Similarly, deriving a function-value analogue of \Cref{thm:UB_Sconvex_w_Q} proves challenging due to the presence of multiple coupled recursions. While we suspect that the techniques from the strongly convex setting could be extended to address this case, we leave such an investigation to future work. Instead, in this section, we present results under the more restrictive \Cref{ass:zeta} where we can give uniform upper bounds on consensus error terms.

Our proof proceeds in two steps: (i) we first establish the function-value analogue of \Cref{thm:UB_Sconvex_w_Q_zeta}; and then (ii) we apply a convex-to-strongly-convex reduction by running Local SGD on a suitably regularized objective. The proof of the strongly convex guarantee in (i) mirrors the structure of the analysis in the previous section. We begin by proving a one-step lemma---analogous to \Cref{lem:iterate_error_second_recursion_main_body_w_Q}---for function sub-optimality that accounts for both $\tau$ and $Q$. We then use the basic consensus error bound from \eqref{eq:old_consensus_error_UB}, together with a new bound on the fourth moment of the consensus error in terms of $\zeta$, to complete the argument.

Bypassing the reduction to strongly convex optimization remains an open question. In particular, it would require establishing a one-step lemma for general convex functions that still incorporates the higher-order terms $Q$ and $\tau$---a direction that we leave for future work.

Our one-step recursion lemma in the strongly convex setting (proved in \Cref{app:function_value}) is stated below.

\begin{lemma}\label{lem:function_error_hard_main}
    Assume the problem instance satisfies \Cref{ass:strongly_convex,ass:smooth_second,ass:smooth_third,ass:stoch_bounded_second_moment,ass:tau}. Then, for step-size $\eta < \frac{1}{H}$ and all $t \in [0, T-1]$, the iterates of Local SGD satisfy (for some $x^\star \in S^\star$):
    \begin{align*}
        \ee\sb{F(x_t)} - F(x^\star) &\leq \left(\frac{1}{\eta} - \frac{\mu}{2}\right)\ee\sb{\norm{x_t - x^\star}^2} - \frac{1}{\eta}\ee\sb{\norm{x_{t+1} - x^\star}^2} + \frac{\eta \sigma_2^2}{M} \\
        &\quad + \frac{8\tau^2}{\mu}\cdot \textcolor{blue}{\frac{1}{M} \sum_{m \in [M]} \ee\sb{\norm{x_t - x_t^m}^2}} + \frac{2Q^2}{\mu} \cdot \textcolor{red}{\frac{1}{M} \sum_{m \in [M]} \ee\sb{\norm{x_t - x_t^m}^4}}\enspace.
    \end{align*}
\end{lemma}

This recursion simultaneously tracks both function sub-optimality and iterate error. As a result, a careful telescoping argument is required to cancel out the iterate error terms. Similar to \Cref{lem:iterate_error_second_recursion_main_body_w_Q}, the recursion features two types of consensus error: the \textcolor{blue}{second moment} and the \textcolor{red}{fourth moment}. These arise from incorporating both \Cref{ass:smooth_third,ass:tau} into the analysis.

We prove the following upper bound on the fourth moment of the consensus error (see \Cref{lem:cons_error_fourth})\footnote{Compared to \eqref{eq:old_consensus_error_UB}, the fourth moment bound depends differently on $\sigma_2$ and $\sigma_4$. In particular, since $\sigma_4$ can be larger than $\sigma_2$ in general, the additional $K$ factor in the third term of \eqref{eq:consensus_fourth_specific} may be non-negligible. However, if $\sigma_2$ and $\sigma_4$ are of similar magnitude, then \eqref{eq:consensus_fourth_specific} implies \eqref{eq:old_consensus_error_UB} up to constant factors.}:

\begin{equation}\label{eq:consensus_fourth_specific}
    \textcolor{red}{\frac{1}{MT} \sum_{m \in [M],\ t \in [0, T-1]} \ee\sb{\norm{x_t - x_t^m}^4}} \leq 2620\eta^4 K^4 H^4 \zeta^4 + 5000\eta^4 K^2 \sigma_2^4 + 320\eta^4 \sigma_4^4 K\enspace.
\end{equation}

Combining \Cref{lem:function_error_hard_main} with the consensus error bounds from \eqref{eq:old_consensus_error_UB} and \eqref{eq:consensus_fourth_specific} yields the following result (proved in \Cref{app:zeta_func}), which serves as the function-value analogue of \Cref{thm:UB_Sconvex_w_Q_zeta}.

\begin{theorem}[Informal, Function Sub-optimality with $Q,\ \zeta,\ \tau$]\label{thm:UB_Sconvex_func_w_Q_zeta}
    Assume the problem instance satisfies \Cref{ass:strongly_convex,ass:smooth_second,ass:smooth_third,ass:stoch_bounded_second_moment,ass:stoch_bounded_fourth_moment,ass:bounded_optima,ass:zeta,ass:tau}. Then, for a suitable choice of step-size $\eta$, a weighted Local SGD iterate $\hat{x}$ with initialization $x_0 = 0$ satisfies for large enough $R$ and $K$\footnote{For precise constraints on $R$ and $K$, see the full statement in \Cref{app:zeta_func}.}:
    \begin{align*}
        \ee\sb{F(\hat{x})} - F(x^\star) &= \tilde{\mathcal{O}}\Bigg(
            \mu B^2 e^{-KR / 4\kappa}
            + \frac{Q^2 H^4 \zeta^4}{\mu^5 R^4}
            + \frac{Q^2 \sigma_2^4}{\mu^5 K^2 R^4}
            + \frac{Q^2 \sigma_4^4}{\mu^5 K^3 R^4}
            + \frac{\tau^2 H^2 \zeta^2}{\mu^3 R^2}
            + \frac{\tau^2 \sigma_2^2}{\mu^3 K R^2}\\
            &\qquad\qquad\qquad + \frac{\sigma_2^2}{\mu M K R}
        \Bigg)\enspace.
    \end{align*}
\end{theorem}

We do not repeat the full discussion of the convergence behavior, as it closely mirrors that of \Cref{thm:UB_Sconvex_w_Q_zeta}. One notable distinction, however, is that the guarantee here is for a \textit{weighted average} of the iterates across all machines---not the final iterate. Such averaging is standard in convex optimization~\citep{lacoste2012simpler}. Establishing similarly strong guarantees for the final iterate, even in the convex setting, remains an active area of research~\citep{liu2023revisiting}. Addressing this is beyond the scope of this thesis.

We now apply Local SGD to the regularized objective on each machine $m\in[M]$: 
\(
F_{m,\mu}(x) := F_m(x) + \frac{\mu}{2}\norm{x}^2
\).
Using \Cref{thm:UB_Sconvex_func_w_Q_zeta}, we obtain a convergence guarantee for the regularized average objective $F_\mu(x) := \frac{1}{M} \sum_{m \in [M]} F_{m,\mu}(x)$. To translate this guarantee into one for the original (unregularized) objective $F(x)$, we invoke the following standard inequality (proved in \Cref{app:zeta_func}):
\[
F(\hat{x}) - F(x^\star) \leq F_\mu(\hat{x}) - \min_{x_\mu^\star \in \rr^d} F_\mu(x_\mu^\star) + \frac{\mu}{2} \norm{x^\star}^2\enspace,
\]
where $x^\star \in S^\star$ denotes an optimum of the original objective $F$.

The regularization strength $\mu$ presents a trade-off: increasing $\mu$ improves the conditioning of $F_\mu$ and accelerates convergence, but also worsens the approximation error due to the $\frac{\mu}{2} \norm{x^\star}^2$ term. To obtain the final guarantee, we optimize this trade-off by carefully tuning $\mu$, which leads to the following result.

\begin{theorem}[Informal, Function Sub-optimality with $Q$ in the Convex Setting]\label{thm:UB_convex_w_Q}
    Assume the problem instance satisfies \Cref{ass:convex,ass:smooth_second,ass:smooth_third,ass:stoch_bounded_second_moment,ass:stoch_bounded_fourth_moment,ass:bounded_optima,ass:tau,ass:zeta}. Then, for a suitable step-size $\eta$, an appropriate regularization strength $\mu$, and a weighted Local SGD iterate $\hat{x}$ initialized at $x_0 = 0$, we have (for some $x^\star \in S^\star$):
    \begin{align*}
        \ee\sb{F(\hat{x})} - F(x^\star) = \tilde{\mathcal{O}}\Bigg(&\frac{H B^2}{K R} + \frac{\tau^{1/2} H^{1/2} \zeta^{1/2} B^{3/2}}{R^{1/2}} + \frac{\tau^{1/2} \sigma_2^{1/2} B^{3/2}}{K^{1/4} R^{1/2}} + \frac{Q^{1/3} B^{5/3} H^{2/3} \zeta^{2/3}}{R^{2/3}} \\
        &+ \frac{Q^{1/3} B^{5/3} \sigma_2^{2/3}}{K^{1/3} R^{2/3}} + \frac{Q^{1/3} B^{5/3} \sigma_4^{2/3}}{K^{1/2} R^{2/3}} + \frac{\sigma_2 B}{\sqrt{M K R}} \Bigg)\enspace,
    \end{align*}
    provided that
    \[
    R = \Omega\left(
        \frac{1}{K} 
        + \frac{\sigma_2^2}{H^2 B^2 M K} 
        + \frac{\tau \zeta}{B H} 
        + \frac{Q \zeta^2}{H B}
        + \frac{\tau \sigma_2}{H^2 B \sqrt{K}} 
        + \frac{\sigma_2 Q^{1/2}}{H^{3/2} \sqrt{B K}} 
        + \frac{Q^{1/2} \sigma_4}{H^{3/2} \sqrt{B} K^{3/4}}
    \right)\enspace.
    \]
\end{theorem}
\begin{remark}[Extreme Communication Efficiency]
    Observe that in the quadratic homogeneous setting, i.e., when $\tau = \zeta = Q = 0$, our upper bound recovers extreme communication efficiency, consistent with the intuition provided by the lower bound. More generally, in the regime when $K$ is large enough we can write the communication complexity in terms of the target accuracy $\epsilon$,
    \begin{align*}
        R = \tilde\ooo\rb{\frac{\tau \zeta}{HB} + \frac{Q\zeta^2}{HB} + \frac{\tau H\zeta B^3}{\epsilon^2} + \frac{Q^{1/2}B^{5/2}H\zeta}{\epsilon^{3/2}}}\enspace.
    \end{align*}
    Notably the communication complexity improves with both smaller second-order heterogeneity and third-order smoothness. Unfortunately, since we rely on \Cref{ass:zeta}, the upper bound can not expose the dependence on $\phi_\star$ and $\zeta_\star$. Having said that when $\tau=0$, we can replace $\zeta$ by $\zeta_\star$ in the above communication complexity due to \Cref{prop:zeta_bound}.
\end{remark}

\begin{remark}[Comparison to Existing Results]
     In the homogeneous setting when $\sigma_2 = \sigma_4$ our rate recovers the upper bound of \citet{yuan2020federated}, which also incorporates third-order smoothness $Q$. Although we lack a matching lower bound, we suspect this rate is tight in the homogeneous regime. 
     
     We can also compare our result with the upper bound of \citet{woodworth2020minibatch}, which does not account for dependence on $\tau$ or $Q$ but depends on $\zeta$. To facilitate this comparison, it is helpful to interpret $\zeta$ not just as the constant from \Cref{ass:zeta}, but as a measure of actual gradient heterogeneity across clients. We then compare the required upper bounds on $\zeta$ for achieving a target suboptimality $\epsilon$. Assuming $K$ is large enough to ignore terms involving $1/K$, the bounds become:
    \begin{align*}
        \zeta_{\text{old}} &= \ooo\left(\frac{\epsilon^{3/2} R}{H^{3/2} B^2}\right)
        \qquad \text{vs} \qquad
        \zeta_{\text{ours}} = \ooo\left(\min\left\{
        \frac{\epsilon^{3/2} R}{(Q B)^{1/2} HB^2},\;
        \frac{\epsilon^2 R}{\tau H B^3}
        \right\}\right)\enspace.
    \end{align*}
    In the regime where $Q$ and $\tau$ are small, our requirements on $\zeta$—that is, the gradient heterogeneity—are significantly less stringent.
\end{remark}

While we do not fully resolve the min-max complexity of Local SGD under \Cref{ass:convex,ass:smooth_second,ass:smooth_third,ass:stoch_bounded_second_moment,ass:stoch_bounded_fourth_moment,ass:zeta_star,ass:phi_star,ass:tau}, the results in this chapter represent tangible progress toward understanding Local SGD's convergence behavior under higher-order assumptions. We hope these insights will guide future investigations.

\section{Empirical Study: Distributed Linear Regression}\label{sec:ch5.4}

\begin{figure}[H]
  \centering
  \begin{subfigure}[t]{0.48\textwidth}
    \centering
    \includegraphics[width=\linewidth]{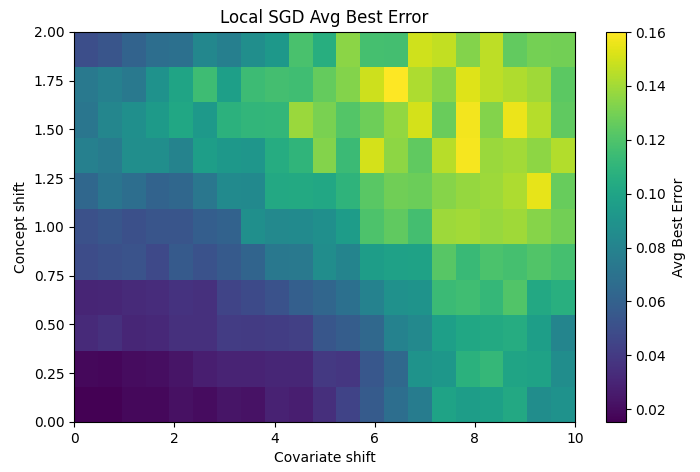}
    \caption{
      Heatmap of the average best final $\ell_2$ error of Local SGD after $R=5$ communication rounds as a function of covariate shift $\tau$ (horizontal axis) and concept shift $\zeta_\star$ (vertical axis).   
    }
    \label{fig:heatmap}
  \end{subfigure}
  \quad
  \begin{subfigure}[t]{0.48\textwidth}
    \centering
    \includegraphics[width=\linewidth]{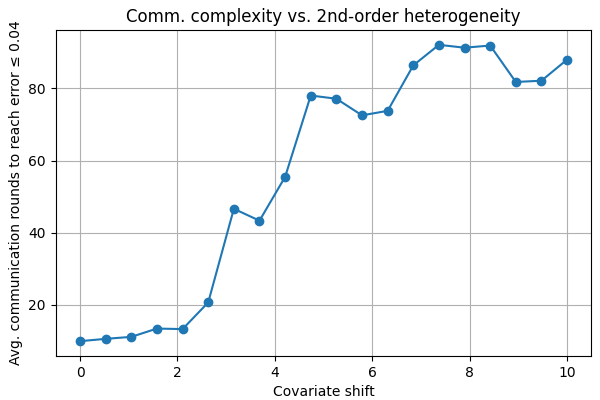}
    \caption{
      Communication complexity of Local SGD versus covariate shift $\tau$, for a fixed concept shift $\zeta_\star=1.0$ to reach an $\ell_2$ error $\le0.04$. We allow up to $R_{\max}=100$ rounds and plot the mean number of rounds to target.
    }
    \label{fig:comm_rounds}
  \end{subfigure}
  \caption{\textbf{Impact of First- and Second-Order Heterogeneity on Local SGD.} In both figures, we use $d = 5$, $M = 20$ clients, $K = 10$ local steps, and a noise level of $\sigma_{noise} = 0.1$. The step-size is tuned over a logarithmic grid in $[10^{-3}, 10^{-1}]$, and the error is averaged over multiple trials. For (a), we report the mean error over $n_{runs} = 20$ trials for each $(\tau, \zeta_\star)$ pair, tuning the step-size separately in each trial. Similarly, in (b), we average over $n_{runs} = 20$ trials for each $\tau$, again tuning the step-size independently per trial. We discuss in \Cref{app:experiments} how to interpret the numerical values of $\tau$, $\zeta_\star$ in our plots' axes.}
  \label{fig:combined_local_sgd}
\end{figure}

We consider a linear regression task, where for each client $m \in [M]$, the data consists of covariate-label pairs $z_m := (\beta_m, y_m) \sim \mathcal{D}_m$ with Gaussian covariates $\beta_m \sim \mathcal{N}(\mu_m, I_d) \in \rr^d$ and labels $y_m \sim \inner{x_m^\star}{\beta_m} + \mathcal{N}(0, \sigma_{\text{noise}}^2)$  generated using a ground truth model $x_m^\star \in \rr^d$. Each client minimizes the mean squared error, $f(x; (\beta_m, y_m)) = \frac{1}{2}(y_m - \inner{x}{\beta_m})^2$ leading to an expected loss:
    \[
    F_m(x) = \frac{1}{2}(x - x_m^\star)^\top (\mu_m \mu_m^\top + I_d)(x - x_m^\star) + \frac{1}{2} \sigma_{\text{noise}}^2 \enspace.
    \]
    Under suitable bounds on $\mu_m$, $\Sigma_m$, and $\sigma_{\text{noise}}$, this problem satisfies \Cref{ass:strongly_convex,ass:smooth_second,ass:stoch_bounded_second_moment,ass:bounded_optima} for bounded $x$. Furthermore, we have $\norm{\nabla^2 F_m(x) - \nabla^2 F_n(x)} \leq (\norm{\mu_m} + \norm{\mu_n}) \cdot \norm{\mu_m - \mu_n}$ for any $m,n\in[M]$. So \Cref{ass:tau} quantifies the \textbf{covariate shift} across clients. Meanwhile, \Cref{ass:zeta_star} reflects the \textbf{concept shift} via the bound $\norm{x_m^\star - x_n^\star} \leq \zeta_\star$. 

    In \Cref{fig:combined_local_sgd}, we examine the convergence behavior of Local SGD on the synthetic linear regression task. In \Cref{fig:heatmap}, we decouple first- and second-order heterogeneity by independently varying the means $\mu_m$ and the ground truths $x_m^\star$. We observe that Local SGD performs well only when both types of heterogeneity are small. This highlights why earlier works that did not account for second-order heterogeneity (\Cref{ass:tau}) were unable to explain the effectiveness of Local SGD fully. In \Cref{fig:comm_rounds}, we fix the first-order heterogeneity and plot the communication complexity required to reach a target accuracy as a function of $\tau$. As expected, we find a monotonic relationship, further reinforcing the connection between second-order heterogeneity and the communication efficiency of Local SGD. 
    
    Importantly, when varying the heterogeneity, we ensure we do not inadvertently make the individual optimization problems harder, for example, by increasing the condition number $\kappa$ or the radius $B$. In \Cref{app:experiments}, we describe how we control for this and include additional experiments.

\textbf{Practical Implications for Federated Learning.}  
Our results highlight that the performance of Local SGD depends critically on the structure of data heterogeneity. In practice, this suggests distinguishing between heterogeneity in optimal predictors (first-order, measured by $\zeta_\star$ and $\phi_\star$) and curvature or feature distributions (second-order, measured by $\tau$). For example, $\zeta_\star$ may be small for learning in overparameterized settings while $\tau$ remains significant. Large local steps ($K$) can still yield good performance and communication savings in such cases. But when $\tau$ is very large, aggressive local updates with a fixed step-size can cause instability. We recommend tuning $\eta$ as a function of $K$ and using diagnostic signals---such as consensus error growth or curvature estimates---to adjust training parameters. Estimating $\tau$ from local and running statistics could help guide such choices in practice.

\chapter{Local Update Algorithms for Non-Convex Functions}\label{ch:non_convex}

In this chapter, we propose a new local-update algorithm for the non-convex setting, which can be interpreted as a variance-reduced extension of Local SGD. We provide both upper bounds for the convergence of our algorithm and lower bounds for distributed zero-respecting algorithms (cf.~\Cref{def:zero_respecting}), thereby demonstrating that our method is nearly minimax optimal. The overarching aim of this chapter is to demonstrate that the second-order heterogeneity assumption (\Cref{ass:tau}) remains a critical factor in the non-convex regime. Specifically, we argue that small values of $\tau$ are essential for local updates to offer algorithmic advantages. Our main contributions are as follows:
\begin{enumerate}
    \item We establish a new lower bound in \Cref{thm:lb} for distributed zero-respecting algorithms, which explicitly depends on the heterogeneity parameters $\zeta$ and $\tau$ (\Cref{ass:zeta,ass:tau}). Analogous to the convex setting, this result indicates that a low $\tau$ can lead to improved communication complexity.

    \item We also prove a lower bound for centralized algorithms in \Cref{thm:lb_cent}, showing that their communication complexity does not improve even when heterogeneity is bounded. Moreover, this bound is tight---it matches the convergence rate of existing mini-batch algorithms (cf.~\Cref{sec:minibatch_storm}).

    \item We introduce a new communication-efficient algorithm, \textsc{CE-LSGD} and show that \textsc{CE-LSGD} is \emph{minimax optimal} under deterministic gradient oracles and \emph{nearly optimal} under stochastic oracles (see \Cref{thm:alg1,thm:lb,thm:alg1_full_main} and the discussion in \Cref{subsec:other_pers}).

    \item Finally, we analyze the trade-off between oracle and communication complexity in the regime of low target accuracy $\epsilon$---relevant to overparameterized deep learning models---showing \textsc{CE-LSGD} achieves optimal complexity trade-offs with a simpler variance-reduction structure (\Cref{fig:3region}, \Cref{fig:intermittent}).
\end{enumerate}

\subsection*{Outline and Relevant References}

The results in this chapter are based on our joint work~\citep{patel2022towards} with co-authors Lingxiao Wang, Blake Woodworth, Brian Bullins, and Nathan Srebro.

\Cref{sec:ch6.1} introduces the additional assumptions required in the non-convex setting. The overall setup follows a standard formulation, widely studied in prior work on distributed non-convex optimization~\citep{karimireddy2019error, karimireddy2020mime, murata2021bias}. This section also presents a new \emph{centralized lower bound} for heterogeneous objectives. While the bound is novel in the distributed context, it is derived by adapting the serial lower bound of~\citet{arjevani2019lower}.

\Cref{sec:ch6.2} presents our algorithm, \textsc{CE-LSGD}, along with its convergence analysis. The algorithm is closely related to \textsc{BVR-LSGD}~\citep{murata2021bias}, but improves upon it by requiring fewer and lighter heavy-batch computations. Our convergence guarantee, together with a new lower bound for distributed zero-respecting algorithms, establishes that \textsc{CE-LSGD} is \emph{minimax optimal} under exact oracles and \emph{nearly optimal} in the stochastic case. The construction of our lower bound builds on the non-convex hard instance proposed by~\citet{carmon2020lower} for serial optimization and leverages techniques from~\citet{arjevani2015communication}, which have also been employed in other works~\citep{woodworth2020minibatch, zhang2020fedpd}. A detailed comparison of related algorithms and their assumptions appears in \Cref{table:comparision}. Notably,~\citet{karimireddy2020mime} were among the first to highlight the role of second-order heterogeneity in distributed non-convex optimization.

Finally, \Cref{sec:exps} presents an empirical evaluation of \textsc{CE-LSGD}, demonstrating that its performance aligns with our theoretical predictions.

\section{Additional Assumptions and a Centralized Lower Bound}\label{sec:ch6.1}
We first recall that in the non-convex setting, optimization guarantees for solving problem \eqref{prob:scalar} are stated in terms of the stationarity of the outputted model on the average objective $F$. In particular, throughout this chapter, when we say a model $\hat x$ satisfies $\epsilon$ sub-optimality, when $\ee\|\nabla F(\hat{x})\|_2^2\leq \epsilon$. In the non-convex setting it is also common to assume the following function value equivalent of \Cref{ass:bounded_optima}. 

\begin{assumption}[Bounded Function Sub-optimality]\label{ass:bounded_func_subopt}
    We assume that for all $x^\star\in S^\star$ we have $$F(0) - F(x^\star)\leq \Delta\enspace.$$
\end{assumption}

\begin{table*}[!thp]
\centering
        \begin{tabular}{ll}
			\toprule
			\toprule
			Method (Reference) & \multirow{2}{*}{Convergence Rate, i.e. $\ee\norm{\nabla F(\hat{x})}^2 \preceq $}  \\
			(Oracles used)&\\
			\midrule
                \midrule
			\textsc{SCAFFOLD}$^\dagger$, \textsc{MB-SGD}$^\dagger$ \citep{karimireddy2020scaffold}
			& \multirow{2}{*}{$\frac{\Delta H}{R} + \left(\frac{\sigma_2^2\Delta L}{MKR}\right)^{1/2}$}  \\ 
			(Stochastic) &\\
			\midrule
			 \textsc{MB-STORM} (Theorem \ref{thm:minibatch_storm}) \citep{ cutkosky2019momentum}
			 & \multirow{2}{*}{$\frac{\Delta H}{R} + \frac{\sigma_2^2}{MKR} + \left(\frac{\sigma_2\Delta H}{MKR}\right)^{2/3}$} \\
			(Stochastic)&\\
			\midrule 
			Lower Bound (Centralized)
			& \multirow{2}{*}{$\frac{\Delta H}{R} + \frac{\sigma_2^2}{MKR} + \left(\frac{\sigma_2\Delta H}{MKR}\right)^{2/3} $}\\ 
			(Theorem \ref{thm:lb_cent})&\\
			\midrule
			\midrule
            \textsc{STEM}~\citep{khanduri2021stem}
			& \multirow{2}{*}{$(\Delta H + \sigma_2^2 + H^2\zeta^2)\left(\frac{1}{R} + \frac{1}{(MKR)^{2/3}}\right)$}\\ 
			(Stochastic)&\\
			\midrule
          \textsc{BVR-L-SGD}*~\citep{murata2021bias}, \textsc{CE-LSGD} (Theorem \ref{thm:alg1})
			& \multirow{2}{*}{$\frac{\Delta\tau}{R} + \textcolor{brown}{\frac{\Delta H}{\sqrt{K}R}} + \frac{\sigma^2}{MKR} + \left(\frac{\sigma\Delta H}{MKR}\right)^{2/3}$} \\ 
		   (Stochastic) &\\
		    \midrule
            \textsc{CE-LGD} (Theorem \ref{thm:alg1})		   
			& \multirow{2}{*}{$\frac{\Delta\tau}{R} + \textcolor{brown}{\frac{\Delta H}{KR}}$} \\ 
		   (Exact)&\\
		    \midrule
            Lower Bound
			& \multirow{2}{*}{$\min\left\{\frac{\Delta\tau}{R}, \frac{H^2\zeta^2}{R}\right\} + \textcolor{brown}{\frac{\Delta L}{KR}} + \frac{\sigma^2}{MKR} + \left(\frac{\sigma\Delta H}{MKR}\right)^{2/3}$}\\ 
			(Theorem \ref{thm:lb})&\\
			\bottomrule
			\bottomrule
		\end{tabular}

    \caption{Comparison of convergence rates for various algorithms in the intermittent communication setting (cf. \Cref{fig:IC}). The quantities $\zeta$ and $\tau$ refer to the heterogeneity assumptions (\Cref{ass:zeta,ass:tau}); note that $\tau \leq 2H$ and can be significantly smaller than $H$ in practice. *See Section \ref{subsec:other_pers} for a detailed comparison with BVR-L-SGD. $^\dagger$The variance term is optimal in these rates, as the corresponding analyses do not rely on the mean-squared smoothness assumption (cf.~\Cref{def:oracle_first_multi_point}).}
\label{table:comparision} 
\end{table*}

We will also assume access to a more powerful stochastic gradient oracle in the non-convex setting, which is necessary for implementing variance-reduced algorithms.

\begin{definition}[Stochastic Multi-point First-order Oracle]\label{def:oracle_first_multi_point}
For each machine $m \in [M]$, we assume access to an oracle $\mathcal{O}_m : (\mathbb{R}^d)^{\otimes n} \times \Delta(\mathcal{Z}) \to (\mathbb{R}^d)^{\otimes n}$, such that for any $x_1, \dots, x_n \in \mathbb{R}^d$, the oracle samples a \emph{random datum} $z \sim \mathcal{D}_m$ and returns
\[
\left(\{s_z(x_i)\}_{i \in [n]}, \{g_z(x_i)\}_{i \in [n]}\right),
\]
satisfying the following properties for all $i \in [n]$:
\begin{align*}
    \mathbb{E}[s_z(x_i)\,|\,x_i] &= F_m(x_i)\enspace, \\
    \mathbb{E}[g_z(x_i)\,|\,x_i] &= \nabla F_m(x_i)\enspace, \\
    \mathbb{E}\left[\|g_z(x_i) - \nabla F_m(x_i)\|_2^2\,|\,x_i\right] &\leq \sigma_2^2\enspace.
\end{align*}
Moreover, the gradients satisfy $H$-\emph{mean smoothness}, i.e., for all $x, y \in \mathbb{R}^d$,
\[
\mathbb{E}_{z \sim \mathcal{D}_m}\left[\|g(x;z) - g(y;z)\|_2|x,y\right] \leq H\|x - y\|_2\enspace.
\]
\end{definition}

\begin{remark}[Smoothness v/s Mean-smoothness]
The mean-smoothness property is essential for achieving an oracle complexity of $\mathcal{O}(1/\epsilon^{3/2})$ in the serial setting ($M = 1$) for finding an $\epsilon$-stationary point~\citep{arjevani2019lower}, as opposed to the $\mathcal{O}(1/\epsilon^2)$ complexity of standard SGD. While it is common to distinguish the oracle's $\bar{H}$-mean-smoothness from the objective's $H$-smoothness~\citep{arjevani2019lower}, we do not make this distinction here.

For example, consider the square loss function discussed in \Cref{subsec:square_loss}. For any $z = (a, b) \in \supp{\mathcal{D}_m}$ with $a \in \mathbb{R}^d$ and $b \in \mathbb{R}$, we have
\(
\nabla_x f^{\text{square}}(x; (a,b)) = a\left(\inner{a}{x} - b\right)
\).
Suppose the oracle returns $g(x; z) = \nabla_x f^{\text{square}}(x; (a,b))$. Then for all $x, y \in \mathbb{R}^d$,
\begin{align*}
    \mathbb{E}_{(a,b) \sim \mathcal{D}_m} \left[\left\|a\left(\inner{a}{x} - b\right) - a\left(\inner{a}{y} - b\right)\right\|_2|x,y\right] 
    &= \mathbb{E}_{(a,b)} \left[\left\|aa^\top (x - y)\right\|_2|x,y\right]\enspace,\\
    &\leq \mathbb{E}_{(a,b)\sim \ddd_m}\left[\|aa^\top\|_2\right] \cdot \|x - y\|_2\enspace.
\end{align*}
This shows that the mean-smoothness constant is given by $\mathbb{E}_{(a,b) \sim \mathcal{D}_m}[\|aa^\top\|_2]$, while the smoothness constant of the objective $F_m$ would be $\|\mathbb{E}_{(a,b)}[aa^\top]\|_2$. These two quantities can differ significantly; in particular, the smoothness of $F_m$ can be much larger than the mean-smoothness of the oracle.
\end{remark}

\begin{remark}
\citet{arjevani2019lower} showed that if a first-order oracle satisfies only the bounded variance condition---without the stronger mean-squared smoothness property---then any algorithm requires at least $\Omega(1/\epsilon^2)$ oracle queries to find an $\epsilon$-stationary point. This lower bound explains the suboptimal oracle complexity of distributed algorithms such as Local SGD, SCAFFOLD~\citep{karimireddy2020scaffold}, and mini-batch SGD, which are typically analyzed under this weaker oracle model (cf.~\Cref{table:comparision}).
\end{remark}

With this definition in hand, we will first state the following lower bound for all centralized zero-respecting algorithms (cf. \Cref{def:zero_respecting}).

\begin{theorem}[Centralized Lower Bound]\label{thm:lb_cent}
For any problem instance satisfying \Cref{ass:smooth_second,ass:bounded_func_subopt,ass:tau,ass:zeta} every algorithm $A\in \aaa^{cent}_{ZR}$ equipped with a two-point stochastic oracle on all machines (cf. \Cref{def:oracle_first_multi_point}) must output $x^{A}_R$ such that (for a numerical constant $c_{10}$), $$\ee\left[\norm{\nabla F(x^A_R)}^2\right] \geq c_{10}\cdot\rb{\frac{\Delta H}{R} + \frac{\sigma_2^2}{MKR} + \left(\frac{\sigma_2 \Delta H}{MKR}\right)^{2/3}}\enspace.$$  
\end{theorem}
The proof of this theorem follows the known oracle complexity lower bounds \citep{carmon2020lower, arjevani2019lower}, and we include it in \Cref{app:chap6}. This theorem shows that, mini-batch SARAH/STORM, which are centralized algorithms, already achieve the optimal communication and oracle complexity (see Table \ref{table:comparision}) for algorithms in $\aaa_{ZR}^{cent}$ optimizing smooth non-convex problems. Note that the lower bound result holds for all $\tau,\ \zeta$, which highlights the limitation of the centralized baselines, showing they \textbf{can not} improve with lower heterogeneity (also see \Cref{rem:mb_rate_is_tight}). Specific existing \textit{local-update} algorithms, such as \textsc{MimeMVR} \citep{karimireddy2020mime} and \textsc{BVR-L-SGD} \citep{murata2021bias}, can indeed improve upon centralized algorithms in the low-heterogeneity regime. In the next section, we will quantify this improvement and demonstrate that our algorithm strictly outperforms the centralized baselines and nearly matches our lower bound for algorithms in $\aaa_{ZR}$.

\section{Our New Local Update Algorithm}\label{sec:ch6.2}

\begin{algorithm}[!hpt]
	\caption{Communication Efficient Local Stochastic Gradient Descent (CE-LSGD)}\label{alg:fed_vr}
	\begin{algorithmic}[1]
		\INPUT Initialization $x_0$, iteration number $R$, step size $\eta$, parameters $b_0$, $b$, $P$ and $\beta\in[0,1]$
        \STATE Let $x_{-1}=x_0$
		\FOR{$r=0,1,\ldots, R-1$} 
		\STATE \textbf{if} $r=0$ set $\rho=1$, $Q=1$, $B=b_0$ \textbf{else} set $\rho=\beta$, $Q=P$, $B=Q$
        \STATE \textbf{Communicate (send)} $(x_{r},x_{r-1})$ to clients
		\FORC{$m \in [M]$} 
		\STATE Sample $\bbb_{r}^m\sim \mathcal{D}_m^{\otimes B}$, get $\nabla F_{m,\bbb_{{r}}^m}(x_{r})$, $\nabla F_{m,\bbb_{{r}}^m}(x_{r-1})$, where $|\bbb_r^m|=B$
		\STATE \textbf{Communicate (rec)} $\big(\nabla F_{m,\bbb_{{r}}^m}(x_{r}),\nabla F_{m,\bbb_{{r}}^m}(x_{r-1})\big)$ to the server
		\ENDFORC
		\STATE $v_{r}=\frac{1}{M}\sum_{m=1}^M\nabla F_{m,\bbb_{{r}}^m}(x_{r})+(1-\rho)\left(v_{r-1}-\frac{1}{M}\sum_{m=1}^M\nabla F_{m,\bbb_{{r}}^m}(x_{r-1})\right)$
		\STATE \textbf{Communicate (send)} $(x_r,v_{r})$ to client $\tilde m_r$, where $\tilde m_r \sim Unif\left([M]\right)$
		\FORC{$\tilde m_r$} 
		\STATE $w^{\tilde m_r}_{r+1,1}:=w^{\tilde m_r}_{r+1,0}:=x_{r},v_{r,0}^{\tilde m_r}:=v_{r}$ 
		\FOR{$k=1,\ldots, Q$}
		    \STATE Sample $\bbb_{r,k}^{\tilde m}\sim \mathcal{D}_{\tilde m}^{\otimes b}$, get $\nabla F_{\tilde m,\bbb_{{r,k}}^{\tilde m}}(w^{\tilde m_r}_{r+1,k})$, $\nabla F_{\tilde m,\bbb_{{r,k}}^{\tilde m}}(w^{\tilde m_r}_{r+1,k-1})$, where $|\bbb_{r,k}^{\tilde m}|=b$
        	\STATE $v_{r,k}^{\tilde m_r}=v_{r,k-1}^{\tilde m_r}+\nabla F_{\tilde m,\bbb_{{r,k}}^{\tilde m}}(w^{\tilde m_r}_{r+1,k})-\nabla F_{\tilde m,\bbb_{{r,k}}^{\tilde m}}(w^{\tilde m_r}_{r+1,k-1})$
		\STATE $w^{\tilde m_r}_{r+1,k+1}=w^{\tilde m_r}_{r+1,k}-\eta v_{r,k}^{\tilde m_r}$
		\ENDFOR
		\STATE \textbf{Communicate (rec)} $\big(w^{\tilde m_r}_{r+1,Q+1}\big)$ to the server
		\ENDFORC
		\STATE  Let $x_{r+1}=w^{\tilde m_r}_{r+1,Q+1}$
		\ENDFOR
        \OUTPUT  Choose $\tilde{x}$ uniformly from $\{w^{\tilde m_r}_{r,k}\}_{r\in[R],k\in[Q]}$
	\end{algorithmic}
\end{algorithm}

In this section, we introduce our communication-efficient algorithm, denoted \textsc{CE-LSGD}, and describe it in Algorithm~\ref{alg:fed_vr}. For each machine $m \in [M]$, we define the mini-batch stochastic gradient as
\[
\nabla F_{m,\mathcal{B}^m}(x) := \frac{1}{|\mathcal{B}^m|} \sum_{l \in \mathcal{B}^m} g(x; z_l \sim \mathcal{D}_m),
\]
where $\mathcal{B}^m$ denotes a mini-batch of size $|\mathcal{B}^m|$ obtained by querying the oracle $\mathcal{O}_m$.

At each iteration of Algorithm~\ref{alg:fed_vr}, the algorithm performs \textbf{two rounds} of communication—i.e., two back-and-forth exchanges between the server and all clients. The additional communication round, captured in lines 4 to 9, is used to update the variance-reduced gradient $v_r$ using the current and previous server models, $x_r$ and $x_{r-1}$, respectively. In the rest of this section, we use the iteration index $R$ and the communication complexity of Algorithm~\ref{alg:fed_vr} interchangeably.

To implement Algorithm~\ref{alg:fed_vr} in the intermittent communication (IC) setting with $K$ local steps between two communication rounds, we choose the input parameters as $P = K$ and $b = 1$ (see line 14 of Algorithm~\ref{alg:fed_vr}). We assume this setting throughout the section. As discussed previously, mini-batch algorithms such as mini-batch \textsc{STORM} can also operate in the IC setting by making $K$ oracle calls at the same point in each communication round. Notably, our method reduces to mini-batch \textsc{STORM} when the number of local updates is set to $Q = 1$ (see Appendix~\ref{sec:minibatch_storm}).

The core of our proposed method lies in the construction of the variance-reduced gradient $v_r$ and the local gradient estimator $v_{r,k}^m$ (lines 9 and 15 of Algorithm~\ref{alg:fed_vr}). This construction is inspired by the variance reduction techniques used in SARAH~\citep{nguyen2017sarah} and SPIDER~\citep{fang2018spider}. Intuitively, the estimation error between $v_{r,k}^m$ and the true gradient $\nabla F(w_{r+1,k}^m)$ can be decomposed into two dominant terms:
\begin{itemize}
    \item \colorbox{pink}{$\mathbb{E}[\|v_r - \nabla F(x_r)\|^2]$}, the error due to stale information in the global variance-reduced gradient;
    \item \colorbox{yellow}{$\tau^2 K \sum_{k=1}^K \mathbb{E}[\|w_{r+1,k}^m - w_{r+1,k-1}^m\|^2]$}, which quantifies the accumulated local drift due to data heterogeneity.
\end{itemize}
The first term is controlled by momentum-based variance reduction~\citep{cutkosky2019momentum}, and is dominated by a term that vanishes as the iterates converge: \colorbox{pink}{$H^2 \mathbb{E}[\|x_r - x_{r-1}\|^2]$}. The second term also vanishes during convergence and scales with $\tau^2$, indicating that lower heterogeneity enables more aggressive local updates and faster convergence.

\vspace{1em}
We now state the convergence guarantees of \textsc{CE-LSGD} in the intermittent communication setting.

\begin{theorem}[Convergence of \textsc{CE-LSGD}]\label{thm:alg1}
Suppose the problem instance satisfies \Cref{ass:smooth_second,ass:bounded_func_subopt,ass:tau,ass:zeta}. Then:
\begin{itemize}[leftmargin=1em]
    \item[(a)] If each client $m \in [M]$ has access to a stochastic two-point oracle (cf.~\Cref{def:oracle_first_multi_point}) and $\frac{\Delta H}{R} = \mathcal{O} \left( \frac{\sigma_2^2}{\sqrt{MK}} \right)$, then Algorithm~\ref{alg:fed_vr}, with
    \begin{align*}
        \beta = \max\left\{\frac{1}{R}, \frac{(\Delta H)^{2/3}(MK)^{1/3}}{\sigma_2^{4/3} R^{2/3}}\right\}\enspace,\quad 
        b_0 = K R\enspace,\quad \text{and}\quad 
        \eta = \min\left\{ \frac{1}{H}, \frac{1}{K\tau}, \frac{(\beta M)^{1/2}}{H K^{1/2}} \right\}\enspace,
    \end{align*}
    outputs $\tilde{x}$ satisfying
    \[
    \mathbb{E}\left[\|\nabla F(\tilde{x})\|^2\right] \leq c_{11} \cdot \left( \frac{\Delta \tau}{R} + \frac{\Delta H}{\sqrt{K} R} + \frac{\sigma_2^2}{MKR} + \left( \frac{\sigma_2 \Delta H}{MKR} \right)^{2/3} \right)\enspace.
    \]
    
    \item[(b)] If each client $m \in [M]$ has a deterministic two-point oracle, then using $\beta = 1$ and $\eta = \min\left\{ \frac{1}{H}, \frac{1}{K\tau} \right\}$, Algorithm~\ref{alg:fed_vr} satisfies
    \[
    \mathbb{E}\left[\|\nabla F(\tilde{x})\|^2\right] \leq c_{12} \cdot \left( \frac{\Delta \tau}{R} + \frac{\Delta H}{K R} \right)\enspace,
    \]
\end{itemize}
where $c_{11}, c_{12}$ are numerical constants.
\end{theorem}

The proof, presented in Appendix~\ref{sec:app_proof_alg1}, follows from a careful tuning of the parameters $\beta$ and $b_0$ while controlling the two dominant terms mentioned above. To demonstrate that our convergence rate is nearly optimal, we establish the following lower bound (proved in \Cref{sec:non_convex_lb}):

\begin{theorem}[Lower Bound]\label{thm:lb}
Let the problem instance satisfy \Cref{ass:smooth_second,ass:bounded_func_subopt,ass:tau,ass:zeta}. Then any algorithm $A \in \mathcal{A}_{\text{zr}}$, using two-point first-order oracles on all machines (cf.~\Cref{def:oracle_first_multi_point}), outputs $x^A_R$ satisfying
\[
\mathbb{E}\left[\|\nabla F(x^A_R)\|^2\right] \geq c_{13} \cdot \left( \min\left\{ \frac{H^2 \zeta^2}{R}, \frac{\Delta \tau}{R} \right\} + \frac{\Delta H}{K R} + \frac{\sigma_2^2}{MKR} + \left( \frac{\sigma_2 \Delta H}{MKR} \right)^{2/3} \right),
\]
for some universal constant $c_{13}$.
\end{theorem}

\begin{remark}
By comparing the upper and lower bounds under the \Cref{ass:smooth_second,ass:stoch_bounded_second_moment,ass:bounded_func_subopt,ass:tau,ass:zeta}, we make two key observations:
\begin{enumerate}
    \item In the \textbf{deterministic setting} ($\sigma_2 = 0$), our upper bound matches the lower bound exactly, implying that CE-LSGD is \textbf{minimax optimal}. This improves upon all existing methods in this setting.
    
    \item In the \textbf{stochastic setting} ($\sigma_2 > 0$), the convergence bound of CE-LSGD is optimal up to the second term, where our upper bound includes a $\Delta H / (\sqrt{K} R)$ term, while the lower bound achieves $\Delta H / (K R)$. We discuss this discrepancy in more detail in Section~\ref{sec:stochastic_gap}.
\end{enumerate}
\end{remark}

Our construction for Theorem~\ref{thm:lb} builds on the non-convex hard instance proposed by~\citet{carmon2020lower} for serial non-convex optimization, by partitioning the instance across machines carefully. This construction technique has previously been employed to give lower bounds in the heterogeneous setting~\citep{arjevani2015communication, woodworth2020minibatch, zhang2020fedpd}. 

While \textsc{BVR-L-SGD}~\citep{murata2021bias} achieves a similar upper bound to our method (see Table~\ref{table:comparision}), this is expected, as several variance-reduced algorithms—e.g., SPIDER~\citep{fang2018spider}, SARAH~\citep{nguyen2017sarah}, and momentum-based variants~\citep{cutkosky2019momentum}—are simultaneously optimal in the sequential setting. Nevertheless, our algorithm requires fewer and lighter-weight variance reduction steps, making it more scalable in distributed settings. In the next section, we further compare these methods in detail.

\subsection{The Perspective of Reducing Communication}\label{subsec:other_pers}
\begin{figure*}[!thp]
\centering
    \includegraphics[width=0.8\textwidth]{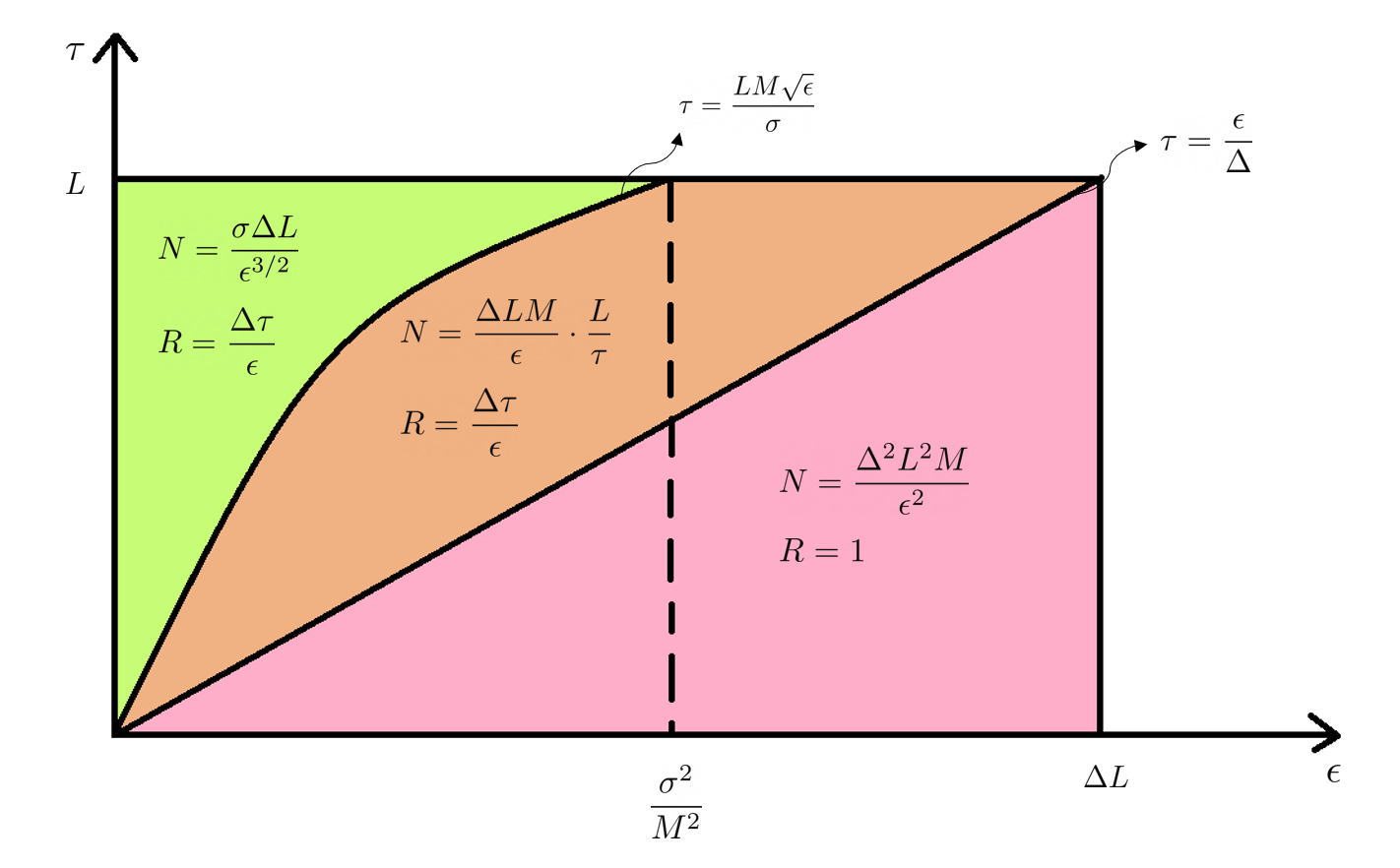}
        \caption{
      Illustration of the best communication complexity $R$ and oracle complexity $N$ that our method can obtain for different $\epsilon$ and $\tau$. Green regime: Our method can obtain optimal communication and oracle complexities. Orange regime: Our method achieves optimal communication with a larger oracle complexity. Red regime: Our method requires only one round of communication, achieving a larger oracle complexity. $H$ and $\tau$ are the smoothness and second-order heterogeneity parameters, respectively. \textbf{TODO: change smoothness constant}
    } \label{fig:3region}
\end{figure*}

Thus far, we have analyzed convergence rates under the intermittent communication (IC) model, assuming fixed values of $K$ and $R$. An alternative and often more practical perspective is to minimize the overall communication complexity required to reach an $\epsilon$-approximate stationary point, while still achieving the optimal oracle complexity. Using our convergence guarantees, both communication and oracle complexities----denoted by $R$ and $N$ respectively---can be expressed as functions of $\epsilon$, which facilitates this analysis.

This perspective is particularly relevant in federated learning (FL), where communication often dominates the wall-clock time due to device heterogeneity and intermittent availability, which slows down synchronous updates. Motivated by this, we summarize the communication and oracle complexities achieved by our method (\textsc{CE-LSGD}) and by \textsc{BVR-L-SGD}~\citep{murata2021bias} in Figure~\ref{fig:3region}, focusing on optimization with stochastic oracles.

The figure illustrates three regimes based on the scaling of the heterogeneity parameter $\tau$ relative to $\epsilon$. Our primary focus is on the green regime, defined by the condition $\epsilon^{1/2} \in (0, \tau\sigma_2 / (H M)]$, which is particularly relevant in deep learning. In such settings, overparameterization often leads to very small target accuracies $\epsilon$, making this regime practically significant across a wide range of values of $\tau$.\footnotemark

\footnotetext{We discuss the remaining regimes in the proof of~\Cref{thm:alg1_full}, presented in~\Cref{sec:app_proof_alg1}.}

In the green regime, both \textsc{CE-LSGD} and \textsc{BVR-L-SGD} require $K = \sigma_2 H / (\tau M \epsilon^{1/2})$ local steps to attain optimal communication and oracle complexities. However, the two methods differ significantly in how variance reduction is implemented:

\begin{itemize}
    \item \textsc{BVR-L-SGD} requires multiple heavy-batch stochastic gradient computations per machine during $S = H \Delta / (\sigma_2 \sqrt{\epsilon})$ communication rounds (cf.~\cite{murata2021bias} for details of their algorithm). Specifically, it uses batch size $b_{\max}$, with the ratio 
    \[
    \rho_{\text{BVR}} := \frac{b_{\max}}{K} = \frac{\sigma_2 \tau}{H \epsilon^{1/2}}
    \]
    quantifying the computational overhead compared to a standard mini-batch, or even compared to the lighter communication rounds within \textsc{BVR-L-SGD} itself.

    \item \textsc{CE-LSGD}, in contrast, requires only a single heavy-batch gradient computation per machine with batch size $b_0 = \sigma_2^3 / (H \Delta M \epsilon^{1/2})$, giving
    \[
    \rho_{\text{our}} := \frac{b_0}{K} = \frac{\sigma_2^2 \tau}{H^2 \Delta}\enspace.
    \]
    Crucially, this satisfies
    \[
    \frac{\rho_{\text{our}}}{\rho_{\text{BVR}}} = \frac{\sigma_2 \epsilon^{1/2}}{H \Delta} \leq 1\enspace,
    \]
    indicating that \textsc{CE-LSGD} not only incurs fewer heavy-batch operations, but each such operation is also lighter.
\end{itemize}

\begin{figure}[!thp]
\centering
    \includegraphics[width=0.8\textwidth]{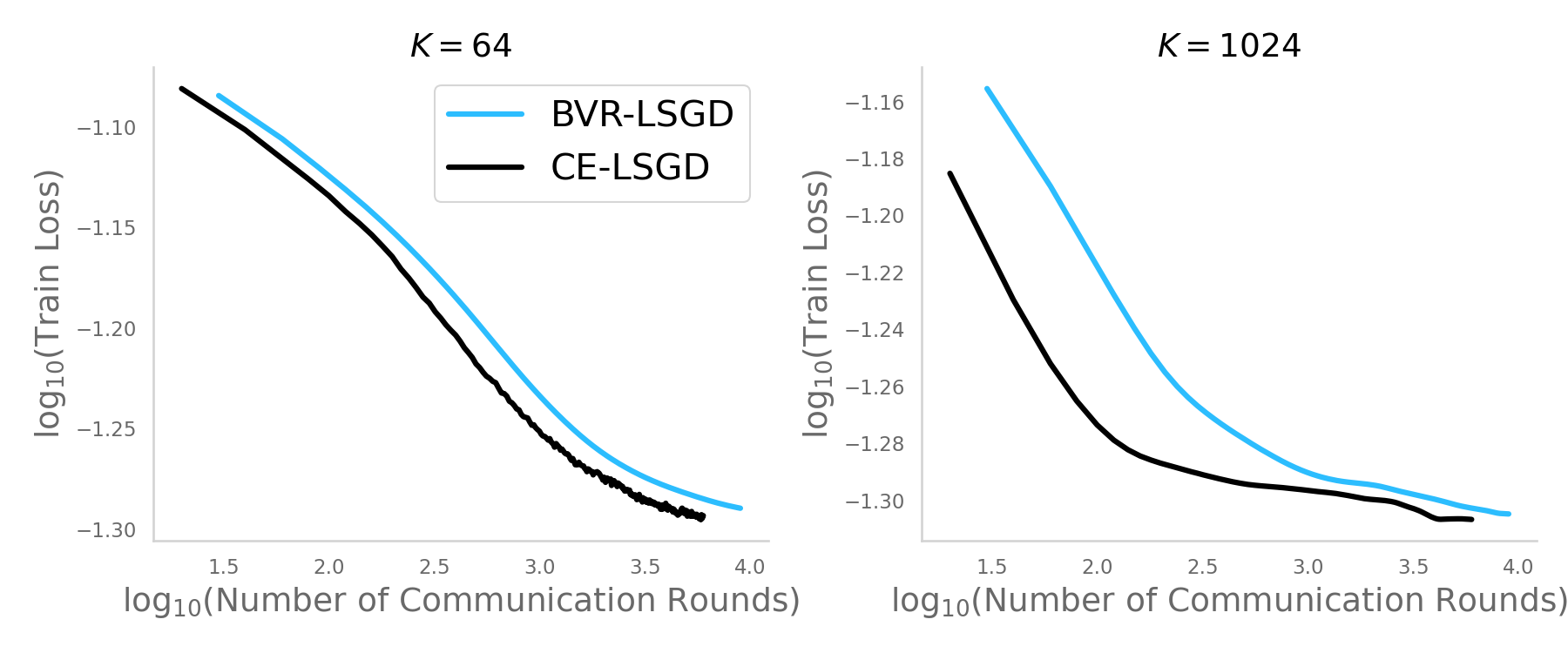}
  \caption{Training loss of CE-LSGD and BVR-L-SGD on CIFAR-10 data-set versus the number of communication rounds in the intermittent communication setting with different local-updates $K$. We use $M=10$ machines, and synthetically generate heterogeneous data-sets (see Section \ref{sec:exps}) with $q=0.1$. All oracle queries use a mini-batch of size $b=16$, i.e., each machine has $Kb$ oracle queries between two communication rounds. We note that our method exhibits faster convergence in all settings, highlighting its communication efficiency. Fixed step-sizes $\eta$ for both the methods were tuned in $\{0.001, 0.005, 0.01, 0.05, 0.1, 0.5\}$ (to obtain best loss) following \citep{murata2021bias}, our method set the momentum $\beta=0.3$, $b_{max}^{our}=K$, while $b_{max}^{BVR} = 5000$ according to \citep{murata2021bias}.}\label{fig:intermittent}
\end{figure}

If one were to implement both methods in the IC model by distributing the large-batch computation across multiple rounds while respecting the local budget $K = \sigma_2 H / (\tau M \epsilon^{1/2})$, then both algorithms would exhibit an effective communication complexity of $\mathcal{O}(\Delta \tau / \epsilon)$. In theory, this washes out the computational differences up to numerical constants.

However, as demonstrated in \Cref{fig:intermittent}, this asymptotic equivalence does not translate to practice: \textsc{CE-LSGD} consistently converges faster than \textsc{BVR-L-SGD}, owing to its fewer and lighter heavy-batch operations.

\subsection{The Gap in the Stochastic Setting}\label{sec:stochastic_gap} According to the results in Table \ref{table:comparision}, there is a gap between the convergence rates of \textsc{CE-LSGD} and \textsc{CE-LGD}, which doesn't go away when $\sigma_2=0$. In particular, the \textcolor{brown}{brown} term in \textsc{CE-LGD}'s upper bound, which doesn't depend on $\sigma_2$, matches the corresponding term in the lower bound, but the \textcolor{brown}{brown} term in \textsc{CE-LSGD}'s upper bound is worse by a factor of $1/\sqrt{K}$. This result comes from a more pessimistic choice of step size in the stochastic setting. 

To further elucidate this, consider a more general communication model. Recall that each machine makes $K$ queries in the IC setting between two communication rounds. We can instead consider the model where each machine is allowed to make $Kb$ queries, but at most $K$ different inputs. Centralized algorithms will make just $Kb$ queries at the same input. For instance, in this model, \textsc{MB-SGD} or \textsc{MB-STORM} will make $R$ updates with batch size $MKb$. However, local update algorithms can make $K$ \textit{``mini-batch''} style queries, i.e., make $b$ repeated queries at the current local iterate. This oracle model has been studied for hierarchical parallelism \citep{lin2018don}. For instance, let's say each machine has access to a GPU. Then, each local update should use the largest batch size $b=b_{max}$ that saturates the GPU's capacity (such as its memory) without requiring additional parallel run-time compared to $b=1$. Modern specialized hardware for deep learning (including FPGAs, TPUs, etc.) is designed with such parallelism, and $b_{max}$ is usually much larger than $1$ \citep{shallue2018measuring}. Thus, if energy usage (i.e., the number of oracle queries) is a non-concern and achieving an accurate solution as quickly as possible is the primary goal, then it is beneficial to consider this hierarchical setting. We can attain the following convergence guarantee for \textsc{CE-LSGD} in this setting. 

\begin{theorem}\label{thm:alg1_full_main}
Suppose we have a problem instance satisfying \Cref{ass:smooth_second,ass:bounded_func_subopt,ass:tau,ass:zeta}, each client $m\in[M]$ has a stochastic two-point oracle (cf. \Cref{def:oracle_first_multi_point}) which it uses through $b$-calls for every single query, and assume that $\frac{\Delta H}{R} \leq \frac{\sigma_2^2}{\sqrt{MKb}}$. Then the output $\tilde x$ of Algorithm \ref{alg:fed_vr} using $\beta=\max\left\{\frac{1}{R}, \frac{(\Delta H)^{2/3}(MKb)^{1/3}}{\sigma_2^{4/3}R^{2/3}}\right\}$, $b_0=KbR$ and $\eta=\min\left\{\frac{1}{H}, \frac{1}{K\tau},\frac{\sqrt{b}}{\sqrt{K}H}, \frac{(\beta MKb)^{1/2}}{HK}\right\}$, satisfies the following
$$\ee\|\nabla F(\tilde x)\|^2\leq c_{14}\cdot\rb{ \frac{\Delta \tau}{R} +\frac{\Delta H}{KR}+\frac{\Delta H}{R\sqrt{Kb}}+\left(\frac{\sigma_2\Delta  H}{MKbR}\right)^{2/3}+\frac{\sigma_2^2}{MKbR}}\enspace.$$
\end{theorem}

When $b=1$, this reduces to Theorem \ref{thm:alg1} since the third term in the upper bound always dominates the second term. In the exact setting as we show in Appendix \ref{sec:app_proof_alg1}, the last three terms go away altogether. Using arguments similar to the ones given in Appendix \ref{sec:non_convex_lb} (to prove Theorem \ref{thm:lb}), we can show that every term except the third term is tight in Theorem \ref{thm:alg1_full_main}. We currently don't know how to eliminate the loose third term, but as is apparent from the theorem, setting $b=K$ suffices to recover the min-max optimal guarantee, even in the stochastic setting.

\section{Empirical Study}\label{sec:exps}

We evaluate the performance of our method by optimizing a two-layer fully connected network for multi-class classification on the CIFAR-10 \citep{krizhevsky2009learning} dataset. Since we are in a heterogeneous setting, we need to generate a dataset artificially. We follow the same data processing procedure as in \citet{murata2021bias}. We first make sure that all the ten classes in CIFAR-10 have the same number of samples (roughly around 5000), and assign $q\times 100\%$ of class $m$'s samples to client $m\in[10]$, where $q$ is chosen from $\{0.1,0.35,0.6,0.85\}$. For each class $m$, we evenly split the remaining $(1-q)\times 100\%$ samples to the other nine clients except client $m$. Thus, $q$ controls the heterogeneity of our dataset, with small $q$ corresponding to small heterogeneity.

\begin{figure*}[!ht]%
  \begin{center}
    \includegraphics[width=0.6\textwidth]{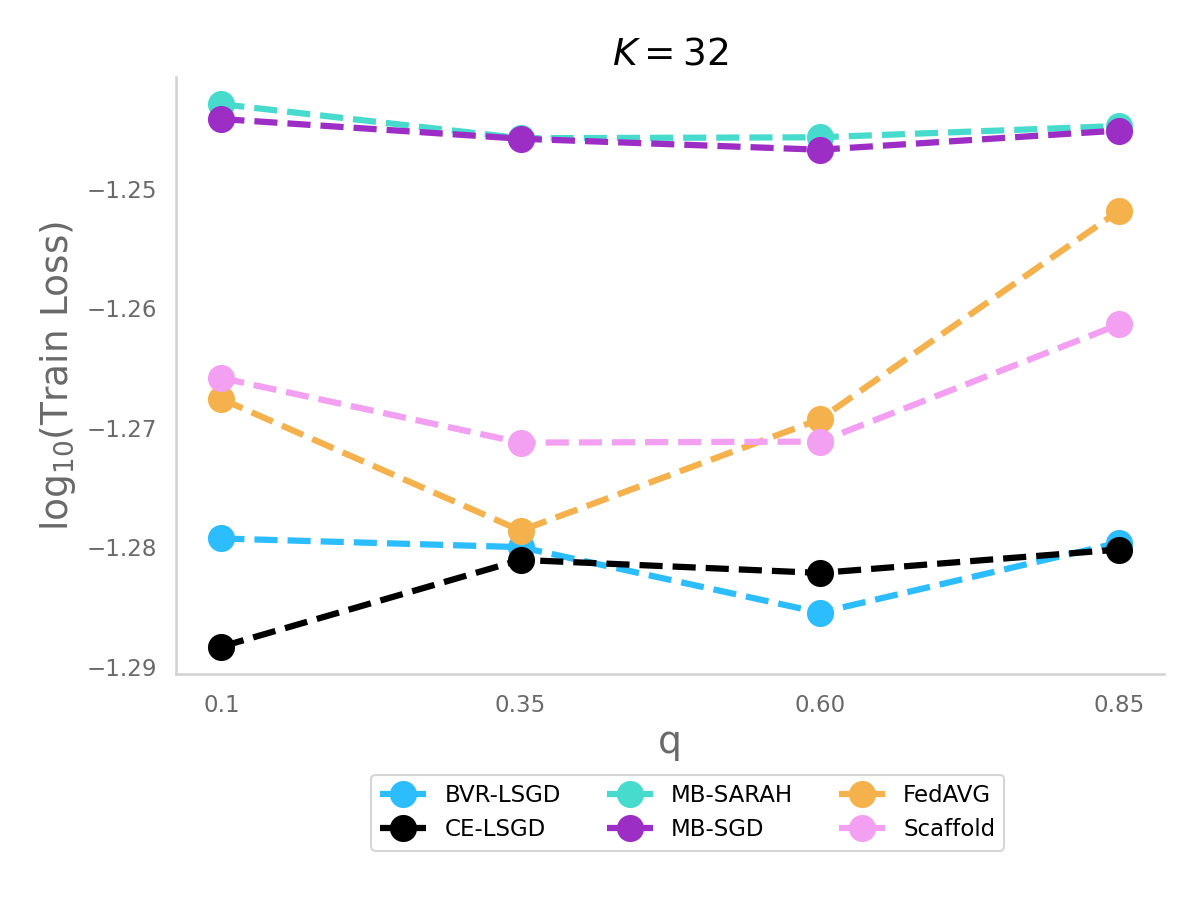}
  \end{center}
  \caption{Comparing CE-LSGD to centralized and local-update methods, for fixed $K=32$ and varying heterogeneity controlled by $q$ on CIFAR-10 data-set. Like Figure \ref{fig:intermittent}, we use mini-batch size $b=16$ for each oracle query. Thus, each method makes $Kb$ oracle queries every round per machine. All the methods for different $q$ are tuned separately, following a similar hyperparameter search as in Figure \ref{fig:intermittent}.}
\end{figure*}

We perform two different experiments. In the first experiment, we directly compare our method, i.e., \textsc{CE-LSGD}, with \textsc{BVR-L-SGD} in the intermittent communication setting (see Figure \ref{fig:intermittent}). We observe that while both methods converge to a similar quality of solution eventually, our method is more communication-efficient. In the second experiment, we compare our method with BVR-L-SGD \citep{murata2021bias} as well as \textsc{FedAvg} \citep{mcmahan2016communication}, \textsc{SCAFFOLD} \citep{karimireddy2020scaffold}, \textsc{MB-SARAH} \citep{nguyen2017sarah} and \textsc{MB-SGD} \citep{dekel2012optimal} for the same number of updates/iterations. The last two methods are centralized baselines, and we use the local computation to compute a mini-batch stochastic gradient. We again observe that \textsc{CE-LSGD} and \textsc{BVR-L-SGD} have comparable performance, which is superior to that of all the other methods. 

This concludes our discussion for the non-convex setting. The key takeaway from the non-convex setting is that second-order heterogeneity can help us characterize the communication complexity of optimization, much like the convex case in \Cref{ch:fixed_point,ch:upper_bounds}.

\chapter{Distributed Online and Bandit Convex Optimization}\label{ch:online}

So far in this thesis, we have assumed that the data distribution on each machine remains fixed over time: for each machine \( m \in [M] \), the distribution \( \mathcal{D}_m \)  does not vary across different queries to the oracle (as defined in \Cref{def:oracle_first}). This fixed-distribution setup, formalized in Problem~\eqref{prob:scalar}, underpins much of the theory in distributed stochastic optimization. However, many real-world applications---such as mobile keyboard prediction~\citep{hard2018federated, chen2019federated, google_ai_blog_2021}, autonomous driving~\citep{elbir2020federated, nguyen2022deep}, voice assistants~\citep{hao2020apple, googleVoice}, and recommendation systems~\citep{shi2021federated, liang2021fedrec++, khan2021payload}---involve \emph{sequential} decision-making in dynamic environments. In these settings, data is generated in real-time and often cannot be stored due to memory or privacy constraints. Moreover, these services must improve continuously while deployed, which requires that all models remain reasonably accurate at all times.

To address these challenges, this chapter develops a systematic theory of \emph{federated online optimization}. We identify settings where collaboration among clients provably helps and characterize when such benefits do not arise. Our contributions are as follows:
\begin{enumerate}
    \item In Theorems~\ref{thm:first_lip} and~\ref{thm:first_smth}, we show that collaboration offers no worse-case benefit when each machine has access to full gradients (first-order feedback) and simply running online gradient descent independently on each client is already min-max optimal.

    \item Motivated by the limitations of first-order feedback, we study the federated adversarial linear bandits problem. We propose a new one-point feedback algorithm, \textsc{FedPOSGD} (Algorithm~\ref{alg:fed_pogd}), and prove that collaboration improves regret in high-dimensional settings (Theorem~\ref{thm:favlb}), outperforming non-collaborative baselines.

    \item Next, we extend our analysis to general convex cost functions using two-point bandit feedback. We analyze a natural online variant of Local SGD, \textsc{FedOSGD} (Algorithm~\ref{alg:fed_ogd}), and prove that collaboration reduces stochastic gradient variance, enabling tighter regret bounds (Theorems~\ref{thm:bd_grad_first_stoch} and~\ref{thm:smooth_first_stoch}). We also demonstrate that two-point feedback strictly improves regret, even for linear objectives (Corollary~\ref{coro:linear}), showing that multi-point feedback can outperform one-point methods in federated adversarial bandits.
\end{enumerate}

\subsection*{Outline and Relevant References}

Although many recent attempts \citep{Wang2020Distributed,dubey2020differentially,huang2021federated,li2022asynchronous,he2022simple,gauthier2022resource, gogineni2022communication, dai2022addressing, kuh2021real, mitra2021online} have been made towards tackling online optimization for FL, most existing theoretical works \citep{Wang2020Distributed,dubey2020differentially,huang2021federated,li2022asynchronous} study \textit{``stochastic"} adversaries. The results in this chapter are the first of their kind, as they tackle fully adaptive adversaries and are based on our work \cite{patel2023federated}, co-authored with Lingxiao Wang, Aadirupa Saha, and Nathan Srebro. 

In \Cref{sec:ch7.1,sec:ch7.2} we describe our regret minimization problem, which is a direct extension of the classic online convex optimization problem~\cite{hazan2016introduction} to the distributed setting. Other theoretical works that have considered stochastic adversaries~\citep{Wang2020Distributed,dubey2020differentially,huang2021federated,li2022asynchronous} have examined similar regret notions, differing primarily in the distinction between regret and pseudo regret. \Cref{sec:first_does_not_help} discusses new lower bounds for our problem under first-order feedback based on careful reductions to folklore lower bounds and a recent lower bound due to \citet{woodworth2021even} for stochastic optimization.

\Cref{sec:falb} provides our algorithm with a single zeroth-order feedback, which takes inspiration from the classical zeroth-order algorithm of~\citet{flaxman2004online} as well as lazy mirror descent methods~\citep{nesterov2009primal,bubeck2015convex,yuan2021federated}. Finally, \Cref{sec:2falb} discusses our other new algorithm, which utilizes two-point bandit feedback, along with improved regret guarantees. Two-point feedback is well-studied in the single-agent setting, and our algorithm and analyses are indeed inspired by existing results, due to \citet{duchi2015optimal,shamir2017optimal}.

\section{Distributed Regret Minimization}\label{sec:ch7.1}
The challenges of an online environment call for new methods that can enable collaboration across machines while being robust to changing data distributions—i.e., distribution shift. We formalize this setting, in line with the rest of the thesis, through the following distributed regret minimization problem over \( M \) machines and a time horizon of length \( T \):
\begin{align}\label{eq:online_problem}
\frac{1}{MT}\sum_{m \in [M],\ t \in [T]} f_t^m(x_t^m) - \min_{\|x^\star\|_2 \leq B} \frac{1}{MT}\sum_{m \in [M],\ t \in [T]} f_t^m(x^\star)\enspace,
\end{align}
where \( f_t^m \) is a convex cost function revealed to machine \( m \) at time \( t \), \( x_t^m \) is the model selected by that machine based on available history, and the comparator \( x^\star \), shared across all machines and time steps, satisfies \( \|x^\star\|_2 \leq B \) (cf. \Cref{ass:bounded_optima}). We study this problem under the \emph{intermittent communication} setting (see \Cref{fig:IC}), where machines play fresh models at every time step but are allowed to communicate only \( R \) times over the \( T \) steps, with \( K = T/R \) steps between communication rounds. This formulation generalizes the classical federated optimization setup~\citep{mcmahan2016communication, kairouz2019advances} in Problem~\eqref{prob:scalar}, introducing new challenges arising from sequential decision-making and potentially adversarial cost sequences. Unlike standard federated learning, which aims to learn a single high-quality consensus model, the objective here is to generate a sequence of models that perform well at every round.

While many recent works~\citep{Wang2020Distributed,dubey2020differentially,huang2021federated,li2022asynchronous,he2022simple,gauthier2022resource,gogineni2022communication,dai2022addressing,kuh2021real,mitra2021online} have tackled Problem~\eqref{eq:online_problem}, most focus on the \emph{``stochastic online''} setting, where the functions \( \{f_t^m\} \) are sampled from distributions fixed at time $t=0$. This assumption fails to capture various real-world scenarios involving unmodeled perturbations, distribution shifts, or even adaptively chosen cost sequences. The stochastic online setup is not far from the static-distribution formulation of Problem~\eqref{prob:scalar}, where machines interact with first-order oracles as defined in \Cref{def:oracle_first}.

Although several recent works~\citep{gauthier2022resource, gogineni2022communication, dai2022addressing, kuh2021real, mitra2021online, he2022simple} have underscored the significance of adaptive settings, our theoretical understanding of regret guarantees for Problem~\eqref{eq:online_problem}—particularly under intermittent communication—remains limited. The objective of this chapter is to advance this understanding by studying distributed online and bandit convex optimization against \emph{adaptive} adversaries capable of generating worst-case sequences of cost functions.

In the next section, we begin by generalizing some of the core assumptions from \Cref{ch:setting} to the online setting.

\section{Our Setting and Some Baselines}\label{sec:ch7.2}

We denote the average cost function at any time step \( t \in [T] \) by \( f_t(\cdot) := \frac{1}{M} \sum_{m \in [M]} f_t^m(\cdot) \). We use \( \mathbb{I}_A \) to denote the indicator function for an event \( A \), and \( \mathbb{B}_2(B) \subset \mathbb{R}^d \) to denote the \( \ell_2 \)-ball of radius \( B \) centered at the origin.

\subsection{Regularity Conditions}

As the name of this chapter suggests, we focus on convex cost functions.

\begin{assumption}\label{ass:online_convex}
    For all \( t \in [T] \) and \( m \in [M] \), the function \( f_t^m: \mathbb{R}^d \to \mathbb{R} \) is differentiable and convex.
\end{assumption}

We also consider two types of Lipschitz conditions.

\begin{assumption}[Bounded Gradients]\label{ass:online_bounded_gradients}
    For all \( t \in [T] \) and \( m \in [M] \), the function \( f_t^m \) is \( G \)-Lipschitz, i.e.,
    \begin{align*}
        |f_t^m(x) - f_t^m(y)| \leq G \|x - y\|_2\enspace, \qquad \forall x, y \in \mathbb{R}^d\enspace.
    \end{align*}
    For differentiable functions, this is equivalent to assuming \( \|\nabla f_t^m(x)\| \leq G \) for all \( x \in \mathbb{R}^d \).
\end{assumption}

\begin{assumption}[Lipschitz Gradients]\label{ass:online_smooth_second}
    For all \( t \in [T] \) and \( m \in [M] \), the function \( f_t^m \) has \( H \)-Lipschitz gradients, i.e.,
    \begin{align*}
        \|\nabla f_t^m(x) - \nabla f_t^m(y)\| \leq H \|x - y\|_2\enspace, \qquad \forall x, y \in \mathbb{R}^d\enspace.
    \end{align*}
    This is equivalent to each \( f_t^m \) being \( H \)-second-order smooth.
\end{assumption}

These assumptions generalize \Cref{ass:bounded_gradients,ass:smooth_second} to the online setting. We will also consider the following special case:

\begin{assumption}[Linear Cost Functions]\label{ass:online_linear}
    For all \( t \in [T] \) and \( m \in [M] \), the function \( f_t^m \) is linear, i.e.,
    \begin{align*}
        \nabla f_t^m(a x + b y) = a \nabla f_t^m(x) + b \nabla f_t^m(y)\enspace, \qquad \forall a, b \in \mathbb{R},\ x, y \in \mathbb{R}^d\enspace.
    \end{align*}
\end{assumption}

This assumption is satisfied in adversarial linear bandit problems, one of the most commonly studied settings in online optimization. Note that linear functions satisfy \Cref{ass:online_smooth_second} with \( H = 0 \), making them the ``smoothest'' convex functions.

\subsection{Adversary Model}

In the most general setting, each machine may face arbitrary functions from a function class \( \mathcal{F} \) at each time step—for example, the class of convex and smooth functions satisfying \Cref{ass:online_convex,ass:online_smooth_second}. We analyze algorithms under this general setting, often referred to as an \textbf{\textit{adaptive}} adversary model.

Specifically, we allow the adversary to generate functions based on the history of past models but not on the internal randomness of the learning algorithms. Formally, define the filtration at time \( t \in [T] \) as\footnote{We use \( \sigma(S) \) to denote the sigma-algebra generated by a set \( S \).}
\begin{align*}
    \mathcal{H}_t := \sigma\left( \left\{ \{x_l^n\}_{l \in [t-1]}^{n \in [M]}, \{f_l^n\}_{l \in [t-1]}^{n \in [M]} \right\} \right)\enspace.
\end{align*}
Let \( P \in \mathcal{P} \) denote an adversary in a given class \( \mathcal{P} \)\footnote{Compare to the problem class discussed in \Cref{sec:min_max}.}. We assume the adversary outputs a distribution \( \mathcal{E}_t \in \Delta(\mathcal{F}^{\otimes M}) \) over functions at each time step:
\[
    P(\mathcal{H}_t, A) = \mathcal{E}_t \quad \text{and} \quad \{f_t^m\}_{m \in [M]} \sim \mathcal{E}_t\enspace,
\]
where \( A \in \mathcal{A} \) denotes the learning algorithm. In Section~\ref{sec:related}, we will discuss when randomization helps or does not help the adversary.

In contrast to the \emph{adaptive} setting, a \emph{stochastic adversary} must fix a joint distribution \( \mathcal{E} \in \Delta(\mathcal{F}^{\otimes M}) \) ahead of time and sample independently at each round:
\[
    \{f_t^m\}_{m \in [M]} \sim \mathcal{E} \quad \text{for all } t \in [T]\enspace.
\]
We denote the class of such stochastic adversaries by \( \mathcal{P}_{\text{stoc}} \). A detailed discussion of this distinction appears in Section~\ref{sec:related}. In particular, note that for Problem~\eqref{prob:scalar} $\eee = \ddd_1\times \dots \times \ddd_M$, i.e., the product distribution of the machines' data distributions.

As in the stochastic setting, it is useful to impose a notion of data heterogeneity to limit the adversary’s power. We use the following assumption, which controls the variation of gradients across machines.

\begin{assumption}[Bounded First-Order Heterogeneity]\label{ass:online_zeta}
    Suppose \( f_t^m \) satisfies \Cref{ass:online_convex,ass:online_smooth_second} for all \( t \in [T] \) and \( m \in [M] \). Then there exists \( \hat{\zeta} \leq 2G \) such that for all \( x \in \mathbb{R}^d \),
    \[
        \frac{1}{M} \sum_{m \in [M]} \| \nabla f_t^m(x) - \nabla f_t(x) \|_2^2 \leq \hat{\zeta}^2\enspace.
    \]
\end{assumption}
\begin{remark}
    The parameter \( \hat{\zeta} \) quantifies the degree of allowable heterogeneity across machines and thus controls the power of the adversary. A larger value of \( \hat{\zeta} \) permits the adversary to choose more diverse cost functions across machines. In particular, when \( \hat{\zeta} \geq 2G \), the constraint in \Cref{ass:online_zeta} becomes vacuous, allowing the adversary to select any collection of \( M \) functions satisfying \Cref{ass:online_convex,ass:online_bounded_gradients}. In contrast, when \( \hat{\zeta} = 0 \), the gradients of all functions must be identical at every point, forcing the adversary to assign the same cost function to every machine at each time step.
\end{remark}

\begin{remark}[Comparison to \Cref{ass:zeta}]
    In the stochastic setting, \Cref{ass:zeta} bounds the \emph{expected} gradient variance across machines, where each cost function is of the form \( f_t^m(x) = f(x; z_t^m) \) with \( z_t^m \sim \mathcal{D}_m \). That is, the expectation is taken over the randomness of sampling from \( \mathcal{D}_m \). In contrast, \Cref{ass:online_zeta} imposes a \emph{pointwise} bound on the deviation of gradients across machines at each time step. To see the relationship between the two assumptions, observe that:
    \begin{align*}
        &\ee\sb{\frac{1}{M} \sum_{m \in [M]} \| \nabla f_t^m(x) - \nabla f_t(x) \|_2^2}= \ee\sb{\frac{1}{M} \sum_{m \in [M]} \| \nabla f(x; z_t^m) - \nabla F(x) \|_2^2}\enspace,\\
        &\quad= \ee\sb{\frac{1}{M} \sum_{m \in [M]} \| \nabla f(x; z_t^m) - \nabla F_m(x) + \nabla F_m(x) - \nabla F(x) \|_2^2}\enspace,\\
        &\quad\leq 2\ee\sb{\frac{1}{M} \sum_{m \in [M]} \norm{\nabla f(x; z_t^m) - \nabla F_m(x)}} + \frac{2}{M}\sum_{m \in [M]}\norm{\nabla F_m(x) - \nabla F(x)}^2\enspace,\\
        &\quad\leq^{\text{(\Cref{ass:stoch_bounded_second_moment})}} 2\sigma_2^2 + \frac{2}{M^2}\sum_{m,n \in [M]}\norm{\nabla F_m(x) - \nabla F_n(x)}^2\enspace,\\
        &\quad\leq^{\text{(\Cref{ass:zeta})}} 2\sigma_2^2 + 2H^2\zeta^2\enspace.
    \end{align*}
    Thus, even in the homogeneous case where \( \mathcal{D}_m = \mathcal{D} \) for all \( m \), the parameter \( \zeta \) in \Cref{ass:zeta} is zero, but \( \hat{\zeta} \) in \Cref{ass:online_zeta} may still be nonzero unless \( \mathcal{D} \) is a Dirac distribution (i.e., has zero variance). In this sense, \Cref{ass:online_zeta} provides a stricter, distribution-free control on cross-machine heterogeneity that holds deterministically at each round, rather than in expectation over the data.
\end{remark}

We now introduce an assumption on the average optimal function value, analogous to \Cref{ass:bounded_optima,ass:bounded_func_subopt} in the stochastic setting.

\begin{assumption}[Average Value at Optima]\label{ass:online_bounded_optimal}
    For all \( x^\star \in \arg\min_{x \in \mathbb{B}_2(B)} \sum_{t \in [T]} f_t(x) \), we have\footnote{Note that \( F_\star \) can always be defined retrospectively once all cost functions have been revealed.}
    \[
        \frac{1}{T} \sum_{t \in [T]} f_t(x^\star) \leq F_\star\enspace.
    \]
\end{assumption}

\begin{remark}
    For non-negative functions satisfying \Cref{ass:online_convex,ass:online_smooth_second}, \Cref{ass:online_bounded_optimal} implies a bound on the average squared gradient norm at the global optimum:
    \begin{align*}
        \frac{1}{T} \sum_{t \in [T]} \| \nabla f_t(x^\star) \|_2^2 
        \leq^{\text{(\Cref{ass:online_smooth_second})}} \frac{1}{T} \sum_{t \in [T]} 2H \left(f_t(x^\star) - \min_{x_t^\star \in \mathbb{R}^d} f_t(x_t^\star)\right) \leq 2H F_\star\enspace.
    \end{align*}
    Thus, \Cref{ass:online_bounded_optimal} serves as an online analogue of \Cref{ass:zeta_star,ass:phi_star}, capturing a form of first-order regularity at the optimum across time (cf. \Cref{rem:other_first_order_ass}).
\end{remark}

\subsection{Oracle Model} 

We consider three types of oracle access to the cost functions in this paper. At each time step \( t \in [T] \), every machine \( m \in [M] \) interacts with its local cost function \( f_t^m \) through one of the following modes of feedback:

\begin{enumerate}
    \item \textbf{Gradient access}: the machine receives the gradient \( \nabla f_t^m(x_t^m) \) at a single point \( x_t^m \in \mathbb{R}^d \); this is referred to as \emph{first-order feedback}.
    \item \textbf{Single function value}: the machine receives the function value \( f_t^m(x_t^m) \) at a single point \( x_t^m \); this corresponds to \emph{one-point bandit feedback}.
    \item \textbf{Two function values}: the machine receives function values \( \rb{f_t^m(x_t^{m,1}), f_t^m(x_t^{m,2})} \) at two points \( x_t^{m,1}, x_t^{m,2} \in \mathbb{R}^d \); this is referred to as \emph{two-point bandit feedback}.
\end{enumerate}

The oracle model formally specifies how the learner (agent) interacts with the environment. We define these oracles below.

\begin{definition}[Online First-Order Oracle]\label{def:online_first}
    Each machine \( m \in [M] \) is equipped with an oracle \( \mathcal{O}_m : \mathbb{R}^d \times [T] \to \mathbb{R}^d \) such that, for any time \( t \in [T] \), querying the oracle with \( x_t^m\in\rr^d \) yields
    \[
        \mathcal{O}_m(x_t^m; t) = \nabla f_t^m(x_t^m)\enspace.
    \]
\end{definition}

This model parallels the stochastic first-order oracle in \Cref{def:oracle_first}. In particular, in the stochastic setting where \( f_t^m(x) = f(x; z_t^m) \), querying either oracle yields the same gradient.\footnote{Strictly speaking, the stochastic oracle in \Cref{def:oracle_first} may return any unbiased estimate of \( \nabla F_m(x) \), but in typical learning problems this corresponds to \( \nabla f(x; z_t^m) \).}

We next define a generalization of the bandit oracle to the online setting, analogous to the multi-point oracle in \Cref{def:oracle_first_multi_point}, but now for zeroth-order information:

\begin{definition}[Online Bandit Multi-Point Oracle]\label{def:online_zero_n}
    Each machine \( m \in [M] \) is equipped with an oracle 
    \[
        \mathcal{O}_m : (\mathbb{R}^d)^{\otimes n} \times [T] \to \mathbb{R}^n
    \]
    such that, for any \( t \in [T] \), querying the oracle with \( x_t^{m,1}, \dots, x_t^{m,n} \in \rr^d\) returns
    \[
        \mathcal{O}_m\rb{x_t^{m,1}, \dots, x_t^{m,n}; t} = \rb{f_t^m(x_t^{m,1}), \dots, f_t^m(x_t^{m,n})}\enspace.
    \]
\end{definition}

In this paper, we consider the case \( n = 1 \) (one-point feedback) and \( n = 2 \) (two-point feedback). Note that we always evaluate regret at the points where the oracle is queried, consistent with our deployment-centric view: querying an oracle is equivalent to deploying a model and thus incurring the corresponding cost.

Under one-point feedback, Problem~\eqref{eq:online_problem} remains well-defined: machine \( m \) simply queries the oracle at \( x_t^m \), and incurs loss \( f_t^m(x_t^m) \). In the two-point feedback setting, if machine \( m \) queries points \( x_t^{m,1} \) and \( x_t^{m,2} \) at time \( t \), it incurs the cumulative cost 
\(
    f_t^m(x_t^{m,1}) + f_t^m(x_t^{m,2}),
\)
(cf. \Cref{thm:bd_grad_first_stoch}).

\subsection{Algorithm Class}

We assume that the algorithm on each machine may depend on its entire local history, as well as any information shared through communication. Using the notation that the input and output of the oracle on machine $m\in[M]$ and time $t\in[T]$ is denoted by $I_t^m$ $O_t^m$ respectively we can formally define the information available to machine \( m\in[M] \) at time $t\in[T]$ as
\[
    \mathcal{G}_t^m := \sigma\rb{\left\{
        \{I_l^n\}_{n \in [M],\ l \in [\delta(t-1)]},\ 
        \{O_l^n\}_{n \in [M],\ l \in [\delta(t-1)]},\ 
        \{I_l^m\}_{l \in [t-1]},\ 
        \{O_l^m\}_{l \in [t-1]}
    \right\}}\enspace,
\]
where \( \delta(t) = t - t \bmod K \) denotes the most recent communication round before or at time \( t \). This construction captures the information structure of the intermittent communication (IC) setting.

We denote the class of online algorithms with the above information structure by \( \mathcal{A}_{\text{online-IC}} \), and further specify the type of oracle access using superscripts:
\begin{itemize}
    \item \( \mathcal{A}_{\text{online-IC}}^1 \): first-order feedback (gradient access),
    \item \( \mathcal{A}_{\text{online-IC}}^0 \): one-point bandit feedback,
    \item \( \mathcal{A}_{\text{online-IC}}^{0,2} \): two-point bandit feedback.
\end{itemize}

Formally, for algorithms \( \{A_m\}_{m \in [M]} \in \mathcal{A}_{\text{online-IC}}^1 \) or \( \mathcal{A}_{\text{online-IC}}^0 \), each machine’s model at time \( t \) is given by
\[
    A_m(\mathcal{G}_t^m) = X_t^m \in \Delta(\mathbb{R}^d)\enspace,
\]
i.e., a randomized selection of a single point in \( \mathbb{R}^d \). In contrast, for algorithms \( \{A_m\}_{m \in [M]} \in \mathcal{A}_{\text{online-IC}}^{0,2} \), which operate under two-point bandit feedback, the output at time \( t \) is a randomized pair of points:
\[
    A_m(\mathcal{G}_t^m) = (X_t^{m,1}, X_t^{m,2}) \in \Delta(\mathbb{R}^d \times \mathbb{R}^d)\enspace.
\]

In both cases, the variables \( X \) represent the internal randomization used by the algorithm to determine the model(s) \( x \) played by the machine. As we will observe later in this chapter, such randomization is essential in the bandit-feedback setting to ensure low regret.

\subsection{Min-Max Regret}

We now have all the components in place to define the appropriate notion of min-max complexity for Problem~\eqref{eq:online_problem} (cf. \Cref{sec:min_max}). Let \( \mathcal{P} \) denote a class of adversaries and \( \mathcal{A} \) a class of algorithms with single-point feedback. The corresponding min-max regret is defined as:
\begin{align}\label{eq:P2}
    \mathcal{R}(\mathcal{P}, \mathcal{A}) := \min_{A \in \mathcal{A}} \max_{P \in \mathcal{P}} 
    \mathbb{E}_{P,A}\left[
        \frac{1}{MT} \sum_{t \in [T],\ m \in [M]} f_t^m(x_t^m) 
        - \min_{x^\star \in \mathbb{B}_2(B)} \frac{1}{MT} \sum_{t \in [T],\ m \in [M]} f_t^m(x^\star)
    \right]\enspace,
\end{align}
where the expectation is taken over both sources of randomness: (1) the algorithm's internal randomization in selecting models \( \{x_t^m \sim A(\mathcal{G}_t^m)\} \), and (2) the adversary’s selection of functions \( \{\{f_t^m\}_{m \in [M]} \sim \mathcal{P}(\mathcal{H}_t, A)\}_{t \in [T]} \).\footnote{Recall that \( \mathcal{H}_t \) denotes the full history (or sigma algebra) of functions and models played by all machines up to time \( t-1 \).}

In the case of two-point feedback, the min-max regret is similarly defined as:
\begin{align*}
    \mathcal{R}(\mathcal{P}, \mathcal{A}) := \min_{A \in \mathcal{A}} \max_{P \in \mathcal{P}} 
    \mathbb{E}_{P,A}\left[
        \frac{1}{2MT} \sum_{t \in [T],\ m \in [M],\ j \in [2]} f_t^m(x_t^{m,j}) 
        - \min_{x^\star \in \mathbb{B}_2(B)} \frac{1}{MT} \sum_{t \in [T],\ m \in [M]} f_t^m(x^\star)
    \right]\enspace,
\end{align*}
where the expectation over the algorithm is with respect to the random choice of \( (x_t^{m,1}, x_t^{m,2}) \sim A(\mathcal{G}_t^m) \) for all \( t \in [T] \) and \( m \in [M] \).

The goal of this chapter is to characterize (up to numerical constants) the min-max regret \( \mathcal{R}(\mathcal{P}, \mathcal{A}) \) for the different adversary and algorithm classes introduced earlier in this section.

\begin{remark}[Randomization and Min-Max Games]
    In this min-max formulation, the second player—the adversary—does not benefit from randomization. That is, the worst-case regret is attained by a deterministic choice of functions. Hence, we can restrict attention to deterministic adversary classes \( \mathcal{P} \) and drop the expectation with respect to \( P \) in the above definitions.\footnote{It becomes meaningful to retain adversarial randomization when comparing to weaker benchmarks, such as those in Problem (P3) in \Cref{sec:related}.}

    It is also important to note that the max player (the adversary) does not have access to the internal randomness of the min player (the algorithm). This asymmetry is crucial in the analysis and matches the standard formulation of online learning games (cf. Problem (P1) in \Cref{sec:related}).
\end{remark}
In the next section we will first study online optimization with first-order feedback. 

\section{Collaboration Does Not Help with First-Order Feedback}\label{sec:first_does_not_help}

In this section, we demonstrate that collaboration does not improve the min-max regret for the adaptive online optimization Problem~\eqref{eq:online_problem} under first-order feedback---even though it is known to help in the stochastic setting for solving Problem~\eqref{prob:scalar}. Specifically, we characterize the min-max complexity for optimizing cost functions satisfying \Cref{ass:online_bounded_gradients,ass:online_smooth_second} using the algorithm class \( \mathcal{A}_{\text{IC}}^1 \), i.e., algorithms that receive one gradient per cost function on each machine at each time step.

This problem is well understood in the serial (single-machine) setting, i.e., when \( M = 1 \). In particular, Online Gradient Descent (OGD)~\citep{zinkevich2010parallelized,hazan2016introduction} is known to attain the min-max regret for both Lipschitz and smooth cost functions~\citep{woodworth2021even}. A natural question is whether a distributed version of OGD remains min-max optimal when \( M > 1 \). Surprisingly, the answer is negative.

To make this precise, we introduce a simple non-collaborative algorithm: each machine runs OGD independently, without any communication (see Algorithm~\ref{alg:NC_OGD}). We show that this algorithm, which belongs to \( \mathcal{A}_{\text{IC}}^1 \), is in fact min-max optimal—implying that collaboration yields no improvement in this setting.

\vspace{1em}

{\LinesNumbered
\SetAlgoVlined
\begin{algorithm2e}[!ht]
	\caption{Non-Collaborative \textsc{OGD} (\( \eta \))}\label{alg:NC_OGD}
	Initialize \( x_0^m = 0 \) on all machines \( m \in [M] \)\;
	\For(\tcp*[f]{Across total time steps}){\( t \in \{0, \ldots, KR - 1\} \)}{
		\For(\tcp*[f]{Each machine runs independently}){\( m \in [M] \) \textbf{in parallel}}{
			Play model \( x_t^m \) and observe cost function \( f_t^m(\cdot) \)\;
			\textcolor{red}{\textbf{Incur loss} \( f_t^m(x_t^m) \)}\;
			Compute gradient \( \nabla f_t^m(x_t^m) \)\;
			Update: \( x_{t+1}^m \gets x_t^m - \eta \cdot \nabla f_t^m(x_t^m) \)\;
		}
	}
\end{algorithm2e}
}

\vspace{1em}

We now state the main theorems showing the optimality of Algorithm~\ref{alg:NC_OGD}.

\begin{theorem}[Optimality for Lipschitz Functions]\label{thm:first_lip}
    Let \( \mathcal{P} \) be a problem class satisfying \Cref{ass:online_convex,ass:online_bounded_gradients,ass:online_zeta}. Then the min-max regret satisfies
    \[
        \mathcal{R}(\mathcal{P}, \mathcal{A}_{\text{IC}}^1) = \Theta\left( \frac{GB}{\sqrt{T}} \right),
    \]
    and Algorithm~\ref{alg:NC_OGD} achieves this optimal rate.
\end{theorem}

\begin{theorem}[Optimality for Smooth Functions]\label{thm:first_smth}
    Let \( \mathcal{P} \) be a problem class satisfying \Cref{ass:online_convex,ass:online_smooth_second,ass:online_bounded_optimal,ass:online_zeta}. Then the min-max regret satisfies
    \[
        \mathcal{R}(\mathcal{P}, \mathcal{A}_{\text{IC}}^1) = \Theta\left( \frac{HB^2}{T} + \frac{\sqrt{HF_\star}B}{\sqrt{T}} \right),
    \]
    and Algorithm~\ref{alg:NC_OGD} achieves this optimal rate.
\end{theorem}

Proofs of both these theorems are provided in \Cref{app:first_order_online_lb} and follow simply from our observations about related problem classes in \Cref{sec:related}. These results establish that, under first-order feedback, collaboration across machines offers no benefit: the non-collaborative baseline is already min-max optimal. Notably, the heterogeneity parameter \( \hat{\zeta} \) in \Cref{ass:online_zeta} has no impact on the min-max complexity.

The key idea behind these proofs is to construct adversarial instances in which all machines observe the same cost function at each time step. In this way, even though machines act independently, they receive identical information from their oracle queries—effectively simulating a centralized algorithm. This "coordinated attack" renders collaboration superfluous. 

For example, in the proof of \Cref{thm:first_lip}, we place the same linear function on every machine at each time step,\footnote{This also implies that assuming both \Cref{ass:online_bounded_gradients} and \Cref{ass:online_smooth_second}, or alternatively \Cref{ass:online_linear}, does not help: linear functions have the smallest possible second-order smoothness—namely, zero.} ensuring that all machines receive identical gradients. Since \( \hat{\zeta} = 0 \) in this construction, no assumption on heterogeneity can alter the min-max complexity. 

Interestingly, these hard instances (in \Cref{thm:first_lip,thm:first_smth}) are themselves stochastic---they belong to \( \mathcal{P}_{\text{stoc}} \)---and the only adversarial power exploited is the ability to coordinate identical cost functions across machines. This example also highlights a broader insight: when machines have access to exact first-order oracles, there is little to be gained from collaboration. In contrast, in the stochastic setting, collaboration was beneficial in part because of the variance in gradient estimates. 

This motivates us to consider weaker oracle models---particularly those in which the learner receives only partial or noisy feedback. The most natural such model in the online setting is bandit feedback, which we investigate in the next section.

\section{Collaboration Helps with Bandit Feedback}\label{sec:falb}

In this section, we turn to the more challenging setting where machines receive only bandit (zeroth-order) feedback. We begin by studying a significant instance of Problem~\eqref{eq:online_problem}, namely the setting of \emph{federated adversarial linear bandits}. We then extend our results to the more general setting of \emph{federated bandit convex optimization with two-point feedback} in the next section.

\subsection{Federated Adversarial Linear Bandits}

Federated linear bandits represent an essential application of the general formulation in~\eqref{eq:online_problem}, and have recently garnered significant attention. However, most existing works~\citep{Wang2020Distributed,huang2021federated,li2022asynchronous,he2022simple} focus on the stochastic setting, and do not address the more challenging case of adaptive adversaries—leaving it unclear whether collaboration can improve performance under worst-case cost sequences.

We propose and analyze the setting of \emph{federated adversarial linear bandits}, a natural extension of the classical single-agent adversarial linear bandit problem~\citep{bubeck2012regret} to the federated environment. Formally, at each time step \( t \in [T] \), each machine \( m \in [M] \) selects an action \( x_t^m \in \mathbb{R}^d \), while simultaneously the environment selects a loss vector \( \beta_t^m \in \mathbb{B}_2(G) \subset \mathbb{R}^d \). The machine then suffers linear loss: \( f_t^m(x_t^m) = \langle \beta_t^m, x_t^m \rangle \).

The goal is to generate a sequence \( \{x_t^m\}_{t \in [T],\ m \in [M]} \) that minimizes the expected regret:
\begin{align}\label{eq:problem_falb}
    \mathbb{E} \left[
        \sum_{m, t} \langle \beta_t^m, x_t^m \rangle - 
        \min_{\|x^\star\| \leq B} \sum_{m, t} \langle \beta_t^m, x^\star \rangle
    \right]\enspace,
\end{align}
where the expectation is over the algorithm’s internal randomness.

To address this problem, we propose a new algorithm, \textsc{FedPOSGD} (Federated Projected Online Stochastic Gradient Descent), which operates with one-point bandit feedback. The algorithm is described in detail in Algorithm~\ref{alg:fed_pogd}.

\vspace{1em}
{\LinesNumbered
\SetAlgoVlined
\begin{algorithm2e}[!t]
	\caption{\textsc{FedPOSGD} (\( \eta, \delta \)) with One-Point Bandit Feedback}\label{alg:fed_pogd}
	Initialize \( x_0^m = 0 \) on all machines \( m \in [M] \)\;
	\For(\tcp*[f]{Across total time steps}){\( t \in \{0, \ldots, KR-1\} \)}{
		\For(\tcp*[f]{Each machine runs in parallel}){\( m \in [M] \)}{
			Project to feasible region: \( w_t^m = \textbf{Proj}(x_t^m) \)\;
			Sample direction \( u_t^m \sim \text{Unif}(\mathbb{S}^{d-1}) \)\;
			Query function at \( w_t^{m,1} = w_t^m + \delta u_t^m \)\;
			\textcolor{red}{\textbf{Incur loss} \( f_t^m(w_t^{m,1}) \)}\;
			Estimate gradient: \( g_t^m = df(w_t^{m,1}) u_t^m / \delta \)\;
			\eIf{\( (t+1) \bmod K = 0 \)}{
				\textcolor{blue}{\textbf{Send to server:}} \( x_t^m - \eta g_t^m \)\;
				Server computes average: \( x_{t+1} = \frac{1}{M} \sum_{m \in [M]} (x_t^m - \eta g_t^m) \)\;
				\textcolor{blue}{\textbf{Broadcast to clients:}} \( x_{t+1}^m \gets x_{t+1} \)\;
			}{
				Local update: \( x_{t+1}^m \gets x_t^m - \eta g_t^m \)\;
			}
		}
	}
\end{algorithm2e}
}

\paragraph{Gradient Estimator.}
The estimator \( g_t^m \) in line 8 is based on the one-point bandit gradient method of \citet{flaxman2004online}, but adapted to operate at the projected point \( w_t^m = \textbf{Proj}(x_t^m) = \arg\min_{\|w\| \leq B} \|w - x_t^m\| \). For linear cost functions \( f_t^m(x) = \langle \beta_t^m, x \rangle \), this estimator is unbiased:
\[
    \mathbb{E}_{u_t^m}[g_t^m] = \nabla f_t^m(x)\enspace,
\]
and its variance is bounded by~\citep{hazan2016introduction}, for a constant \( c_{15} \),
\begin{align}\label{eq:single_feedback_variance}
    \mathbb{E}_{u_t^m} \left[ \| g_t^m - \nabla f_t^m(x) \|_2^2 \right] 
    \leq c_{15} \cdot \left( \frac{d \|\beta_t^m\|_2 (\|x\|_2 + \delta)}{\delta} \right)^2\enspace.
\end{align}
Thus, the projection step is crucial to keep the variance bounded. However, it also complicates aggregation across machines. To address this, we perform the updates in the \emph{unprojected} space (line 14), inspired by lazy mirror descent methods~\citep{nesterov2009primal,bubeck2015convex,yuan2021federated}.

We now state the main guarantee for \textsc{FedPOSGD}.

\begin{theorem}[Regret of \textsc{FedPOSGD} for Federated Adversarial Linear Bandits]\label{thm:favlb}
Assume we have cost functions satisfying \Cref{ass:online_convex,ass:online_bounded_gradients,ass:online_linear,ass:online_zeta}. Let \( \eta = \frac{B}{G\sqrt{T}} \cdot \min\left\{1, \frac{\sqrt{M}}{dB}, \frac{1}{\mathbb{I}_{K > 1} \sqrt{dB} K^{1/4}}, \frac{\sqrt{G}}{\mathbb{I}_{K > 1} \sqrt{\hat{\zeta} K}} \right\} \) and \( \delta = B \). Then, for a constant \( c_{16} \), the average regret at the queried points \( \{w_t^{m,1}\} \) satisfies:
\[
    \frac{1}{MKR} \sum_{t \in [KR],\ m \in [M]} \mathbb{E} \left[ f_t^m(w_t^{m,1}) - f_t^m(x^\star) \right] 
    \leq c_{16} \cdot \left( \frac{GB}{\sqrt{KR}} + \frac{GBd}{\sqrt{MKR}} 
    + \mathbb{I}_{K > 1} \left[ \frac{GB\sqrt{d}}{K^{1/4} \sqrt{R}} + \frac{\sqrt{G \hat{\zeta}} B}{\sqrt{R}} \right] \right)\enspace,
\]
where \( x^\star \in \arg\min_{x \in \mathbb{B}_2(B)} \sum_{t \in [KR]} f_t(x) \), and the expectation is over the algorithm's internal randomness.
\end{theorem}

\paragraph{Implications of Theorem~\ref{thm:favlb}.} Compared to the non-collaborative baseline—where each machine independently runs SCRiBLe~\citep{hazan2016introduction} or Algorithm~\ref{alg:NC_OGD} with one-point feedback—the regret is \( \mathcal{O}(GBd / \sqrt{KR}) \). When \( d = \mathcal{O}(\sqrt{K}) \), \textsc{FedPOSGD} outperforms the baseline. In particular, when \( d = \Omega(\sqrt{K}M) \), the regret is dominated by the \( \mathcal{O}(GBd / \sqrt{MKR}) \) term, implying a \emph{linear speedup} in the number of machines. Although the regret also decreases with smaller \( \hat{\zeta} \), this benefit is negligible in high dimensions, since the \( \hat{\zeta} \)-dependent term becomes dominated when \( d = \Omega(\sqrt{K}) \).

\paragraph{Limitations of Algorithm~\ref{alg:fed_pogd}.} While \textsc{FedPOSGD} achieves meaningful gains, it suffers from three key limitations:
\begin{enumerate}
    \item It requires an additional \emph{projection step} before querying the function.
    \item Its regret bound scales linearly with \( d \), which can be prohibitive in high-dimensional settings.
    \item It does not fully exploit low heterogeneity (\( \hat{\zeta} \)) in regimes where collaboration yields improvements.
\end{enumerate}
To address these issues, we now turn to two-point feedback algorithms in the next section.

\section{Better Rates with Two-Point Bandit Feedback}\label{sec:2falb}
We now consider distributed bandit convex optimization with \emph{two-point feedback}, where at each time step, machines may query their cost functions at two locations (but do not have access to gradients). We show improved regret guarantees for general Lipschitz smooth functions, and specialize these results to both adversarial linear bandits and functions satisfying second-order smoothness (\Cref{ass:online_smooth_second}).

Two-point feedback is well-studied in the single-agent setting, where it enables optimal horizon dependence for regret using simple algorithms~\citep{duchi2015optimal,shamir2017optimal}. Here, we go beyond linear losses and consider general convex cost functions. Our proposed method is an online variant of the \textsc{FedAvg} or \textsc{Local-SGD} algorithm, adapted to work with two-point bandit feedback. We refer to this algorithm as \textsc{FedOSGD} and describe it in Algorithm~\ref{alg:fed_ogd}.

{\LinesNumbered
\SetAlgoVlined
\begin{algorithm2e}[!thp]
	\caption{\textsc{FedOSGD} ($ \eta, \delta$) with two-point bandit feedback}
    \label{alg:fed_ogd}
		Initialize $x_0^m=0$ on all machines $m\in[M]$\\
            \For{$t\in\{0, \ldots, KR-1\}$}{
		      \For{$m \in [M]$ \textbf{in parallel}}{
		        Sample $u_t^m\sim Unif(\ss_{d-1})$, i.e., a random unit vector\\          
                    Query function $f_t^m$ at points $(x_t^{m,1}, x_t^{m,2}):=(x_t^{m} + \delta u_t^m, x_t^{m} - \delta u_t^m)$\\
                    \textcolor{red}{\textbf{Incur loss} $(f_t^m(x_t^{m} + \delta u_t^m) + f_t^m(x_t^{m} - \delta u_t^m))$}\\
                    Compute stochastic gradient at point $x_t^m$ as $g_t^m = \frac{d(f(x_t^m+\delta u_t^m) - f(x_t^m-\delta u_t^m))u_t^m}{2\delta}$ \\  
                    \If{$(t+1)\mod K =0$}{ 
                    \textcolor{blue}{\textbf{Communicate to server:}} $\rb{x_t^m - \eta\cdot g_t^m}$\\
                    On server $x_{t+1} \gets \frac{1}{M}\sum_{m\in[M]}\left(x_t^m - \eta\cdot g_t^m\right)$\\
                    \textcolor{blue}{\textbf{Communicate to machine:}} $x_{t+1}^m \gets x_{t+1}$}
                    \Else{
                    $x_{t+1}^m \gets x_{t}^m - \eta\cdot g_t^m$
                    }
                }
		}	
\end{algorithm2e}
}

The key idea in \textsc{FedOSGD} is that the estimator in line 7, originally proposed by~\citet{shamir2017optimal}, is an unbiased estimate of the gradient of the smoothed function 
\[
    \hat{f}_t^m(x) := \mathbb{E}_{u_t^m}[f_t^m(x + \delta u_t^m)]\enspace,
\]
i.e., \( \mathbb{E}_{u_t^m}[g_t^m] = \nabla \hat{f}_t^m(x) \), with bounded variance:
\[
    \mathbb{E}_{u_t^m} \left[ \| g_t^m - \nabla \hat{f}_t^m(x) \|^2 \right] \leq d G^2\enspace,
\]
where \( G \) is the Lipschitz constant of \( f_t^m \)~\citep[Lemmas 3, 5]{shamir2017optimal}.

Equipped with this gradient estimator, we can prove the following guarantee for Lipschitz cost functions using \textsc{FedOSGD}.

\begin{theorem}[Regret of \textsc{FedOSGD} for Lipschitz Functions]
\label{thm:bd_grad_first_stoch}
Assume we have cost functions satisfying \Cref{ass:online_convex,ass:online_bounded_gradients,ass:online_zeta}. Let $\eta = \frac{B}{G\sqrt{T}}\cdot\min\left\{1, \frac{\sqrt{M}}{\sqrt{d}}, \frac{1}{\ii_{K>1}\sqrt{K}d^{1/4}}\right\}$, and $\delta = \frac{Bd^{1/4}}{\sqrt{R}}\rb{1+ \frac{d^{1/4}}{\sqrt{MK}}}$, then the queried points $\{x_t^{m,j}\}_{t,m,j=1}^{T,M,2}$ of Algorithm \ref{alg:fed_ogd} satisfy (for some numerical constant $c_{17}$):
\begin{align*}
   \frac{1}{2MKR}\sum_{t\in [KR], m\in[M], j\in[2]}\ee\sb{f_t^m(x_t^{m,j})-f_t^m(x^\star)} \leq c_{17}\cdot\rb{  
        \frac{GB}{\sqrt{KR}} + \frac{GB\sqrt{d}}{\sqrt{MKR}} + \ii_{K>1}\cdot\frac{GBd^{1/4}}{\sqrt{R}}}\enspace,
\end{align*}
    where $x^\star \in \arg\min_{x\in\bb_2(B)}\sum_{t\in [KR]}f_t(x)$, and the expectation is w.r.t. the choice of function-value queries.
\end{theorem}

\textbf{Implication of Theorem~\ref{thm:bd_grad_first_stoch}.}
When $K=1$, the average regret reduces to the first two terms, which match the known optimal rates for two-point bandit feedback~\citep{duchi2015optimal,hazan2016introduction} (see \Cref{app:two_pt_lb}), establishing the optimality of \textsc{FedOSGD} in this setting.

For $K > 1$, we compare our result against a non-collaborative baseline, which runs Algorithm~\ref{alg:NC_OGD} on each machine using the two-point gradient estimator of~\citet{shamir2017optimal}. This baseline achieves an average regret of $\mathcal{O}(GB\sqrt{d}/\sqrt{KR})$. Therefore, when $d = \Omega(K^2)$, \textsc{FedOSGD} outperforms the non-collaborative approach. Moreover, when $d = \Omega(K^2M^2)$, the regret of \textsc{FedOSGD} is dominated by $\mathcal{O}(GB\sqrt{d}/\sqrt{MKR})$, implying a \emph{linear speed-up} in the number of machines compared to the non-collaborative baseline.

We emphasize that the Lipschitz continuity assumption is crucial for bounding the variance of the gradient estimator used in Algorithm~\ref{alg:fed_ogd}. While alternative estimators exist that do not rely on Lipschitzness or bounded gradients~\citep{flaxman2004online}, they typically require stronger assumptions—such as bounded function values—or incur additional complexity, such as projection steps (cf.\ Algorithm~\ref{alg:fed_pogd}).

Despite these benefits, one limitation remains: the regret bound does not improve with smaller $\hat{\zeta}$ (cf.\ \Cref{ass:online_zeta}). To address this, we now turn our attention to cost functions that satisfy both \Cref{ass:online_bounded_gradients,ass:online_smooth_second}, for which we will derive refined guarantees.

\begin{theorem}[Informal, Regret of \textsc{FedOSGD} for Smooth Functions]
\label{thm:smooth_first_stoch}
Assume we have cost functions satisfying \Cref{ass:online_convex,ass:online_bounded_gradients,ass:online_smooth_second,ass:online_bounded_optimal,ass:online_zeta}. If we choose appropriate $\eta, \delta$ (c.f., Lemma \ref{lem:smooth_first_stoch} in Appendix \ref{app:smth}), the queried points $\{x_t^{m,j}\}_{t,m,j=1}^{T,M,2}$ of Algorithm \ref{alg:fed_ogd} satisfy (for a numerical constant $c_{18}$):
    \begin{align*}
       \frac{1}{2MT}&\sum_{t\in [T], m\in[M], j\in[2]}\ee\sb{f_t^m(x_t^{m,j})-f_t^m(x^\star)} \leq c_{18}\cdot \Bigg( 
        \frac{HB^2}{KR} + \frac{\sqrt{HF_\star}B}{\sqrt{KR}} + \frac{GB}{\sqrt{KR}}+ \frac{ GB\sqrt{d}}{\sqrt{MKR}} \\
        &+\ii_{K>1}\cdot\min\cb{\frac{H^{1/3}B^{4/3}G^{2/3}d^{1/3}}{K^{1/3}R^{2/3}} + \frac{H^{1/3}B^{4/3}\hat\zeta^{2/3}}{R^{2/3}} + \frac{\sqrt{\hat\zeta G}Bd^{1/4}}{K^{1/4}\sqrt{R}} + \frac{\hat\zeta B}{\sqrt{R}}, \frac{GBd^{1/4}}{K^{1/4}\sqrt{R}} + \frac{\sqrt{G\hat\zeta}B}{\sqrt{R}}}\Bigg)\enspace,    
    \end{align*}
    where $x^\star \in \arg\min_{x\in \bb_2(B)}\sum_{t\in [KR]}f_t(x)$, and the expectation is w.r.t. the choice of function-value queries.
\end{theorem}

The regret bound in Theorem~\ref{thm:smooth_first_stoch} is somewhat technical due to the generality of the smooth setting. To better interpret its implications, we consider the simpler case of linear functions with bounded gradients. Specifically, we assume the cost functions additionally satisfy \Cref{ass:online_linear}, which implies $H = 0$. Under this setting, we obtain the following informal corollary (see Appendix~\ref{sec:modify}):

\begin{corollary}[Informal, Regret of \textsc{FedOSGD} for Linear Functions]
\label{coro:linear}
Suppose the cost functions satisfy \Cref{ass:online_convex,ass:online_bounded_gradients,ass:online_linear,ass:online_bounded_optimal,ass:online_zeta}. If we choose the same step-size $\eta$ and smoothing parameter $\delta$ as in Theorem~\ref{thm:smooth_first_stoch} with $H=0$, then the queried points $\{x_t^{m,j}\}_{t,m,j=1}^{T,M,2}$ of Algorithm~\ref{alg:fed_ogd} satisfy (for some numerical constant $c_{19}$):
\begin{align*}
    \frac{1}{2MT} \sum_{\substack{t \in [T], m \in [M], \\ j \in [2]}} \mathbb{E} \left[ f_t^m(x_t^{m,j}) - f_t^m(x^\star) \right] 
    \leq c_{19} \cdot \left(
        \frac{GB}{\sqrt{KR}} 
        + \frac{GB\sqrt{d}}{\sqrt{MKR}} 
        + \mathbb{I}_{K > 1} \cdot \left[
            \frac{\sqrt{G\hat\zeta}\,Bd^{1/4}}{K^{1/4} \sqrt{R}} 
            + \frac{\hat\zeta B}{\sqrt{R}}
        \right]
    \right)\enspace,
\end{align*}
where $x^\star \in \arg\min_{x \in \bb_2(B)} \sum_{t \in [KR]} f_t(x)$, and the expectation is w.r.t. the choice of function-value queries.
\end{corollary}

\textbf{Implications of Theorem~\ref{thm:smooth_first_stoch} and Corollary~\ref{coro:linear}:}
For linear cost functions, the last two terms in the average regret bound vanish when $\hat\zeta = 0$, and the bound improves monotonically as $\hat\zeta$ decreases. In fact, when $K = 1$ or $\hat\zeta = 0$, the regret simplifies to $\ooo\left(GB/\sqrt{KR} + GB\sqrt{d}/\sqrt{MKR}\right)$, which is optimal (cf.~Appendix~\ref{app:two_pt_lb}). This shows that \textsc{FedOSGD} can effectively exploit small heterogeneity in the system.

More broadly, whenever 
\[
K \leq \min\left\{\frac{G^2}{\hat\zeta^2d},\ \frac{G^2d}{\hat\zeta^2M^2}\right\} = \frac{G^2}{\hat\zeta^2}\cdot\min\cb{\frac{1}{d},\ \frac{d}{M^2}}\enspace,
\]
\textsc{FedOSGD} again achieves the optimal regret rate of $\ooo\left(GB/\sqrt{KR} + GB\sqrt{d}/\sqrt{MKR}\right)$. In particular, the smaller the instantaneous first-order heterogeneity $\hat\zeta$, the higher the local updates can be made while obtaining optimal regret. In comparison, the non-collaborative baseline~\citep{shamir2017optimal} achieves only $\ooo\left(GB\sqrt{d}/\sqrt{KR}\right)$, so collaboration provides a clear advantage, especially in high-dimensional settings.

\paragraph{\textbf{Single vs.\ Two-Point Feedback.}} For federated adversarial linear bandits, Algorithm~\ref{alg:fed_ogd} (\textsc{FedOSGD}) achieves the regret bound:
\begin{align*}
    \frac{GB}{\sqrt{KR}} + \frac{ GB\sqrt{d}}{\sqrt{MKR}} + \mathbb{I}_{K>1} \cdot \left( \frac{\sqrt{G\hat\zeta}\,Bd^{1/4}}{K^{1/4}\sqrt{R}} + \frac{\hat\zeta B}{\sqrt{R}} \right).
\end{align*}
In contrast, Algorithm~\ref{alg:fed_pogd} (\textsc{FedPOSGD}) with one-point feedback yields:
\begin{align*}
    \frac{GB}{\sqrt{KR}} + \frac{GB d}{\sqrt{MKR}} + \mathbb{I}_{K>1} \cdot \left( \frac{GB\sqrt{d}}{K^{1/4}\sqrt{R}} + \frac{\sqrt{G\hat\zeta}B}{\sqrt{R}} \right).
\end{align*}
Thus, \textsc{FedOSGD} achieves strictly better regret bounds in both $d$ and $\hat\zeta$, while also avoiding the need for projection steps. These improvements demonstrate that access to richer feedback (via two-point queries) can substantially enhance performance in federated adversarial linear bandit settings.

\chapter{Conclusion}\label{ch:conclusion}

This thesis presents a unified and heterogeneity-aware theory of local update algorithms in distributed optimization. The core contributions span conceptual, technical, and algorithmic dimensions:

\begin{itemize}
    \item We develop a min-max complexity framework that distinguishes between first- and second-order heterogeneity, offering sharper and more interpretable models for analyzing distributed learning algorithms.
    \item Across both convex and non-convex settings, we establish that small second-order heterogeneity is necessary and sufficient for local update algorithms like Local SGD to outperform centralized or mini-batch methods. This insight underpins our lower and upper bounds, unifying their conclusions under a common theme.
    \item In some regimes where local updates are suboptimal, we prove the min-max optimality of classical methods such as mini-batch SGD, clarifying the boundary between algorithmic effectiveness and structural limitations.
    \item We contribute a new fixed-point perspective on Local SGD, revealing a heterogeneity-aware form of implicit bias and conditioning that affects its behavior.
    \item Our consensus-error analysis framework enables improved upper bounds, especially under third-order smoothness, and accommodates more relaxed assumptions than previous approaches.
    \item We design and analyze \textsc{CE-LSGD}, a new communication-efficient local update algorithm for the non-convex setting, and prove its near-optimality.
    \item Finally, we establish a theory of federated online optimization, showing when collaboration helps (and when it doesn't) in both full-information and bandit feedback regimes.
\end{itemize}

Together, these contributions offer a principled and comprehensive picture of the role local updates can play in federated and distributed optimization. They bring clarity to longstanding questions about when local computation helps, why it helps, and how much can be gained in different regimes of heterogeneity.

\paragraph{Outlook and Open Problems.} Several avenues remain open for future work. Technically, it would be valuable to refine the consensus-error framework further, especially for algorithms with adaptive step sizes, compression, or partial participation. Extending our upper bounds to general convex settings, where we currently cannot completely get rid of \Cref{ass:zeta} is another important challenge. Conceptually, a better understanding of structured heterogeneity---such as clustering, task similarity, or adversarial noise---could inspire new adaptive algorithms.

More broadly, local updates raise pressing questions about fairness, privacy, and personalization, especially as federated learning is deployed in systems that must respect data sovereignty and user constraints. Bridging optimization theory with these broader considerations will be a central task in the years ahead. It is my hope that this thesis helps lay a theoretical foundation for these future explorations.





\appendix
\chapter{Additional Details for Chapter 2}\label{app:chap2}
\section{Proof of Equation \ref{eq:motivation}}
\begin{proof}
    Note that the update on machine $m$ leading up to communication round $r$ is as follows for $k\in[0,K-1]$ and $m=1$,
    \begin{align*}
     &x_{r,k+1}^1[1] = x_{r,k}^1[1] - \eta H(x_{r,k}^1[1]-x^\star[1]),\\
     &\Rightarrow x_{r,k+1}^1[1] = x^\star[1] + (1-\eta H)^{k+1}(x_{r,0}^1[1]-x^\star[1]),\\
     &\Rightarrow x_{r,K}^1[1] = x^\star[1] + (1-\eta H)^{K}(x_{r-1}[1]-x^\star[1]),\\
     &\Rightarrow x_{r,K}^1[1] - x_{r-1}[1] = (1-(1-\eta H)^{K})(x^\star[1]-x_{r-1}[1]).
    \end{align*}
    On the second dimension, the iterates don't move at all for $m=1$, 
    \begin{align*}
        x_{r,K}^1[2] - x_{r-1}[2] = 0.
    \end{align*}
    Writing a similar expression for the second machine and averaging these updates we get,
    \begin{align*}
        \frac{1}{2}\sum_{m\in[2]}(x_{r,K}^m-x_{r-1}) = \frac{1}{2}(1-(1-\eta H)^{K})(x^\star-x_{r-1}).
    \end{align*}
    This gives the update for communication round $r$ as follows,
    \begin{align*}
        &x_r = x_{r-1} + \frac{\beta}{2}(1-(1-\eta H)^{K})(x^\star-x_{r-1}),\\
        &\Rightarrow x_r - x^\star = \rb{1 - \frac{\beta}{2}(1-(1-\eta H)^{K})}(x_{r-1}-x^\star),\\
        &\Rightarrow x_R = x^\star + \rb{1 - \frac{\beta}{2}(1-(1-\eta H)^{K})}^R(x_{0}-x^\star),\\
        &\Rightarrow x_R = \rb{1 - \rb{1 - \frac{\beta}{2}(1-(1-\eta H)^{K})}^R}x^\star,
    \end{align*}
    which finishes the proof.
\end{proof}

\section{Proof of \texorpdfstring{\Cref{prop:zeta_limitation}}{TEXT}}
\begin{proof}
Note the following for any $m\in[M]$ using triangle inequality,
\begin{align*}
    sup_{x\in \rr^d}\norm{\nabla F_m(x) - \nabla F(x)} &= sup_{x\in \rr^d}\norm{(A_m-A)x + b_m-b},\\
    &\geq sup_{x\in \rr^d}\norm{(A_m-A)x} - \norm{b_m-b}.
\end{align*}
Denote the matrix $C_m:= A_m-A = [c_{m,1}, \dots, c_{m,d}]$ using its column vectors. Then take $x = \delta e_i$ where $e_i$ is the $i$-th standard basis vector to note in the above inequality,
\begin{align*}
        sup_{x\in \rr^d}\norm{\nabla F_m(x) - \nabla F(x)} &\geq \delta \norm{(A_m-A)e_i} - \norm{b_m-b},\\
        &\geq \delta \norm{c_{m,i}} - \norm{b_m}-\norm{b}.
    \end{align*}
    Assuming $\norm{b_m}, \norm{b}$ are finite, since we can take $\delta\to \infty$ we must have $\norm{c_{m,i}}=0$ for all $i\in[d]$ if $\zeta < \infty$. This implies that $c_{m,i} = 0$ for all $i\in[d]$, or in other words $A_m = A$. Since this is true for all $m\in[M]$, the machines must have the same Hessians, and thus they can only differ upto linear terms.
\end{proof}
\chapter{Additional Details for Chapter 3}\label{app:chap3}
\section{Proof of \texorpdfstring{\Cref{lem:condition_number,thm:new_LSGD_lower_bound}}{TEXT}}
We will first prove \Cref{lem:condition_number}.
\begin{proof}
Let $A$ be the Hessian of $F$. Observe that we have $F(x) - F(x^\star) = \frac{1}{2}(x - x^\star)^TA(x - x^\star)$.

Let $v_1$ and $v_2$ be the eigenvectors of norm $1$ of $A$ with the greatest and least eigenvalues, respectively. Assume $x^\star := -B\left(\frac{v_1 + v_2}{\sqrt{2}}\right)$, which ensures $\|x^\star\|_2 = B$. Then, solving for the GD iterates in closed form, we have 
\begin{align*}
    x_R - x^\star &= x_{R-1} - x^\star - \eta A\rb{x_{R-1} - x^\star}\enspace,\\
    &= \rb{I - \eta A}\rb{x_{R-1} - x^\star}\enspace,\\
    &=^{\text{(a)}} \rb{v_1v_1^T + v_2v_2^T - \eta Hv_1v_1^T - \eta \mu v_2v_2^T}^R\rb{x_{0} - x^\star}\enspace,\\
    &= \rb{ (1- \eta H)v_1v_1^T + (1 - \eta \mu)v_2v_2^T}^R\rb{x_{0} - x^\star}\enspace,\\
    &= \rb{ (1- \eta H)^Rv_1v_1^T + (1 - \eta \mu)^Rv_2v_2^T}\rb{- x^\star}\enspace,\\
    &= \frac{B}{\sqrt{2}}\left(1 - \eta H\right)^Rv_1 + \frac{B}{\sqrt{2}}\left(1 - \eta \frac{H}{\kappa}\right)^Rv_2\enspace. 
\end{align*}
where in (a) we use the eigenvalue decomposition of $A = Hv_1v_1^T + \mu v_2v_2^T$ and the fact that for orthonormal vectors $v_1,\ v_2$ we have $I_2 = v_1v_1^T + v_2v_2^T$. 
Observe that if $\eta \geq \frac{3}{H}$, then the iterates explode and we have $F(x_R) \geq F(x_0) \geq \Omega\left(HB^2\right)$.

If $\eta \leq \frac{3}{H}$, then using the fact that $\kappa \geq 6$, we have
\begin{align*}
    F(x_R) - F(x^\star) &\geq^{\text{(a)}} \frac{1}{2}\left(\frac{B}{\sqrt{2}}\left(1 - \frac{3}{\kappa}\right)^Rv_2\right)^T A\left(\frac{B}{\sqrt{2}}\left(1 - \frac{3}{\kappa}\right)^Rv_2\right)\enspace,\\
    &= \frac{B^2}{4}\left(1 -  \frac{3}{\kappa}\right)^{2R}v_2^TAv_2 \enspace,\\
    &= \frac{B^2}{4}\left(1 - \frac{3}{\kappa}\right)^{2R}\frac{H}{\kappa}\enspace,\\
    &\geq^{\text{(b)}}\frac{HB^2}{4R}\rb{1-\frac{6R}{\kappa}}\enspace,
\end{align*}
where in (a) we lower bound by the function sub-optimality only in the second component corresponding to $v_2$; and in (b) we assume $\kappa\geq 3$ and Bernoulli's inequality. Finally using $\kappa = 12 R$ we get the lower bound $\frac{HB^2}{8R}$. The result follows.
\end{proof}
 
Now we are ready to prove the lower bound in \Cref{thm:new_LSGD_lower_bound}. 
\begin{proof}
    First we will see how to get the leading and most important term $\frac{HB^2}{R}$ in the lower bound. 
    
    We will consider a two-dimensional problem in the noiseless setting for this proof, as we do not want to understand the dependence on $\sigma$ or $d$. Define $A_1 := \left(\begin{array}{cc}
            1 & 0 \\
            0 & 0
        \end{array}\right)$ and $A_2:=vv^T$, where $v=(\alpha,\sqrt{1-\alpha^2})$ and $\alpha\in (0,1)$. For even $m$, let 
    \begin{equation*}
        F_m(x):= \frac{H}{2}(x-x^*)^TA_1(x-x^*)\enspace.
    \end{equation*}
    For odd $m$, let 
    \begin{equation*}
        F_m(x):= \frac{H}{2}(x-x^*)^TA_2(x-x^*)\enspace.
    \end{equation*}
    Note that $A_1$ and $A_2$ are rank-1 and have eigenvalues $0$ and $1$, and thus they satisfy Assumption \ref{ass:smooth_second}. Furthermore, both the functions have a shared optimizer $x^\star$. It is easy to verify that,
    \begin{equation*}
        (I-\eta HA_i)^K = I-(1-(1-\eta H)^K)A_i=: I-\widetilde{\eta}HA_i,
    \end{equation*}
    where $\widetilde{\eta}:=(1-(1-\eta H)^K) / H$. Note that the above property will be crucial for our construction, and we can not satisfy this property if our matrices are not ranked one. For any $x$, let us denote the centered iterate by $\tilde{x} := x - x^\star$. Then, for any $r\in[R]$  and $m\in[M]$ we have
    \begin{equation*}
        \tilde{x}^{m}_{r,K} = (I-\eta HA_m)^K \tilde{x}_{r-1} = (I-\widetilde{\eta}HA_m)\tilde{x}_{r-1}.
    \end{equation*}
    Using this, we can write the updates between two communication rounds as,
    \begin{align*}
        \tilde{x}_{r} &= \tilde{x}_{r-1} + \frac{\beta}{M}\sum_{m\in[M]}\rb{\tilde{x}_{r,K}^m - \tilde{x}_{r-1}},\\
        &= \tilde{x}_{r-1} - \frac{\beta}{M}\sum_{m\in[M]}\widetilde{\eta}HA_m\tilde{x}_{r-1},\\
        &= \rb{I - \beta \widetilde{\eta}H A}\tilde{x}_{r-1},\\
        &= \tilde{x}_{r-1} - \beta \widetilde{\eta}\nabla F(x_{r-1}),
    \end{align*}
    where we used that $F(x)=\frac{H}{2}(x-x^\star)^TA(x-x^\star)$, for $A = (1-a)A_1 + aA_2$ and 
    \begin{equation*}
        a := \left\{
            \begin{array}{ll}
                1/2 &  \text{if $M$ is even}, \\
                (M+1)/2M & \text{otherwise}.
            \end{array}
        \right.
    \end{equation*}
    This implies the iterates of local GD across communication rounds are equivalent to GD on $F(x)$ with step size $\beta (1-(1-\eta H)^K) / H$. Combining this observation with Lemma \ref{lem:condition_number} about the function sub-optimality of gradient descent updates will finish the proof. To use the lemma, however, we need to verify our average function $F$ has condition number $\Omega(R)$. We can explicitly compute the eigenvalues of $A$ as follows,
    \begin{align*}
        \lambda_1 &= \frac{1}{2} + \sqrt{\frac{1}{4}-(a-a^2)(1-\alpha^2)},\\ 
        \lambda_2 &= \frac{1}{2} - \sqrt{\frac{1}{4}-(a-a^2)(1-\alpha^2)}.
    \end{align*}
    Note that $\lim_{\alpha \rightarrow 1} \lambda_1 = 1$, and  $\lim_{\alpha \rightarrow 1} \lambda_2 = 0$ and thus $\lim_{\alpha \rightarrow 1} \lambda_1/\lambda_2 = \infty$. Since $\lambda_1/\lambda_2=1$ when $\alpha=0$, by the intermediate value theorem, we can choose $\alpha$ to get $\kappa=\Omega(R)$ for the average objective $F$. Thus, we can use Lemma \ref{lem:condition_number} and finish the proof.

    Now we will combine this lower bound of $HB^2/R$ with the previous hard instance of \citet{glasgow2022sharp} using \Cref{ass:phi_star}. To do so, we place the two instances on disjoint coordinates, increasing the dimensionality of our hard instance. This is a standard technique to combine lower bounds. We first recall the lower bound due to \citet{glasgow2022sharp} (up to numerical constant)\footnote{ \citet{glasgow2022sharp} do not state their lower bound in terms of \Cref{ass:phi_star}, but a weaker first-order heterogeneity. However their hard instance satisfies \Cref{ass:phi_star} and can thus be translated to our setting.},
    \begin{align*}
        \frac{HB^2}{KR} + \frac{\sigma_2 B}{\sqrt{MKR}} + \min\cb{\frac{\sigma_2 B}{\sqrt{KR}}, \frac{H^{1/2}\sigma_2^{2/3}B^{4/3}}{K^{1/3}R^{2/3}}} + \min\cb{\frac{\phi_\star^2}{H}, \frac{H^{1/3}\phi_\star^{2/3}B^{4/3}}{R^{2/3}}}\enspace.
    \end{align*}
    
    To get rid of the terms with the minimum function in the lower bound of \cite{glasgow2022sharp}, we note the following,
    \begin{itemize}
        \item $\frac{\sigma B}{\sqrt{KR}} \geq \frac{(H\sigma_2^2B^4)^{1/3}}{K^{1/3}R^{2/3}}$ implies that $\frac{\sigma_2 B}{\sqrt{KR}} \leq \frac{HB^2}{R}$, and
        \item $\frac{\zeta_\star^2}{H} \geq \frac{(H\zeta_\star^2B^4)^{1/3}}{R^{2/3}}$ implies that $\frac{\zeta_\star^2}{H} \leq \frac{HB^2}{R}$.
    \end{itemize}
    These observations allow us to avoid the minimum operations, thus concluding the proof of the theorem.
\end{proof}

\section{Proof of Theorem~\ref{thm:AIlb_zeta0}}
For  even $m$, let 
\begin{equation}
    F_m(x) := \frac{H}{2}\left((q^2 + 1)(q - x_1)^2 + \sum_{i = 1}^{\lfloor{(d - 1)/2}\rfloor} (qx_{2i} - x_{2i + 1})^2\right),
\end{equation}
and for odd $m$, let 
\begin{equation}
    F_m(x) =  \frac{H}{2}\left(\sum_{i = 1}^{\lfloor{d/2}\rfloor} (qx_{2i - 1} - x_{2i})^2\right).
\end{equation}
Thus we have
\begin{equation}
    F(x) = \mathbb{E}_m[F_m(x)] = \frac{H}{2}\left((q^2 + 1)(q - x_1)^2 + \sum_{i = 1}^{\d} (qx_{i} - x_{i + 1})^2\right).
\end{equation}

Observe that the optimum of $F$ is attained at $x^\star$, where $x^\star_i = q^i$. 
Theorem~\ref{thm:AIlb_zeta0} improves on the previous best lower bounds by introducing the term $\frac{HB^2}{R^2}$. Combining the following lemma with standard arguments to achieve the $\frac{\sigma B}{\sqrt{MKR}}$ suffices to prove Theorem~\ref{thm:AIlb_zeta0}.

\begin{lemma}\label{lem:AI}
For any $K \geq 2, R, M, H, B, \sigma$, there exist $f(x; xi)$ and distributions $\{\mathcal{D}_m\}$, each satisfying \Cref{ass:smooth_second,ass:bounded_optima,ass:stoch_bounded_second_moment}, and together satisfying $\frac{1}{M} \sum_{m=1}^M \|\nabla F_m(x^{\star})\|_2^2 = 0$, such that for initialization $x^{(0, 0)}=0$, the final iterate $\hat{x}$ of any zero-respecting with $R$ rounds of communication and $KR$ gradient computations per machine satisfies
\begin{align}
 \mathbb{E}\left[F(\hat{x})\right] - F(x^{\star}) \succeq \frac{HB^2}{R^2}.
\end{align}
\end{lemma}
\begin{proof}
    Consider the division of functions onto machines described above for some sufficiently large $d$.

    Let $q = 1 - \frac{1}{R}$, and let $t = \frac{1}{2}\log_q\left(\frac{B^2}{R}\right)$. We begin at the iterate $x_0$, where the coordinate $(x_0)_i = q^i$ for all $i < t$, and $(x_0)_i = 0$ for $i \geq t$. Observe that $\|x_0 - x^\star\|^2 \leq \sum_{i = t}^{\infty} q^{2i} \leq \frac{q^{2t}}{1 - q^2} \leq R q^{2t} \leq B^2$.
    
    Observe that for any zero-respecting algorithm, on odd machines, if for any $i$, we have $x_{2i}^m = x_{2i + 1}^m = 0$, then after any number of local computations, we still have $x_{2i + 1} = 0$. Similarly, on even machines, if for any $i$, we have $x_{2i - 1}^m = x_{2i}^m = 0$, then after any number of local computations, we still have $x_{2i} = 0$. 

    Thus, after $R$ rounds of communication, on all machines, we have $x_i^m = 0$, for all $i > t + R$. Thus for $d$ sufficiently large, we have $\|\hat{x} - x^\star\|^2 \geq \sum_{i = t +  R + 1}^d{q^{2i}} \geq \frac{q^{2t + 2R + 2} - q^{2d}}{1 - q^2} = \Omega\left(B^2 q^{2R + 2}\right) = \Omega(B^2)$ since $q = 1 - \frac{1}{R}$.

    Now observe that the Hessian of $F$ is a tridiagonal Toeplitz matrix with diagonal entries $H (q^2 + 1)$ and off-diagonal entries $-Hq$. It is well-known (see e.g., \cite{golub2005cme}) that the $d$ eigenvalues of $\tilde{M}$ are 
    $(1 + q^2)H + 2qH\cos\left(\frac{i \pi}{d + 1}\right)$ for $i = 1, \ldots, d$. Thus since $\cos(x) \geq -1$, we know that $F$ has strong-convexity parameter at least $H(q^2 + 1 - 2q) = \Omega\left(\frac{H}{R^2}\right) $, so we have $F(\hat{x}) - F(x^*) \geq \Omega\left(B^2\right)\Omega\left(\frac{H}{R^2}\right),$ which gives the desired result.
\end{proof}

\section{Proof of \texorpdfstring{\Cref{thm:new_LSGD_lower_bound_with_tau}}{TEXT}}

To prove the theorem, we will first show the following lemma.
\begin{lemma}\label{lem:just_tau_lb}
    There exists a convex quadratic function for $x\in\rr^3$ satisfying \Cref{ass:smooth_second,ass:bounded_optima,ass:tau}, such the Local SGD iterate $\bar x_{R}$, when initialized at zero and for any choice of step-sizes $\eta,\ \beta >0$ must have $F(\bar x_{R}) - F(x^\star) = \Omega\rb{\frac{\tau B^2}{R}}$.
\end{lemma}
\begin{proof}
We consider the quadratic functions defined by the following two Hessians for $\tau \leq H$,
\begin{align*}
    A_1 = \begin{bmatrix}
        \tau \hat A_1 & 0\\
        0 & H
    \end{bmatrix} \quad \text{and} \quad A_2 = \begin{bmatrix}
        \tau \hat A_2 & 0\\
        0 & H
    \end{bmatrix}\enspace,
\end{align*}
where we for some $\alpha\in(0,1)$,
\begin{align*}
    \hat A_1 &:= \begin{bmatrix}
        1 & 0\\
        0 & 0
    \end{bmatrix}\enspace, \quad \text{and}\\
    \hat A_2 &:= vv^T = (\alpha, \sqrt{1-\alpha^2})(\alpha, \sqrt{1-\alpha^2})^T = \begin{bmatrix}
        \alpha^2 & \alpha\sqrt{1-\alpha^2}\\
        \alpha\sqrt{1-\alpha^2} & 1- \alpha^2
    \end{bmatrix}\enspace. 
\end{align*}
Note about the spectrum of $A_1$,
\begin{align*}
    \text{Spec}\rb{A_1} = \cb{0, \tau, H}\enspace.
\end{align*}
Similarly for $A_2$ we note that,
\begin{align*}
    &\text{det}\rb{\hat A_2 - \lambda I_2} = 0\enspace,\\
    &\Rightarrow (\lambda - \alpha^2)(\lambda -1 + \alpha^2) = \alpha^2(1-\alpha^2)\enspace,\\
    &\Rightarrow \lambda^2 -\rb{\alpha^2 + 1 - \alpha^2}\lambda = 0\enspace,\\
    &\Rightarrow \lambda \in \cb{0,\ 1}\enspace.
\end{align*}
which implies that also for, 
\begin{align*}
    \text{Spec}\rb{A_2} = \cb{0, \tau, H}\enspace.
\end{align*}
Thus objectives defined by both these Hessians $A_1$ and $A_2$ are $H$-smooth. Further, we can notice the following about the difference between these Hessians,
\begin{align*}
    \text{Spec}\rb{A_1 - A_2} &= \tau\cdot\text{Spec}\rb{\hat A_1 - \hat A_2}\cup \cb{0}\enspace,\\
    &= \tau\cdot\text{Spec}\rb{\begin{bmatrix}
        1-\alpha^2 & -\alpha\sqrt{1-\alpha^2}\\
        -\alpha\sqrt{1-\alpha^2} & -(1-\alpha^2)
    \end{bmatrix}}\cup \cb{0}\enspace,\\
    &= \cb{-\tau\sqrt{1-\alpha^2}, 0, \tau\sqrt{1-\alpha^2}}\cup \cb{0}\enspace,
\end{align*}
which implies that,
\begin{align*}
    \norm{A_1 - A_2} = \tau \sqrt{1-\alpha^2} \leq \tau\enspace.
\end{align*}
Now we shall split the objectives on each machine as follows,
For even $m$, let 
\begin{equation*}
    F_m(x):= \frac{1}{2}(x-x^*)^TA_1(x-x^*)\enspace.
\end{equation*}
For odd $m$, let 
\begin{equation*}
    F_m(x):= \frac{1}{2}(x-x^*)^TA_2(x-x^*)\enspace.
\end{equation*}
Note that the iterates after $K$ local updates leading up to communication round $r$ on machine $m$ gives,
\begin{align*}
    \tilde x_{r,K}^m &= \rb{I-\eta A_i}^K\tilde x_{r-1}\enspace,
\end{align*}
where we denote $\tilde x_{r,k}^m = x_{r,k}^m - x^\star$ for all $k\in[0,K]$ and $\tilde x_r = \bar x_r-x^\star$ for all $r\in[0,R]$. For odd machines it is straightforward that,
\begin{align*}
    \rb{I-\eta A_1}^K = \begin{bmatrix}
       (1-\eta\tau)^K & 0 & 0 \\
       0 & 1 & 0\\
       0 & 0 & (1-\eta H)^K
    \end{bmatrix} = \begin{bmatrix}
       I_2 - \rb{\frac{1-(1-\eta\tau)^K}{\tau}}\tau \hat A_1 & 0 \\
       0 & (1-\eta H)^K
    \end{bmatrix} \enspace.
\end{align*}
For even machines, the above can also be noted using $v_{\perp}$ as the unit vector orthogonal to $v$,
\begin{align*}
    \rb{I-\eta A_2}^K &= \begin{bmatrix}
       I_2 - \eta \tau \hat A_2 & 0 \\
       0 & 1-\eta H
    \end{bmatrix}^K = \begin{bmatrix}
       (I_2 - \eta \tau vv^T)^K & 0 \\
       0 & (1-\eta H)^K
    \end{bmatrix}\enspace,\\
    &= \begin{bmatrix}
        \rb{vv^T + v_{\perp}v_{\perp}^T - \eta\tau vv^T}^K & 0\\
        0 & (1-\eta H)^K
    \end{bmatrix}\enspace,\\ 
    &=^{\text{(a)}} \begin{bmatrix}
        (1-\eta\tau)^Kvv^T + v_{\perp}v_{\perp}^T & 0\\
        0 & (1-\eta H)^K
    \end{bmatrix}\enspace,\\
    &= \begin{bmatrix}
        I_2 - vv^T + (1-\eta\tau)^Kvv^T & 0\\
        0 & (1-\eta H)^K
    \end{bmatrix}\enspace,\\
    &= \begin{bmatrix}
        I_2 - \rb{\frac{1-(1-\eta\tau)^K}{\tau}}\tau \hat A_2 & 0\\
        0 & (1-\eta H)^K
    \end{bmatrix}\enspace,
\end{align*}
where in (a) we note that $\rb{(1-\eta\tau)vv^T + v_{\perp}v_{\perp}^T}^2 = (1-\eta\tau)^2vv^T + v_{\perp}v_{\perp}^T$. This implies for the local updates with $\tilde\eta := \rb{\frac{1-(1-\eta\tau)^K}{\tau}}$ for all $m\in[M]$, 
\begin{align*}
    \tilde{x}_{r,K}^m = \begin{bmatrix}
        I_2 - \tilde\eta\tau \hat A_m & 0\\
        0 & (1-\eta H)^K
    \end{bmatrix}\tilde{x}_{r-1} = \begin{bmatrix}
        \rb{I_2 - \tilde \eta\tau\hat A_i}\tilde{x}_{r-1}[1:2]\\
        (1-\eta H)^K\tilde{x}_{r-1}[3]
    \end{bmatrix}\enspace.
\end{align*}
Now, using the calculations so far, we can write the updates between two communication rounds as,
\begin{align*}
    \tilde{x}_{r} &= \tilde{x}_{r-1} + \frac{\beta}{M}\sum_{m\in[M]}\rb{\tilde{x}_{r,K}^m - \tilde{x}_{r-1}}\enspace,\\
    &= \tilde{x}_{r-1} - \frac{\beta}{M}\sum_{m\in[M]}\begin{bmatrix}
        \tilde \eta \tau \hat A_{(m-1)\ \text{mod}(2) + 1}\tilde{x}_{r-1}[1:2]\\
        (1-(1-\eta H)^K)\tilde{x}_{r-1}[3]
    \end{bmatrix}\enspace,\\
    &= \begin{bmatrix}
        \rb{I_2 - \beta\tilde\eta A[1:2;1:2]}\tilde{x}_{r-1}[1:2]\\
        \rb{1- \beta\rb{1-(1-\eta H)^K}}\tilde{x}_{r-1}[3]
    \end{bmatrix}\enspace.
\end{align*}
The above calculation implies that the third coordinate evolves as synchronized gradient descent with $KR$ iterations, while the first two coordinates evolve with step size $\beta\tilde \eta$ and a hessian matrix of $A[1:2;1:2]$ (i.e., the top-left $2\times 2$ block of $A$ the average Hessian) for $R$ iterations\footnote{We don't need to restrict the step-sizes because \Cref{lem:condition_number} works for any step-size.}. Now note that $A[1:2; 1:2] = \tau(1-a)\hat A_1 + \tau a\hat A_2$ and 
\begin{equation*}
    a := \left\{
        \begin{array}{ll}
            1/2 &  \text{if $M$ is even}, \\
            (M+1)/2M & \text{otherwise}.
        \end{array}
    \right.
\end{equation*}
Now, all we need to do is apply Lemma \ref{lem:condition_number} to the first two dimensions. To be able to do so we need to be able to choose a condition number $\kappa = \Omega(R)$ for $A[1:2;1:2]$, in particular $\Omega(\kappa)$. Let us first consider the case with even machines, i.e., when $a=1/2$. Then note that,
\begin{align*}
    A[1:2;1:2] = \tau\frac{\hat A_1 + \hat A_2}{2} = \frac{\tau}{2}\begin{bmatrix}
        1+\alpha^2 & \alpha\sqrt{1-\alpha^2}\\
        \alpha\sqrt{1-\alpha^2} & 1-\alpha^2
    \end{bmatrix}\enspace,
\end{align*}
which implies for the spectrum of the matrix,
\begin{align*}
    \text{Spec}(A[1:2;1:2]) = \frac{\tau}{2}\cb{1-\alpha, 1+\alpha}\enspace,
\end{align*}
which in turn guarantees that,
\begin{align*}
    \kappa\rb{A[1:2;1:2]} = \frac{1+\alpha}{1-\alpha}\enspace,
\end{align*}
which can indeed be made $\Omega(R)$ by picking an $\alpha$ close enough to $1$. Now let us look at the case when $M$ is odd and $a= \frac{M+1}{2M}$,
\begin{align*}
    A[1:2;1:2] = \frac{\tau}{2M}\begin{bmatrix}
        M-1 + (M+1)\alpha^2 & (M+1)\alpha\sqrt{1-\alpha^2}\\
        (M+1)\alpha\sqrt{1-\alpha^2} & (M+1)\rb{1-\alpha^2}
    \end{bmatrix}\enspace,
\end{align*}
which using simple calculations as before implies for the spectrum of the matrix,
\begin{align*}
    \text{Spec}\rb{A[1:2;1:2]} = \frac{\tau}{2M}\cb{M-\sqrt{1-\alpha^2 + M^2\alpha^2}, M + \sqrt{1-\alpha^2 + M^2\alpha^2}}\enspace,
\end{align*}
which implies that,
\begin{align*}
    \kappa\rb{A[1:2;1:2]} = \frac{M + \sqrt{1-\alpha^2 + M^2\alpha^2}}{M - \sqrt{1-\alpha^2 + M^2\alpha^2}}\enspace,
\end{align*}
which can which can indeed be made $\Omega(R)$ by picking an $\alpha$ close enough to $1$. Finally this allows us to use \Cref{lem:condition_number} which implies that the progress on the first two coordinates is lower bounded by $\frac{\tau B^2}{R}$ for any choice of hyperparameters. To make this more explicit, note the following for any model $\hat x$,
\begin{align*}
    F(\hat x)  - F(x^\star) &= \frac{1}{2}x^TAx - (Ax^\star)^Tx\enspace,\\
    &= \frac{1}{2}x^TAx - (x^\star)^TAx\enspace,\\
    &= \frac{1}{2}x[1:2]^TA[1:2;1:2]x[1:2] - (x^\star[1:2])^TA[1:2;1:2]x[1:2]\\
    &\quad + \frac{H}{2}x[3]^2 - H x^\star[3]x[3]\enspace,\\
    &\geq \frac{1}{2}x[1:2]^TA[1:2;1:2]x[1:2] - (x^\star[1:2])^TA[1:2;1:2]x[1:2]\enspace,\\
    &=: F_{1:2}(\hat x)  - F_{1:2}(x^\star)\enspace,
\end{align*}
where we define a different quadratic objective $F_{1:2}:\rr^2\to\rr^2$ using the top left two-dimensional block of the Hessian $A$. This implies that we can lower bound the sub-optimality $F(x_R) - F(x^\star)$ by $\frac{\tau B^2}{R}$, which finishes the proof of the lemma.
\end{proof}

Now to conclude the proof of \Cref{thm:new_LSGD_lower_bound_with_tau}, we first note the following tight lower bound for the homogeneous setting due to Glasgow et al.~\cite{glasgow2022sharp} for the local SGD iterate $\bar x_R$, which we recall also uses a quadratic hard instance satisfying \Cref{ass:convex,ass:smooth_second,ass:stoch_bounded_second_moment,ass:bounded_optima},
\begin{align*}
    F(\bar x_{R}) - F(x^\star)  &= \Omega\Bigg( \frac{H B^2}{K R} + \frac{\sigma_2 B}{\sqrt{M K R}} + \min\left\{ \frac{\sigma_2 B}{\sqrt{K R}}, \frac{H^{1/3} \sigma_2^{2/3} B^{4/3}}{K^{1/3} R^{2/3}} \right\} \Bigg)\enspace.
\end{align*}
We also recall the heterogeneous lower bound due to Glasgow et al.~\cite{glasgow2022sharp} using a quadratic hard instance satisfying \Cref{ass:convex,ass:smooth_second,ass:stoch_bounded_second_moment,ass:bounded_optima,ass:phi_star}, and apply it on $\tau$-smooth problems (instead of $H$-smooth in their construction, as they do not decouple $\tau$ and $H$ in their construction),
\begin{align*}
        F(\bar x_{R}) - F(x^\star)  &= \Omega\Bigg(\min\left\{ \tau\phi_\star^2, \frac{\tau \phi_\star^{2/3} B^{4/3}}{R^{2/3}} \right\} \Bigg)\enspace.
\end{align*}
To translate their bound to our setting we also set $\zeta_\star$ (in their lower bound, not to be confused with our \Cref{ass:zeta_star}) as $\phi_\star$ to account for the different definitions of first-order heterogeneity in their paper and ours (c.f., \Cref{ass:phi_star}). Combining \Cref{lem:just_tau_lb} with the above two lower bounds from Glasgow et al.~\cite{glasgow2022sharp} by placing different hard instances on disjoint co-ordinates and noting the independent evolution in the gradient descent iterates (which is made explicit in \Cref{lem:just_tau_lb}) completes the proof of \Cref{thm:new_LSGD_lower_bound_with_tau}.
\chapter{Additional Details for Chapter 4}\label{app:chap4}
\section{Some Technical Lemmas}
\begin{lemma}\label{lemma: matrix telescope}
    Let $A$ and $B$ be two positive-semidefinite matrices. We have: 
    \begin{align*}
        A^k - B^k = \sum_{j=0}^{k-1} A^{k-1-j}(A-B)B^j
    \end{align*}
\end{lemma}
\begin{proof}
    we prove by induction. For $k=1$ we have: 
    \begin{align*}
        A - B = \sum_{j=0}^0 A^{-j} (A-B)B^j = A-B
    \end{align*}
    for $k+1$ we have: 
    \begin{align*}
        A^{k+1} - B^{k+1} = AA^k - BB^k = AA^k - AB^k + AB^k - BB^k = A(A^k-B^k) + (A-B)B^k
    \end{align*}
    for the first term in the above equality we have: 
    \begin{align*}
        A(A^k-B^k) = A \sum_{j=0}^{k-1} A^{k-1-j}(A-B)B^j = \sum_{j=0}^{k-1} A^{k-j}(A-B)B^j
    \end{align*}
    By adding the second term we have: 
    \begin{align*}
        A^{k+1} - B^{k+1} = \sum_{j=0}^{k-1} A^{k-j}(A-B)B^j + (A-B)B^k = \sum_{j=0}^{k} A^{k-j}(A-B)B^j
    \end{align*}
    which completes the proof. 
\end{proof}

\begin{lemma}\label{lem:g_non_increasing}
    Let $g(K) = \frac{1-(1-\eta H)^{K}}{1-(1-\eta \mu)^K}$, where $\eta < 1/H$ and $0<\mu \leq H$, then $g$ is a non-increasing function.
\end{lemma}
\begin{proof}
    To see this note for $k\in\zz_{\geq 1}$, while denoting $0< a := 1-\eta H \leq 1-\eta\mu =: b <1$,
\begin{align*}
    g(k) &= \frac{1-a^K}{1-b^K}\enspace,\\
    &= \frac{1-a}{1-b}\cdot \frac{1+a + \dots + a^{k-1}}{1 + b + \dots + b^{k-1}}\enspace,\\
    &=: \frac{1-a}{1-b}\cdot \frac{S_k(a)}{S_k(b)}\enspace,
\end{align*}
where we defined the geometric sum $S_k(\cdot)$ for ease of notation. Using this we get that,
\begin{align*}
    \rb{g(k) - g(k+1)}\frac{1-b}{1-a} &= \frac{S_k(a)}{S_k(b)} - \frac{S_{k+1}(a)}{S_{k+1}(b)}\enspace,\\
    &= \frac{S_k(a)}{S_k(b)} - \frac{S_{k}(a) + a^k}{S_{k}(b) +b^k}\enspace,\\
    &= \frac{S_k(a)(S_{k}(b) +b^k) - (S_{k}(a) + a^k)S_k(b)}{S_k(b)(S_{k}(b) +b^k)}\enspace,\\
    &= \frac{a^kb^k}{S_k(b)(S_{k}(b) +b^k)}\rb{\frac{S_k(a)}{a^k} - \frac{S_k(b)}{b^k}}\enspace,\\
    &= \frac{a^kb^k}{S_k(b)(S_{k}(b) +b^k)}\sum_{i=0}^{k-1}\rb{\frac{a^i}{a^k} - \frac{b^i}{b^k}}\enspace,\\
    &= \frac{a^kb^k}{S_k(b)(S_{k}(b) +b^k)}\sum_{i=0}^{k-1}\rb{a^{i-k} - b^{i-k}}\enspace,\\
    &\geq^{\text{(a<b)}} 0\enspace.
\end{align*}
Thus $g(\cdot)$ is a non-increasing function proving our earlier claim.
\end{proof}

We will need the following lemma about the Lipschitzness of a specific matrix polynomial.

\begin{lemma}\label{lem:R_Lipschitz}
Let $A_m,A_n\in\mathbb{R}^{d\times d}$ be symmetric positive-definite
matrices whose spectra lie inside the interval $[\mu ,H]\subset(0,1/\eta)$,
with $0<\mu\le H$ and $0<\eta<1/H$.
Fix an integer $K\ge 1$ and define the polynomial
$$
   R(\lambda)=1-\bigl(1-\eta\lambda\bigr)^{K}-\eta K\lambda,
   \qquad \lambda\in\mathbb{R}.
$$
Extend $R$ to symmetric matrices by functional calculus,
$R(X)=I-\bigl(I-\eta X\bigr)^{K}-\eta KX$.
Then
$$
   \bigl\|R(A_m)-R(A_n)\bigr\|_2
   \;\le\;
   L\,\bigl\|A_m-A_n\bigr\|_2,
   \qquad
   L=\eta K\!\left[1-(1-\eta H)^{K-1}\right].
$$
\end{lemma}
\begin{proof}

\noindent\emph{Step 1: A scalar Lipschitz constant.}
Direct differentiation gives
$$
   R'(\lambda)=\eta K\bigl[(1-\eta\lambda)^{K-1}-1\bigr],
$$
which is non-positive and increasing on $[\mu ,H]$.
Hence
$$
   L=\sup_{\lambda\in[\mu ,H]}|R'(\lambda)|
     =\eta K\!\left[1-(1-\eta H)^{K-1}\right].
$$

\medskip
\noindent\emph{Step 2: Fréchet derivative.}
Write $X=U\operatorname{diag}(\lambda_1,\dots,\lambda_d)U^{\top}$ and set
$F=U^{\top}EU$ for any symmetric perturbation $E$.
The Daleckii–Krein formula yields
$$
   D R[X](E)\;=\;U\!\left(M\odot F\right)U^{\top},
   \qquad
   M_{ij}=\frac{R(\lambda_i)-R(\lambda_j)}{\lambda_i-\lambda_j}.
$$
Because $-R$ is operator-monotone on $[\mu ,H]$, the matrix $M$ is
positive-semidefinite and its entries satisfy $|M_{ij}|\le L$.

\medskip
\noindent\emph{Step 3: Schur-multiplier estimate.}
If a PSD matrix $M$ has entries bounded by $L$, then for every
$G\in\mathbb{R}^{d\times d}$
$$
   \|M\odot G\|_2\;\le\;(\max_i M_{ii})\,\|G\|_2
   \;\le\;L\,\|G\|_2.
$$
Applying this with $G=F$ gives
$$
   \|D R[X](E)\|_2\;\le\;L\,\|E\|_2.
$$

\medskip
\noindent\emph{Step 4: Integration along a line segment.}
Set $\Delta:=A_m-A_n$ and $A(t):=A_n+t\Delta$ for $t\in[0,1]$.
Define $\Phi(t):=R\!\bigl(A(t)\bigr)$.
Step 3 implies $\|\Phi'(t)\|_2\le L\|\Delta\|_2$ for all $t$, so
$$
   \bigl\|R(A_m)-R(A_n)\bigr\|_2
   =\bigl\|\Phi(1)-\Phi(0)\bigr\|_2
   \le\int_0^1\|\Phi'(t)\|_2\,dt
   \le L\,\|\Delta\|_2.
$$
This is precisely the claimed bound.
\end{proof}


\begin{lemma}\label{lem:fact_C_m_A_m}
Let \( A_1, \dots, A_M \in \mathbb{R}^{d \times d} \) be symmetric positive semidefinite matrices, and let \( 1/H>\eta > 0 \) and \( K \in \mathbb{N} \). Define
\[
C_m := I - (I - \eta A_m)^K, \quad C := \frac{1}{M} \sum_{m=1}^M C_m \enspace.
\]
Suppose the kernel intersection is trivial:
\[
\bigcap_{m=1}^M \ker(A_m) = \{0\} \enspace.
\]
Assume further that \( \eta < 1/\lambda_{\max}(A_m) \) for all \( m \) (where \( \lambda_{\max}(A_m)\leq H \) denotes the largest eigenvalue of \( A_m \).

Then:
\begin{enumerate}
    \item For each \( m \), \( \ker(C_m) = \ker(A_m) \).
    \item The matrix \( C \) is full rank: \( \operatorname{im}(C) = \mathbb{R}^d \), i.e., \( \ker(C) = \{0\} \).
\end{enumerate}
\end{lemma}

\begin{proof}
\textbf{Part (1):} Since \( A_m \succeq 0 \), its eigenvalues lie in \( [0, \lambda_{\max}(A_m)] \). Then \( I - \eta A_m \) has eigenvalues in \( [1 - \eta \lambda_{\max}(A_m), 1] \subset (0, 1] \), so:
\[
C_m = I - (I - \eta A_m)^K = \sum_{j=1}^K \binom{K}{j} (-\eta A_m)^j \enspace,
\]
a matrix polynomial in \( A_m \). Due to the polynomial structure it is easy to see that,
\[
\ker(C_m) \supseteq \ker(A_m) \enspace.
\]
To see the other side note, we will prove the contrapositive. Suppose that $v\not\in \ker(A_m)$, but $v\in \ker(C_m)$, then
\begin{align*}
    C_mv = v - (I-\eta A_m)^Kv = 0\quad \Rightarrow \quad \norm{v} = \norm{(I-\eta A_m)^Kv} < \norm{v}\enspace,
\end{align*}
which is a contradiction. Thus $v\not\in \ker(A_m)$ implies that, $v\not\in \ker(C_m)$, or in other words,
\[
\ker(C_m) \subseteq \ker(A_m) \enspace.
\]
This proves the first part of the statement that $\ker(A_m) = \ker(C_m)$.

\textbf{Part (2):} Now suppose for contradiction that \( C v = 0 \) for some \( v \ne 0 \). Then:
\[
\sum_{m=1}^M C_m v = 0 \quad \Rightarrow \quad \langle C v, v \rangle = \frac{1}{M} \sum_{m=1}^M \langle C_m v, v \rangle = 0 \enspace.
\]
Since each \( C_m \succeq 0 \), it must be that \( \langle C_m v, v \rangle = 0 \Rightarrow C_m v = 0 \Rightarrow v \in \ker(C_m) = \ker(A_m) \) for all \( m \). So:
\[
v \in \bigcap_{m=1}^M \ker(A_m) = \{0\} \enspace,
\]
contradicting \( v \ne 0 \). Hence \( \ker(C) = \{0\} \), and since \( C \) is symmetric, \( \operatorname{im}(C) = \mathbb{R}^d \).
\end{proof}
\chapter{Additional Details for Chapter 5}\label{app:chap5}
This appendix is a self-contained guide on analyzing Local SGD using a bound on its consensus error. We will first begin by establishing the notation we will use throughout the appendix. 

\section{Notation and Outline of the Upper Bounds' Proofs}\label{app:outline}
Recall that the algorithm we would like to analyze is local SGD in the intermittent communication setting. In particular, we assume the algorithm runs over $R\in\nn$ communication rounds, with $K\in\nn$ local update steps between each communication round and total $T=KR$ time steps. We also assume we have $M\in\nn$ machines/clients/agents with each agent $m\in[M]$ sampling from their data distribution $\ddd_m \in \Delta(\zzz)$. These samples from the data distribution are used to calculate the stochastic gradients for each machine for each time step. In particular, at time $t\in[0,T]$ agent $m$ calculated $g_t^m := \nabla f(x_t^m; z_t^m)$ where $z_t^m\sim \ddd_m$. We recall the local SGD updates that use these stochastic gradients for all $t\in[0,T-1]$ and $m\in[M]$,
\begin{align*}
    x_{t+1}^m &:= x_t^m - \eta g_t^m && \text{if}\qquad t+1\mod K \neq 0\enspace,\\
    x_{t+1}^m &:= \frac{1}{M}\sum_{n\in[M]}\rb{x_t^n - \eta g_t^n } && \text{if}\qquad t+1\mod K = 0\enspace.
\end{align*}
We will often denote the stochastic noise by $\xi_{t}^m := \nabla F_m(x_t^m) - g_t^m$. We will also define the ``ghost iterate'' for all times $t\in[0,T]$ which may or may not be physically computed depending on the time $t$,
\begin{align*}
    x_t &:= \frac{1}{M}\sum_{m\in[M]}x_t^m\enspace.
\end{align*}
Considering these iterations, we will define several quantities in the analyses throughout the appendix. We include this notation in \Cref{tab:notation} for ease of reference. 
\begin{table}[H]
\centering
\renewcommand{\arraystretch}{1.5}
\setlength{\arrayrulewidth}{0.5mm}
\begin{tabular}{c l}
\toprule
\textbf{Symbol} & \textbf{Definition} \\ \midrule
$A(t)$ & $\ee\sb{\norm{x_t - x^\star}^2},\ \forall\ t\in[0,T]$\\ \midrule
$B(t)$ & $\ee\sb{\norm{x_t - x^\star}^4},\ \forall\ t\in[0,T]$\\ \midrule
$C(t)$ & $\frac{1}{M^2}\sum_{m,n\in[M]}\ee\sb{\norm{x_t^m - x_t^n}^2},\ \forall\ t\in[0,T]$\\ \midrule
$D(t)$ & $\frac{1}{M^2}\sum_{m,n\in[M]}\ee\sb{\norm{x_t^m - x_t^n}^4},\ \forall\ t\in[0,T]$\\ \midrule
$E(t)$ & $\ee\sb{F(x_t)} - \min_{x^\star\in \rr^d}F(x^\star),\ \forall\ t\in[0,T]$\\ \midrule
$\delta(t)$ & $t - t\mod{(K)},\ \forall\ t\in[0,T]$\\ \midrule
$g_t^m$ & $\nabla f(x_t^m; z_t^m),\ z_t^m\sim \ddd_m,\ \forall\ t\in[0,T],\ m\in[M]$\\ \midrule
$\xi_t^m$ & $\nabla F_m(x_t^m) - \nabla f(x_t^m; z_t^m),\ z_t^m\sim \ddd_m,\ \forall\ t\in[0,T],\ m\in[M]$\\ \midrule
$g_t$ & $g_t := \frac{1}{M}\sum_{m\in[M]}g_t^m,\ \forall\ t\in[0,T]$\\\midrule
$\xi_t$ & $\xi_t := \frac{1}{M}\sum_{m\in[M]} \xi_t^m,\ \forall\ t\in[0,T]$\\ \midrule
$\hhh_t$ & $\sigma\rb{\cb{z_0^m}_{m=1}^{M}, \dots, \cb{z_{t-1}^m}_{m=1}^{M}},\ \forall\ t\in[1,T]$\\
\bottomrule
\end{tabular}
\vspace{2mm}
\caption{Summary of the notation used in the appendix.}
\label{tab:notation}
\end{table}
With the above notation in mind, our analysis aims to provide upper bounds for $A(KR)$ and $E(KR)$ as a function of problem-dependent parameters that appear in all our assumptions. To do this: 
\begin{itemize}
    \item We will first state some technical lemmas in \Cref{app:technical}. 
    \item Then in \Cref{app:recursions} we state recursions across communication rounds for the sequences $A(\cdot)$, $B(\cdot)$, and $E(\cdot)$ in terms of the consensus error sequences $C(\cdot)$ and $D(\cdot)$. These recursions\footnote{We note that we are less explicit about randomness in the proof of these recursions and the following results. In particular, we often omit repetitive steps using the tower rule and conditional expectations to shorten the already complex proofs. We urge the reader to familiarize themselves with applying these techniques by first reading the proof of \Cref{lem:iterate_error_second_recursion}.} highlight the need to control the consensus error sequences $C(\cdot)$ and $D(\cdot)$. 
    \item In \Cref{app:zeta_results} we first control the consensus error by relying on the strongest \Cref{ass:zeta}. In the following sections, we relax this need for the $\zeta$ assumption and do a more fine-grained analysis of the consensus error.
    \item In \Cref{app:double_recursions} we provide more fine-grained recursions for $C(\cdot)$ and $D(\cdot)$, which depend on $A(\cdot)$ and $B(\cdot)$. These recursions are coupled and our main technical contribution is unrolling them carefully and simplying to provide new upper bounds. 
    \item \Cref{app:together} then brings together the results from \Cref{app:recursions} and \Cref{app:double_recursions} and provides convergence guarantees in terms of the step size $\eta$. Then we tune the step-size and obtain all the upper bounds from the main body of the thesis.
\end{itemize}

\section{Useful Technical Lemmas}\label{app:technical}
We will first prove some standard technical lemmas that are useful in the analysis of first-order algorithms.
\subsection{Simple Analytical Lemmas}

We will also use the following inequality several times, essentially a variant of the A.M.-G.M. inequality.
\begin{lemma}\label{lem:mod_am_gm}
    For any $a,b\in\rr$ and $\gamma >0$ we have,
    \begin{align*}
        (a+b)^2 &\leq \rb{1 + \frac{1}{\gamma}}a^2 + \rb{1+\gamma}b^2\enspace,\\
        (a+b)^4 &\leq \rb{1 + \frac{1}{\gamma}}^3a^4 + \rb{1+\gamma}^3b^4\enspace.
    \end{align*}
    \begin{proof}
        Note the following,
        \begin{align*}
            (a+b)^2 &= a^2 + b^2 + 2ab\enspace,\\
            &= a^2 + b^2 + 2\rb{\frac{a}{\sqrt{\gamma}}}\rb{\sqrt{\gamma} b}\enspace,\\
            &\leq^{\text{(A.M.-G.M. Inequality)}} a^2 + b^2 + \frac{a^2}{\gamma} + \gamma b^2\enspace,\\
            &\leq \rb{1 + \frac{1}{\gamma}}a^2 + \rb{1 + \gamma}b^2\enspace,
        \end{align*}
        which proves the first statement of the lemma. To get the second statement we will just apply the first statement twice as follows,
        \begin{align*}
            (a+b)^4 &\leq \rb{\rb{1 + \frac{1}{\gamma}}a^2 + \rb{1+\gamma}b^2}^2\enspace,\\
            &\leq \rb{1 + \frac{1}{\gamma}}\rb{\rb{1 + \frac{1}{\gamma}}a^2}^2 + \rb{1+\gamma}\rb{\rb{1+\gamma}b^2}^2\enspace,\\
            &=  \rb{1 + \frac{1}{\gamma}}^3a^4 + \rb{1+\gamma}^3b^4\enspace,
        \end{align*}
        which proves the second statement of the lemma.
    \end{proof}
\end{lemma}

\begin{lemma}\label{lem:mod_am_gm_three_terms}
    For any $a,b,c\in\rr$ we have,
    \begin{align*}
        (a+b+c)^2 &\leq 3a^2 + 3b^2 +3c^2\enspace,\\
        (a+b+c)^4 &\leq 27 a^4 + 27 b^4 + 27c^4\enspace.
    \end{align*}
\end{lemma}
\begin{proof}
    We note the following,
    \begin{align*}
        (a+b+c)^2 &= a^2 + b^2 + c^2 + 2ab + 2bc +2ca\enspace,\\
        &\leq^{\text{(A.M.-G.M. inequality)}} a^2 + b^2 + c^2 + (a^2 + b^2) + (b^2 + c^2) + (c^2 + a^2)\enspace,\\
        &= 3(a^2 + b^2 + c^2)\enspace,
    \end{align*}
    which proves the first statement. For the second statement using the first statement note the following,
    \begin{align*}
        (a+b+c)^4 &\leq \rb{3a^2 + 3b^2 +3c^2}^2\enspace,\\
        &\leq 3\rb{3a^2}^2 + 3\rb{3b^2}^2 + 3\rb{3c^2}^2\enspace,\\
        &= 27 a^4 + 27 b^4 + 27c^4\enspace,
    \end{align*}
    which proves the lemma.
\end{proof}

\begin{lemma}\label{lem:series_bound_one}
    Let $x\in(0,1)$ and $K>1$ then we have
    \begin{align*}
        \sum_{i=1}^{K-1}x^{i-1}i^2&\leq \frac{K}{(1-x)^2}\enspace.
    \end{align*}
\end{lemma}
\begin{proof}
    Note the following,
    \begin{align*}
        \sum_{i=1}^{K-1}x^{i-1}i^2 &\leq K\sum_{i=1}^{K-1}ix^{i-1}\enspace,\\
        &= K\nabla_x\rb{\sum_{i=1}^{K-1}x^{i}}\enspace,\\
        &=K\nabla_x\rb{x\frac{1-x^K}{1-x}}\enspace,\\
        &= K\frac{1-x^K}{1-x} + Kx\frac{1-Kx^{K-1} + (K-1)x^K}{(1-x)^2}\enspace,\\
        &= K\frac{1-x^K-x+x^{K+1}}{(1-x)^2} + K\frac{x-Kx^{K} + (K-1)x^{K+1}}{(1-x)^2}\enspace,\\
        &=K\frac{1-(K+1)x^K + Kx^{K+1}}{(1-x)^2}\enspace,\\
        &\leq \frac{K}{(1-x)^2}\enspace, 
    \end{align*}
    where in the last inequality we just note that $1-(K+1)x^K + Kx^{K+1}\leq 1$.
    This proves the lemma.
\end{proof}

\subsection{Useful Facts about Stochastic Noise}
Throughout this sub-section, we will assume \Cref{ass:stoch_bounded_second_moment}. Recall the following standard lemmas about the stochastic gradient noise,
\begin{lemma}[Averaged Stochastic Noise Second Moment]\label{lem:stoch_noise_second}
    For $t\in[0, T-1]$ we have,
    \begin{align*}
        \ee\sb{\norm{\xi_t}^2} &\leq \frac{\sigma_2^2}{M}\enspace.
    \end{align*}
\end{lemma}
\begin{remark}
    As the following proof will highlight we can also prove a higher upper bound of $\frac{\sigma_4^4}{M^3} + \frac{3(M-1)\sigma_2^4}{M^3}$ which can be much tighter when $\sigma_4 >> \sigma_2$. However, in this thesis we do not consider those regimes, and hence choose to upper bound $\sigma_2$ by $\sigma_4$. 
\end{remark}
\begin{proof}
    Recall that at any time step $t\in[0, T-1]$,
    \begin{align*}
        \ee\sb{\norm{\xi_t}^2} &= \ee\sb{\norm{\frac{1}{M}\sum_{m\in[M]}\xi_t^m}^2}\enspace,\\
        &=^{\text{(Tower rule)}} \ee\sb{\ee\sb{\norm{\frac{1}{M}\sum_{m\in[M]}\rb{g_t^m - \nabla f(x_t^m; z_t^m)}}^2|\hhh_t}}\enspace,\\
        &=^{\text{(a)}} \ee\sb{\frac{1}{M^2}\sum_{m\in[M]}\ee\sb{\norm{\rb{g_t^m - \nabla f(x_t^m; z_t^m)}}^2|\hhh_t}}\enspace,\\
        &\leq^{\text{(\Cref{ass:stoch_bounded_second_moment})}} \frac{1}{M^2}\sum_{m\in[M]}\sigma_2^2 = \frac{\sigma_2^2}{M}\enspace,
    \end{align*}
    where (a) uses the fact that for all $m\neq n$, $z_t^m \perp z_t^n|\hhh_t$, i.e., $\xi_t^1, \dots, \xi_t^M$ are independent conditioned on the history $\hhh_t$.
\end{proof}
We can also give the following stronger bound on the fourth moment of the stochastic noise.
\begin{lemma}[Averaged Stochastic Noise Fourth Moment]\label{lem:stoch_noise_fourth}
    For $t\in[0, T-1]$ we have,
    \begin{align*}
        \ee\sb{\norm{\xi_t}^4} &\leq \frac{3\sigma_4^4}{M^2}\enspace.
    \end{align*}
\end{lemma}
\begin{proof}
    Recall that at any time step $t\in[0, T-1]$,
    \begin{align*}
        \ee\sb{\norm{\xi_t}^4} &= \ee\sb{\norm{\frac{1}{M}\sum_{m\in[M]}\xi_t^m}^4}\enspace,\\
         &= \ee\sb{\rb{\norm{\frac{1}{M}\sum_{m\in[M]}\xi_t^m}^2}^2} = \ee\sb{\rb{\frac{1}{M^2}\sum_{m,n\in[M]}\inner{\xi_t^m}{\xi_t^n}}^2}\enspace,\\
         &= \frac{1}{M^4}\sum_{l,m,n,o\in[M]}\ee\sb{\inner{\xi_t^l}{\xi_t^m}\inner{\xi_t^n}{\xi_t^o}}\enspace,\\
         &=^{\text{(Tower Rule)}}\frac{1}{M^4}\sum_{l,m,n,o\in[M]}\ee\sb{\ee\sb{\inner{\xi_t^l}{\xi_t^m}\inner{\xi_t^n}{\xi_t^o}|\hhh_t}}\enspace,
    \end{align*}
    Recall that for all $m\neq n$, $z_t^m \perp z_t^n|\hhh_t$, i.e., $\xi_t^1, \dots, \xi_t^M$ are independent conditioned on the history $\hhh_t$. In the above sum, the only non-zero terms are the ones where either $l=m=n=o$, or where the set $\cb{l,m,n,o}$ has two distinct values, each repeated twice. There are $M$ terms of the first kind, and $3M(M-1)$ terms of the second kind (first choose two colours out of $M$, then choose two indices out of $\cb{l,m,n,o}$ which divides into two groups, i.e., total $\frac{M(M-1)}{2}\times \frac{4!}{2!2!}$). Using this we get,
    \begin{align*}
        \ee\sb{\norm{\xi_t}^4} &= \frac{1}{M^4}\rb{\sum_{l\in[M]}\ee\sb{\norm{\xi_t^l}^4} + 3\sum_{l\neq m\in[M]}\ee\sb{\norm{\xi_t^l}^2}\ee\sb{\norm{\xi_t^m}^2}}\enspace,\\
        &\leq^{\text{(\Cref{ass:stoch_bounded_second_moment,ass:stoch_bounded_fourth_moment})}} \frac{1}{M^4}\rb{M\sigma_4^4 + 3M(M-1)\sigma_2^4}\enspace,\\
        &\leq \frac{3\sigma_4^4}{M^2}\enspace,
    \end{align*}
    where we use that $\ee\sb{\norm{x}^2}\leq \sqrt{\ee\sb{\norm{x}^4}}$ due to Jensen's inequality to upper bound $\sigma_2$ by $\sigma_4$. This proves the lemma.
\end{proof}

\begin{lemma}[Averaged Stochastic Noise Third Moment]\label{lem:stoch_noise_third}
    For $t\in[0, T-1]$ we have,
    \begin{align*}
        \ee\sb{\norm{\xi_t}^3} &\leq \frac{\sqrt{3}\sigma_4^2\sigma_2}{M^{3/2}}\enspace.
    \end{align*}
\end{lemma}
\begin{proof}
    This result follows from simply noting the previous two lemmas, and the fact that,
    \begin{align*}
        \ee\sb{\norm{\xi_t}^3} &= \ee\sb{\norm{\xi_t}^2\norm{\xi_t}}\leq^{\text{(Cauchy Shwartz)}} \sqrt{\ee\sb{\norm{\xi_t}^4}}\sqrt{\ee\sb{\norm{\xi_t}^2}}\enspace,\\
        &\leq^{\text{(\Cref{lem:stoch_noise_second,lem:stoch_noise_fourth})}} \sqrt{\frac{\sigma_2^2}{M}}\sqrt{\frac{3\sigma_4^4}{M^2}} \leq \frac{\sqrt{3}\sigma_4^2\sigma_2}{M^{3/2}}\enspace,
    \end{align*}
    which proves the lemma. 
\end{proof}

We can also note the following about the difference of the stochastic noise on two machines.
\begin{lemma}[Second Moment of Difference]\label{lem:stoch_diff_second}
    For $t\in[0,T-1]$ and for $m\neq n\in[M]$ we have,
    \begin{align*}
        \ee\sb{\norm{\xi_t^m - \xi_t^n}^2} &\leq 2\sigma_2^2\enspace.
    \end{align*}
\end{lemma}
\begin{proof}
    Note the following for $m\neq n\in[M]$, and for $t\in[0,T-1]$
    \begin{align*}
        \ee\sb{\norm{\xi_t^m - \xi_t^n}^2} &= \ee\sb{\norm{\xi_t^m}^2 + \norm{\xi_t^m}^2 + 2\inner{\xi_t^m}{\xi_t^n}}\enspace,\\
        &=^\text{(a), (Tower Rule)} \ee\sb{\norm{\xi_t^m}^2} + \ee\sb{\norm{\xi_t^n}^2} + 2\ee\sb{\inner{\ee\sb{\xi_t^m|\hhh_t}}{\ee\sb{\xi_t^n|\hhh_t}}}\enspace,\\
        &\leq^{\text{(\Cref{ass:stoch_bounded_second_moment}), (b)}} 2\sigma_2^2\enspace,
    \end{align*}
    where in (a) we used that $\xi_t^m\perp\xi_t^n|\hhh_t$; and in (b) we used that $\ee\sb{\xi_t^m|\hhh_t} = \ee\sb{\xi_t^n|\hhh_t} =0$. This proves the lemma.
\end{proof}

\begin{lemma}[Fourth Moment of Difference]\label{lem:stoch_diff_fourth}
    For $t\in[0,T-1]$ and for $m\neq n\in[M]$ we have,
    \begin{align*}
        \ee\sb{\norm{\xi_t^m - \xi_t^n}^4} &\leq 8\sigma_4^4\enspace.
    \end{align*}
\end{lemma}
\begin{proof}
    Note the following for $m\neq n\in[M]$, and for $t\in[0,T-1]$
    \begin{align*}
        \ee\sb{\norm{\xi_t^m - \xi_t^n}^4} &= \ee\sb{\rb{\norm{\xi_t^m}^2 + \norm{\xi_t^m}^2 + 2\inner{\xi_t^m}{\xi_t^n}}^2}\enspace,\\
        &=^{\text{(a)}} \ee\sb{\norm{\xi_t^m}^4} + \ee\sb{\norm{\xi_t^n}^4} +4\ee\sb{\rb{\inner{\xi_t^m}{\xi_t^n}}^2}\\
        &\quad + 2\ee\sb{\norm{\xi_t^m}^2}\ee\sb{\norm{\xi_t^n}^2} + 2\ee\sb{\norm{\xi_t^m}^2\xi_t^m}^T\cancelto{0}{\ee\sb{\xi_t^n}}\\
        &\quad + 2\ee\sb{\norm{\xi_t^n}^2\xi_t^n}^T\cancelto{0}{\ee\sb{\xi_t^m}}\enspace,\\
        &\leq^{\text{(Cauchy Shwartz)}} \ee\sb{\norm{\xi_t^m}^4} + \ee\sb{\norm{\xi_t^n}^4} + 6\ee\sb{\norm{\xi_t^m}^2}\ee\sb{\norm{\xi_t^n}^2}\enspace,\\
        &\leq^{\text{(\Cref{ass:stoch_bounded_second_moment,ass:stoch_bounded_fourth_moment})}} 2\sigma_4^4 + 6\sigma_2^4\enspace,\\
        &\leq 8\sigma_4^4\enspace,
    \end{align*}
    where in (a) we used that $\xi_t^m\perp\xi_t^n|\hhh_t$ along with tower rule several times like in previous lemmas.  This finishes the proof.
\end{proof}

\section{Deriving Round-wise Recursions for Errors}\label{app:recursions}
In this section, we derive several recursions that prove useful later in the analysis and form the core of our proof. An informed reader would note that the ideas and in some cases the entire recursions occur in previous works \citep{patel2024limits,woodworth2020minibatch,koloskova2020unified,stich2018local}. 
\subsection{Second Moment of the Error in Iterates}
The main result of this sub-section is the following result, which relates $A(\cdot)$ to $C(\cdot)$ and $D(\cdot)$. 
\begin{lemma}\label{lem:iterate_error_second_recursion}
    Assume we have a problem instance satisfying \Cref{ass:strongly_convex,ass:smooth_second,ass:smooth_third,ass:stoch_bounded_second_moment,ass:tau}. Then assuming $\eta < \frac{1}{H}$ we have for all $t\in[0,T-1]$,
    \begin{align*}
        A(t+1) \leq \rb{1-\eta\mu}A(t) + \frac{\eta}{\mu}\cdot\min\cb{\textcolor{red}{2Q^2D(t) + 2\tau^2C(t)}, \textcolor{blue}{H^2C(t)}}+ \frac{\eta^2\sigma_2^2}{M}\enspace.
    \end{align*}
    This also implies that for all $r\in[R]$,
    \begin{align*}
        A(Kr) &\leq \rb{1-\eta\mu}^KA(K(r-1))+ \rb{1-\rb{1-\eta\mu}^K}\frac{\eta\sigma^2}{\mu M}\\
        &\quad + \frac{\eta }{\mu}\sum_{j=K(r-1)}^{Kr-1}(1-\eta\mu)^{Kr-1-j} \min\cb{\textcolor{red}{2Q^2D(j) + 2\tau^2C(j)}, \textcolor{blue}{H^2C(j)}} \enspace.
    \end{align*}
\end{lemma}
\begin{remark}
    Note that the above lemma implies that if $Q$ and $\tau$ are both zero---i.e., we have a quadratic problem with no second-order heterogeneity---then we will achieve extreme communication efficiency, matching the convergence rate of mini-batch SGD, with $KR$ communication rounds. As such, the trade-off between the \textcolor{red}{red} and \textcolor{blue}{blue} upper bounds is that the former allows us to exploit higher-order assumptions, but we need to be able to bound the fourth moment of the consensus error. In contrast, the latter only requires a bound on the second moment of the consensus error.
\end{remark}
\begin{proof}
    We note the following about the progress made in a single iteration for $t\in[0,T-1]$,
    \begin{align*}
        &A(t+1) = \ee\sb{\norm{x_{t+1} - x^\star}^2}\enspace,\\
        &=^{\text{(Tower rule)}} \ee\sb{\ee\sb{\norm{x_t - x^\star - \frac{\eta}{M}\sum_{m\in[M]}g_t^m}^2\bigg|\hhh_t}}\enspace,\\
        &=^{\text{(a)}} \ee\sb{\norm{x_t - x^\star - \frac{\eta}{M}\sum_{m\in[M]}\nabla F_m(x_t^m)}^2} + \eta^2\ee\sb{\norm{\frac{1}{M}\sum_{m\in[M]}\xi_t^m}^2}\enspace,\\
        &\leq^{\text{(\Cref{lem:stoch_noise_second})}} \ee\sb{\norm{x_t - \eta \nabla F(x_t) -x^\star + \eta \nabla F(x_t) - \frac{\eta}{M}\sum_{m\in[M]}\nabla F_m(x_t^m)}^2} + \frac{\eta^2\sigma^2}{M}\enspace,\\
        &\leq^\text{(\Cref{lem:mod_am_gm}), (b)} \rb{1 + \frac{\eta\mu}{1-\eta\mu}}\rb{1-\eta\mu}^2\ee\sb{\norm{x_t - x^\star}^2}\\
        &\quad + \rb{1 + \frac{1-\eta\mu}{\eta\mu}}\eta^2\ee\sb{\norm{\frac{1}{M}\sum_{m\in[M]}\rb{\nabla F_m(x_t) - \nabla F_m(x_t^m)}}^2}+ \frac{\eta^2\sigma_2^2}{M}\enspace,\\
        &= \rb{1-\eta\mu}\ee\sb{\norm{x_t - x^\star}^2} + \frac{\eta}{\mu}\cdot\ee\sb{\norm{\frac{1}{M}\sum_{m\in[M]}\rb{\nabla F_m(x_t) - \nabla F_m(x_t^m)}}^2}+ \frac{\eta^2\sigma_2^2}{M}\enspace,\\
        &\leq^{\text{(c)}} \rb{1-\eta\mu}\ee\sb{\norm{x_t - x^\star}^2} + \frac{\eta}{\mu}\cdot\ee\sb{\norm{\frac{1}{M}\sum_{m\in[M]}\nabla^2 F_m(\hat x_t^m)(x_t - x_t^m)}^2}+ \frac{\eta^2\sigma_2^2}{M}\enspace,
        \end{align*}
        where in (a) we used the fact that $\xi_{t}^1, \dots, \xi_{t}^1\in m\hhh_t$ i.e., they are measurable/``non-random'' under $\hhh_t$ and zero-mean, which allows us to ignore the cross-terms while squaring; in (b) we use the fact that $\eta < 1/H$ which implies that $0\preceq I-\eta\nabla^2 F(\cdot) \preceq (1-\eta\mu)\cdot I_d$ and also that $(1-\eta\mu)>0$; and in (c) we note that due to the mean-value theorem there exists some $\hat x_{t}^m$ which is a convex combination of $x_t^m$ and $x_t$. From this point, we can proceed in two different ways. First, to get the \textcolor{blue}{blue} upper bound we just use smoothness as follows,   
        \begin{align*}
        A(t+1) &\leq \rb{1-\eta\mu}\ee\sb{\norm{x_t - x^\star}^2} + \textcolor{blue}{\frac{\eta}{\mu}\cdot\ee\sb{\norm{\frac{1}{M}\sum_{m\in[M]}\nabla^2 F_m(\hat x_t^m)(x_t - x_t^m)}^2}}+ \frac{\eta^2\sigma_2^2}{M}\enspace,\\
        &\leq^{\text{(Jensen's Inequality)}} \rb{1-\eta\mu}\ee\sb{\norm{x_t - x^\star}^2} + \textcolor{blue}{\frac{\eta}{\mu}\cdot\frac{1}{M}\sum_{m\in[M]}\ee\sb{\norm{\nabla^2 F_m(\hat x_t^m)(x_t - x_t^m)}^2}}\\
        &\quad + \frac{\eta^2\sigma_2^2}{M}\enspace,\\
        &\leq \rb{1-\eta\mu}\ee\sb{\norm{x_t - x^\star}^2} + \textcolor{blue}{\frac{\eta H^2}{\mu}\cdot\frac{1}{M}\sum_{m\in[M]}\ee\sb{\norm{x_t^m-x_t}^2}}+ \frac{\eta^2\sigma_2^2}{M}\enspace,\\
        &\leq^{\text{(Jensen's Inequality)}} \rb{1-\eta\mu}\ee\sb{\norm{x_t - x^\star}^2} + \textcolor{blue}{\frac{\eta H^2}{\mu}\cdot\frac{1}{M^2}\sum_{m,n\in[M]}\ee\sb{\norm{x_t^m-x_t^n}^2}}+ \frac{\eta^2\sigma_2^2}{M}\enspace,\\
        &= \rb{1-\eta\mu}A(t) + \textcolor{blue}{\frac{\eta H^2}{\mu}C(t)}+ \frac{\eta^2\sigma_2^2}{M}\enspace,
        \end{align*}
        which proves one part of the lemma. To get the \textcolor{red}{red} upper bound, we will use second-order heterogeneity and third-order smoothness as follows,
        \begin{align*}
        A(t+1)&\leq \rb{1-\eta\mu}\ee\sb{\norm{x_t - x^\star}^2} + \textcolor{red}{\frac{\eta}{\mu}\cdot\ee\sb{\norm{\frac{1}{M}\sum_{m\in[M]}\nabla^2 F_m(\hat x_t^m)(x_t - x_t^m)}^2}}+ \frac{\eta^2\sigma_2^2}{M}\enspace,\\
        &= \rb{1-\eta\mu}\ee\sb{\norm{x_t - x^\star}^2}\\
        &\quad +\textcolor{red}{\frac{\eta}{\mu}\cdot\ee\sb{\norm{\frac{1}{M}\sum_{m\in[M]}\rb{\nabla^2 F_m(\hat x_t^m) -\nabla^2 F_m(x_t) + \nabla^2 F_m(x_t) - \nabla^2 F(x_t)}(x_t - x_t^m)}^2}}\\
        &\quad + \frac{\eta^2\sigma_2^2}{M}\enspace,\\
        &\leq^{\text{(Jensen's Inequality), (\Cref{lem:mod_am_gm})}} \rb{1-\eta\mu}\ee\sb{\norm{x_t - x^\star}^2}\\
        &\quad + \textcolor{red}{\frac{2\eta}{\mu M}\sum_{m\in[M]}\ee\sb{\norm{\rb{\nabla^2 F_m(\hat x_t^m) -\nabla^2 F_m(x_t)}(x_t - x_t^m)}^2}}\\
        &\quad + \textcolor{red}{\frac{2\eta}{\mu M}\sum_{m\in[M]}\ee\sb{\norm{\rb{\nabla^2 F(x_t) - \nabla^2 F_m(x_t)}(x_t - x_t^m)}^2}}+ \frac{\eta^2\sigma_2^2}{M}\enspace,\\
        &\leq^{\text{(\Cref{ass:smooth_third,ass:tau})}} \rb{1-\eta\mu}\ee\sb{\norm{x_t - x^\star}^2} + \textcolor{red}{\frac{2\eta Q^2}{\mu M}\sum_{m\in[M]}\ee\sb{\norm{\hat x_t^m - x_t}^2\norm{x_t - x_t^m}^2}}\\
        &\quad + \textcolor{red}{\frac{2\eta \tau^2}{\mu M}\sum_{m\in[M]}\ee\sb{\norm{x_t - x_t^m}^2}}+ \frac{\eta^2\sigma_2^2}{M}\enspace,\\
        &\leq^{\text{(a)}}\rb{1-\eta\mu}\ee\sb{\norm{x_t - x^\star}^2} + \textcolor{red}{\frac{2\eta Q^2}{ \mu M}\sum_{m\in[M]}\ee\sb{\norm{x_t - x_t^m}^4}} + \textcolor{red}{\frac{2\eta \tau^2}{\mu M}\sum_{m\in[M]}\ee\sb{\norm{x_t - x_t^m}^2}}\\
        &\quad + \frac{\eta^2\sigma_2^2}{M}\enspace,\\
        &\leq^{\text{(Jensen's Inequality)}}\rb{1-\eta\mu}\ee\sb{\norm{x_t - x^\star}^2} + \textcolor{red}{\frac{2\eta Q^2}{ \mu M^2}\sum_{m,n\in[M]}\ee\sb{\norm{x_t^n - x_t^m}^4}}\\
        &\quad + \textcolor{red}{\frac{2\eta \tau^2}{\mu M^2}\sum_{m,n\in[M]}\ee\sb{\norm{x_t^n - x_t^m}^2}} + \frac{\eta^2\sigma^2}{M}\enspace,\\
        &= \rb{1-\eta\mu}A(t) + \textcolor{red}{\frac{2\eta Q^2D(t)}{\mu}} + \textcolor{red}{\frac{2\eta \tau^2C(t)}{\mu}} + \frac{\eta^2\sigma_2^2}{M}\enspace,
    \end{align*}
    where in (a) we use that $\norm{\hat x_t^m - x_t} \leq \norm{x_t^m - x_t}$ for all $m\in[M]$. This proves the second part of the upper bound, thus finishing the proof for the first statement of the lemma. Note that for $r\in [R]$ we can re-write this result as follows,
    \begin{align*}
        A(Kr) &\leq \rb{1-\eta\mu}A(Kr-1) + \frac{\eta}{\mu}\cdot\min\cb{\textcolor{red}{2Q^2D(Kr-1) + 2\tau^2C(Kr-1)}, \textcolor{blue}{H^2C(Kr-1)}} + \frac{\eta^2\sigma_2^2}{M}\enspace,\\
        &\leq \rb{1-\eta\mu}^KA(K(r-1)) + \frac{\eta }{\mu}\sum_{j=K(r-1)}^{Kr-1}(1-\eta\mu)^{Kr-1-j} \min\cb{\textcolor{red}{2Q^2D(j) + 2\tau^2C(j)}, \textcolor{blue}{H^2C(j)}} + \frac{\eta\sigma_2^2}{\mu M}\enspace,
    \end{align*}
    where in the second inequality we just unrolled the recursion till the time-step of the previous communication. This finishes the proof of the lemma.
\end{proof}

It would also be helpful to state the following lemma, which talks about the convergence on individual machines between two communication rounds. 
\begin{lemma}[Single Machine SGD Second Moment]\label{lem:single_second_recursion}
    Assume we have a problem instance satisfying \Cref{ass:strongly_convex,ass:stoch_bounded_second_moment}. Then for any machine $m\in[M]$, for $t\in[0,T]$, and for $k\geq0$ we have the following for $\eta< \frac{1}{H}$,
    \begin{align*}
        \ee\sb{\norm{x_{\delta(t)+k}^m - x^\star}^2} &\leq (1-\eta\mu)^{2k}\ee\sb{\norm{x_{\delta(t)} - x_m^\star}^2} + \rb{1 - (1-\eta\mu)^{2k}}\cdot\frac{\eta\sigma_2^2}{\mu}\enspace.
    \end{align*}
\end{lemma}
The above lemma follows the usual strongly convex analysis of SGD (see, for instance, \cite{stich2019unified}), since we can rely on that between two communication rounds. 

\subsection{Fourth Moment of the Error in Iterates}\label{app:fourth_error_iterate}
It would also be useful to state the following recursion on the fourth moment of the iterates, as the recursion would appear in the analysis of the fourth moment of the consensus error. 
\begin{lemma}\label{lem:iterate_error_fourth_recursion}
    Assume we have a problem instance satisfying \Cref{ass:strongly_convex,ass:smooth_second,ass:smooth_third,ass:stoch_bounded_second_moment,ass:stoch_bounded_fourth_moment,ass:tau}. Then assuming $\eta < \frac{1}{H}$ we have for all $t\in[0,T-1]$,
    \begin{align*}
        B(t+1) &\leq (1-\eta\mu)B(t) + \rb{\frac{\eta H^4}{\mu^3} + \frac{16\eta^3\sigma_2^2Q^2}{\mu M}}D(t) + \frac{8\eta^2\sigma_2^2(1-\eta\mu)}{M}A(t) + \frac{16\eta^3\sigma_2^2\tau^2}{\mu M}C(t) + \frac{9\eta^4\sigma_4^4}{M^2}\enspace.
    \end{align*}
    This also implies that for $r\in[R]$ we have,
    \begin{align*}
        B(Kr) &\leq (1-\eta\mu)^KB(K(r-1)) + \rb{\frac{\eta H^4}{\mu^3} + \frac{16\eta^3\sigma_2^2Q^2}{\mu M}}\sum_{j=K(r-1)}^{Kr-1}(1-\eta\mu)^{Kr-1-j}D(j)\\
        &\quad + \frac{8\eta^2\sigma_2^2}{M}\sum_{j=K(r-1)}^{Kr-1}(1-\eta\mu)^{Kr-j}A(j) + \frac{16\eta^3\sigma_2^2\tau^2}{\mu M}\sum_{j=K(r-1)}^{Kr-1}(1-\eta\mu)^{Kr-1-j}C(j) + \frac{9\eta^3\sigma_4^4}{\mu M^2}\enspace.
    \end{align*}
\end{lemma}
\begin{proof}
    For $t\in[0,T-1]$ we note the following,
    \begin{align*}
        &\ee\sb{\norm{x_{t+1} - x^\star}^4}\\
        &= \ee\sb{\norm{x_t - x^\star - \frac{\eta}{M}\sum_{m\in[M]}\nabla F_m(x_t^m) + \frac{\eta}{M}\sum_{m\in[M]}\xi_{t}^m}^4}\enspace,\\
        &= \ee\Bigg[\bigg(\norm{x_t - x^\star - \frac{\eta}{M}\sum_{m\in[M]}\nabla F_m(x_t^m)}^2 + \norm{\frac{\eta}{M}\sum_{m\in[M]}\xi_{t}^m}^2\\
        &\qquad\qquad\qquad\qquad + 2\eta\inner{x_t - x^\star - \frac{\eta}{M}\sum_{m\in[M]}\nabla F_m(x_t^m)}{\frac{1}{M}\sum_{m\in[M]}\xi_{t}^m}\bigg)^2\Bigg]\enspace,\\
        &= \ee\sb{\norm{x_t - x^\star - \frac{\eta}{M}\sum_{m\in[M]}\nabla F_m(x_t^m)}^4} + \eta^4\ee\sb{\norm{\frac{1}{M}\sum_{m\in[M]}\xi_{t}^m}^4}\\
        &\quad + 4\eta^2\ee\sb{\rb{\inner{x_t - x^\star - \frac{\eta}{M}\sum_{m\in[M]}\nabla F_m(x_t^m)}{\frac{1}{M}\sum_{m\in[M]}\xi_{t}^m}}^2}\\
        &\quad + 2\eta^2\ee\sb{\norm{x_t - x^\star - \frac{\eta}{M}\sum_{m\in[M]}\nabla F_m(x_t^m)}^2\norm{\frac{1}{M}\sum_{m\in[M]}\xi_{t}^m}^2}\\
        &\quad + 4\eta\ee\sb{\norm{x_t - x^\star - \frac{\eta}{M}\sum_{m\in[M]}\nabla F_m(x_t^m)}^2\rb{x_t - x^\star - \frac{\eta}{M}\sum_{m\in[M]}\nabla F_m(x_t^m)}}^T \cancelto{0}{\ee\sb{\frac{1}{M}\sum_{m\in[M]}\xi_{t}^m}}\\
        &\quad + 4\eta^3\ee\sb{\norm{x_t - x^\star - \frac{\eta}{M}\sum_{m\in[M]}\nabla F_m(x_t^m)}\norm{\frac{1}{M}\sum_{m\in[M]}\xi_{t}^m}^3}\enspace,\\
        &\leq^{\text{(Cauchy Shwartz)}} \ee\sb{\norm{x_t - x^\star - \frac{\eta}{M}\sum_{m\in[M]}\nabla F_m(x_t^m)}^4} + \eta^4\ee\sb{\norm{\frac{1}{M}\sum_{m\in[M]}\xi_{t}^m}^4}\\
        &\quad + 4\eta^2\ee\sb{\rb{\norm{x_t - x^\star - \frac{\eta}{M}\sum_{m\in[M]}\nabla F_m(x_t^m)}\norm{\frac{1}{M}\sum_{m\in[M]}\xi_{t}^m}}^2}\\
        &\quad + 2\eta^2\ee\sb{\norm{x_t - x^\star - \frac{\eta}{M}\sum_{m\in[M]}\nabla F_m(x_t^m)}^2\norm{\frac{1}{M}\sum_{m\in[M]}\xi_{t}^m}^2}\\
        &\quad + 4\eta^3\ee\sb{\norm{x_t - x^\star - \frac{\eta}{M}\sum_{m\in[M]}\nabla F_m(x_t^m)}\norm{\frac{1}{M}\sum_{m\in[M]}\xi_{t}^m}^3}\enspace,\\
        &=^{\text{(Tower Rule)}} \ee\sb{\norm{x_t - x^\star - \frac{\eta}{M}\sum_{m\in[M]}\nabla F_m(x_t^m)}^4} + \eta^4\ee\sb{\norm{\frac{1}{M}\sum_{m\in[M]}\xi_{t}^m}^4}\\
        &\quad + 6\eta^2\ee\sb{\ee\sb{\norm{x_t - x^\star - \frac{\eta}{M}\sum_{m\in[M]}\nabla F_m(x_t^m)}^2\norm{\frac{1}{M}\sum_{m\in[M]}\xi_{t}^m}^2\Bigg|\hhh_t}}\\
        &\quad + 4\eta^3\ee\sb{\ee\sb{\norm{x_t - x^\star - \frac{\eta}{M}\sum_{m\in[M]}\nabla F_m(x_t^m)}\norm{\frac{1}{M}\sum_{m\in[M]}\xi_{t}^m}^3\Bigg|\hhh_t}}\enspace,\\
        &=^{\text{(a)}} \ee\sb{\norm{x_t - x^\star - \frac{\eta}{M}\sum_{m\in[M]}\nabla F_m(x_t^m)}^4} + \eta^4\ee\sb{\norm{\frac{1}{M}\sum_{m\in[M]}\xi_{t}^m}^4}\\
        &\quad + 6\eta^2\ee\sb{\norm{x_t - x^\star - \frac{\eta}{M}\sum_{m\in[M]}\nabla F_m(x_t^m)}^2\ee\sb{\norm{\frac{1}{M}\sum_{m\in[M]}\xi_{t}^m}^2\Bigg|\hhh_t}}\\
        &\quad + 4\eta^3\ee\sb{\norm{x_t - x^\star - \frac{\eta}{M}\sum_{m\in[M]}\nabla F_m(x_t^m)}\ee\sb{\norm{\frac{1}{M}\sum_{m\in[M]}\xi_{t}^m}^3\Bigg|\hhh_t}}\enspace,\\
    &\leq^{\text{(\Cref{lem:stoch_noise_second,lem:stoch_noise_fourth,lem:stoch_noise_third}), (b)}} \ee\sb{\norm{x_t - x^\star - \frac{\eta}{M}\sum_{m\in[M]}\nabla F_m(x_t^m)}^4} + \frac{3\eta^4\sigma_4^4}{M^2}\\ 
        &\quad + \frac{6\eta^2\sigma_2^2}{M}\ee\sb{\norm{x_t - x^\star - \frac{\eta}{M}\sum_{m\in[M]}\nabla F_m(x_t^m)}^2}+ \frac{4\sqrt{3}\eta^3\sigma_4^2\sigma_2}{M^{3/2}}\sqrt{\ee\sb{\norm{x_t - x^\star - \frac{\eta}{M}\sum_{m\in[M]}\nabla F_m(x_t^m)}^2}}\enspace,\\
        &= \ee\sb{\norm{x_t - x^\star - \frac{\eta}{M}\sum_{m\in[M]}\nabla F_m(x_t^m)}^4} + \frac{3\eta^4\sigma^4}{M^2}+ \frac{6\eta^2\sigma_2^2}{M}\ee\sb{\norm{x_t - x^\star - \frac{\eta}{M}\sum_{m\in[M]}\nabla F_m(x_t^m)}^2}\\
        &\quad + 4\sqrt{\rb{\frac{3\eta^4\sigma_4^4}{M^2}}\rb{\frac{\eta^2\sigma_2^2}{M}\ee\sb{\norm{x_t - x^\star - \frac{\eta}{M}\sum_{m\in[M]}\nabla F_m(x_t^m)}^2}}}\enspace,\\
        &\leq^{\text{(A.M.-G.M. Inequality)}} \ee\sb{\norm{x_t - x^\star - \frac{\eta}{M}\sum_{m\in[M]}\nabla F_m(x_t^m)}^4} + \frac{9\eta^4\sigma_4^4}{M^2}\\ 
        &\quad + \frac{8\eta^2\sigma_2^2}{M}\ee\sb{\norm{x_t - x^\star - \frac{\eta}{M}\sum_{m\in[M]}\nabla F_m(x_t^m)}^2}\enspace,\\
        &= \ee\sb{\norm{x_t - x^\star - \eta \nabla F(x_t) + \eta \nabla F(x_t) -\frac{\eta}{M}\sum_{m\in[M]}\nabla F_m(x_t^m)}^4} + \frac{9\eta^4\sigma_4^4}{M^2}\\ 
        &\quad + \frac{8\eta^2\sigma_4^2}{M}\ee\sb{\norm{x_t - x^\star - \eta \nabla F(x_t) + \eta \nabla F(x_t) - \frac{\eta}{M}\sum_{m\in[M]}\nabla F_m(x_t^m)}^2}\enspace,\\
        &=^{\text{(c)}} \ee\sb{\norm{\rb{I-\eta\nabla^2F(\hat x_t)}\rb{x_t - x^\star} + \eta \nabla F(x_t) -\frac{\eta}{M}\sum_{m\in[M]}\nabla F_m(x_t^m)}^4} + \frac{9\eta^4\sigma_4^4}{M^2}\\ 
        &\quad + \frac{8\eta^2\sigma_2^2}{M}\ee\sb{\norm{\rb{I-\eta\nabla^2F(\hat x_t)}\rb{x_t - x^\star} + \eta \nabla F(x_t) - \frac{\eta}{M}\sum_{m\in[M]}\nabla F_m(x_t^m)}^2}\enspace,\\
        &\leq^{\text{(\Cref{lem:mod_am_gm}), (d)}} \rb{1 + \frac{\eta\mu}{1-\eta\mu}}^3(1-\eta\mu)^4\ee\sb{\norm{x_t - x^\star}^4}\\ 
        &\quad + \rb{1 + \frac{1-\eta\mu}{\eta\mu}}^3\eta^4\ee\sb{\norm{\frac{1}{M}\sum_{m\in[M]}\rb{\nabla F_m(x_t)- \nabla F_m(x_t^m)}}^4} + \frac{9\eta^4\sigma_4^4}{M^2}\\ 
        &\quad + \frac{8\eta^2\sigma_2^2}{M}\rb{1 + \frac{\eta\mu}{1-\eta\mu}}(1-\eta\mu)^2\ee\sb{\norm{x_t - x^\star}^2}\\
        &\quad + \frac{8\eta^2\sigma_2^2}{M}\rb{1 + \frac{1-\eta\mu}{\eta\mu}}\eta^2\ee\sb{\norm{\frac{1}{M}\sum_{m\in[M]}\rb{\nabla F_m(x_t)- \nabla F_m(x_t^m)}}^2}\enspace,\\
        &\leq^{(\text{Jensen's Inequality)}} (1-\eta\mu)\ee\sb{\norm{x_t - x^\star}^4} + \frac{\eta}{\mu^3M}\sum_{m\in[M]}\ee\sb{\norm{\rb{\nabla F_m(x_t)- \nabla F_m(x_t^m)}}^4} + \frac{9\eta^4\sigma_4^4}{M^2}\\ 
        &\quad + \frac{8\eta^2\sigma_2^2(1-\eta\mu)}{M}\ee\sb{\norm{x_t-x^\star}^2} + \frac{8\eta^3\sigma_2^2}{\mu M}\ee\sb{\norm{\frac{1}{M}\sum_{m\in[M]}\rb{\nabla F_m(x_t)- \nabla F_m(x_t^m)}}^2}\enspace,\\
        &\leq^{\text{(\Cref{ass:smooth_second}), (d)}} (1-\eta\mu)B(t) + \frac{\eta H^4}{\mu^3}D(t) + \frac{8\eta^2\sigma_2^2(1-\eta\mu)}{M}A(t) + \frac{9\eta^4\sigma_4^4}{M^2}\\
        &\quad + \frac{8\eta^3\sigma_2^2}{\mu M}\ee\sb{\norm{\frac{1}{M}\sum_{m\in[M]}\rb{\nabla^2 F_m(\hat x_t^m) -\nabla^2 F_m(x_t) + \nabla^2 F_m(x_t) - \nabla^2 F(x_t)}(x_t - x_t^m)}^2}\enspace,\\
        &\leq^{\text{(\Cref{ass:smooth_third,ass:tau}), (e)}} (1-\eta\mu)B(t) + \frac{\eta H^4}{\mu^3}D(t) + \frac{8\eta^2\sigma_2^2(1-\eta\mu)}{M}A(t) + \frac{9\eta_4^4\sigma^4}{M^2}\\
        &\quad + \frac{8\eta^3\sigma_2^2}{\mu M}\rb{\frac{2Q^2}{M}\sum_{m\in[M]}\ee\sb{\norm{x_t-x_t^m}^4} + \frac{2\tau^2}{M}\sum_{m\in[M]}\ee\sb{\norm{x_t-x_t^m}^2}}\enspace,\\
        &\leq^{\text{(Jensen's Inequality)}} (1-\eta\mu)B(t) + \frac{\eta H^4}{\mu^3}D(t) + \frac{8\eta^2\sigma_2^2(1-\eta\mu)}{M}A(t)\\
        &\quad + \frac{9\eta^4\sigma_4^4}{M^2}+ \frac{8\eta^3\sigma_2^2}{\mu M}\rb{2Q^2D(t) + 2\tau^2C(t)}\enspace,\\
        &= (1-\eta\mu)B(t) + \rb{\frac{\eta H^4}{\mu^3} + \frac{16\eta^3\sigma_2^2Q^2}{\mu M}}D(t) + \frac{8\eta^2\sigma_2^2(1-\eta\mu)}{M}A(t) + \frac{16\eta^3\sigma_2^2\tau^2}{\mu M}C(t) + \frac{9\eta^4\sigma_4^4}{M^2}\enspace,
    \end{align*}
    where in (a) we used the fact that $x_t - x^\star - \frac{\eta}{M}\sum_{m\in[M]}\nabla F_m(x_t^m)\in m\hhh_t$; in (b) we used the Jensen's inequality $\ee\sb{\norm{y}}\leq \sqrt{\ee\sb{\norm{y}^2}}$; in (c) we use mean value theorem to conclude that there exists some $\hat x_t$ which is a convex combination of $x_t$ and $x^\star$ such that $\nabla F(x_t) = \nabla F(x^\star) + \nabla^2 F(\hat{x}_t)\rb{x_t - x_\star}$; in (d) we apply mean value theorem to find a $\hat x_t^m$ which is a convex combination of $x_t$ and $x_t^m$ such that $\nabla F_m(x_t) = \nabla F_m(x_t^m) + \nabla^2 F_m(\hat x_t^m)\cdot\rb{x_t-x_t^m}$; and in (e) we used the fact that $\norm{\hat x_t^m - x_t} \leq \norm{x_t^m - x_t}$. This finishes the proof of the first statement of the lemma. By letting $t+1=Kr$ for some $r\in[R]$, and unrolling till the previous communication round we get,
    \begin{align*}
        B(Kr) &\leq (1-\eta\mu)B(Kr-1) + \rb{\frac{\eta H^4}{\mu^3} + \frac{16\eta^3\sigma_2^2Q^2}{\mu M}}D(Kr-1)\\
        &\quad + \frac{8\eta^2\sigma_2^2(1-\eta\mu)}{M}A(Kr-1) + \frac{16\eta^3\sigma_2^2\tau^2}{\mu M}C(Kr-1) + \frac{9\eta^4\sigma_4^4}{M^2}\enspace,\\
        &\leq (1-\eta\mu)^KB(K(r-1)) + \rb{\frac{\eta H^4}{\mu^3} + \frac{16\eta^3\sigma_2^2Q^2}{\mu M}}\sum_{j=K(r-1)}^{Kr-1}(1-\eta\mu)^{Kr-1-j}D(j)\\
        &\quad + \frac{8\eta^2\sigma_2^2}{M}\sum_{j=K(r-1)}^{Kr-1}(1-\eta\mu)^{Kr-j}A(j) + \frac{16\eta^3\sigma_2^2\tau^2}{\mu M}\sum_{j=K(r-1)}^{Kr-1}(1-\eta\mu)^{Kr-1-j}C(j) + \frac{9\eta^3\sigma_4^4}{\mu M^2}\enspace,
    \end{align*}
    where in the last inequality for the last term we used that $\sum_{j=K(r-1)}^{Kr-1}(1-\eta\mu)^{Kr-1-j} \leq \frac{1- (1-\eta\mu)^K}{\eta\mu}\leq \frac{1}{\eta\mu}$. This proves the second statement of the lemma. 
\end{proof}

It would also be helpful to state the following lemma, which talks about the convergence on individual machines between two communication rounds measuring the fourth moment of the error. 
\begin{lemma}[Single Machine SGD Fourth Moment]\label{lem:single_fourth_recursion}
    For any machine $m\in[M]$, for $t\in[0,T]$, and for $k\geq0$ we have the following for $\eta< \frac{1}{H}$,
    \begin{align*}
        \ee\sb{\norm{x_{\delta(t)+k}^m - x^\star}^4} &\leq (1-\eta\mu)^{4k}\ee\sb{\norm{x_{\delta(t)} - x_m^\star}^4} + 8\eta^2\sigma_2^2k(1-\eta\mu)^{2k}\ee\sb{\norm{x_{\delta(t)} - x_m^\star}^2} + \frac{11\eta^2\sigma_4^4}{\mu^2}\enspace.
    \end{align*}
    We can also get the following simpler bound,
    \begin{align*}
        \ee\sb{\norm{x_{\delta(t)+k}^m - x^\star}^4} &\leq (1-\eta\mu)^{3k}\ee\sb{\norm{x_{\delta(t)} - x_m^\star}^4}+ \frac{16\eta\sigma_4^4}{\mu^3}\enspace.
    \end{align*}
\end{lemma}
\begin{proof}
    For any machine $m\in[M]$ note the following for $t\geq\delta(t)$,
    \begin{align*}
        &\ee\sb{\norm{x_{t+1}^m - x_m^\star}^4}\\
        &= \ee\sb{\norm{x_t^m - x_m^\star - \eta \nabla F_m(x_t^m)  + \eta \xi_{t}^m}^4}\enspace,\\
        &= \ee\sb{\rb{\norm{x_t^m - x_m^\star - \eta \nabla F_m(x_t^m)}^2 + \eta^2\norm{\xi_{t}^m}^2 + 2\eta\inner{x_t^m - x_m^\star - \eta \nabla F_t^m(x_t^m)}{\xi_{t}^m}}^2}\enspace,\\
        &\leq \ee\sb{\norm{x_t^m - x_m^\star - \eta \nabla F_m(x_t^m)}^4} + \eta^4\ee\sb{\norm{\xi_{t}^m}^4} + 4\eta^2\ee\sb{\norm{x_t^m - x_m^\star - \eta \nabla F_m(x_t^m)}^2}\ee\sb{\norm{\xi_{t}^m}^2}\\
        &\quad + 2\eta^2\ee\sb{\norm{x_t^m - x_m^\star - \eta \nabla F_m(x_t^m)}^2}\ee\sb{\norm{\xi_{t}^m}^2} + 4\eta^3\ee\sb{\norm{x_t^m - x_m^\star - \eta \nabla F_m(x_t^m)}}\ee\sb{\norm{\xi_{t}^m}^3}\\
        &\quad + 4\eta\ee\sb{\norm{x_t^m - x_m^\star - \eta \nabla F_m(x_t^m)}^2(x_t^m - x_m^\star - \eta \nabla F_m(x_t^m))^T}\cancelto{0}{\ee\sb{\xi_{t}^m}}\enspace,\\
        &\leq^{\text(a)} \ee\sb{\norm{x_t^m - x_m^\star - \eta \nabla F_m(x_t^m)}^4} + \eta^4\sigma_4^4 + 6\eta^2\sigma_2^2\ee\sb{\norm{x_t^m - x_m^\star - \eta \nabla F_m(x_t^m)}^2}\\
        &\quad + 4\eta^3\sigma_4^2\sigma_2\ee\sb{\norm{x_t^m - x_m^\star - \eta \nabla F_m(x_t^m)}}\enspace,\\
        &\leq^{\text{(Jensen's Inequality)}} \ee\sb{\norm{x_t^m - x_m^\star - \eta \nabla F_m(x_t^m)}^4}  + 6\eta^2\sigma_2^2\ee\sb{\norm{x_t^m - x_m^\star - \eta \nabla F_m(x_t^m)}^2}\\
        &\quad + 4\eta^3\sigma_4^2\sigma_2\sqrt{\ee\sb{\norm{x_t^m - x_m^\star - \eta \nabla F_m(x_t^m)}^2}} + \eta^4\sigma_4^4\enspace,\\
        &\leq^{\text{(A.M.-G.M. Inequality)}} \ee\sb{\norm{x_t^m - x_m^\star - \eta \nabla F_m(x_t^m)}^4}  + 6\eta^2\sigma_2^2\ee\sb{\norm{x_t^m - x_m^\star - \eta \nabla F_m(x_t^m)}^2}\\
        &\quad + 4\eta^3\rb{\frac{\eta\sigma_4^4}{2} + \frac{\sigma_2^2}{2\eta}\ee\sb{\norm{x_t^m - x_m^\star - \eta \nabla F_m(x_t^m)}}^2} + \eta^4\sigma_4^4\enspace,\\
        &= \ee\sb{\norm{x_t^m - x_m^\star - \eta \nabla F_m(x_t^m)}^4} + 3\eta^4\sigma_4^4 + 8\eta^2\sigma_2^2\ee\sb{\norm{x_t^m - x_m^\star - \eta \nabla F_m(x_t^m)}^2},\\
        &\leq^{\text{(b)}} \textcolor{red}{(1-\eta\mu)^4\ee\sb{\norm{x_t^m - x_m^\star}^4} + 3\eta^4\sigma_4^4 + 8\eta^2\sigma_2^2(1-\eta\mu)^2\ee\sb{\norm{x_t^m - x_m^\star}^2}}\enspace,\\
        &\leq^{\text{(\Cref{lem:single_second_recursion})}} (1-\eta\mu)^4\ee\sb{\norm{x_t^m - x_m^\star}^4} + 3\eta^4\sigma_4^4\\
        &\quad + 8\eta^2\sigma_2^2(1-\eta\mu)^2\rb{(1-\eta\mu)^{2(t-\delta(t))}\ee\sb{\norm{x_{\delta(t)} - x_m^\star}^2} + \frac{\eta\sigma_2^2}{\mu}}\enspace,\\
        &= (1-\eta\mu)^4\ee\sb{\norm{x_t^m - x_m^\star}^4} + 3\eta^4\sigma_4^4 + 8\eta^2\sigma_2^2(1-\eta\mu)^{2(t+1-\delta(t))}\ee\sb{\norm{x_{\delta(t)} - x_m^\star}^2} + \frac{8\eta^3\sigma_2^4(1-\eta\mu)^2}{\mu} \enspace,\\
        &\leq^{\text{(c)}} (1-\eta\mu)^4\ee\sb{\norm{x_t^m - x_m^\star}^4} + 8\eta^2\sigma_2^2(1-\eta\mu)^{2(t+1-\delta(t))}\ee\sb{\norm{x_{\delta(t)} - x_m^\star}^2} + \frac{11\eta^3\sigma_4^4}{\mu} \enspace,\\
        &= (1-\eta\mu)^{4(t+1-\delta(t))}\ee\sb{\norm{x_{\delta(t)}^m - x_m^\star}^4} + \frac{11\eta^2\sigma_4^4}{\mu^2}\\
        &\quad + 8\eta^2\sigma_2^2\sum_{j=\delta(t)}^{t}(1-\eta\mu)^{4(t-j)}(1-\eta\mu)^{2(j+1-\delta(t))}\ee\sb{\norm{x_{\delta(t)} - x_m^\star}^2}\enspace,\\
        &\leq (1-\eta\mu)^{4(t+1-\delta(t))}\ee\sb{\norm{x_{\delta(t)} - x_m^\star}^4} + 8\eta^2\sigma_2^2\ee\sb{\norm{x_{\delta(t)} - x_m^\star}^2}\sum_{j=\delta(t)}^{t}(1-\eta\mu)^{2(t+1-\delta(t))} + \frac{11\eta^2\sigma_4^4}{\mu^2}\enspace,\\
        &\leq (1-\eta\mu)^{4(t+1-\delta(t))}\ee\sb{\norm{x_{\delta(t)} - x_m^\star}^4}\\
        &\quad + 8\eta^2\sigma_2^2\rb{t+1-\delta(t)}(1-\eta\mu)^{2(t+1-\delta(t))}\ee\sb{\norm{x_{\delta(t)} - x_m^\star}^2} + \frac{11\eta^2\sigma_4^4}{\mu^2}\enspace,
    \end{align*}
    where in (a) we used the fact that $\ee\sb{\norm{\xi_t^m}^3}\leq \sqrt{\ee\sb{\norm{\xi_t^m}^2}\ee\sb{\norm{\xi_t^m}^4}}$ and \Cref{ass:stoch_bounded_second_moment,ass:stoch_bounded_fourth_moment}; in (b) we used that $\|x_t^m - \nabla F_m(x_t^m) - x_m^\star\| \leq (1-\eta\mu)\norm{x_t^m - x_m^\star}$ for $\eta < \frac{1}{H}$; in (c) we use that $\eta < \frac{1}{H} \leq \frac{1}{\mu}$ which implies that $3\eta^4\sigma_4^4 \leq \frac{3\eta^3\sigma_4^4}{\mu}$ and that $\sigma_2\leq \sigma_4$. We gave the above analysis for $t+1>\delta(t)$, thus it can be translated for $k>0$ as follows,
    \begin{align*}
        \ee\sb{\norm{x_{\delta(t)+k}^m - x_m^\star}^4} &\leq (1-\eta\mu)^{4k}\ee\sb{\norm{x_{\delta(t)} - x_m^\star}^4} + 8\eta^2\sigma_2^2k(1-\eta\mu)^{2k}\ee\sb{\norm{x_{\delta(t)} - x_m^\star}^2} + \frac{11\eta^2\sigma_4^4}{\mu^2}\enspace.
    \end{align*}
    Since when $k=0$, this is still a valid upper bound, we have proven the first lemma statement. To get the simpler upper bound, we will complete the square, proceeding from the red term in the above analysis as follows,
    \begin{align*}
        \ee\sb{\norm{x_{t+1}^m - x_m^\star}^4} &\leq (1-\eta\mu)^4\ee\sb{\norm{x_t^m - x_m^\star}^4} + 3\eta^4\sigma_4^4 + 8\eta^2\sigma_2^2(1-\eta\mu)^2\ee\sb{\norm{x_t^m - x_m^\star}^2}\enspace,\\
        &\leq^{\text{(Jensen's Inequality), (a)}} (1-\eta\mu)^4\ee\sb{\norm{x_t^m - x_m^\star}^4} + 16\eta^4\sigma_4^4\\
        &\quad + 8\eta^2\sigma_4^2(1-\eta\mu)^2\sqrt{\ee\sb{\norm{x_t^m - x_m^\star}^4}}\enspace,\\
        &= \rb{(1-\eta\mu)^2\sqrt{\ee\sb{\norm{x_t^m - x_m^\star}^4}} + 4\eta^2\sigma_4^2}^2\enspace,
    \end{align*}
    where in (a) we used $\sigma_2\leq \sigma_4$. Taking the square root of both sides, we get,
    \begin{align*}
        \sqrt{\ee\sb{\norm{x_{t+1}^m - x_m^\star}^4}} &\leq (1-\eta\mu)^2\sqrt{\ee\sb{\norm{x_t^m - x_m^\star}^4}} + 4\eta^2\sigma_4^2\enspace,\\
        &\leq (1-\eta\mu)^{2(t+1-\delta(t))}\sqrt{\ee\sb{\norm{x_{\delta(t)}^m - x_m^\star}^4}} + \frac{4\eta\sigma_4^2}{\mu}\enspace.
    \end{align*}
    Finally using \Cref{lem:mod_am_gm} and taking a whole square we get,
    \begin{align*}
        \ee\sb{\norm{x_{t+1}^m - x_m^\star}^4} &\leq \rb{1 + \frac{\eta\mu}{1-\eta\mu}}(1-\eta\mu)^{4(t+1-\delta(t))}\ee\sb{\norm{x_{\delta(t)}^m - x_m^\star}^4} + \rb{1 + \frac{1-\eta\mu}{\eta\mu}}\frac{16\eta^2\sigma_4^4}{\mu^2}\enspace,\\
        &\leq (1-\eta\mu)^{3(t+1-\delta(t))}\ee\sb{\norm{x_{\delta(t)}^m - x_m^\star}^4} + \frac{16\eta\sigma_4^4}{\mu^3}\enspace. 
    \end{align*}
    We proved this for $t+1>\delta(t)$, but clearly it also holds when $t+1=\delta(t)$, which implies that for all $k\geq 0$,
    \begin{align*}
        \ee\sb{\norm{x_{\delta(t)+k}^m - x^\star}^4} &\leq (1-\eta\mu)^{3k}\ee\sb{\norm{x_{\delta(t)} - x_m^\star}^4}+ \frac{16\eta\sigma_4^4}{\mu^3}\enspace,
    \end{align*}
    thus finishing the proof of the lemma.
\end{proof}

\subsection{Function Value Error}\label{app:function_value}
The main result of this sub-section relates $E(\cdot)$ to $C(\cdot)$ and $D(\cdot)$.
\begin{lemma}[Section D.4, \citet{patel2024limits}]\label{lem:function_error_hard}
    Assume we have a problem instance satisfying \Cref{ass:strongly_convex,ass:smooth_second,ass:smooth_third,ass:stoch_bounded_second_moment,ass:tau}. Then assuming $\eta \leq \frac{1}{2H}$ we have for all $t\in[0,T-1]$,
    \begin{align*}
        E(t) \leq \rb{\frac{1}{\eta} - \frac{\mu}{2}} \ee \sb{\norm{x_t-x^\star}^2} - \frac{1}{\eta}\ee \sb{\norm{x_{t+1}-x^\star}^2}+ \frac{6\tau^2}{\mu}C(t) + \frac{6Q^2 }{\mu}D(t) + \frac{\eta \sigma_2^2}{M} \enspace .
    \end{align*}
\end{lemma}
\begin{proof}
The proof follows a similar approach to \cite{yuan2020federated}. Starting with the distance from the optimal point and taking the conditional expectation on the previous iterate $x^m_t, \forall m \in [M]$, we have: 
\begin{align}
    &\mathbb{E} \sb{\norm{ x_{t+1} - x^{\star}}^2|\hhh_t} \nonumber\\
    &= \mathbb{E} \sb{\norm{ x_{t} - x^{\star} - \frac{\eta}{M} \sum_{m=1} ^{M} \nabla F_m(x^m_t) + \frac{\eta}{M} \sum_{m=1} ^{M} \nabla F_m(x^m_t) - \frac{\eta}{M} \sum_{m=1} ^{M} g^m_t}^2|\hhh_t}\enspace, \nonumber\\
    &\leq^{(\text{\Cref{lem:stoch_noise_second}})}  \norm{ x_{t} - x^{\star} - \frac{\eta}{M} \sum_{m=1} ^{M} \nabla F_m(x^m_t)}^2 + \frac{\eta^2 \sigma_2^2}{M}\enspace, \nonumber\\
    &= \norm{x_{t} - x^{\star} - \eta \nabla F(x_t) + \eta \nabla F(x_t) - \frac{\eta}{M} \sum_{m=1} ^{M} \nabla F_m(x^m_t)}^2 + \frac{\eta^2 \sigma_2^2}{M}\enspace, \nonumber\\
    &\leq^{(\text{\Cref{lem:mod_am_gm}})} \left(1+\frac{\eta \mu}{2}\right)\norm{x_{t} - x^{\star} - \eta \nabla F(x_t)}^2 + \eta^2\left(1+\frac{2}{\eta \mu}\right) \norm{\nabla F(x_t) - \frac 1 M \sum_{m=1} ^{M} \nabla F_m(x^m_t)}^2+ \frac{\eta^2 \sigma_2^2}{M}  \enspace. \label{eq: first decrease lemma}
\end{align}
For the first term in (\ref{eq: first decrease lemma}) we have: 
\begin{align*}
    \norm{ x_{t} - x^{\star} - \eta \nabla F( x_t)}^2 = \norm{ x_{t} - x^{\star}}^2 + \eta^2 \norm{\nabla F(x_t)}^2 -2 \eta \Bigr \langle  x_{t} - x^{\star}, \nabla F( x_t)  \Bigr \rangle \,.
\end{align*}
For the second term in the above equation, we have: 
\begin{align*}
    \eta^2 \norm{\nabla F(x_t)}^2 &= \eta^2 \norm{\nabla F(x_t) - \nabla F(x^{\star})}^2 \\
    &\leq^{(\text{\Cref{ass:smooth_second,rem:self_bounding}})} 2H\eta^2 \Bigr[  F(x_t) - F(x^{\star}) \Bigr] \,.
\end{align*}
For the third term in the equality, we have using strong convexity, i.e., \Cref{ass:strongly_convex}: 
\begin{align*}
    -2 \eta \Bigr \langle x_{t} - x^{\star}, \nabla F( x_t)  \Bigr \rangle \leq - 2\eta \Bigr[ F( x_t) - F(x^{\star}) \Bigr] - \eta \mu  \norm{x_t - x^{\star}}^2\enspace.
\end{align*}
Now by putting everything together, we have: 
\begin{align*}
    \norm{x_{t} - x^{\star} - \eta \nabla F(x_t)}^2 &= \norm{x_{t} - x^{\star}}^2 + \eta^2 \norm{\nabla F(\bar x)}^2 -2 \eta \Bigr \langle x_{t} - x^{\star}, \nabla F(x_t)  \Bigr \rangle  \enspace,\\
    &\leq \norm{x_{t} - x^{\star}}^2 + 2H\eta^2 \Bigr[  F(x_t) - F(x^{\star}) \Bigr] - 2\eta \Bigr[ F(x_t) - F(x^{\star}) \Bigr] - \eta \mu  \norm{ x_t - x^{\star}}^2 \enspace.
\end{align*}
With the choice of $\eta \leq \frac{1}{2H}$ we have:
\begin{align*}
    \norm{x_{t} - x^{\star} - \eta \nabla F(x_t)}^2 \leq (1-\eta \mu) \norm{x_{t} - x^{\star}}^2 - \eta \Bigr[ F(x_t) - F(x^{\star}) \Bigr] \enspace.
\end{align*}
Multiplying both sides by $(1+\frac{\eta \mu}{2})$ we have: 
\begin{align*}
    \left(1+\frac{\eta \mu}{2}\right)\norm{ x_{t} - x^{\star} - \eta \nabla F( x_t)}^2 &\leq \left(1+\frac{\eta \mu}{2}\right)(1-\eta \mu) \norm{x_{t} - x^{\star}}^2 - \eta \left(1+\frac{\eta \mu}{2}\right) \Bigr[ F(x_t) - F(x^{\star}) \Bigr] \enspace,\\
    &\leq \left(1-\frac{\eta \mu}{2}\right)\norm{x_{t} - x^{\star}}^2 - \eta \Bigr[ F(x_t) - F(x^{\star}) \Bigr]\enspace.
\end{align*}
For the second term in (\ref{eq: first decrease lemma}) we have: 
\begin{align*}
    &\eta^2\left(1+\frac{2}{\eta \mu}\right) \norm{\frac 1 M \sum_{m=1} ^{M} \nabla F_m(x^m_t) - \nabla F( x_t)}^2 \\
    &\leq^{\text{(a)}} \frac{3\eta}{ \mu} \norm{\frac 1 M \sum_{m=1} ^{M} \nabla F_m(x^m_t) - \nabla F(x_t)}^2\enspace, \\
    &= \frac{3\eta}{ \mu}  \norm{\frac 1 M \sum_{m=1} ^{M}\Bigr(\nabla F_m(x^m_t) - \nabla F(x^m_t) + \nabla F(x_t) - \nabla F_m(x_t) \Bigr) + \frac 1 M \sum_{m=1} ^{M} \nabla F(x^m_t)-\nabla F(x_t)}^2\enspace, \\
    & \leq^{\text{(b)}}  \frac {6\eta}{\mu M} \sum_{m=1} ^{M} \norm{\nabla F_m(x^m_t) - \nabla F(x^m_t) + \nabla F(x_t) - \nabla F_m(\bar x_t)}^2 + \frac{8\eta}{\mu} \norm{\frac 1 M \sum_{m=1} ^{M} \nabla F(x^m_t)-\nabla F(x_t)}^2\enspace, \\
    &\leq^{\text{(c)}} \frac {6\eta\tau^2}{\mu M} \sum_{m=1} ^{M}\norm{x_t^m-x_t}^2 + \frac{6\eta}{\mu} \norm{\frac 1 M \sum_{m=1} ^{M} \nabla F(x^m_t)-\nabla F(x_t)}^2  \enspace,
\end{align*}
where in (a) we used that $\eta \leq 1/2H \leq 1/2\mu < 1/\mu$; in (b) we used Jensen's inequality and \Cref{lem:mod_am_gm}; and in (c) we used that the function $F_m-F$ is $\tau$-second-order-smooth. For the second term in the above inequality, we have: 
\begin{align*}
    & \frac{6\eta}{\mu} \norm{\frac 1 M \sum_{m=1} ^{M} \nabla F(x^m_t)-\nabla F(x_t)}^2 \\
    &= \frac{6\eta}{\mu}  \norm{\frac 1 M \sum_{m=1} ^{M}\Bigr(\nabla F(x^m_t)-\nabla F(x_t)-\nabla^2 F(x_t)^{\top}(x^m_t-x_t)\Bigr) + \underbrace{\frac 1 M \sum_{m=1} ^{M}\nabla^2 F(x_t)^{\top}(x^m_t-x_t)}_{=0}}^2\enspace, \\
    &= \frac{6\eta}{\mu}  \Biggr(\norm{\frac 1 M \sum_{m=1} ^{M}\Bigr(\nabla F(x^m_t)-\nabla F(x_t)-\nabla^2 F(x_t)^{\top}(x^m_t-x_t)\Bigr)}\Biggr)^2\enspace, \\
    &\leq^{(\text{Triangle Inequality})} \frac{6\eta}{\mu}  \Biggr(\frac{1}{M}\sum_{m=1} ^{M}\norm{\nabla F(x^m_t)-\nabla F(x_t)-\nabla^2 F(x_t)^{\top}(x^m_t-x_t)}\Biggr)^2\enspace, \\
    &=^{\text{(a)}} \frac{6\eta}{\mu}  \Biggr(\frac{1}{M}\sum_{m=1} ^{M}\norm{\nabla^2 F(\hat x^m_t)(x^m_t-x_t)-\nabla^2 F(x_t)^{\top}(x^m_t-x_t)}\Biggr)^2\enspace, \\
    &\leq^{\text{(\text{\Cref{ass:smooth_third}})}} \frac{6\eta}{\mu}  \Biggr(\frac{1}{M}\sum_{m=1} ^{M}Q \norm{\hat x^m_t-x_t}\norm{x^m_t-x_t} \Biggr)^2\enspace, \\
    &\leq \frac{6\eta}{\mu}  \Biggr(\frac{1}{M}\sum_{m=1} ^{M}Q \norm{x^m_t-x_t}^2 \Biggr)^2\enspace, \\
    &\leq^{(\text{Jensen's Inequality})} \frac{6Q^2 \eta}{\mu M} \sum_{m=1} ^{M} \norm{x^m_t-x_t}^4\enspace,
\end{align*}
where in (a) we use mean-value theorem on $\nabla F(\cdot)$ with $\hat x^m_t$ some point on the line-segment connecting $x_t^m$ and $x_t$.
Now by plugging everything back into (\ref{eq: first decrease lemma}) we have: 
\begin{align*}
    \mathbb{E} \sb{\norm{\bar x_{t+1} - x^{\star}}^2|\hhh_t} &\leq \left(1-\frac{\eta \mu}{2}\right)\norm{x_{t} - x^{\star}}^2 - \eta \Bigr[ F(x_t) - F(x^{\star}) \Bigr] + \frac {6\eta\tau^2}{\mu M} \sum_{m=1} ^{M}\norm{x_t^m-x_t}^2\\
    &\quad + \frac{6Q^2 \eta}{\mu M} \sum_{m=1}^M \norm{x^m_t - \bar x_t}^4  + \frac{\eta^2 \sigma_2^2}{M} \enspace.
\end{align*}
Finally, dividing both sides by $\eta$, rearranging the terms, taking the expectation, and recalling the definitions of $C(\cdot)$ and $D(\cdot)$, we get: 
\begin{align*}
    \mathbb{E} \Bigr[ F(x_t) - F(x^{\star}) \Bigr] &\leq \left(\frac 1 \eta -\frac{ \mu}{2}\right) \mathbb{E} \sb{\norm{x_{t} - x^{\star}}^2} - \frac 1 \eta \mathbb{E} \sb{\norm{x_{t+1} - x^{\star}}^2} + \frac{6 \tau^2}{\mu}  C(t) + \frac{6Q^2 }{\mu } D(t) + \frac{\eta \sigma_2^2}{M}\enspace.
\end{align*}
This finishes the proof.
\end{proof}

We also recall the more straightforward recursion, which does not explicitly depend on $Q$, used in several existing results. This can be derived using a similar and simpler proof strategy as in the above lemma.
\begin{lemma}[Lemma 7, \citet{woodworth2020minibatch}]\label{lem:function_error_easy}
    Assume we have a problem instance satisfying \Cref{ass:strongly_convex,ass:smooth_second,ass:stoch_bounded_second_moment}. Then assuming $\eta \leq \frac{1}{10H}$ we have for all $t\in[0,T-1]$,
    \begin{align*}
        E(t) \leq \rb{\frac{1}{\eta} - \mu} \ee \sb{\norm{x_t-x^\star}^2} - \frac{1}{\eta}\ee \sb{\norm{x_{t+1}-x^\star}^2}+ 2HC(t) + \frac{3\eta\sigma_2^2}{M} \enspace.
    \end{align*}
\end{lemma}


\section{Uniform Control over the Consensus Error and Analysis using \texorpdfstring{\Cref{ass:zeta}}{TEXT}}\label{app:zeta_results}
In this section, we will use \Cref{ass:zeta} to derive uniform bounds on the consensus error, which we can then utilize in the recursions developed in the previous section to provide formal convergence guarantees.

\subsection{Upper Bound on Second Moment of Consensus Error}
In this subsection, we restate the upper bound on the second moment of consensus error from the work \cite{woodworth2020minibatch}. We do not claim any novelty and include this lemma for completeness. 
\begin{lemma}[Lemma 8 from \cite{woodworth2020minibatch}]\label{lem:consensus_error_second_zeta} 
    For all $t \in [0,T]$ under \Cref{ass:strongly_convex,ass:stoch_bounded_second_moment,ass:zeta} with a stepsize $\eta \leq \frac{1}{2H}$ and $K\geq 2$ we have,
    \begin{align}
        C(t) \leq 3 K^2 \eta^2 H^2 \zeta^2+  6 K \sigma_2^2 \eta^2  \enspace .
    \end{align}
\end{lemma}

\begin{proof}
    Note the following about the second moment of the difference between the iterates on two machines $m,n\in[M]$ when $t > \delta(t)$,
    \begin{align*}
    &\ee\sb{\norm{x^m_t - x^n_t}^2}\\ 
    &= \ee \sb{\norm{x^m_{t-1} - x^n_{t-1} - \eta g^m_{t-1} + \eta g^n_{t-1}}^2} \enspace, \\
    &\leq \ee \sb{\norm{x^m_{t-1} - x^n_{t-1} - \eta \nabla F_m (x^m_{t-1}) + \eta \nabla F_n (x^n_{t-1})}^2} + 2\eta^2 \sigma_2^2 \enspace, \\
    &= \ee \sb{\norm{x^m_{t-1} - x^n_{t-1} - \eta \left(\nabla F_m (x^m_{t-1}) - \nabla F_m (x^n_{t-1}) \right) - \eta \left( \nabla F_m (x^n_{t-1}) - \nabla F_n (x^n_{t-1})\right)}^2} + 2\eta^2 \sigma_2^2 \enspace, \\
    &\leq^{\text{(a)}} \ee\sb{\norm{x^m_{t-1} - x^n_{t-1} - \eta \nabla^2 F_m(c) (x^m_{t-1} - x^n_{t-1})- \eta \left(\nabla F_m (x^n_{t-1})- \nabla F_n (x^n_{t-1}) \right)}^2}+ 2\eta^2 \sigma_2^2\enspace,\\
    &= \ee\sb{\norm{\left(I - \eta \nabla^2 F_m(c)\right) (x^m_{t-1} - x^n_{t-1}) - \eta \left(\nabla F_m (x^n_{t-1})- \nabla F_n (x^n_{t-1}) \right)}^2} + 2\eta^2 \sigma_2^2\enspace,\\
    &\leq^{(\text{b})} \rb{1 + \frac{1}{K-1}}\ee\sb{\norm{\left(I - \eta \nabla^2 F_m(c)\right) (x^m_{t-1} - x^n_{t-1})}^2}\\
    &\quad + K \eta^2 \ee\sb{\norm{\nabla F_m (x^n_{t-1})- \nabla F_n (x^n_{t-1})}^2} + 2\eta^2 \sigma_2^2\enspace,\\
    &\leq \rb{1 + \frac{1}{K-1}}(1-\eta \mu)^2\ee\sb{\norm{x^m_{t-1} - x^n_{t-1}}^2} + K \eta^2 H^2 \zeta^2 + 2\eta^2 \sigma_2^2\enspace, \\
    &\leq \rb{1 + \frac{1}{K-1}}\ee\sb{\norm{x^m_{t-1} - x^n_{t-1}}^2} + K \eta^2 H^2 \zeta^2 + 2\eta^2 \sigma_2^2\enspace,
\end{align*}
where in (a) we use the mean value theorem to find a $c$ between $x_{t-1}^m$ and $x_{t-1}^n$ such that $\nabla F_m(x_{t-1}^m) - \nabla F_m(x_{t-1}^n) = \nabla^2 F(c)\cdot(x_{t-1}^m - x_{t-1}^n)$; and in (b) we apply \Cref{lem:mod_am_gm} with $\gamma = K-1$. Unrolling the recursion gives us, 
\begin{align*}
    \ee\sb{\norm{x^m_t - x^n_t}^2} &\leq \rb{1 + \frac{1}{K-1}}^{t-\delta(t)}\ee\sb{\norm{x^m_{\delta(t)} - x^n_{\delta(t)}}^2}\\
    &\quad + \rb{1 + \frac{1}{K-1}}^{K-1}(t-\delta(t))\rb{K \eta^2 H^2 \zeta^2 + 2\eta^2 \sigma_2^2}\enspace,\\
    &\leq 3 K^2 \eta^2 H^2 \zeta^2 + 6 \eta^2 K \sigma_2^2\enspace,
\end{align*}
where in the last inequality we used that $\rb{1 + \frac{1}{K-1}}^{K-1} \leq 3$ for all $K$ and that at time $\delta(t)$ the machines last synchronized there models so $x_{\delta(t)}^m = x_{\delta(t)}^n$. Averaging this over $m,n$ finishes the proof.
\end{proof}

\subsection{Upper Bound on Fourth Moment of Consensus Error}

In this subsection, we prove a fourth-moment upper bound on consensus error using similar techniques as those in \citet{woodworth2020minibatch,yuan2020federated}. 
\begin{lemma}[Lemma 12 from \cite{patel2024limits}]\label{lem:cons_error_fourth}
    For all $t \in [0,T]$ under \Cref{ass:strongly_convex,ass:stoch_bounded_second_moment,ass:stoch_bounded_fourth_moment,ass:zeta} with a step-size $\eta \leq \frac{1}{2H}$ and $K\geq 2$ we have,
    \begin{align*}
        D(t) \leq 2620\eta^4K^4 H^4 \zeta^4 + 5000\eta^4 K^2\sigma_2^4 + 320\eta^4\sigma_4^4K\enspace.
    \end{align*}
\end{lemma}

\begin{proof}
Note the following about the fourth moment of the difference between the iterates on two machines $m,n\in[M]$,
    \begin{align*}
        &\ee\sb{\norm{x_t^m - x_t^n}^4}\\
        &= \ee\sb{\norm{x_{t-1}^m - x_{t-1}^n - \eta g_{t-1}^m +\eta g_{t-1}^n}^4}\enspace,\\
        &= \ee\sb{\rb{\norm{x_{t-1}^m - x_{t-1}^n - \eta \nabla F_m(x_{t-1}^m) + \eta \nabla F_n(x_{t-1}^n) + \eta \xi_{t-1}^m - \eta \xi_{t-1}^n}^2}^2}\enspace,\\
        &= \ee\Bigg[\bigg(\norm{x_{t-1}^m - x_{t-1}^n - \eta \nabla F_m(x_{t-1}^m) + \eta \nabla F_n(x_{t-1}^n)}^2 + \eta^2 \norm{\xi_{t-1}^m -\xi_{t-1}^n}^2\\
        &\quad + 2\eta\inner{x_{t-1}^m - x_{t-1}^n - \eta \nabla F_m(x_{t-1}^m) + \eta \nabla F_n(x_{t-1}^n)}{\xi_{t-1}^m - \xi_{t-1}^n}\bigg)^2\Bigg]\enspace,\\
        &= \ee\sb{\norm{x_{t-1}^m - x_{t-1}^n - \eta \nabla F_m(x_{t-1}^m) + \eta \nabla F_n(x_{t-1}^n)}^4} + \eta^4 \ee\sb{\norm{\xi_{t-1}^m -\xi_{t-1}^n}^4}\\
        &\quad +4\eta^2\ee\sb{\rb{\inner{x_{t-1}^m - x_{t-1}^n - \eta \nabla F_m(x_{t-1}^m) + \eta \nabla F_n(x_{t-1}^n)}{\xi_{t-1}^m - \xi_{t-1}^n}}^2}\\
        &\quad + 2\eta^2\ee\sb{\norm{x_{t-1}^m - x_{t-1}^n - \eta \nabla F_m(x_{t-1}^m) + \eta \nabla F_n(x_{t-1}^n)}^2\norm{\xi_{t-1}^m -\xi_{t-1}^n}^2}\\
        &\quad +4\eta^3\ee\sb{\inner{x_{t-1}^m - x_{t-1}^n - \eta \nabla F_m(x_{t-1}^m) + \eta \nabla F_n(x_{t-1}^n)}{\xi_{t-1}^m - \xi_{t-1}^n}\norm{\xi_{t-1}^m - \xi_{t-1}^n}^2}\\
        &\quad + 4\eta\ee\Bigg[\norm{x_{t-1}^m - x_{t-1}^n - \eta \nabla F_m(x_{t-1}^m) + \eta \nabla F_n(x_{t-1}^n)}^2\\
        &\qquad\qquad\qquad \rb{x_{t-1}^m - x_{t-1}^n - \eta \nabla F_m(x_{t-1}^m) + \eta \nabla F_n(x_{t-1}^n)}\Bigg]\cdot \cancelto{0}{\ee\sb{\xi_{t-1}^m - \xi_{t-1}^n}}\enspace, \\
        &\leq^{\text{(C.S. Inequality, \Cref{lem:stoch_diff_fourth})}} \ee\sb{\norm{x_{t-1}^m - x_{t-1}^n - \eta \nabla F_m(x_{t-1}^m) + \eta \nabla F_n(x_{t-1}^n)}^4} + 8\sigma_4^4\eta^4\\
        &\quad + 6\eta^2\ee\sb{\norm{x_{t-1}^m - x_{t-1}^n - \eta \nabla F_m(x_{t-1}^m) + \eta \nabla F_n(x_{t-1}^n)}^2}\ee\sb{\norm{\xi_{t-1}^m -\xi_{t-1}^n}^2}\\
        &\quad + \textcolor{blue}{4\eta^3\ee\sb{\norm{x_{t-1}^m - x_{t-1}^n - \eta \nabla F_m(x_{t-1}^m) + \eta \nabla F_n(x_{t-1}^n)}}}\textcolor{red}{\ee\sb{\norm{\xi_{t-1}^m - \xi_{t-1}^n}^3}}\tag{a}\enspace,
\end{align*}
In order to bound the term $\textcolor{red}{\ee\sb{\norm{\xi_{t-1}^m - \xi_{t-1}^n}^3}}$ we use Cauchy-Schwarz Inequality:
\begin{align*}
    \ee \sb{\norm{\xi^m_{t-1} - \xi^n_{t-1}}^3} &= \ee \sb{\norm{\xi^m_{t-1} - \xi^n_{t-1}}\cdot \norm{\xi^m_{t-1} - \xi^n_{t-1}}^2} \\
    &\leq \sqrt{\ee \sb{\norm{\xi^m_{t-1} - \xi^n_{t-1}}^2} \ee \sb{\norm{\xi^m_{t-1} - \xi^n_{t-1}}^4}} \stackrel{\text{(\Cref{lem:stoch_diff_second,lem:stoch_diff_fourth})}}{\leq} 4\sqrt{\sigma_4^4\sigma_2^2}\enspace.
\end{align*}
Also the term $\textcolor{blue}{4\eta^3\ee\sb{\norm{x_{t-1}^m - x_{t-1}^n - \eta \nabla F_m(x_{t-1}^m) + \eta \nabla F_n(x_{t-1}^n)}}}$ can be bounded as: 
\begin{align*}
    &\ee\sb{\norm{x_{t-1}^m - x_{t-1}^n - \eta \nabla F_m(x_{t-1}^m) + \eta \nabla F_n(x_{t-1}^n)}} \\
    &\quad\stackrel{\text{(Jensen's Inequality)}}{\leq} \sqrt{\ee\sb{\norm{x_{t-1}^m - x_{t-1}^n - \eta \nabla F_m(x_{t-1}^m) + \eta \nabla F_n(x_{t-1}^n)}^2}}
\end{align*}
Putting everything back into (a) gives us: 
\begin{align*}
        \ee \sb{\norm{x^n_t - x^m_t}^4} &\leq^{\text{(\Cref{ass:stoch_bounded_second_moment})}} \ee\sb{\norm{x_{t-1}^m - x_{t-1}^n - \eta \nabla F_m(x_{t-1}^m) + \eta \nabla F_n(x_{t-1}^n)}^4} + 8\eta^4 \sigma_4^4\\
        &\quad +12\eta^2\sigma_2^2\ee\sb{\norm{x_{t-1}^m - x_{t-1}^n - \eta \nabla F_m(x_{t-1}^m) + \eta \nabla F_n(x_{t-1}^n)}^2}\\
        &\quad +\textcolor{blue}{16\eta^3\sqrt{\sigma_4^4\sigma_2^2 \ee\sb{\norm{x_{t-1}^m - x_{t-1}^n - \eta \nabla F_m(x_{t-1}^m) + \eta \nabla F_n(x_{t-1}^n)}^2}}}\enspace,
\end{align*}
To bound the \textcolor{blue}{third} term in the above inequality, we use the A.M. - G.M. Inequality $\sqrt{ab} \leq \frac{a}{2\gamma} + \frac{\gamma b}{2}$ for $\gamma >0$. Let $\gamma = \eta, a=\sigma_2^2 \ee\sb{\norm{x_{t-1}^m - x_{t-1}^n - \eta \nabla F_m(x_{t-1}^m) + \eta \nabla F_n(x_{t-1}^n)}^2}, b=\sigma_4^4$. We have:
\begin{align*}
    &16\eta^3\sqrt{\sigma_4^4\sigma_2^2 \ee\sb{\norm{x_{t-1}^m - x_{t-1}^n - \eta \nabla F_m(x_{t-1}^m) + \eta \nabla F_n(x_{t-1}^n)}^2}} \\
    &= 16\eta^3\sqrt{(\sigma_4^4) \left(\sigma_2^2 \ee\sb{\norm{x_{t-1}^m - x_{t-1}^n - \eta \nabla F_m(x_{t-1}^m) + \eta \nabla F_n(x_{t-1}^n)}^2}\right)}\enspace, \\
    &\leq 16\eta^3 \left(\frac{\eta \sigma_4^4}{2} + \frac{\sigma_2^2}{2\eta}\ee\sb{\norm{x_{t-1}^m - x_{t-1}^n - \eta \nabla F_m(x_{t-1}^m) + \eta \nabla F_n(x_{t-1}^n)}^2} \right)\enspace,
\end{align*}
So we have: 
\begin{align*}
        &\ee \sb{\norm{x^n_t - x^m_t}^4} \\
        &\leq \ee\sb{\norm{x_{t-1}^m - x_{t-1}^n - \eta \nabla F_m(x_{t-1}^m) + \eta \nabla F_n(x_{t-1}^n)}^4} + 8\eta^4 \sigma_4^4 \\
        &\quad +12\eta^2\sigma_2^2\ee\sb{\norm{x_{t-1}^m - x_{t-1}^n - \eta \nabla F_m(x_{t-1}^m) + \eta \nabla F_n(x_{t-1}^n)}^2}\\
        &\quad +16\eta^3\rb{\frac{\eta\sigma_4^4}{2} +  \frac{\sigma_2^2}{2\eta}\ee\sb{\norm{x_{t-1}^m - x_{t-1}^n - \eta \nabla F_m(x_{t-1}^m) + \eta \nabla F_n(x_{t-1}^n)}^2}}\enspace,\\
        &= \ee\sb{\norm{x_{t-1}^m - x_{t-1}^n - \eta \nabla F_m(x_{t-1}^m) + \eta \nabla F_n(x_{t-1}^n)}^4}  \\ 
        &\quad +20\eta^2\sigma_2^2\ee\sb{\norm{x_{t-1}^m - x_{t-1}^n - \eta \nabla F_m(x_{t-1}^m) + \eta \nabla F_n(x_{t-1}^n)}^2} + 16\eta^4 \sigma_4^4 \enspace,\\
        &= \ee\sb{\norm{x_{t-1}^m - x_{t-1}^n - \eta \rb{\nabla F_m(x_{t-1}^m) - \nabla F_m(x_{t-1}^n)} + \eta \rb{\nabla F_n(x_{t-1}^n) -\nabla F_m(x_{t-1}^n) }}^4} \\
        &\quad +20\eta^2\sigma_2^2\ee\sb{\norm{x_{t-1}^m - x_{t-1}^n - \eta \rb{\nabla F_m(x_{t-1}^m) - \nabla F_m(x_{t-1}^n)} + \eta \rb{\nabla F_n(x_{t-1}^n) -\nabla F_m(x_{t-1}^n) }}^2}\\
        &\quad + 16\eta^4 \sigma_4^4 \enspace,
\end{align*}
Now by using \Cref{lem:mod_am_gm} with $\gamma=K-1$ we have: 
\begin{align*}
        \ee\sb{\norm{x^m_t - x^n_t}^4} &\leq \textcolor{blue}{\rb{1+\frac{1}{K-1}}^3\ee\sb{\norm{x_{t-1}^m - x_{t-1}^n - \eta \rb{\nabla F_m(x_{t-1}^m) - \nabla F_m(x_{t-1}^n)}}^4}}\\ 
        &\quad + \eta^4 K^3\ee\sb{\norm{\nabla F_n(x_{t-1}^n) -\nabla F_m(x_{t-1}^n)}^4}\\
        &\quad + \textcolor{red}{20\eta^2\sigma_2^2\rb{1+ \frac{1}{K-1}}\ee\sb{\norm{x_{t-1}^m - x_{t-1}^n - \eta \rb{\nabla F_m(x_{t-1}^m) - \nabla F_m(x_{t-1}^n)}}^2}}\\
        &\quad + 20\eta^4\sigma_2^2 K \ee\sb{\norm{\nabla F_n(x_{t-1}^n) -\nabla F_m(x_{t-1}^n)}^2}+ 16\eta^4 \sigma_4^4 \enspace,
\end{align*}
From the mean-value theorem we know that $\nabla F(x) - \nabla F(y) = \nabla^2 F(c)(x-y)$ for some $c = \lambda x + (1-\lambda) y$ and $\lambda\in[0,1]$. By applying this to the \textcolor{blue}{first} and \textcolor{red}{third} term of the above inequality we have: 
\begin{align*}
    &\rb{1+\frac{1}{K-1}}^3\ee\sb{\norm{x_{t-1}^m - x_{t-1}^n - (\eta \nabla F_m(x_{t-1}^m) - \eta\nabla F_m(x_{t-1}^n))}^4} \\
    &= \rb{1+\frac{1}{K-1}}^3\ee\sb{\norm{x_{t-1}^m - x_{t-1}^n - \eta \nabla^2 F_m(c)(x_{t-1}^m-x_{t-1}^n)}^4} \enspace, \\
    &= \rb{1+\frac{1}{K-1}}^3\ee\sb{\norm{(I- \eta \nabla^2 F_m(c))(x_{t-1}^m - x_{t-1}^n)}^4}\enspace, \\
    &\leq^{(\text{\Cref{ass:strongly_convex}})} \rb{1+\frac{1}{K-1}}^3 (1- \eta \mu)^4 \ee \sb{\norm{x_{t-1}^m-x_{t-1}^n}^4}\enspace,
\end{align*}
With the same approach for the \textcolor{red}{third} term we have: 
\begin{align*}
    &20\eta^2\sigma_2^2\rb{1+ \frac{1}{K-1}}\ee\sb{\norm{x_{t-1}^m - x_{t-1}^n - \eta\rb{ \nabla F_m(x_{t-1}^m) - \eta\nabla F_m(x_{t-1}^n)}}^2} \\
    &\leq 20\eta^2\sigma_2^2\rb{1+ \frac{1}{K-1}}(1-\eta \mu)^2 \ee\sb{\norm{x_{t-1}^m - x_{t-1}^n}^2}\enspace,
\end{align*}
Putting all of these bounds together gives us: 
\begin{align*}
        &\ee \sb{\norm{x^n_t - x^m_t}^4} \\
        &\leq \rb{1+\frac{1}{K-1}}^3(1-\eta\mu)^4\ee\sb{\norm{x_{t-1}^m - x_{t-1}^n}^4} + \eta^4 K^3\ee\sb{\norm{\nabla F_n(x_{t-1}^n) -\nabla F_m(x_{t-1}^n)}^4}\\
        &\quad + 20\eta^2\sigma^2\rb{1+ \frac{1}{K-1}}(1-\eta\mu)^2\ee\sb{\norm{x_{t-1}^m - x_{t-1}^n}^2} + 20\eta^4\sigma_2^2 K \ee\sb{\norm{\nabla F_n(x_{t-1}^n) -\nabla F_m(x_{t-1}^n)}^2}\\
        &\quad+ 16\eta^4 \sigma_4^4\enspace,\\
        &\leq^{(\text{\Cref{ass:zeta}})} \rb{1+\frac{1}{K-1}}^3(1-\eta\mu)^4\ee\sb{\norm{x_{t-1}^m - x_{t-1}^n}^4} + \eta^4K^3 H^4 \zeta^4\\
        &\quad + 20\eta^2\sigma_2^2\rb{1+ \frac{1}{K-1}}(1-\eta\mu)^2\ee\sb{\norm{x_{t-1}^m - x_{t-1}^n}^2} + 20\eta^4\sigma_2^2K H^2 \zeta^2+ 16\eta^4 \sigma_4^4\enspace,\\
        &\leq^{(\text{\Cref{lem:consensus_error_second_zeta}})} \rb{1+\frac{1}{K-1}}^3\ee\sb{\norm{x_{t-1}^m - x_{t-1}^n}^4} + \eta^4K^3 H^4 \zeta^4\\
        &\quad + 20\eta^2\sigma_2^2\rb{1+ \frac{1}{K-1}}(1-\eta\mu)^2\rb{3K\sigma_2^2\eta^2 + 6K^2\eta^2 H^2 \zeta^2} + 20\eta^4\sigma_2^2K H^2 \zeta^2+ 16\eta^4 \sigma_4^4\enspace,\\
        &\leq^{\text{(a)}} \rb{1+\frac{1}{K-1}}^3\ee\sb{\norm{x_{t-1}^m - x_{t-1}^n}^4} + \eta^4K^3 H^4 \zeta^4 + 120\eta^4\sigma_2^4K + 16\eta^4\sigma_4^4 +260 \eta^4\sigma_2^2K^2 H^2 \zeta^2\enspace,\\
        &\leq \rb{1+\frac{1}{K-1}}^{3(K-1)}\rb{\eta^4K^4 H^4 \zeta^4 + 120\eta^4\sigma_2^4K^2 + 16\eta^4\sigma_4^4K +260 \eta^4\sigma_2^2K^3 H^2 \zeta^2}\enspace,\\
        &\leq^{\text{(b)}} 20\rb{\eta^4K^4 H^4 \zeta^4 + 120\eta^4\sigma_2^4K^2 + 16\eta^4\sigma_4^4K +260 \eta^4\sigma_2^2K^3 H^2 \zeta^2}\enspace,\\
        &=20\rb{\eta^4K^4 H^4 \zeta^4 + 120\eta^4\sigma_2^4K^2 + 16\eta^4\sigma_4^4K +260 \sqrt{\eta^4K^4H^4\zeta^4}\sqrt{\eta^4\sigma_2^4K^2}}\enspace,\\
        &\leq^{\text{(A.M.-G.M. Inequality)}} 20\rb{131\eta^4K^4 H^4 \zeta^4 + 250\eta^4\sigma_2^4K^2 + 16\eta^4\sigma_4^4K}\enspace,\\
        &\leq 2620\eta^4K^4 H^4 \zeta^4 + 5000\eta^4 K^2\sigma_2^4 + 320\eta^4\sigma_4^4K\enspace,
    \end{align*}
    where in (a) we use $K\geq 2$ to bound $\frac{1}{K-1}$ by one; and in (b) we used that $(1+1/x)^x \leq 20$ for all $x\geq0$. 
    Finally averaging this over $m,n\in[M]$ implies,
    \begin{align*}
        \frac{1}{M}\sum_{m\in[M]}\ee\sb{\norm{x_t-x_t^m}^4} &\leq \frac{1}{M^2}\sum_{m,n\in[M]}\ee\sb{\norm{x_t^n-x_t^m}^4}\enspace,\\ 
        &\leq 2620\eta^4K^4 H^4 \zeta^4 + 5000\eta^4 K^2\sigma_2^4 + 320\eta^4\sigma_4^4K\enspace,
    \end{align*}
    which proves the lemma.
\end{proof}

\subsection{Convergence in Iterates}\label{app:zeta_iterate}

In this sub-section, we provide a convergence guarantee for the iterates of local SGD incorporating \Cref{ass:smooth_third,ass:tau}. We do so by using the \textcolor{red}{red} upper bound from \Cref{lem:iterate_error_second_recursion}.

\begin{lemma}[Convergence with $\zeta$, $\tau$ and $Q$]
    Assume we have a problem instance satisfying \Cref{ass:strongly_convex,ass:smooth_second,ass:smooth_third,ass:bounded_optima,ass:stoch_bounded_second_moment,ass:stoch_bounded_fourth_moment,ass:tau,ass:zeta} the Local SGD satisfies the following convergence guarantee assuming $\eta \leq 1/2H$: 
    \begin{align*}
        A(T) &\leq \rb{1-\eta\mu}^{KR} B^2 + \frac{5240 Q^2 \eta^4 K^4  H^4 \zeta^4}{\mu^2} + \frac{10000 Q^2 \eta^4 K^2 \sigma_2^4}{\mu^2}+\frac{640 Q^2 \eta^4 K \sigma_4^4}{\mu^2}+\frac{6\tau^2\eta^2 K^2  H^2 \zeta^2}{\mu^2}\\
        &\quad +\frac{12\tau^2\eta^2 K \sigma_2^2}{\mu^2}+\frac{\eta \sigma_2^2}{\mu M}\enspace.
    \end{align*}
    Furthermore choosing $\eta = \min\cb{\frac{1}{2H}, \frac{1}{\mu KR}\ln\rb{\frac{B^2}{\epsilon}}}$, where we define $$\epsilon = \max\cb{\frac{5240 Q^2  H^4 \zeta^4}{\mu^6R^4}, \frac{10000 Q^2 \sigma_2^4}{\mu^6K^2R^4}, \frac{640 Q^2 \sigma_4^4}{\mu^6K^3R^4}, \frac{6\tau^2  H^2 \zeta^2}{\mu^4R^2}, \frac{12\tau^2 \sigma_2^2}{\mu^4 KR^2}, \frac{\sigma_2^2}{\mu^2 MKR}, \epsilon_{target}}\enspace,$$
    where $\epsilon_{target}$ is some target precision greater than or equal to machine precision, we get the following convergence guarantee,
    \begin{align*}
        A(T) &= \tilde\ooo\Bigg(e^{-KR/2\kappa} B^2 + \frac{Q^2  H^4 \zeta^4}{\mu^6R^4} + \frac{Q^2 \sigma_2^4}{\mu^6K^2R^4} +\frac{Q^2 \sigma_4^4}{\mu^6K^3R^4}+\frac{\tau^2  H^2 \zeta^2}{\mu^4R^2}+\frac{\tau^2 \sigma_2^2}{\mu^4 KR^2}+\frac{\sigma_2^2}{\mu^2 MKR}\Bigg)\enspace.
    \end{align*}
\end{lemma}
\begin{proof}
Use the \textcolor{red}{red} upper bound for one-step progress from \Cref{lem:iterate_error_second_recursion}. We first restate the one-step lemma using the \textcolor{red}{red} upper bound,
    \begin{align*}
        A(KR) &\leq \rb{1-\eta\mu}A(KR-1) + \frac{2\eta Q^2}{\mu}D(KR-1)+ \frac{2\eta \tau^2}{\mu}C(KR-1)+ \frac{\eta^2\sigma_2^2}{M}\enspace,\\
        &\leq \rb{1-\eta\mu}^KA(K(R-1))+ \frac{2\eta }{\mu}\sum_{j=K(R-1)}^{KR-1}(1-\eta\mu)^{KR-1-j} \rb{Q^2D(j)+\tau^2C(j)}\\
        &\quad + \rb{1-(1-\eta\mu)^K}\frac{\eta\sigma_2^2}{\mu M}\enspace,\\
        &\leq \rb{1-\eta\mu}^KA(K(R-1))\\
        &\quad +\frac{2Q^2\eta }{\mu}\sum_{j=K(R-1)}^{KR-1}(1-\eta\mu)^{KR-1-j} \rb{2620\eta^4K^4 H^4 \zeta^4 + 5000\eta^4 K^2\sigma_2^4 + 320\eta^4\sigma_4^4K} \\
        &\quad+\frac{2\tau^2\eta }{\mu}\sum_{j=K(R-1)}^{KR-1}(1-\eta\mu)^{KR-1-j} \rb{3K^2\eta^2  H^2 \zeta^2+6K\eta^2\sigma_2^2} + \rb{1-(1-\eta\mu)^K}\frac{\eta\sigma_2^2}{\mu M}\enspace,\\
        &\leq \rb{1-\eta\mu}^KA(K(R-1))\\
        &\quad+\rb{\frac{5240 Q^2 \eta^5 K^4  H^4 \zeta^4}{\mu}+ \frac{10000 Q^2 \eta^5 K^2 \sigma_2^4}{\mu} + \frac{640 Q^2 \eta^5 K \sigma_4^4}{\mu}}\sum_{j=K(R-1)}^{KR-1}(1-\eta\mu)^{KR-1-j}\\
        &\quad+\rb{\frac{6\tau^2\eta^3 K^2  H^2 \zeta^2}{\mu}+\frac{12\tau^2\eta^3 K \sigma_2^2}{\mu}}\sum_{j=K(R-1)}^{KR-1}(1-\eta\mu)^{KR-1-j} + \rb{1-(1-\eta\mu)^K}\frac{\eta\sigma_2^2}{\mu M}\enspace.\tag{a}
    \end{align*}
    Note that we can simplify the summation as follows,
    \begin{align*}
        &\sum_{j=K(R-1)}^{KR-1}(1-\eta\mu)^{KR-1-j}= \sum_{i=0}^{K-1}(1-\eta \mu)^i = \frac{1-(1-\eta \mu)^K}{\eta \mu}\enspace.
    \end{align*}
    Plugging the above result back into (a) gives us, 
    \begin{align*}
        A(KR) &\leq \rb{1-\eta\mu}^KA(K(R-1)) + \rb{1-(1-\eta\mu)^K}\rb{\frac{5240 Q^2 \eta^4 K^4  H^4 \zeta^4}{\mu^2} + \frac{10000 Q^2 \eta^4 K^2 \sigma_2^4}{\mu^2}}\\
        &+\quad  \rb{1-(1-\eta\mu)^K}\rb{\frac{640 Q^2 \eta^4 K \sigma_4^4}{\mu^2} +\frac{6\tau^2\eta^2 K^2  H^2 \zeta^2}{\mu^2} +\frac{12\tau^2\eta^2 K \sigma_2^2}{\mu^2}+\frac{\eta \sigma_2^2}{\mu M}} \enspace,
    \end{align*}
    Now we unroll the above inequality over $R$ rounds and we have, 
    \begin{align*}
        A(KR) &\leq \rb{1-\eta\mu}^{KR} B^2 + \frac{5240 Q^2 \eta^4 K^4  H^4 \zeta^4}{\mu^2} + \frac{10000 Q^2 \eta^4 K^2 \sigma_2^4}{\mu^2}+\frac{640 Q^2 \eta^4 K \sigma_4^4}{\mu^2}+\frac{6\tau^2\eta^2 K^2  H^2 \zeta^2}{\mu^2}\\
        &\quad +\frac{12\tau^2\eta^2 K \sigma_2^2}{\mu^2}+\frac{\eta \sigma_2^2}{\mu M}\enspace.
    \end{align*}
    This proves the first statement of the lemma. We can further simplify the upper bound as follows,
    \begin{align*}
        A(KR) &\leq e^{-\eta\mu KR} B^2 + \frac{5240 Q^2 \eta^4 K^4  H^4 \zeta^4}{\mu^2} + \frac{10000 Q^2 \eta^4 K^2 \sigma_2^4}{\mu^2}+\frac{640 Q^2 \eta^4 K \sigma_4^4}{\mu^2}+\frac{6\tau^2\eta^2 K^2  H^2 \zeta^2}{\mu^2}\\
        &\quad +\frac{12\tau^2\eta^2 K \sigma_2^2}{\mu^2}+\frac{\eta \sigma_2^2}{\mu M}\enspace.
    \end{align*}
    To achieve the final bound, we need to tune the step size. First, note that besides the first term, all the other terms are increasing functions of $\eta$. This means if we pick the step-size as described in the lemma statement, we will recover all but the first term in the convergence rate (up to logarithmic powers in $\ln\rb{\frac{B^2}{\epsilon}}$) by choosing the upper bound given by $\eta= \frac{1}{\mu KR}\ln(B^2/\epsilon)$ . The choice of $\epsilon$ is such that this logarithmic term never blows up, and is determined by the dominating term in the convergence rate (barring the first term). Now, for the first term, since it is a decreasing function in $\eta$, we can't choose one of the upper bounds implied by the choice of the step-size. We need to consider two cases\footnote{This step-size tuning is fairly common to get a convergence rate for SGD in the strongly convex setting. For instance, see other works such as \citet{ghadimi2012optimal,nemirovski1994efficient,bach2010self,stich2019unified}.}:
    \begin{itemize}
        \item When $\frac{1}{2H}\leq \frac{1}{\mu KR}\ln\rb{\frac{B^2}{\epsilon}}$, then we get first term in the convergence rate, so clearly the upper bound in the lemma statement is valid.
        \item When $\frac{1}{2H}\geq \frac{1}{\mu KR}\ln\rb{\frac{B^2}{\epsilon}}$, then the first term $e^{-\eta\mu KR}B^2 = e^{-\ln(B^2/\epsilon)}B^2 = \epsilon$. Since $\epsilon$ always matches one of the terms in the convergence rate (up to numerical constants) or the target accuracy $\epsilon_{target}$ (whichever is larger), we can upper bound the first term in the convergence rate with one of the other terms in the rate. This makes the upper bound in the lemma statement valid. 
    \end{itemize}
    The $\tilde\ooo()$ hides all the numerical constants and logarithmic powers in $\ln\rb{\frac{B^2}{\epsilon}}$. This proves the second statement of the lemma.
\end{proof}

\subsection{Convergence in Function Value}\label{app:zeta_func}
In this subsection we will five upper bounds in terms of the function-sub-optimality, using the uniform upper bounds we have developed on the consensus error in the previous sub-section. Specifically, we will combine \Cref{lem:function_error_hard} with \Cref{lem:consensus_error_second_zeta,lem:cons_error_fourth}, resulting in the following theorem:
\begin{lemma}[Convergence with $\zeta$, $\tau$ and $Q$]
    Assume we have a problem instance satisfying \Cref{ass:strongly_convex,ass:smooth_second,ass:smooth_third,ass:bounded_optima,ass:stoch_bounded_second_moment,ass:stoch_bounded_fourth_moment,ass:tau,ass:zeta} and $KR\geq 4\kappa \ln 2$. Choose $\eta := \min\cb{\frac{1}{2H}, \frac{2}{\mu KR}\ln\rb{\frac{\mu B^2}{\epsilon}}}$, where we define $$\epsilon = \max\cb{ \frac{6\tau^2}{\mu}\rb{\frac{3 H^2 \zeta^2}{\mu^2R^2}+  \frac{6 \sigma_2^2}{\mu^2KR^2}} + \frac{6Q^2}{\mu}\rb{\frac{2620 H^4 \zeta^4}{\mu^4R^4} + \frac{5000 \sigma_2^4}{\mu^4K^2R^4} + \frac{320\sigma_4^4}{\mu^4K^3R^4}} + \frac{\sigma_2^2}{\mu MKR}, \epsilon_{target}}\enspace,$$
    where $\epsilon_{target}$ is some target precision greater than or equal to machine precision. We assume $\epsilon \leq \frac{\mu B^2}{2}$. Also define the weighted Local SGD iterate $\hat x := \frac{1}{W}\sum_{t\in[0,T-1]}w_{t}x_t$ where $w_t := \rb{1-\frac{\eta\mu}{2}}^{T-1-t}$ and their sum $W := \sum_{t=0}^{T-1}w_t$. Then we can get the following convergence guarantee for $\hat x$ (where $x_0=0$ and $x^\star\in S^\star$),
    \begin{align*}
         &\ee\sb{F\rb{\hat x}} - F(x^\star)\\ 
         &\qquad= \tilde\ooo\rb{\mu B^2e^{-KR/4\kappa}+ \frac{\tau^2H^2 \zeta^2}{\mu^3R^2}+  \frac{\tau^2\sigma_2^2}{\mu^3KR^2} + \frac{Q^2H^4 \zeta^4}{\mu^5R^4} + \frac{Q^2\sigma_2^4}{\mu^5K^2R^4} + \frac{Q^2\sigma_4^4}{\mu^5K^3R^4} + \frac{\sigma_2^2}{\mu MKR}}\enspace.
    \end{align*}
\end{lemma}
\begin{proof}
    We first recall the recursion from \Cref{lem:function_error_hard} and then upper bound the consensus error terms from \Cref{lem:consensus_error_second_zeta,lem:cons_error_fourth},
    \begin{align*}
        E(t) &\leq \rb{\frac{1}{\eta} - \frac{\mu}{2}} \ee \sb{\norm{x_t-x^\star}^2} - \frac{1}{\eta}\ee \sb{\norm{x_{t+1}-x^\star}^2}+ \frac{6\tau^2}{\mu}C(t) + \frac{6Q^2 }{\mu}D(t) + \frac{\eta \sigma_2^2}{M} \enspace,\\
        &\leq^{\text{(\Cref{lem:consensus_error_second_zeta,lem:cons_error_fourth})}}  \rb{\frac{1}{\eta} - \frac{\mu}{2}} \ee \sb{\norm{x_t-x^\star}^2} - \frac{1}{\eta}\ee \sb{\norm{x_{t+1}-x^\star}^2}+ \frac{6\tau^2}{\mu}\rb{3 K^2 \eta^2 H^2 \zeta^2+  6 K \sigma_2^2 \eta^2}\\
        &\quad \frac{6Q^2 }{\mu}\rb{2620\eta^4K^4 H^4 \zeta^4 + 5000\eta^4 K^2\sigma_2^4 + 320\eta^4\sigma_4^4K} + \frac{\eta \sigma_2^2}{M} \enspace,\\
        &=:\rb{\frac{1}{\eta} - \frac{\mu}{2}} \ee \sb{\norm{x_t-x^\star}^2} - \frac{1}{\eta}\ee \sb{\norm{x_{t+1}-x^\star}^2}+ \Phi\enspace,
    \end{align*}
    where $\Phi:= \frac{6\tau^2}{\mu}\rb{3 K^2 \eta^2 H^2 \zeta^2+  6 K \sigma_2^2 \eta^2} + \frac{6Q^2 }{\mu}\rb{2620\eta^4K^4 H^4 \zeta^4 + 5000\eta^4 K^2\sigma_2^4 + 320\eta^4\sigma_4^4K} + \frac{\eta \sigma_2^2}{M}$. Now for all $t\in[0,T-1]$ define weights $w_t = \rb{1-\frac{\eta\mu}{2}}^{T-1-t}$ and their sum $W = \sum_{t=0}^{T-1}w_t = \frac{2(1 - (1-\eta\mu/2)^T)}{\eta\mu}$. With this in hand we consider the function sub-optimality of the weighted average of the ghost iterates as follows\footnote{Note that while not all ghost iterates are computed at a given time step, we can always compute them post training, i.e., at time step $T$.},
    \begin{align*}
        &\ee\sb{F\rb{\frac{1}{W}\sum_{t\in[0,T-1]}w_tx_t}} - F(x^\star)\\ 
        &\leq^{\text{(Jensen's Inequality)}} \frac{1}{W}\sum_{t\in[0,T-1]}w_t\rb{\ee\sb{F(x_t)} - F(x^\star)}\enspace,\\
        &\leq \frac{1}{\eta W}\sum_{t\in[0,T-1]}w_t\rb{\rb{1 - \frac{\eta\mu}{2}} \ee \sb{\norm{x_t-x^\star}^2} - \ee \sb{\norm{x_{t+1}-x^\star}^2}} +\Phi\enspace,\\
        &= \frac{\mu}{2(1 - (1-\eta\mu/2)^T)}\sum_{t\in[0,T-1]}\rb{w_{t-1}\ee \sb{\norm{x_t-x^\star}^2} - w_t\ee \sb{\norm{x_{t+1}-x^\star}^2}} +\Phi\enspace,\\
        &= \frac{\mu}{2(1 - (1-\eta\mu/2)^T)}\rb{w_0\norm{x^\star}^2 - w_{T-1}\ee \sb{\norm{x_{T}-x^\star}^2}}+\Phi\enspace,\\
        &\leq \frac{\mu (1-\eta\mu/2)^{T-1}}{2(1 - (1-\eta\mu/2)^T)}B^2+\Phi\enspace,\\
        &\leq \mu B^2e^{-\eta\mu T/2}\cdot\frac{1}{(1-\eta\mu/2)2(1 - (1-\eta\mu/2)^T)} + \Phi\enspace.
    \end{align*}
    Now note that $\eta\leq \frac{1}{2H}\leq \frac{1}{2\mu}$ which implies that $\frac{1}{1-\eta\mu/2}\leq \frac{4}{3}$. Furthermore, assuming that the exponential term in the denominator is small, i.e., $\textcolor{red}{(1-\eta\mu/2)^T \leq e^{-\eta\mu T/2} \leq 1/2}$ we can simplify the upper bound as,
    \begin{align*}
        \ee\sb{F\rb{\frac{1}{W}\sum_{t\in[0,T-1]}w_tx_t}} - F(x^\star) &\leq 2\mu B^2e^{-\eta\mu T/2} + \Phi\enspace.
    \end{align*}
    Now to tune the step-size we will use a similar strategy as in the previous lemmas in this section. We first note that all the terms in $\Phi$ are increasing in $\eta$, so we can choose any choice of our step-size in the theorem to bound them. We will choose $\eta = \frac{2}{\mu KR}\ln\rb{\frac{\mu B^2}{\epsilon}}$ and then ignoring logarithmic powers of $\ln\rb{\frac{\mu B^2}{\epsilon}}$ this gives us an upper bound on $\Phi$, which also matches the theorem statement (barring the exponential term) up to numerical constants,
    \begin{align*}
        \frac{6\tau^2}{\mu}\rb{\frac{12 H^2 \zeta^2}{\mu^2R^2}+  \frac{24 \sigma_2^2}{\mu^2KR^2}} + \frac{96Q^2}{\mu}\rb{\frac{2620 H^4 \zeta^4}{\mu^4R^4} + \frac{5000 \sigma_2^4}{\mu^4K^2R^4} + \frac{320\sigma_4^4}{\mu^4K^3R^4}} + \frac{2\sigma_2^2}{\mu MKR} =: \bar\Phi\enspace.
    \end{align*}
    Using this we define our $\epsilon = \max\cb{\bar{\Phi}, \epsilon_{target}}$ where $\epsilon_{target}$ is the target accuracy. Now to bound the exponential term, we again consider two cases,
    \begin{itemize}
        \item When $\frac{1}{2H}\leq \frac{2}{\mu KR}\ln\rb{\frac{\mu B^2}{\epsilon}}$, then we get first term in the convergence rate, so clearly the upper bound in the lemma statement is valid.
        \item When $\frac{1}{2H}\geq \frac{2}{\mu KR}\ln\rb{\frac{\mu B^2}{\epsilon}}$, then the first term $e^{-\eta\mu KR}\mu B^2 = e^{-\ln(\mu B^2/\epsilon)}\mu B^2 = \epsilon$. Since $\epsilon$ always matches the rest of the convergence terms (up to logarithmic factors) or the target accuracy $\epsilon_{target}$ (whichever is larger), we can upper bound the first term in the convergence rate with one of the other terms in the rate. This makes the upper bound in the lemma statement valid. 
    \end{itemize}
    Finally it remains to check how to satisfy the \textcolor{red}{red constraint}. We note that the following two conditions are sufficient to ensure it, for either choice of the step-size,
    \begin{itemize}
        \item If $\eta = \frac{1}{2H}$ then $e^{-KR/4\kappa} \leq \frac{1}{2}$ is implied by assuming $KR \geq 4\kappa\ln(2)$.
        \item If $\eta = \frac{2}{\mu KR}\ln(\mu B^2/\epsilon)$ then $e^{-\ln(\mu B^2/\epsilon)} \leq \frac{1}{2}$ is implies by $\epsilon \leq \frac{\mu B^2}{2}$. 
    \end{itemize}
    We precisely assume these two conditions in the theorem statement which finishes the proof. 
\end{proof}

\subsubsection{Convergence in Function Value in the Convex Setting}
We will finally derive the analogue of the previous lemma, in the general convex setting, i.e., when $\mu=0$. To do so we will use the general convex to strongly convex reduction using $l_2$ regularization. This technique is standard in the literature, see e.g.~\citet{hazan2016introduction}. For the sake of completeness, we repeat the argument in the following proof.
\begin{lemma}[Convergence in Function Value with $\zeta$, $\tau$ and $Q$]
    Assume we have a problem instance satisfying \Cref{ass:convex,ass:smooth_second,ass:smooth_third,ass:bounded_optima,ass:stoch_bounded_second_moment,ass:stoch_bounded_fourth_moment,ass:tau,ass:zeta} and $$R\geq \max \cb{\frac{4}{K}, \frac{\sigma_2^2}{4H^2B^2MK}, \frac{4\tau \zeta}{BH}, \frac{Q\zeta^2}{8HB}, \frac{\tau \sigma_2}{4H^2 B\sqrt{K}}, \frac{\sigma_2 Q^{1/2}}{\sqrt{8BK}H^{3/2}}, \frac{Q^{1/2}\sigma_4}{\sqrt{8B}H^{3/2}K^{3/4}}}\enspace.$$ 
    We run Local SGD to optimize the $\mu$-strongly convex objective $F(x) + \frac{\mu \norm{x}^2}{2}$, where we pick $$\mu = \max\cb{\frac{2}{\eta KR}\ln\rb{2\eta H KR}, \sqrt{\frac{2\Phi'}{B^2}}}\enspace,$$ for
    $$\Phi' := 6\tau^2\rb{3 K^2 \eta^2 H^2 \zeta^2+  6 K \sigma_2^2 \eta^2} + 6Q^2\rb{2620\eta^4K^4 H^4 \zeta^4 + 5000\eta^4 K^2\sigma_2^4 + 320\eta^4\sigma_4^4K}\enspace,$$
    and use the step-size,
    $$\eta = \min \cb{\frac{1}{2H}, \sqrt{\frac{B}{\tau K^2 RH\zeta}}, \sqrt{\frac{B}{\tau K^{3/2}R\sigma_2}}, \sqrt[3]{\frac{B}{QK^3RH^2\zeta^2}}, \sqrt[3]{\frac{B}{QK^2R\sigma_2^2}}, \sqrt[3]{\frac{B}{QK^{3/2}R\sigma_4^2}}, \sqrt{\frac{B^2M}{\sigma_2^2KR}}}\enspace.$$
    Also define the weighted Local SGD iterate $\hat x := \frac{1}{W}\sum_{t\in[0,T-1]}w_{t}x_t$ where $w_t := \rb{1-\frac{\eta\mu}{2}}^{T-1-t}$ and their sum $W := \sum_{t=0}^{T-1}w_t$. Then we can get the following guarantee for $\hat x$ (where $x_0=0$ and $x^\star\in S^\star$),
    \begin{align*}
        \ee\sb{F\rb{\hat x}} - F(x^\star) &= \tilde\ooo\Bigg(\frac{HB^2}{KR} + \frac{\sqrt{\tau H\zeta B^3}}{R^{1/2}} + \frac{\sqrt{\tau\sigma_2B^3}}{K^{1/4}R^{1/2}} + \frac{Q^{1/3}B^{5/3}H^{2/3}\zeta^{2/3}}{R^{2/3}}+ \frac{Q^{1/3}B^{5/3}\sigma_2^{2/3}}{K^{1/3}R^{2/3}}\\ 
        &\qquad\qquad  + \frac{Q^{1/3}B^{5/3}\sigma_4^{2/3}}{K^{1/2}R^{2/3}} + \frac{\sigma_2 B}{\sqrt{MKR}}\Bigg)\enspace.
    \end{align*}
\end{lemma}
\begin{proof}
Let $F(x)$ be a convex function. We construct a regularized version of this function $F_\mu(x)$ as:
\begin{align*}
    F_\mu(x) = F(x) + \frac{\mu}{2} \norm{x-x_0}^2 \enspace,
\end{align*}
where $\mu >0$. Next we define: 
\begin{align*}
    x^\star_\mu &= \underset{x}{\arg\min} \hspace{1mm}F_\mu(x)  \,,\\
    x^\star &\in \underset{x}{\arg\min}\hspace{1mm} F(x) \,.
\end{align*}
Note that we have $F_\mu(x^\star_\mu) \leq F_\mu(x^\star)$. Then we upper bound the function sub-optimality for the function $F$ for some point $\hat x\in\rr^d$:
\begin{align*}
    F(\hat x) - F(x^\star) &= F_\mu(\hat x) - \frac{\mu}{2} \norm{\hat x - x_0}^2 - F_\mu(x^\star) + \frac{\mu}{2} \norm{x^\star - x_0}^2\enspace,\\
    &\leq F_\mu(\hat x)- F_\mu(x^\star)+ \frac{\mu}{2} \norm{x^\star - x_0}^2\enspace,\\
    &\leq F_\mu(\hat x)- F_\mu(x^\star_\mu)+ \frac{\mu}{2} \norm{x^\star - x_0}^2\enspace,\\
    &\leq F_\mu(\hat x)- F_\mu(x^\star_\mu)+ \frac{\mu}{2}B^2\enspace,
\end{align*}
where in the last inequality we use \Cref{ass:bounded_optima} and $x_0=0$ which matches the setting in which we choose to run Local SGD. Since the choice of $\epsilon$ was arbitrary we can tune $\mu$ in the above upper bound to get the tightest possible upper bound. 

We first recall the upper bound from the previous lemma for optimizing $F_\mu$ before tuning $\eta$,
\begin{align*}
    \ee\sb{F_\mu\rb{\frac{1}{W}\sum_{t\in[0,T-1]}w_tx_t}} - F_\mu(x_\mu^\star) &\leq 2\mu B^2e^{-\eta\mu T/2} + \Phi\enspace,\\
    &\leq^{(H\geq \mu)} 2HB^2e^{-\eta\mu T/2} + \frac{1}{\mu}\Phi' + \frac{\eta\sigma_2^2}{M} \enspace,
\end{align*}
where we define $$\Phi' := 6\tau^2\rb{3 K^2 \eta^2 H^2 \zeta^2+  6 K \sigma_2^2 \eta^2} + 6Q^2\rb{2620\eta^4K^4 H^4 \zeta^4 + 5000\eta^4 K^2\sigma_2^4 + 320\eta^4\sigma_4^4K}\enspace,$$ and note that $\Phi'$ does not depend on $\mu$. This leaves us with two terms in the upper bounds to balance with $\frac{\mu B^2}{2}$. To balance the first term with $\frac{\mu B^2}{2}$ we need to choose $\mu_1 = \frac{2}{\eta KR}\ln\rb{2\eta H KR}$. And to balance the second term with $\frac{\mu B^2}{2}$ we choose $\mu_2 = \sqrt{\frac{2\Phi'}{B^2}}$. This motivates us to pick $\mu = \max\cb{\mu_1, \mu_2}$.  
This choice of $\mu$ gives us the following upper bound,
\begin{align*}
    \ee\sb{F\rb{\frac{1}{W}\sum_{t\in[0,T-1]}w_tx_t}} - F(x^\star) &\leq \ee\sb{F_\mu\rb{\frac{1}{W}\sum_{t\in[0,T-1]}w_tx_t}} - F_\mu(x_\mu^\star) + \frac{\mu B^2}{2}\enspace,\\
    &\leq 2HB^2e^{-\eta\mu T/2} + \frac{1}{\mu}\Phi'+ \frac{\mu B^2}{2}+ \frac{\eta\sigma_2^2}{M}\enspace,\\
    &\leq^{\text{(a)}} \textcolor{red}{2HB^2e^{-\eta\mu_1 T/2} + \frac{1}{\mu_2}\Phi' + \frac{\mu_1 B^2}{2}+ \frac{\mu_2 B^2}{2}+ \frac{\eta\sigma_2^2}{M}}\enspace,
\end{align*}
where in order to see why (a) is true, note that there are two cases, 
\begin{itemize}
    \item When $\mu_1\geq \mu_2$, then we pick $\mu = \mu_1$ and $\frac{1}{\mu_2}\Phi' \leq \frac{\mu_2B^2}{2}$, so the second term is upper bounded by the fourth term in the \textcolor{red}{red upper bound}. The first and the third term in the \textcolor{red}{red upper bound} simply appear from the choice of $\mu=\mu_1$. This makes the upper bound valid. 
    \item Similarly, when $\mu_2\geq \mu_1$, then we pick $\mu = \mu_2$ and $2HB^2e^{-\eta\mu_1 T/2}\leq \frac{\mu_1B^2}{2}$, so the first term is upper bounded by the third term in the \textcolor{red}{red upper bound}. The second and the fourth term in the \textcolor{red}{red upper bound} simply appear from the choice of $\mu=\mu_2$. This makes the upper bound valid. 
\end{itemize}
We can further simplify the \textcolor{red}{red upper bound} as follows,
\begin{align*}
    &\ee\sb{F\rb{\frac{1}{W}\sum_{t\in[0,T-1]}w_tx_t}} - F(x^\star)\\ 
    &\quad\leq 2HB^2e^{-\eta\mu_1 T/2} + \frac{1}{\mu_2}\Phi' + \frac{\mu_1 B^2}{2}+ \frac{\mu_2 B^2}{2}+ \frac{\eta\sigma_2^2}{M}\enspace,\\
    &\quad\leq \frac{B^2}{\eta KR}\rb{1 + \ln\rb{2\eta HKR}} + \sqrt{2\Phi'B^2}+ \frac{\eta\sigma_2^2}{M}\enspace,\\
    &\quad\leq^{\text{(Triangle Inequality)}} \frac{B^2}{\eta KR}\rb{1 + \ln\rb{2\eta HKR}} + 6\eta\tau K H\zeta B + 12\eta \tau \sqrt{K}\sigma_2B + 178\eta^2QBK^2H^2\zeta^2\\ 
    &\qquad + 245\eta^2QBK\sigma_2^2 + 62\eta^2 QB\sqrt{K}\sigma_4^2+ \frac{\eta\sigma_2^2}{M}\enspace.
\end{align*}
We again have all terms increasing in $\eta$ except for the first term. Balancing all terms, and recalling that $\eta \leq \frac{1}{2H}$ we choose the step-size as,
\begin{align*}
    \eta &= \min \cb{\frac{1}{2H}, \sqrt{\frac{B}{\tau K^2 RH\zeta}}, \sqrt{\frac{B}{\tau K^{3/2}R\sigma_2}}, \sqrt[3]{\frac{B}{QK^3RH^2\zeta^2}}, \sqrt[3]{\frac{B}{QK^2R\sigma_2^2}}, \sqrt[3]{\frac{B}{QK^{3/2}R\sigma_4^2}}, \sqrt{\frac{B^2M}{\sigma_2^2KR}}}\enspace,\\
    &=: \min\cb{\eta_1,\eta_2,\eta_3,\eta_4,\eta_5,\eta_6,\eta_7}
\end{align*}
Now to get the final convergence for each term which is increasing in $\eta$, we bound it using the step-size choice that balances it with the first term, and for the first term we upper bound it by summing across all possible step-size choices,
\begin{align*}
    &\ee\sb{F\rb{\frac{1}{W}\sum_{t\in[0,T-1]}w_tx_t}} - F(x^\star)\\
    &\quad\leq \sum_{i=1}^{7}\frac{B^2}{\eta_i KR}\rb{1 + \ln\rb{2\eta_iHKR}} + 6\eta_2\tau K H\zeta B + 12\eta_3 \tau \sqrt{K}\sigma_2B + 178\eta_4^2QBK^2H^2\zeta^2\\ 
    &\qquad + 245\eta_5^2QBK\sigma_2^2 + 62\eta_6^2 QB\sqrt{K}\sigma_4^2+ \frac{\eta_7\sigma_2^2}{M}\enspace,\\
    &\quad= \frac{2HB^2}{KR}\rb{1 + \ln\rb{KR}} + \frac{\sqrt{\tau H\zeta B^3}}{R^{1/2}}\rb{7 + \ln\rb{\frac{2\sqrt{BHR}}{\sqrt{\tau \zeta}}}}\\ 
    &\qquad + \frac{\sqrt{\tau\sigma_2B^3}}{K^{1/4}R^{1/2}}\rb{13+ \ln\rb{\frac{2\sqrt{BH^2}K^{1/4}R^{1/2}}{\sqrt{\tau\sigma_2}}}}+ \frac{Q^{1/3}B^{5/3}H^{2/3}\zeta^{2/3}}{R^{2/3}}\rb{179 + \ln\rb{\frac{2\sqrt[3]{HBR^2}}{Q^{1/3}\zeta^{2/3}}}} \\
    &\qquad  + \frac{Q^{1/3}B^{5/3}\sigma_2^{2/3}}{K^{1/3}R^{2/3}}\rb{246 + \ln\rb{\frac{2HK^{1/3}R^{2/3}B^{1/3}}{Q^{1/3}\sigma_2^{2/3}}}} + \frac{Q^{1/3}B^{5/3}\sigma_4^{2/3}}{K^{1/2}R^{2/3}}\rb{63 + \ln\rb{\frac{2HK^{1/2}R^{2/3}B^{1/3}}{Q^{2/3}\sigma_4^{2/3}}}}\\
    &\qquad  + \frac{\sigma_2 B}{\sqrt{MKR}}\rb{2 + \ln\rb{\frac{2HB\sqrt{MKR}}{\sigma_2}}}\enspace.
\end{align*}
It is worth noting that in the above upper bound in the interesting regimes when $Q,\ \tau,\ \zeta,\ \sigma_2,\ \sigma_4$ tend to zero or $K,\ R$ tend to infinity, the bound doesn't blow up. This would be important when we discuss extreme regimes in the main body of the thesis. Now ignoring the numerical constants and the logarithmic terms results in the following rate,
\begin{align*}
    \ee\sb{F\rb{\frac{1}{W}\sum_{t\in[0,T-1]}w_tx_t}} - F(x^\star)&= \tilde\ooo\Bigg(\frac{HB^2}{KR} + \frac{\sqrt{\tau H\zeta B^3}}{R^{1/2}} + \frac{\sqrt{\tau\sigma_2B^3}}{K^{1/4}R^{1/2}} + \frac{Q^{1/3}B^{5/3}H^{2/3}\zeta^{2/3}}{R^{2/3}}\\ 
    &\quad + \frac{Q^{1/3}B^{5/3}\sigma_2^{2/3}}{K^{1/3}R^{2/3}} + \frac{Q^{1/3}B^{5/3}\sigma_4^{2/3}}{K^{1/2}R^{2/3}} + \frac{\sigma_2 B}{\sqrt{MKR}}\Bigg)\enspace,
\end{align*}
which proves the lemma's upper bound. As a final step, recall that the previous lemma's proof assumed that $(1-\eta\mu/2)^{T/2}\leq \frac{1}{2}$. To see how to ensure this note that it is enough to prove $e^{-\eta \mu KR/4}\leq \frac{1}{2}$. This is a decreasing function in $\eta$ and $\mu$. Since we pick maximum value of $\mu$ out of $\mu_1$ and $\mu_2$, it is enough to show that $e^{-\eta \mu_1 KR/4}\leq \frac{1}{2}$, i.e., $e^{-\ln(2\eta HKR)/2} \leq 1/2$. Simplifying further, this reduces to $\eta \geq \frac{2}{HKR}$. Now potential choice of $\eta$ will result in some constraint as follows:
\begin{enumerate}
    \item $\eta_1 \geq \frac{2}{HKR}$ which gives the constraint $KR \geq 4$;
    \item $\eta_2 \geq \frac{2}{HKR}$ which gives the constraint $R\geq \frac{4\tau \zeta}{BH}$;
    \item $\eta_3 \geq \frac{2}{HKR}$ which gives the constraint $K^{1/2}R \geq \frac{\tau \sigma_2}{4H^2 B}$;
    \item $\eta_4 \geq \frac{2}{HKR}$ which gives the constraint $R\geq \frac{Q\zeta^2}{8HB}$;
    \item $\eta_5 \geq \frac{2}{HKR}$ which gives the constraint $K^{1/2}R \geq \frac{\sigma_2 Q^{1/2}}{\sqrt{8B}H^{3/2}}$;
    \item $\eta_6 \geq \frac{2}{HKR}$ which gives the constraint $K^{3/4}R \geq \frac{Q^{1/2}\sigma_4}{\sqrt{8B}H^{3/2}}$; and
    \item $\eta_7 \geq \frac{2}{HKR}$ which gives the constraint $KR \geq \frac{\sigma_2^2}{4H^2B^2M}$.
\end{enumerate}
This finishes the proof.
\end{proof}


\section{Double Recursions for Consensus Error}\label{app:double_recursions}
In this section, we will relate the consensus error to the iterate errors of the previous communication round. This would allow us to obtain more fine-grained upper bounds on consensus error, which would decay over time and with increased communication. More importantly, this would allow us to remove the dependence on $\zeta$, i.e., \Cref{ass:zeta}.

\subsection{Second Moment of the Consensus Error}\label{app:second_consensus}
We can prove the following bound on the second moment of the consensus error using \Cref{ass:zeta_star,ass:tau}.
\begin{lemma}\label{lem:consensus_error_second_recursion}
Assume we have a problem instance satisfying \Cref{ass:strongly_convex,ass:smooth_second,ass:stoch_bounded_second_moment,ass:zeta_star,ass:phi_star,ass:tau} with continuously doubly differentiable objective functions. Then for all $t\in[0,T]$ assuming $\eta < \frac{1}{H}$ we have for the Local SGD iterates,
    \begin{align*}
        C(t)&\leq 2\eta^2 H^2(t-\delta(t))^2\zeta_\star^2 + \frac{2\eta^3\tau^2(t-\delta(t))^2)\sigma_2^2}{\mu} + 2\eta^2\sigma_2^2(t-\delta(t))\ln(t-\delta(t))\\
        &\quad + 4\eta^2\tau^2\rb{t-\delta(t)}^2(1-\eta \mu)^{2(t-1-\delta(t))}\rb{A(\delta(t)) + \phi_\star^2}\enspace,\\
    &\leq 2\eta^2 H^2K^2\zeta_\star^2 + \frac{2\eta^3\tau^2K^2\sigma_2^2}{\mu} + 2\eta^2\sigma_2^2K\ln(K)+ 4\eta^2\tau^2K^2\rb{A(\delta(t)) + \phi_\star^2}\enspace.
    \end{align*}
    This also implies that for $r\in[R]$,
    \begin{align*}
        \sum_{j=K(r-1)}^{Kr-1}(1-\eta\mu)^{Kr-1-j}C(j) &\leq \frac{1-(1-\eta\mu)^K}{\eta\mu}\rb{2\eta^2 H^2K^2\zeta_\star^2 + \frac{2\eta^3\tau^2K^2\sigma_2^2}{\mu} + 2\eta^2\sigma_2^2K\ln(K)}\\
        &\quad + \frac{1-(1-\eta\mu)^K}{\eta\mu}4\eta^2\tau^2K^2(1-\eta\mu)^{K-2}\rb{A(K(r-1)) + \phi_\star^2}\enspace.
    \end{align*}
\end{lemma}
\begin{proof}
Note the following about the difference of iterates on two machines $m,n\in[M]$ for some time $t>\delta(t)$ (for $t=\delta(t)$ the l.h.s. is zero),
\begin{align*}
    \ee\sb{\norm{x^m_t - x^n_t}^2} &= \ee\sb{\norm{x^m_{t-1} - x^n_{t-1} - \eta g^m_{t-1} + \eta g^n_{t-1}}^2}\enspace,\\
    &\leq^{\text{(\Cref{lem:stoch_diff_second}), (a)}} \ee\sb{\norm{x^m_{t-1} - x^n_{t-1} - \eta \left(\nabla F_m (x^m_{t-1}) - \nabla F_n (x^n_{t-1})\right)}^2} + 2\eta^2 \sigma_2^2\enspace,\\
    &=^{\text{(b)}} \ee\sb{\norm{x^m_{t-1} - x^n_{t-1} - \eta \left( \nabla F_m (x^m_{t-1}) - \nabla F_m (x^n_{t-1}) \right) - \eta \left(\nabla F_m (x^n_{t-1})- \nabla F_n (x^n_{t-1} \right)}^2}\\ 
    &\quad + 2\eta^2 \sigma_2^2\enspace,
\end{align*}
where in (a) we exploited the fact that $\xi_t^m\perp\xi_t^n|\hhh_t$ and $x_{t-1}^m, x_{t-1}^n\in m\hhh_t$ as well as used tower rule to introduce conditional expectation; and in (b) we added and subtracted the term $\nabla F_m(x^n_{t-1})$. By mean value theorem we know that there exists a $c = x^n_{t-1} + \theta (x^m_{t-1} - x^n_{t-1})$ for some $\theta \in [0,1]$ such that: 
\begin{align*}
    \nabla F_m (x^m_{t-1}) - \nabla F_m (x^n_{t-1}) = \nabla^2 F_m(c) (x^m_{t-1} - x^n_{t-1})
\end{align*}
Using this in the above inequality, we get: 
\begin{align*}
    \ee\sb{\norm{x^m_t - x^n_t}^2} &\leq \ee\sb{\norm{x^m_{t-1} - x^n_{t-1} - \eta \nabla^2 F_m(c) (x^m_{t-1} - x^n_{t-1})- \eta \left(\nabla F_m (x^n_{t-1})- \nabla F_n (x^n_{t-1}) \right)}^2}+ 2\eta^2 \sigma_2^2\enspace,\\
    &= \ee\sb{\norm{\left(I - \eta \nabla^2 F_m(c)\right) (x^m_{t-1} - x^n_{t-1}) - \eta \left(\nabla F_m (x^n_{t-1})- \nabla F_n (x^n_{t-1}) \right)}^2} + 2\eta^2 \sigma_2^2\enspace,\\
    &\leq^{\text{(a)}} \rb{1+\frac{1}{\gamma_{t-1}}}\ee\sb{\norm{\left(I - \eta \nabla^2 F_m(c)\right) (x^m_{t-1} - x^n_{t-1})}^2}\\ 
    &\quad + \rb{1+ \gamma_{t-1}}\eta^2 \ee\sb{\norm{\nabla F_m (x^n_{t-1})- \nabla F_n (x^n_{t-1})}^2} + 2\eta^2 \sigma_2^2\enspace,\\
    &\leq^{\text{(b)}} \rb{1+\frac{1}{\gamma_{t-1}}}(1-\eta \mu)^2 \ee\sb{\norm{x^m_{t-1} - x^n_{t-1}}^2}\\ 
    &\quad + \rb{1+ \gamma_{t-1}}\eta^2 \ee\sb{\norm{\nabla F_m (x^n_{t-1})- \nabla F_n (x^n_{t-1})}^2} + 2\eta^2 \sigma_2^2\enspace,\\
    &=\rb{1+\frac{1}{\gamma_{t-1}}}(1-\eta \mu)^2\ee\sb{\norm{x^m_{t-1} - x^n_{t-1}}^2} + 2\eta^2 \sigma_2^2\\ 
    &\quad + \rb{1+ \gamma_{t-1}}\eta^2 \ee\sb{\norm{\nabla F_m (x^n_{t-1}) - \nabla F_m (x_n^\star) - \nabla F_n (x^n_{t-1}) + \nabla F_m (x_n^\star)}^2}\enspace,\\
    &\leq^{\text{(c)}} \rb{1+\frac{1}{\gamma_{t-1}}}(1-\eta \mu)^2\ee\sb{\norm{x^m_{t-1} - x^n_{t-1}}^2}\\ 
    &\quad +2\rb{1+ \gamma_{t-1}}\eta^2  \ee\sb{\norm{\nabla F_m (x^n_{t-1}) - \nabla F_m (x_n^\star) - \nabla F_n (x^n_{t-1}) + \nabla F_n(x_n^\star)}^2}\\
    &\quad + 2\rb{1+ \gamma_{t-1}}\eta^2 \ee\sb{\norm{\nabla F_m(x_n^\star) - \nabla F_m(x_m^\star)}^2} + 2\eta^2 \sigma_2^2\enspace,\\
    &\leq^{\text{(\Cref{ass:smooth_second,ass:zeta_star})}} \rb{1+\frac{1}{\gamma_{t-1}}}(1-\eta \mu)^2\ee\sb{\norm{x^m_{t-1} - x^n_{t-1}}^2}\\
    &\quad + \textcolor{blue}{2\rb{1+ \gamma_{t-1}}\eta^2\ee\sb{\norm{\nabla F_m (x^n_{t-1}) - \nabla F_n (x^n_{t-1}) - \rb{\nabla F_m (x_n^\star) - \nabla F_n(x_n^\star)}}^2}}\\
    &\quad + 2\rb{1+ \gamma_{t-1}}\eta^2H^2\zeta_{\star,m,n}^2 + 2\eta^2\sigma_2^2\enspace,
\end{align*}
where in (a) and (c) we used \Cref{lem:mod_am_gm}; and in (b) we used \Cref{ass:strongly_convex} and the fact that $\eta < 1/H$. We will again use the mean value theorem for the \textcolor{blue}{blue} term in the above inequality. For $v := x^n_{t-1} - x_n^\star$ we have: 
\begin{align*}
    &\nabla F_m (x^n_{t-1}) - \nabla F_n (x^n_{t-1}) - \rb{\nabla F_m (x_n^\star) - \nabla F_n(x_n^\star)}\\
    &\qquad = \int_0^1 \rb{\nabla^2 F_m(x_n^\star + tv) - \nabla^2 F_n(x_n^\star + tv)}v dt \enspace,\\
    &\Rightarrow \norm{\nabla F_m (x^n_{t-1}) - \nabla F_n (x^n_{t-1}) - \rb{\nabla F_m (x_n^\star) - \nabla F_n(x_n^\star)}}\\
    &\qquad\leq^{(\text{\Cref{ass:smooth_second,ass:smooth_third}})} \int_0^1 \norm{\nabla^2 F_m(x_n^\star + tv) - \nabla^2 F_n(x_n^\star + tv)}\norm{v} dt \enspace,\\
    &\Rightarrow \norm{\nabla F_m (x^n_{t-1}) - \nabla F_n (x^n_{t-1}) - \rb{\nabla F_m (x_n^\star) - \nabla F_n(x_n^\star)}} \leq^{(\text{\Cref{ass:tau}})} \tau\norm{x_{t-1}^n - x_n^\star}\enspace,
\end{align*}
where in the first implication above we use the fact that (i) $F_m(x_n^\star + tv) - F_n(x_n^\star + tv)$ is twice-continuously-differentiable\footnote{Recall that this is implied by \Cref{ass:smooth_third} as well.}; (ii) $\norm{\nabla^2 F_m(\cdot) - \nabla^2 F_n(\cdot)}$ is upper bounded due to \Cref{ass:smooth_second} and (iii) we are integrating over a finite domain which implies that the the function $\norm{\rb{\nabla^2 F_m(\cdot) - \nabla^2 F_n(\cdot)}v}$ is integrable over the finite domain $[0,1]$.

Plugging this into the inequality above gives the following,
\begin{align*}
    &\ee\sb{\norm{x^m_t - x^n_t}^2} \\
    &\leq \rb{1+\frac{1}{\gamma_{t-1}}}(1-\eta \mu)^2\ee\sb{\norm{x^m_{t-1} - x^n_{t-1}}^2}\\
    &\quad + 2\rb{1+ \gamma_{t-1}}\eta^2\tau^2\ee\sb{\norm{x_{t-1}^n - x_n^\star}^2} + 2\rb{1+ \gamma_{t-1}}\eta^2H^2\zeta_{\star,m,n}^2 + 2\eta^2\sigma_2^2\enspace,\\
    &\leq \rb{1+\frac{1}{\gamma_{t-1}}}(1-\eta \mu)^2\ee\sb{\norm{x^m_{t-1} - x^n_{t-1}}^2} + 2\rb{1+ \gamma_{t-1}}\eta^2H^2\zeta_{\star,m,n}^2 + 2\eta^2\sigma_2^2\\
    &\quad + 2\rb{1+ \gamma_{t-1}}\eta^2\tau^2\textcolor{red}{\rb{(1-\eta\mu)^{2(t-1-\delta(t))}\ee\sb{\norm{x_{\delta(t)} - x_n^\star}^2} + \rb{1- (1-\eta\mu)^{2(t-1-\delta(t))}}\frac{\eta\sigma_2^2}{\mu}}}\enspace,
    \end{align*}
    where in the last inequality above we just used an upper bound for the convergence of SGD on a single machine $n\in[M]$ (cf. \Cref{lem:single_second_recursion}). As a sanity check note that if $t-1=\delta(t)$ then the \textcolor{red}{red} term becomes $\ee\sb{\norm{x_{\delta(t)} - x_n^\star}^2}$. Continuing further and choosing $\gamma_j = j-\delta(j)$ (note that the term with $1/\gamma_{t-1}$ becomes becomes zero when $t-1=\delta(t)$ as $x_{t-1}^m = x_{t-1}^n$, making this choice well defined), this leads to,
    \begin{align}
    &\ee\sb{\norm{x^m_t - x^n_t}^2} \nonumber\\
    &\leq \prod_{j=\delta(t)}^{t-1}\rb{1 + \frac{1}{\gamma_{j}}}(1-\eta \mu)^2\ee\sb{\norm{x_{\delta(t)} - x_{\delta(t)}}^2}\nonumber\\
    &\quad + 2\eta^2\sum_{j=\delta(t)}^{t-1}\rb{\prod_{i=j+1}^{t-1}\rb{1 + \frac{1}{\gamma_{i}}}(1-\eta \mu)^2}\rb{(1+\gamma_{j})H^2\zeta_{\star,m,n}^2 + \sigma_2^2}\nonumber\\
    &\quad + 2\eta^2\tau^2\sum_{j=\delta(t)}^{t-1}\rb{\prod_{i=j+1}^{t-1}\rb{1 + \frac{1}{\gamma_{i}}}(1-\eta \mu)^2}\rb{1+ \gamma_{j}}(1-\eta\mu)^{2(j-\delta(t))}\ee\sb{\norm{x_{\delta(t)} - x_n^\star}^2}\nonumber\\
    &\quad + \frac{2\eta^3\tau^2\sigma_2^2}{\mu}\sum_{j=\delta(t)}^{t-1}\rb{\prod_{i=j+1}^{t-1}\rb{1 + \frac{1}{\gamma_{i}}}(1-\eta \mu)^2}\rb{1+ \gamma_{j}} \rb{1- (1-\eta\mu)^{2(j-\delta(t))}}\enspace,\nonumber\\
    &=  2\eta^2\sum_{j=\delta(t)}^{t-1}\rb{\prod_{i=j+1}^{t-1}\rb{1 + \frac{1}{\gamma_{i}}}}(1-\eta \mu)^{2(t-1-j)}\rb{(1+\gamma_{j})H^2\zeta_{\star,m,n}^2 + \sigma_2^2}\nonumber\\
    &\quad + 2\eta^2\tau^2\sum_{j=\delta(t)}^{t-1}\rb{\prod_{i=j+1}^{t-1}\rb{1 + \frac{1}{\gamma_{i}}}}(1-\eta \mu)^{2(t-1-j)}\rb{1+ \gamma_{j}}(1-\eta\mu)^{2(j-\delta(t))}\ee\sb{\norm{x_{\delta(t)} - x_n^\star}^2}\nonumber\\
    &\quad + \frac{2\eta^3\tau^2\sigma_2^2}{\mu}\sum_{j=\delta(t)}^{t-1}\rb{\prod_{i=j+1}^{t-1}\rb{1 + \frac{1}{\gamma_{i}}}}(1-\eta \mu)^{2(t-1-j)}\rb{1+ \gamma_{j}} \rb{1- (1-\eta\mu)^{2(j-\delta(t))}}\enspace,\nonumber\\
    &= 2\eta^2\sum_{j=\delta(t)}^{t-1}\rb{\prod_{i=j+1}^{t-1}\frac{i+1-\delta(t)}{i-\delta(t)}}(1-\eta \mu)^{2(t-1-j)}\rb{(j+1-\delta(t))H^2\zeta_{\star,m,n}^2 + \sigma_2^2}\nonumber\\
    &\quad + 2\eta^2\tau^2\sum_{j=\delta(t)}^{t-1}\rb{\prod_{i=j+1}^{t-1}\frac{i+1-\delta(t)}{i-\delta(t)}}\rb{j+1-\delta(t)}(1-\eta \mu)^{2(t-1-\delta(t))}\ee\sb{\norm{x_{\delta(t)} - x_n^\star}^2}\nonumber\\
    &\quad + \frac{2\eta^3\tau^2\sigma_2^2}{\mu}\sum_{j=\delta(t)}^{t-1}\rb{\prod_{i=j+1}^{t-1}\frac{i+1-\delta(t)}{i-\delta(t)}}\rb{j+1-
    \delta(t)} \rb{(1-\eta \mu)^{2(t-1-j)}- (1-\eta\mu)^{2(t-1-\delta(t))}}\enspace,\nonumber\\
    &= 2\eta^2\sum_{j=\delta(t)}^{t-1}\frac{t-\delta(t)}{j+1-\delta(t)}(1-\eta \mu)^{2(t-1-j)}\rb{(j+1-\delta(t))H^2\zeta_{\star,m,n}^2 + \sigma_2^2}\nonumber\\
    &\quad + 2\eta^2\tau^2\sum_{j=\delta(t)}^{t-1}\rb{t-\delta(t)}(1-\eta \mu)^{2(t-1-\delta(t))}\ee\sb{\norm{x_{\delta(t)} - x_n^\star}^2}\nonumber\\
    &\quad + \frac{2\eta^3\tau^2\sigma_2^2}{\mu}\sum_{j=\delta(t)}^{t-1}\rb{t-\delta(t)}\rb{(1-\eta \mu)^{2(t-1-j)}- (1-\eta\mu)^{2(t-1-\delta(t))}}\enspace,\nonumber\\
     &= 2\eta^2(t-\delta(t))\sum_{j=\delta(t)}^{t-1}(1-\eta \mu)^{2(t-1-j)}\rb{H^2\zeta_{\star,m,n}^2 + \frac{\sigma_2^2}{j+1-\delta(t)}}\nonumber\\
     &\quad + 2\eta^2\tau^2\rb{t-\delta(t)}^2(1-\eta \mu)^{2(t-1-\delta(t))}\ee\sb{\norm{x_{\delta(t)} - x_n^\star}^2}\nonumber\\
    &\quad + \frac{2\eta^3\tau^2\sigma_2^2}{\mu}\sum_{j=\delta(t)}^{t-1}\rb{t-\delta(t)}\rb{(1-\eta \mu)^{2(t-1-j)}- (1-\eta\mu)^{2(t-1-\delta(t))}}\enspace,\nonumber\\
    &\leq \textcolor{red}{2\eta^2(t-\delta(t))^2H^2\zeta_{\star,m,n}^2} + 2\eta^2(t-\delta(t))\sigma_2^2\sum_{j=\delta(t)}^{t-1}\frac{1}{j+1-\delta(t)}\nonumber\\
    &\quad + 2\eta^2\tau^2\rb{t-\delta(t)}^2(1-\eta \mu)^{2(t-1-\delta(t))}\ee\sb{\norm{x_{\delta(t)} - x_n^\star}^2} + \textcolor{red}{\frac{2\eta^3\tau^2\sigma_2^2\rb{t-\delta(t)}^2}{\mu}}\enspace,\nonumber\\
    &\leq^{\text{(a)}} \textcolor{red}{(t-\delta(t))^2\rb{2\eta^2 H^2\zeta_{\star,m,n}^2 + \frac{2\eta^3\tau^2\sigma_2^2}{\mu}}} + 2\eta^2\sigma_2^2(t-\delta(t))\ln(t-\delta(t))\nonumber\\
    &\quad + 4\eta^2\tau^2\rb{t-\delta(t)}^2(1-\eta \mu)^{2(t-1-\delta(t))}\rb{\ee\sb{\norm{x_{\delta(t)} - x^\star}^2} + \norm{x^\star-x_n^\star}^2}\enspace,\nonumber\\
    &\leq^{(\text{\Cref{ass:phi_star}})} (t-\delta(t))\rb{2\eta^2 H^2K\zeta_{\star,m,n}^2 + \frac{2\eta^3\tau^2K\sigma_2^2}{\mu} + 2\eta^2\sigma_2^2\ln(K)}\nonumber\\
    &\quad + 4\eta^2\tau^2\rb{t-\delta(t)}^2(1-\eta \mu)^{2(t-1-\delta(t))}\rb{A(\delta(t)) + \phi_{\star,n}^2}\enspace,\label{eq:consensus_second_specific}\\
    &\leq 2\eta^2 H^2K^2\zeta_{\star,m,n}^2 + \frac{2\eta^3\tau^2K^2\sigma_2^2}{\mu} + 2\eta^2\sigma_2^2K\ln(K) + 4\eta^2\tau^2\rb{t-\delta(t)}^2(1-\eta \mu)^{2(t-1-\delta(t))}\rb{A(\delta(t)) + \phi_{\star,n}^2}\enspace,\nonumber
\end{align}
where in (a) we combined the \textcolor{red}{red} terms into one, and used the fact that $$\frac{1-(1-\eta\mu)^{2(t-\delta(t))}}{1-(1-\eta\mu)^2} \leq \frac{1}{\eta\mu(2-\eta\mu)}\leq \frac{1}{\eta\mu}\enspace,$$ because $\eta < 1/H$ and used \Cref{lem:mod_am_gm}. As a sanity check, note that the above bound has the property that when $t=\delta(t)$, it automatically becomes zero (we adopt the notation that $0\cdot(-\infty)$ in the second term becomes $0$). Thus, we can safely drop the assumption that $t>\delta(t)$, making the above bound valid for all values of $t$. Finally, averaging the upper bound over $m,n\in[M]$ and noting the last few inequalities in the calculation proves the lemma's main upper bound. To get the other result, we will use this upper bound with some simplifications. In particular noting that $\delta(j)=K(r-1)$ we get,
\begin{align*}
    &\sum_{j=K(r-1)}^{Kr-1}(1-\eta\mu)^{Kr-1-j}C(j)\\
    &\leq \sum_{j=K(r-1)}^{Kr-1}(1-\eta\mu)^{Kr-1-j}\rb{2\eta^2 H^2K^2\zeta_\star^2 + \frac{2\eta^3\tau^2K^2\sigma_2^2}{\mu} + 2\eta^2\sigma_2^2K\ln(K)}\\
    &\quad + \sum_{j=K(r-1)}^{Kr-1}(1-\eta\mu)^{Kr-1-j}\rb{4\eta^2\tau^2\rb{j-K(r-1)}^2(1-\eta \mu)^{2(j-1-K(r-1))}\rb{A(K(r-1)) + \phi_\star^2}}\enspace,\\
    &= \frac{1-(1-\eta\mu)^K}{\eta\mu}\rb{2\eta^2 H^2K^2\zeta_\star^2 + \frac{2\eta^3\tau^2K^2\sigma_2^2}{\mu} + 2\eta^2\sigma_2^2K\ln(K)}\\
    &\quad + 4\eta^2\tau^2(1-\eta\mu)^{K-2}\rb{A(K(r-1)) + \phi_\star^2}\sum_{j=K(r-1)}^{Kr-1}(1-\eta\mu)^{j-1-K(r-1)}\rb{j-K(r-1)}^2\enspace,\\
    &\leq \frac{1-(1-\eta\mu)^K}{\eta\mu}\rb{2\eta^2 H^2K^2\zeta_\star^2 + \frac{2\eta^3\tau^2K^2\sigma_2^2}{\mu} + 2\eta^2\sigma_2^2K\ln(K)}\\
    &\quad + \frac{1-(1-\eta\mu)^K}{\eta\mu}4\eta^2\tau^2K^2(1-\eta\mu)^{K-2}\rb{A(K(r-1)) + \phi_\star^2}\enspace.
\end{align*}
This finishes the proof. 
\end{proof}


\subsection{Fourth Moment of the Consensus Error}\label{app:fourth_consensus}
\begin{lemma}\label{lem:consensus_error_fourth_recursion}
    Assume we have a problem instance satisfying \Cref{ass:strongly_convex,ass:smooth_second,ass:stoch_bounded_second_moment,ass:zeta_star,ass:phi_star,ass:tau} with continuously doubly differentiable objective functions. Then for all $t\in[0,T]$ assuming $\eta < 1/H$ we have,
    \begin{align*}
        &D(t)\leq \rb{\frac{128\eta^5\tau^4\sigma_2^2}{\mu}(t-\delta(t)) + 320\eta^4\sigma_2^2\tau^2 }(t-\delta(t))^3(1-\eta\mu)^{t-1-\delta(t)}\rb{A(\delta(t)) + \phi_\star^2}\\
        &\quad + 64\eta^4\tau^4(t-\delta(t))^4(1-\eta\mu)^{t-1-\delta(t)}\rb{B(\delta(t)) + \phi_\star^4}\\
        &\quad + \rb{\frac{8\eta^3H^4\zeta_\star^4}{\mu} + \frac{88\eta^5\tau^4\sigma_4^4}{\mu^3} + 160\eta^4K\sigma_2^2H^2\zeta_\star^2 + \frac{160\eta^5\tau^2K\sigma_2^4}{\mu} + 112\eta^4\sigma_4^4\ln(K)}(t-\delta(t))^3\enspace,\\
        &\quad\leq \rb{\frac{128\eta^5\tau^4K^4\sigma_2^2}{\mu} + 320\eta^4\sigma_2^2\tau^2K^3}\rb{A(\delta(t)) + \phi_\star^2} + 64\eta^4\tau^4K^4\rb{B(\delta(t)) + \phi_\star^4}\\
        &\quad + \frac{8\eta^3K^3H^4\zeta_\star^4}{\mu} + \frac{88\eta^5K^3\tau^4\sigma_4^4}{\mu^3} + 160\eta^4K^4\sigma_2^2H^2\zeta_\star^2 + \frac{160\eta^5\tau^2K^4\sigma_2^4}{\mu} + 112\eta^4K^3\sigma_4^4\ln(K)\enspace.
    \end{align*}
    This also implies that for $r\in[R]$,
    \begin{align*}
    &\sum_{j=K(r-1)}^{Kr-1}(1-\eta\mu)^{Kr-1-j}D(j)\\
    &\leq \rb{1-\rb{1-\eta\mu}^K}\rb{\frac{128\eta^4K^4\tau^4\sigma_2^2}{\mu^2} + \frac{320\eta^3K^3\sigma_2^2\tau^2}{\mu}}(1-\eta\mu)^{K-3}\rb{A(K(r-1)) + \phi_\star^2}\\
    &\quad + \rb{1-\rb{1-\eta\mu}^K}\frac{64\eta^3K^4\tau^4}{\mu}(1-\eta\mu)^{K-5}\rb{B(K(r-1)) + \phi_\star^4}\\
    &\quad \rb{1-\rb{1-\eta\mu}^K}\\
    &\quad \times\rb{\frac{8\eta^2K^3H^4\zeta_\star^4}{\mu^2} + \frac{88\eta^4K^3\tau^4\sigma_4^4}{\mu^4} + \frac{160\eta^3K^4\sigma_2^2H^2\zeta_\star^2}{\mu} + \frac{160\eta^4\tau^2K^4\sigma_2^4}{\mu^2} + \frac{112\eta^3K^3\sigma_4^4\ln(K)}{\mu}}\enspace.
    \end{align*}
\end{lemma}
\begin{proof}
    Note the following about the fourth moment of the difference between the iterates on two machines $m,n\in[M]$ for $t>\delta(t)$ (for $t=\delta(t)$ the l.h.s. is zero),
    \begin{align*}
        &\ee\sb{\norm{x_t^m - x_t^n}^4}\\
        &= \ee\sb{\norm{x_{t-1}^m - x_{t-1}^n - \eta g_{t-1}^m +\eta g_{t-1}^n}^4}\enspace,\\
        &= \ee\sb{\rb{\norm{x_{t-1}^m - x_{t-1}^n - \eta \nabla F_m(x_{t-1}^m) + \eta \nabla F_n(x_{t-1}^n) + \eta \xi_{t-1}^m - \eta \xi_{t-1}^n}^2}^2}\enspace,\\
        &= \ee\Bigg[\bigg(\norm{x_{t-1}^m - x_{t-1}^n - \eta \nabla F_m(x_{t-1}^m) + \eta \nabla F_n(x_{t-1}^n)}^2 + \eta^2 \norm{\xi_{t-1}^m -\xi_{t-1}^n}^2\\
        &\quad + 2\eta\inner{x_{t-1}^m - x_{t-1}^n - \eta \nabla F_m(x_{t-1}^m) + \eta \nabla F_n(x_{t-1}^n)}{\xi_{t-1}^m - \xi_{t-1}^n}\bigg)^2\Bigg]\enspace,\\
        &=^{\text{(a)}} \ee\sb{\norm{x_{t-1}^m - x_{t-1}^n - \eta \nabla F_m(x_{t-1}^m) + \eta \nabla F_n(x_{t-1}^n)}^4} + \eta^4 \ee\sb{\norm{\xi_{t-1}^m -\xi_{t-1}^n}^4}\\
        &\quad +4\eta^2\ee\sb{\rb{\inner{x_{t-1}^m - x_{t-1}^n - \eta \nabla F_m(x_{t-1}^m) + \eta \nabla F_n(x_{t-1}^n)}{\xi_{t-1}^m - \xi_{t-1}^n}}^2}\\
        &\quad + 2\eta^2\ee\sb{\norm{x_{t-1}^m - x_{t-1}^n - \eta \nabla F_m(x_{t-1}^m) + \eta \nabla F_n(x_{t-1}^n)}^2\norm{\xi_{t-1}^m -\xi_{t-1}^n}^2}\\
        &\quad +4\eta^3\ee\sb{\inner{x_{t-1}^m - x_{t-1}^n - \eta \nabla F_m(x_{t-1}^m) + \eta \nabla F_n(x_{t-1}^n)}{\xi_{t-1}^m - \xi_{t-1}^n}\norm{\xi_{t-1}^m - \xi_{t-1}^n}^2}\enspace,\\
        &\leq^{\text{(\Cref{lem:stoch_diff_fourth}), (b)}} \ee\sb{\norm{x_{t-1}^m - x_{t-1}^n - \eta \nabla F_m(x_{t-1}^m) + \eta \nabla F_n(x_{t-1}^n)}^4} + 8\eta^4\sigma_4^4\\
        &\quad + 6\eta^2\ee\sb{\norm{x_{t-1}^m - x_{t-1}^n - \eta \nabla F_m(x_{t-1}^m) + \eta \nabla F_n(x_{t-1}^n)}^2}\ee\sb{\norm{\xi_{t-1}^m -\xi_{t-1}^n}^2}\\
        &\quad +4\eta^3\textcolor{blue}{\ee\sb{\norm{x_{t-1}^m - x_{t-1}^n - \eta \nabla F_m(x_{t-1}^m) + \eta \nabla F_n(x_{t-1}^n)}}}\textcolor{red}{\ee\sb{\norm{\xi_{t-1}^m - \xi_{t-1}^n}^3}}\enspace,
\end{align*}
where in (a) we use the fact that $\ee\sb{\xi_{t-1}^m - \xi_{t-1}^n|\hhh_{t-1}} = 0$ and the conditional indepence of stochastic noise i.e., $\cb{\xi_{t-1}^m,\xi_{t-1}^n}~\perp~\cb{x_{t-1}^m, x_{t-1}^n}~|~\hhh_{t-1}$ allowing us to ignore one of the terms while expanding the square; and in (b) we again used this fact about the randomness along with an application of Cauchy Shwartz inequality.

In order to bound the term $\textcolor{red}{\ee \sb{\norm{\xi^m_{t-1} - \xi^n_{t-1}}^3}}$ we use Cauchy-Schwarz Inequality:
\begin{align*}
    \ee \sb{\norm{\xi^m_{t-1} - \xi^n_{t-1}}^3} &= \ee \sb{\norm{\xi^m_{t-1} - \xi^n_{t-1}}\cdot \norm{\xi^m_{t-1} - \xi^n_{t-1}}^2} \\
    &\leq \sqrt{\ee \sb{\norm{\xi^m_{t-1} - \xi^n_{t-1}}^2} \ee \sb{\norm{\xi^m_{t-1} - \xi^n_{t-1}}^4}}\enspace,\\
    &\leq^{\text{(\Cref{lem:stoch_diff_second,lem:stoch_diff_fourth})}} 4\sqrt{\sigma_2^2\sigma_4^4} = \textcolor{red}{4\sigma_2\sigma_4^2}\enspace.
\end{align*}
Also the term $\textcolor{blue}{\ee\sb{\norm{x_{t-1}^m - x_{t-1}^n - \eta \nabla F_m(x_{t-1}^m) + \eta \nabla F_n(x_{t-1}^n)}}}$ can be bounded as\footnote{Technically to apply Jensen the right hand side should be finite. While we do not explicitly show this, by for instance using \Cref{ass:bounded_gradients}, we use this without loss of generality because if the upper bound is not finite it would be reflected in our guarantees.}: 
\begin{align*}
    &\ee\sb{\norm{x_{t-1}^m - x_{t-1}^n - \eta \nabla F_m(x_{t-1}^m) + \eta \nabla F_n(x_{t-1}^n)}}\\ 
    &\quad \stackrel{\text{(Jensen's Inequality)}}{\leq} \textcolor{blue}{\sqrt{\ee\sb{\norm{x_{t-1}^m - x_{t-1}^n - \eta \nabla F_m(x_{t-1}^m) + \eta \nabla F_n(x_{t-1}^n)}^2}}}\enspace.
\end{align*}
Putting everything back together gives us: 
\begin{align*}
        \ee \sb{\norm{x^n_t - x^m_t}^4} &\leq^{\text{(\Cref{lem:stoch_diff_second})}} \ee\sb{\norm{x_{t-1}^m - x_{t-1}^n - \eta \nabla F_m(x_{t-1}^m) + \eta \nabla F_n(x_{t-1}^n)}^4} + 8\eta^4 \sigma_4^4\\
        &\quad +12\eta^2\sigma_2^2\ee\sb{\norm{x_{t-1}^m - x_{t-1}^n - \eta \nabla F_m(x_{t-1}^m) + \eta \nabla F_n(x_{t-1}^n)}^2}\\
        &\quad +16\eta^3\textcolor{orange}{\sqrt{\sigma_2^2\sigma_4^4 \ee\sb{\norm{x_{t-1}^m - x_{t-1}^n - \eta \nabla F_m(x_{t-1}^m) + \eta \nabla F_n(x_{t-1}^n)}^2}}}\enspace.
\end{align*}
To bound the third term in the above inequality, we use the A.M. - G.M. Inequality $\sqrt{ab} \leq \frac{a}{2\gamma} + \frac{\gamma b}{2}$ for $\gamma >0$. Let $\gamma = \eta, a=\sigma_2^2 \ee\sb{\norm{x_{t-1}^m - x_{t-1}^n - \eta \nabla F_m(x_{t-1}^m) + \eta \nabla F_n(x_{t-1}^n)}^2}, b=\sigma_4^4$. We have:
\begin{align*}
    &16\eta^3\sqrt{\sigma_2^2\sigma_4^4 \ee\sb{\norm{x_{t-1}^m - x_{t-1}^n - \eta \nabla F_m(x_{t-1}^m) + \eta \nabla F_n(x_{t-1}^n)}^2}}\\
    &= 16\eta^3\sqrt{(\sigma_4^4) \left(\sigma_2^2 \ee\sb{\norm{x_{t-1}^m - x_{t-1}^n - \eta \nabla F_m(x_{t-1}^m) + \eta \nabla F_n(x_{t-1}^n)}^2}\right)} \\
    &\leq \textcolor{orange}{16\eta^3 \left(\frac{\eta \sigma_4^4}{2} + \frac{\sigma_2^2}{2\eta}\ee\sb{\norm{x_{t-1}^m - x_{t-1}^n - \eta \nabla F_m(x_{t-1}^m) + \eta \nabla F_n(x_{t-1}^n)}^2} \right)}
\end{align*}
Plugging this upper bound and following a similar strategy as in \Cref{lem:consensus_error_second_recursion} we get 
\begin{align*}
        &\ee \sb{\norm{x^n_t - x^m_t}^4} \\
        &\leq \ee\sb{\norm{x_{t-1}^m - x_{t-1}^n - \eta \nabla F_m(x_{t-1}^m) + \eta \nabla F_n(x_{t-1}^n)}^4} + 8\eta^4 \sigma_4^4\\ 
        &\quad +12\eta^2\sigma_2^2\ee\sb{\norm{x_{t-1}^m - x_{t-1}^n - \eta \nabla F_m(x_{t-1}^m) + \eta \nabla F_n(x_{t-1}^n)}^2}\\
        &\quad +16\eta^3\rb{\frac{\eta\sigma_4^4}{2} +  \frac{\sigma_2^2}{2\eta}\ee\sb{\norm{x_{t-1}^m - x_{t-1}^n - \eta \nabla F_m(x_{t-1}^m) + \eta \nabla F_n(x_{t-1}^n)}^2}}\enspace,\\
        &= \ee\sb{\norm{x_{t-1}^m - x_{t-1}^n - \eta \nabla F_m(x_{t-1}^m) + \eta \nabla F_n(x_{t-1}^n)}^4}\\ 
        &\quad +20\eta^2\sigma_2^2\ee\sb{\norm{x_{t-1}^m - x_{t-1}^n - \eta \nabla F_m(x_{t-1}^m) + \eta \nabla F_n(x_{t-1}^n)}^2} + 16\eta^4 \sigma_4^4\enspace,\\
        &= \ee\sb{\norm{x_{t-1}^m - x_{t-1}^n - \eta \nabla F_m(x_{t-1}^m) + \eta \nabla F_m(x_{t-1}^n) - \eta \nabla F_m(x_{t-1}^n) + \eta \nabla F_n(x_{t-1}^n)}^4} + 16\eta^4 \sigma_4^4\\
        &\quad +20\eta^2\sigma_2^2\ee\sb{\norm{x_{t-1}^m - x_{t-1}^n - \eta \nabla F_m(x_{t-1}^m) + \eta \nabla F_m(x_{t-1}^n) - \eta \nabla F_m(x_{t-1}^n) + \eta \nabla F_n(x_{t-1}^n)}^2}\enspace,\\
        &\leq^{\text{(\Cref{lem:mod_am_gm}), (a)}} \rb{1+\frac{1}{\gamma_{t-1}}}^3(1-\eta\mu)^4\ee\sb{\norm{x_{t-1}^m - x_{t-1}^n}^4}\\ 
        &\quad + \rb{1+\gamma_{t-1}}^3\eta^4\ee\sb{\norm{\nabla F_m(x_{t-1}^n)- \nabla F_n(x_{t-1}^n) - \nabla F_m(x_n^\star) + \nabla F_m(x_n^\star)}^4} \\
        &\quad +20\eta^2\sigma_2^2\rb{1+\frac{1}{\gamma_{t-1}}}(1-\eta\mu)^2\ee\sb{\norm{x_{t-1}^m - x_{t-1}^n}^2}\\
        &\quad + 20\eta^4\sigma_2^2(1+\gamma_{t-1})\ee\sb{\norm{\nabla F_m(x_{t-1}^n)- \nabla F_n(x_{t-1}^n) - \nabla F_m(x_n^\star) + \nabla F_m(x_n^\star)}^2} + 16\eta^4 \sigma_4^4\enspace,\\
        &\leq^{\text{(\Cref{lem:mod_am_gm,ass:zeta_star,ass:tau}), (b)}} \rb{1+\frac{1}{\gamma_{t-1}}}^3(1-\eta\mu)^4\ee\sb{\norm{x_{t-1}^m - x_{t-1}^n}^4}\\
        &\quad + 8\rb{1+\gamma_{t-1}}^3\eta^4\rb{\tau^4\ee\sb{\norm{x_{t-1}^n - x_n^\star}^4} + H^4\zeta_{\star,m,n}^4} +20\eta^2\sigma_2^2\rb{1+\frac{1}{\gamma_{t-1}}}(1-\eta\mu)^2\ee\sb{\norm{x_{t-1}^m - x_{t-1}^n}^2}\\
        &\quad + 40\eta^4\sigma_2^2(1+\gamma_{t-1})\rb{\tau^2\ee\sb{\norm{x_{t-1}^n - x_n^\star}^2} + H^2\zeta_{\star,m,n}^2} + 16\eta^4 \sigma_4^4\enspace,
\end{align*}
where in (a) and (b) we again use mean-value theorem along with \Cref{ass:strongly_convex,ass:tau} just as in the proof of \Cref{lem:consensus_error_second_recursion}.
Averaging this over $m,n\in[M]$ we have for all $t>\delta(t)$,
\begin{align*}
    D(t) &\leq \rb{1+\frac{1}{\gamma_{t-1}}}^3(1-\eta\mu)^4D(t-1) + 8\rb{1+\gamma_{t-1}}^3\eta^4\tau^4\textcolor{red}{\frac{1}{M}\sum_{n\in[M]}\ee\sb{\norm{x_{t-1}^n - x_n^\star}^4}}\\
    &\quad + 8\rb{1+\gamma_{t-1}}^3\eta^4H^4\zeta_\star^4 +20\eta^2\sigma_2^2\rb{1+\frac{1}{\gamma_{t-1}}}(1-\eta\mu)^2\textcolor{orange}{C(t-1)}\\
    &\quad + 40\eta^4\sigma_2^2\tau^2(1+\gamma_{t-1})\textcolor{blue}{\frac{1}{M}\sum_{n\in[M]}\ee\sb{\norm{x_{t-1}^n - x_n^\star}^2}} + 40\eta^4\sigma_2^2(1+\gamma_{t-1})H^2\zeta_\star^2 + 16\eta^4 \sigma_4^4\enspace.
\end{align*}
Now we will use a couple of upper bounds that we already have for $\textcolor{red}{\ee\sb{\norm{x_{t-1}^n - x_n^\star}^4}}$ from \Cref{lem:single_fourth_recursion}, $\textcolor{blue}{\ee\sb{\norm{x_{t-1}^n - x_n^\star}^2}}$ from \Cref{lem:single_second_recursion} and $\textcolor{orange}{C(t-1)}$ for $t-1 \geq \delta(t)$ from \eqref{eq:consensus_second_specific} in the proof of \Cref{lem:consensus_error_second_recursion}. This gives us the following with $\gamma_j = j-\delta(j)=j-\delta(t)$ for $j\geq \delta(t)$,
\begin{align*}
    &D(t)\\ 
    &\leq \rb{1+\frac{1}{\gamma_{t-1}}}^3(1-\eta\mu)^4D(t-1) + 8\rb{1+\gamma_{t-1}}^3\eta^4H^4\zeta_\star^4 + \textcolor{red}{\rb{1+\gamma_{t-1}}^3\frac{88\eta^6\tau^4\sigma_4^4}{\mu^2}}\\
    &\quad + \textcolor{red}{8\rb{1+\gamma_{t-1}}^3\eta^4\tau^4\frac{1}{M}\sum_{n\in[M]}\rb{(1-\eta\mu)^{4(t-1-\delta(t))}\ee\sb{\norm{x_{\delta(t)} - x_n^\star}^4}}}\\
    &\quad + \textcolor{red}{8\rb{1+\gamma_{t-1}}^3\eta^4\tau^4\frac{1}{M}\sum_{n\in[M]}\rb{ 8\eta^2\sigma_2^2(t-1-\delta(t))(1-\eta\mu)^{2(t-1-\delta(t))}\ee\sb{\norm{x_{\delta(t)} - x_n^\star}^2}}}\\
    &\quad +\textcolor{orange}{20\eta^2\sigma_2^2\rb{1+\frac{1}{\gamma_{t-1}}}(1-\eta\mu)^2C(t-1)}\\ 
    &\quad + \textcolor{blue}{40\eta^4\sigma_2^2\tau^2(1+\gamma_{t-1})\frac{1}{M}\sum_{n\in[M]}\rb{(1-\eta\mu)^{2(t-1-\delta(t))}\ee\sb{\norm{x_{\delta(t)} - x_n^\star}^2} + \frac{\eta\sigma_2^2}{\mu}}}\\
    &\quad + 40\eta^4\sigma_2^2(1+\gamma_{t-1})H^2\zeta_\star^2 + 16\eta^4 \sigma_4^4\enspace,\\
    &\leq^{\text{(\Cref{lem:mod_am_gm,ass:phi_star,eq:consensus_second_specific})}} \rb{1+\frac{1}{\gamma_{t-1}}}^3(1-\eta\mu)^4D(t-1) + 8\rb{1+\gamma_{t-1}}^3\eta^4H^4\zeta_\star^4\\
    &\quad + 64\rb{1+\gamma_{t-1}}^3\eta^4\tau^4\rb{(1-\eta\mu)^{4(t-1-\delta(t))}\rb{B(\delta(t))+\phi_\star^4}} + \rb{1+\gamma_{t-1}}^3\frac{88\eta^6\tau^4\sigma_4^4}{\mu^2}\\
    &\quad+ 128\rb{1+\gamma_{t-1}}^3\eta^4\tau^4\rb{\eta^2\sigma_2^2(t-1-\delta(t))(1-\eta\mu)^{2(t-1-\delta(t))}\rb{A(\delta(t)) + \phi_\star^2}}\\
    &\quad+\textcolor{orange}{20\eta^2\sigma_2^2\rb{1+\frac{1}{\gamma_{t-1}}}(1-\eta\mu)^2\rb{4\eta^2\tau^2\rb{t-1-\delta(t)}^2(1-\eta \mu)^{2(t-2-\delta(t))}\rb{A(\delta(t)) + \phi_\star^2}}}\\
    &\quad+\textcolor{orange}{20\eta^2\sigma_2^2\rb{1+\frac{1}{\gamma_{t-1}}}(1-\eta\mu)^2\rb{2(t-1-\delta(t))\rb{\eta^2K H^2\zeta_\star^2 + \frac{\eta^3\tau^2K\sigma_2^2}{\mu} + \eta^2\sigma_2^2\ln(K)}}}\\
    &\quad+ 40\eta^4\sigma_2^2\tau^2(1+\gamma_{t-1})\rb{2(1-\eta\mu)^{2(t-1-\delta(t))}\rb{A(\delta(t)) + \phi_\star^2} + \frac{\eta\sigma_2^2}{\mu}}\\
    &\quad+ 40\eta^4\sigma_2^2(1+\gamma_{t-1})H^2\zeta_\star^2 + 16\eta^4 \sigma_4^4\enspace,\\
    &\leq \prod_{j=\delta(t)}^{t-1}\rb{1 + \frac{1}{\gamma_{j}}}^3(1-\eta \mu)^4\cancelto{0}{\ee\sb{\norm{x_{\delta(t)} - x_{\delta(t)}}^4}}\\
    &\quad + \rb{8\eta^4H^4\zeta_\star^4 + \frac{88\eta^6\tau^4\sigma_4^4}{\mu^2}}\sum_{j=\delta(t)}^{t-1}\rb{\prod_{i=j+1}^{t-1}\rb{1 + \frac{1}{\gamma_{i}}}^3(1-\eta \mu)^4}(1+\gamma_j)^3\\
    &\quad + 64\eta^4\tau^4\rb{B(\delta(t)) + \phi_\star^4}\sum_{j=\delta(t)}^{t-1}\rb{\prod_{i=j+1}^{t-1}\rb{1 + \frac{1}{\gamma_{i}}}^3(1-\eta \mu)^4}(1+\gamma_j)^3(1-\eta\mu)^{4(j-\delta(t))}\\
    &\quad + 128\eta^6\tau^4\sigma_2^2\rb{A(\delta(t)) + \phi_\star^2}\\
    &\qquad\qquad\qquad\times\sum_{j=\delta(t)}^{t-1}\rb{\prod_{i=j+1}^{t-1}\rb{1 + \frac{1}{\gamma_{i}}}^3(1-\eta \mu)^4}(1+\gamma_j)^3(j-\delta(t))(1-\eta\mu)^{2(j-\delta(t))}\\
    &\quad + 80\eta^4\sigma_2^2\tau^2\rb{A(\delta(t)) + \phi_\star^2}\sum_{j=\delta(t)}^{t-1}\rb{\prod_{i=j+1}^{t-1}\rb{1 + \frac{1}{\gamma_{i}}}^3(1-\eta \mu)^4}\rb{1 + \gamma_j}(j-\delta(t))(1-\eta\mu)^{2(j-\delta(t))}\\
    &\quad + 40\eta^2\sigma_2^2\rb{\eta^2K H^2\zeta_\star^2 + \frac{\eta^3\tau^2K\sigma_2^2}{\mu} + \eta^2\sigma_2^2\ln(K)}\sum_{j=\delta(t)}^{t-1}\rb{\prod_{i=j+1}^{t-1}\rb{1 + \frac{1}{\gamma_{i}}}^3(1-\eta \mu)^4}(1-\eta\mu)^2\rb{1+\gamma_j}\\
    &\quad + 80\eta^4\sigma_2^2\tau^2\rb{A(\delta(t)) + \phi_\star^2}\sum_{j=\delta(t)}^{t-1}\rb{\prod_{i=j+1}^{t-1}\rb{1 + \frac{1}{\gamma_{i}}}^3(1-\eta \mu)^4}(1+\gamma_j)(1-\eta\mu)^{2(j-
    \delta(t))}\\
    &\quad + 40\eta^4\sigma_2^2\rb{\frac{\eta \tau^2\sigma_2^2}{\mu} + H^2\zeta_\star^2}\sum_{j=\delta(t)}^{t-1}\rb{\prod_{i=j+1}^{t-1}\rb{1 + \frac{1}{\gamma_{i}}}^3(1-\eta \mu)^4}(1+\gamma_j)\\
    &\quad + 16\eta^4\sigma_4^4\sum_{j=\delta(t)}^{t-1}\rb{\prod_{i=j+1}^{t-1}\rb{1 + \frac{1}{\gamma_{i}}}^3(1-\eta \mu)^4}\enspace,\\
    &= \rb{8\eta^4H^4\zeta_\star^4 + \frac{88\eta^6\tau^4\sigma_4^4}{\mu^2}}\sum_{j=\delta(t)}^{t-1}(t-\delta(t))^3(1-\eta\mu)^{4(t-1-j)}\\
    &\quad + 64\eta^4\tau^4\rb{B(\delta(t)) + \phi_\star^4}\sum_{j=\delta(t)}^{t-1}(t-\delta(t))^3(1-\eta\mu)^{4(t-1-\delta(t))}\\
    &\quad + 128\eta^6\tau^4\sigma_2^2\rb{A(\delta(t)) + \phi_\star^2}\sum_{j=\delta(t)}^{t-1}(t-\delta(t))^3(j-\delta(t))(1-\eta\mu)^{4(t-1)-2j -2\delta(t))}\\
    &\quad + 80\eta^4\sigma_2^2\tau^2\rb{A(\delta(t)) + \phi_\star^2}\sum_{j=\delta(t)}^{t-1}\frac{(t-\delta(t))^3}{(j+1-\delta(t))^2}(1-\eta\mu)^{4(t-1)-2j -2\delta(t)}\\
    &\quad + 40\eta^2\sigma_2^2\rb{\eta^2K H^2\zeta_\star^2 + \frac{\eta^3\tau^2K\sigma_2^2}{\mu} + \eta^2\sigma_2^2\ln(K)}\sum_{j=\delta(t)}^{t-1}\frac{(t-\delta(t))^3}{(j+1-\delta(t))^2}(1-\eta\mu)^{4(t-j)-2}\\
    &\quad + 80\eta^4\sigma_2^2\tau^2\rb{A(\delta(t)) + \phi_\star^2}\sum_{j=\delta(t)}^{t-1}\frac{(t-\delta(t))^3}{(j+1-\delta(t))^2}(1-\eta\mu)^{4(t-1) -2j -2\delta(t)}\\
    &\quad + 40\eta^4\sigma_2^2\rb{\frac{\eta \tau^2\sigma_2^2}{\mu} + H^2\zeta_\star^2}\sum_{j=\delta(t)}^{t-1}\frac{(t-\delta(t))^3}{(j+1-\delta(t))^2}(1-\eta\mu)^{4(t-1-j)}\\
    &\quad + 16\eta^4\sigma_4^4\sum_{j=\delta(t)}^{t-1}\frac{(t-\delta(t))^3}{(j+1-\delta(t))^3}(1-\eta\mu)^{4(t-1-j)}\enspace,\\
    &\leq^{\text{(a)}} \rb{\frac{8\eta^3H^4\zeta_\star^4}{\mu} + \frac{88\eta^5\tau^4\sigma_4^4}{\mu^3}}(t-\delta(t))^3 + 64\eta^4\tau^4\rb{B(\delta(t)) + \phi_\star^4}(t-\delta(t))^4(1-\eta\mu)^{4(t-1-\delta(t))}\\
    &\quad + \frac{128\eta^5\tau^4\sigma_2^2}{\mu}\rb{A(\delta(t)) + \phi_\star^2}(t-\delta(t))^4(1-\eta\mu)^{2(t-1-\delta(t))}\\
    &\quad + \textcolor{cyan}{160\eta^4\sigma_2^2\tau^2\rb{A(\delta(t)) + \phi_\star^2}(t-\delta(t))^3(1-\eta\mu)^{2(t-1-\delta(t))}}\\
    &\quad + 80\eta^2\sigma_2^2\rb{\textcolor{red}{\eta^2 H^2K\zeta_\star^2} + \textcolor{blue}{\frac{\eta^3\tau^2K\sigma_2^2}{\mu}} + \textcolor{orange}{\eta^2\sigma_2^2\ln(K)}}(t-\delta(t))^3\\
    &\quad + \textcolor{cyan}{160\eta^4\sigma_2^2\tau^2\rb{A(\delta(t)) + \phi_\star^2}(t-\delta(t))^3(1-\eta\mu)^{2(t-1-\delta(t))}} + 80\eta^4\sigma_2^2\rb{\textcolor{blue}{\frac{\eta \tau^2\sigma_2^2}{\mu}} + \textcolor{red}{H^2\zeta_\star^2}}(t-\delta(t))^3\\
    &\quad  + \textcolor{orange}{32\eta^4\sigma_4^4(t-\delta(t))^3}\enspace,\\
    &\leq^{\text{(b)}} \rb{\frac{8\eta^3H^4\zeta_\star^4}{\mu} + \frac{88\eta^5\tau^4\sigma_4^4}{\mu^3} + \textcolor{red}{160\eta^4K\sigma_2^2H^2\zeta_\star^2} + \textcolor{blue}{\frac{160\eta^5\tau^2K\sigma_2^4}{\mu}} + \textcolor{orange}{112\eta^4\sigma_4^4\ln(K)}}(t-\delta(t))^3\\ 
    &\quad + 64\eta^4\tau^4(t-\delta(t))^4(1-\eta\mu)^{4(t-1-\delta(t))}\rb{B(\delta(t)) + \phi_\star^4}\\
    &\quad + \rb{\frac{128\eta^5\tau^4\sigma_2^2}{\mu}(t-\delta(t)) + \textcolor{cyan}{320\eta^4\sigma_2^2\tau^2}}(t-\delta(t))^3(1-\eta\mu)^{2(t-1-\delta(t))}\rb{A(\delta(t)) + \phi_\star^2}\enspace,
\end{align*}
where in (a) we used that $\sum_{j=\delta(t)}^{t-1}\frac{1}{(j+1-\delta(t))^3} < \sum_{j=\delta(t)}^{t-1}\frac{1}{(j+1-\delta(t))^2} \leq \frac{\pi^2}{6} < 2$; in (b) we used that $\eta < 1/H\leq 1/\mu$ to get the \textcolor{red}{red} and \textcolor{blue}{blue} terms. This finishes the proof of the lemma, once we note that when $t=\delta(t)$, the upper bound is zero, which means we can extend the proof to $t\geq \delta(t)$, which essentially means all $t$.

We can now use this bound to give the following bound for $r\in[R]$,
\begin{align*}
    &\sum_{j=K(r-1)}^{Kr-1}(1-\eta\mu)^{Kr-1-j}D(j)\\
    &\quad\leq \frac{128\eta^5\tau^4\sigma_2^2}{\mu}\rb{A(K(r-1)) + \phi_\star^2}\sum_{j=K(r-1)}^{Kr-1}(1-\eta\mu)^{Kr-1-j}(j-K(r-1))^4(1-\eta\mu)^{2(j-1-K(r-1))}\\
    &\quad + 320\eta^4\sigma_2^2\tau^2\rb{A(K(r-1)) + \phi_\star^2}\sum_{j=K(r-1)}^{Kr-1}(1-\eta\mu)^{Kr-1-j}(j-K(r-1))^3(1-\eta\mu)^{2(j-1-K(r-1))}\\
    &\quad + 64\eta^4\tau^4\rb{B(K(r-1)) + \phi_\star^4}\sum_{j=K(r-1)}^{Kr-1}(1-\eta\mu)^{Kr-1-j} (j-K(r-1))^4(1-\eta\mu)^{4(j-1-K(r-1))}\\
    &\quad \rb{\frac{8\eta^3H^4\zeta_\star^4}{\mu} + \frac{88\eta^5\tau^4\sigma_4^4}{\mu^3} + 160\eta^4K\sigma_2^2H^2\zeta_\star^2 + \frac{160\eta^5\tau^2K\sigma_2^4}{\mu} + 112\eta^4\sigma_4^4\ln(K)}\\
    &\qquad\qquad\qquad\times\sum_{j=K(r-1)}^{Kr-1}(1-\eta\mu)^{Kr-1-j}(j-K(r-1))^3\enspace,\\
    &\quad\leq \frac{128\eta^5K^4\tau^4\sigma_2^2}{\mu}(1-\eta\mu)^{K-3}\rb{A(K(r-1)) + \phi_\star^2}\sum_{j=K(r-1)}^{Kr-1}(1-\eta\mu)^{j-K(r-1)}\\
    &\quad + 320\eta^4K^3\sigma_2^2\tau^2(1-\eta\mu)^{K-3}\rb{A(K(r-1)) + \phi_\star^2}\sum_{j=K(r-1)}^{Kr-1}(1-\eta\mu)^{j-K(r-1)}\\
    &\quad + 64\eta^4K^4\tau^4(1-\eta\mu)^{K-5}\rb{B(K(r-1)) + \phi_\star^4}\sum_{j=K(r-1)}^{Kr-1}(1-\eta\mu)^{3(j-K(r-1))}\\
    &\quad \rb{\frac{8\eta^3K^3H^4\zeta_\star^4}{\mu} + \frac{88\eta^5K^3\tau^4\sigma_4^4}{\mu^3} + 160\eta^4K^4\sigma_2^2H^2\zeta_\star^2 + \frac{160\eta^5\tau^2K^4\sigma_2^4}{\mu} + 112\eta^4K^3\sigma_4^4\ln(K)}\\
    &\qquad\qquad\qquad\times\sum_{j=K(r-1)}^{Kr-1}(1-\eta\mu)^{Kr-1-j}\enspace,\\
    &\leq \rb{1-\rb{1-\eta\mu}^K}\frac{128\eta^4K^4\tau^4\sigma_2^2}{\mu^2}(1-\eta\mu)^{K-3}\rb{A(K(r-1)) + \phi_\star^2}\\
    &\quad + \rb{1-\rb{1-\eta\mu}^K}\frac{320\eta^3K^3\sigma_2^2\tau^2}{\mu}(1-\eta\mu)^{K-3}\rb{A(K(r-1)) + \phi_\star^2}\\
    &\quad + \rb{1-\rb{1-\eta\mu}^K}\frac{64\eta^3K^4\tau^4}{\mu}(1-\eta\mu)^{K-5}\rb{B(K(r-1)) + \phi_\star^4}\\
    &\quad \rb{1-\rb{1-\eta\mu}^K}\\
    &\quad \times\rb{\frac{8\eta^2K^3H^4\zeta_\star^4}{\mu^2} + \frac{88\eta^4K^3\tau^4\sigma_4^4}{\mu^4} + \frac{160\eta^3K^4\sigma_2^2H^2\zeta_\star^2}{\mu} + \frac{160\eta^4\tau^2K^4\sigma_2^4}{\mu^2} + \frac{112\eta^3K^3\sigma_4^4\ln(K)}{\mu}}\enspace,
\end{align*}
which proves the claim.
\end{proof}

\subsection{Should Consensus Error Explode for a Large Step-size?}\label{app:explode}
Note that the results in \Cref{lem:consensus_error_second_recursion,lem:consensus_error_fourth_recursion} suggest that when $K\to\infty$ we must pick $\eta = \ooo\rb{\frac{1}{K}}$ so that the consensus error does not explode. This small step-size was criticized by Wang et al. ~\cite{wang2022unreasonable} through experiments, which showed that even without such a small step-size, the consensus error did not blow up in the regime of large $K$. In the following lemma we show that even with $\eta= \theta\rb{\frac{1}{H}}$, consensus error does not blow up, and saturates to a value that depends on the data heterogeneity \Cref{ass:zeta_star,ass:phi_star,ass:tau}. The lemma relies on just the evolution of iterates on a single machine and the fact that it is decoupled between communication rounds. 

\begin{lemma}[Alternative Bounds on the Consensus Error ]
    Assume we have a problem instance satisfying \Cref{ass:strongly_convex,ass:smooth_second,ass:stoch_bounded_second_moment,ass:zeta_star,ass:phi_star,ass:tau} . Then for any $t\geq \delta(t)$ with $\eta < 1/H$ we have,
    \begin{align*}
        C(t) &\leq 12(1-\eta\mu)^{2(t-\delta(t))}\rb{A(\delta(t)) +\phi_\star^2} + \frac{6\eta\sigma_2^2}{\mu}+ 3\zeta_\star^2\enspace,\\ 
        D(t) &\leq 432(1-\eta\mu)^{3(t-\delta(t))}\rb{B(\delta(t)) + \phi_\star^4} + \frac{864\eta\sigma_4^4}{\mu^3} + 27\zeta_\star^4\enspace.
    \end{align*}
    In particular, when $t-\delta(t)\to\infty$ the upper bounds converge to $\frac{6\eta\sigma^2}{\mu}+ 3\zeta_\star^2$ and $\frac{864\eta\sigma^4}{\mu^3} + 27\zeta_\star^4$ respectively.
\end{lemma}

\begin{proof}
    We note that for any and $m,n\in[M]$
    \begin{align*}
        \ee\sb{\norm{x_t^m - x_t^n}^2} &= \ee\sb{\norm{x_{t}^m -x_m^\star - x_{t}^n + x_n^\star + x_m^\star - x_n^\star}^2}\enspace,\\
        &\leq^{\text{(\Cref{lem:mod_am_gm_three_terms,ass:zeta_star})}} 3\ee\sb{\norm{x_{t}^m -x_m^\star}^2} + 3\ee\sb{\norm{x_{t}^n -x_n^\star}^2} + 3\zeta_{\star,m,n}^2\enspace,\\
        &\leq^{\text{(\Cref{lem:single_second_recursion})}} 3\rb{(1-\eta\mu)^{2(t-\delta(t))}\ee\sb{\norm{x_{\delta(t)} -x_m^\star}^2}+ \frac{\eta\sigma_2^2}{\mu}}\\
        &\quad + 3\rb{(1-\eta\mu)^{2(t-\delta(t))}\ee\sb{\norm{x_{\delta(t)} -x_n^\star}^4}+ \frac{\eta\sigma_2^2}{\mu}}+ 3\zeta_{\star,m,n}^2\enspace.
    \end{align*}
    Averaging this over $m,n\in[M]$,
    \begin{align*}
        C(t) &\leq 6(1-\eta\mu)^{2(t-\delta(t))}\frac{1}{M}\sum_{m\in[M]}\ee\sb{\norm{x_{\delta(t)} -x_m^\star}^2} + \frac{6\eta\sigma_2^2}{\mu}+ 3\zeta_\star^2\enspace,\\
        &\leq^{(\text{\Cref{lem:mod_am_gm}})} 12(1-\eta\mu)^{2(t-\delta(t))}\rb{\ee\sb{\norm{x_{\delta(t)} -x^\star}^2} +\phi_\star^2} + \frac{6\eta\sigma_2^2}{\mu}+ 3\zeta_\star^2\enspace,\\
        &= 12(1-\eta\mu)^{2(t-\delta(t))}\rb{A(\delta(t)) +\phi_\star^2} + \frac{6\eta\sigma_2^2}{\mu}+ 3\zeta_\star^2\enspace,
    \end{align*}
    which proves the first statement. 
    
    For the second result, we similarly note that for any $m,n\in[M]$ and $t\in[0,T]$,
    \begin{align*}
        \ee\sb{\norm{x_t^m - x_t^n}^4} &= \ee\sb{\norm{x_{t}^m -x_m^\star - x_{t}^n + x_n^\star + x_m^\star - x_n^\star}^4}\enspace,\\
        &\leq^{\text{(\Cref{lem:mod_am_gm_three_terms,ass:zeta_star})}} 27\ee\sb{\norm{x_{t}^m -x_m^\star}^4} + 27\ee\sb{\norm{x_{t}^n -x_n^\star}^4} + 27 \zeta_{\star,m,n}^4\enspace,\\
        &\leq^{\text{(\Cref{lem:single_fourth_recursion})}} 27\rb{(1-\eta\mu)^{3(t-\delta(t))}\ee\sb{\norm{x_{\delta(t)} -x_m^\star}^4}+ \frac{16\eta\sigma_4^4}{\mu^3}}\\
        &\quad + 27\rb{(1-\eta\mu)^{3(t-\delta(t))}\ee\sb{\norm{x_{\delta(t)} -x_n^\star}^4}+ \frac{16\eta\sigma_4^4}{\mu^3}}+ 27 \zeta_{\star,m,n}^4\enspace.
    \end{align*}
    Averaging this over $m,n\in[M]$,
    \begin{align*}
        D(t) &\leq  54(1-\eta\mu)^{3(t-\delta(t))}\frac{1}{M}\sum_{m\in[M]}\ee\sb{\norm{x_{\delta(t)} -x_m^\star}^4} + 27 \zeta_\star^4 + \frac{864\eta\sigma_4^4}{\mu^3}\enspace,\\
        &\leq^{\text{(\Cref{lem:mod_am_gm,ass:phi_star})}} 432(1-\eta\mu)^{3(t-\delta(t))}\rb{\ee\sb{\norm{x_{\delta(t)} -x^\star}^4} + \phi_\star^4} + 27 \zeta_\star^4 + \frac{864\eta\sigma_4^4}{\mu^3} \enspace,\\
        &= 432(1-\eta\mu)^{3(t-\delta(t))}\rb{A(\delta(t)) + \phi_\star^4}+ 27 \zeta_\star^4 + \frac{864\eta\sigma_4^4}{\mu^3}\enspace,
    \end{align*}
    which proves the second statement of the lemma. 
\end{proof}

The reason we do not use the above lemma over \Cref{lem:consensus_error_second_recursion,lem:consensus_error_fourth_recursion}, is that our step-size tuning in \Cref{app:together} dictates that we anyways need to use $\eta = \ooo\rb{\frac{1}{\mu KR}}$ to get our convergence guarantees which puts the issue of an exploding consensus error to rest. Having said that the above lemma offers reconciliation with the observations by Wang et al.~\cite{wang2022unreasonable} in the regime when $\eta = \theta\rb{\frac{1}{H}}$.

\section{Putting it All Together}\label{app:together}
In this section, we will combine the one-step recursions as well as the consensus error upper bounds that we developed in \Cref{app:double_recursions,app:second_consensus,app:fourth_consensus}.
\subsection{Convergence in Iterates without Third-order Smoothness}\label{app:together_1}
This subsection will essentially combine the weaker \textcolor{blue}{blue} upper bound from \Cref{lem:iterate_error_second_recursion} with the consensus error upper bound from \Cref{lem:consensus_error_second_recursion}. This would lead to an inequality that we can unroll across communication rounds.
\begin{lemma}
    Assume we have a problem instance satisfying \Cref{ass:strongly_convex,ass:smooth_second,ass:stoch_bounded_second_moment,ass:bounded_optima,ass:zeta_star,ass:phi_star,ass:tau}. Then using Local SGD with $\eta<1/H$ and such that $\rho_1=\rb{1-\eta\mu}^K + \rb{1-\rb{1-\eta\mu}^K}\frac{4\eta^2H^2\tau^2}{\mu^2}K^2(1-\eta\mu)^{K-2}<1$ we can get the following convergence guarantee with initialization $x_0=0$,
    \begin{align*}
        A(KR) &\leq \rho_1^RB^2 + \frac{1-\rb{1-\eta\mu}^K}{1-\rho_1}\cdot\frac{\eta\sigma_2^2}{\mu M} + \frac{1-\rb{1-\eta\mu}^K}{1-\rho_1}\cdot\frac{4\eta^2\tau^2H^2K^2(1-\eta\mu)^{K-2} \phi_\star^2}{\mu^2}\\
        &\quad + \frac{1-(1-\eta\mu)^K}{1-\rho_1}\rb{\frac{2\eta^2 H^4K^2\zeta_\star^2}{\mu^2} + \frac{2\eta^3H^2\tau^2K^2\sigma_2^2}{\mu^3} + \frac{2\eta^2H^2\sigma_2^2K\ln(K)}{\mu^2}}\enspace.
    \end{align*}
\end{lemma}
\begin{proof}
    First recall the round-wise recursion from \Cref{lem:iterate_error_second_recursion} for $r=R$,
    \begin{align*}
        A(KR) &\leq \rb{1-\eta\mu}^KA(K(R-1)) + \frac{\eta H^2}{\mu}\sum_{j=K(R-1)}^{KR-1}(1-\eta\mu)^{KR-1-j}C(j) + \rb{1-\rb{1-\eta\mu}^K}\frac{\eta\sigma_2^2}{\mu M}\enspace,\\
        &\leq^{\text{(\Cref{lem:consensus_error_second_recursion})}} \rb{1-\eta\mu}^KA(K(R-1)) + \rb{1-\rb{1-\eta\mu}^K}\frac{\eta\sigma_2^2}{\mu M}\\
        &\quad \frac{1-(1-\eta\mu)^K}{\mu^2}\rb{2\eta^2 H^4K^2\zeta_\star^2 + \frac{2\eta^3H^2\tau^2K^2\sigma_2^2}{\mu} + 2\eta^2H^2\sigma_2^2K\ln(K)}\\
        &\quad + \frac{1-(1-\eta\mu)^K}{\mu^2}4\eta^2\tau^2H^2K^2(1-\eta\mu)^{K-2}\rb{A(K(r-1)) + \phi_\star^2}\enspace,\\
        &= \rb{\rb{1-\eta\mu}^K + \rb{1-\rb{1-\eta\mu}^K}\frac{4\eta^2H^2\tau^2}{\mu^2}K^2(1-\eta\mu)^{K-2}}A(K(R-1)) + \rb{1-\rb{1-\eta\mu}^K}\frac{\eta\sigma_2^2}{\mu M}\\
        &\quad + \frac{1-(1-\eta\mu)^K}{\mu^2}4\eta^2\tau^2H^2K^2(1-\eta\mu)^{K-2} \phi_\star^2\\
        &\quad + \frac{1-(1-\eta\mu)^K}{\mu^2}\rb{2\eta^2 H^4K^2\zeta_\star^2 + \frac{2\eta^3H^2\tau^2K^2\sigma_2^2}{\mu} + 2\eta^2H^2\sigma_2^2K\ln(K)}\enspace,\\
        &\leq \rho_1^RB^2 + \frac{1-\rb{1-\eta\mu}^K}{1-\rho_1}\cdot\frac{\eta\sigma_2^2}{\mu M} + \frac{1-\rb{1-\eta\mu}^K}{1-\rho_1}\cdot\frac{4\eta^2\tau^2H^2K^2(1-\eta\mu)^{K-2} \phi_\star^2}{\mu^2}\\
        &\quad + \frac{1-(1-\eta\mu)^K}{1-\rho_1}\rb{\frac{2\eta^2 H^4K^2\zeta_\star^2}{\mu^2} + \frac{2\eta^3H^2\tau^2K^2\sigma_2^2}{\mu^3} + \frac{2\eta^2H^2\sigma_2^2K\ln(K)}{\mu^2}}\enspace,
    \end{align*}
    where we defined $\rho_1=\rb{1-\eta\mu}^K + \rb{1-\rb{1-\eta\mu}^K}\frac{4\eta^2H^2\tau^2}{\mu^2}K^2(1-\eta\mu)^{K-2}$. This proves the lemma. 
\end{proof}
We can tune the step-size in the above guarantee, using standard techniques while making sure that $\tau$ is small enough and $K$ is large enough. This gives the following result,
\begin{lemma}[Strongly Convex Functions Iterate Convergence with $\tau, \zeta_\star, \phi_\star$]\label{lem:UB_Sconvex_wo_Q}
    Assume we have a problem instance satisfying \Cref{ass:strongly_convex,ass:smooth_second,ass:stoch_bounded_second_moment,ass:bounded_optima,ass:zeta_star,ass:phi_star,ass:tau} and $R\geq \max\cb{\frac{3H\tau}{\mu^2}\ln\rb{\frac{B^2}{\epsilon}}, \frac{2H\tau}{\mu^2}\ln^{3/2}\rb{\frac{B^2}{\epsilon}}}$ and $K\geq 4$ we can get the following convergence guarantee for local SGD, initialized at $x_0=0$,
    \begin{align*}
        A(KR) &= \tilde\ooo\rb{e^{-\frac{\mu KR}{2H}}B^2 + \frac{\sigma_2^2}{\mu^2MKR} + \frac{\tau^2H^2\phi_\star^2}{\mu^4R^2} + \frac{H^4\zeta_\star^2}{\mu^4R^2} + \frac{H^2\tau^2\sigma_2^2}{\mu^6KR^3} + \frac{H^2\sigma_2^2\ln(K)}{\mu^4KR^2}}\enspace,
    \end{align*}
    where we pick the step-size,
    \begin{align*}
        \eta = \min\cb{\frac{1}{2H}, \frac{1}{\mu KR}\ln\rb{\frac{B^2}{\epsilon}}}\enspace,
    \end{align*}
    for the choice of $\epsilon$,
    \begin{align*}
        \epsilon := \max\cb{\frac{2\sigma_2^2}{\mu^2MKR}, \frac{8\tau^2H^2\phi_\star^2}{\mu^4R^2}, \frac{4H^4\zeta_\star^2}{\mu^4R^2}, \frac{4H^2\tau^2\sigma_2^2}{\mu^6KR^3}, \frac{4H^2\sigma_2^2\ln(K)}{\mu^4KR^2}, \epsilon_{target}}\enspace,
    \end{align*}
    where $\epsilon_{target}$ is a target, which is greater than or equal to the machine precision.
\end{lemma}
\begin{proof}
    We will pick our step-size as follows, where we will later specify our choice of $\epsilon$:
    \begin{align*}
        \eta = \min\cb{\frac{1}{2H}, \frac{1}{\mu KR}\ln\rb{\frac{B^2}{\epsilon}}}\enspace.
    \end{align*}
    We will first derive conditions that are enough to bound $\frac{1-(1-\eta\mu)^K}{1-\rho_1}$ by $2$. Note the following, 
    \begin{align*}
        \frac{1-(1-\eta\mu)^K}{1-\rho_1}\leq 2 &\Leftrightarrow \rho_1 \leq \frac{1 + (1-\eta\mu)^K}{2}\enspace,\\
        &\Leftrightarrow \rb{1-\rb{1-\eta\mu}^K}\frac{4\eta^2H^2\tau^2}{\mu^2}K^2(1-\eta\mu)^{K-2} \leq \frac{1 - (1-\eta\mu)^K}{2}\enspace,\\
        &\Leftrightarrow \frac{4\eta^2H^2\tau^2}{\mu^2}K^2(1-\eta\mu)^{K-2} \leq \frac{1}{2}\enspace,\\
        &\Leftarrow \frac{4H^2\tau^2}{\mu^4 R^2}\ln^2\rb{\frac{B^2}{\epsilon}} \leq \frac{1}{2}\enspace,\\
        &\Leftarrow R\geq \frac{3H\tau}{\mu^2}\ln\rb{\frac{B^2}{\epsilon}}\enspace,
    \end{align*}
    Hence it is sufficient to assume that $\textcolor{red}{R\geq \frac{3H\tau}{\mu^2}\ln\rb{\frac{B^2}{\epsilon}}}$.    
    
    This allows us to simplify the convergence rate from the previous lemma as follows,
    \begin{align*}
        A(KR) &\leq \rho_1^RB^2 + \frac{2\eta\sigma_2^2}{\mu M} + \frac{8\eta^2\tau^2H^2K^2(1-\eta\mu)^{K-2} \phi_\star^2}{\mu^2} + \frac{4\eta^2 H^4K^2\zeta_\star^2}{\mu^2} + \frac{4\eta^3H^2\tau^2K^2\sigma_2^2}{\mu^3}\\
        &\quad + \frac{4\eta^2H^2\sigma_2^2K\ln(K)}{\mu^2}\enspace,\\
        &\leq \rho_1^RB^2 + \frac{2\eta\sigma_2^2}{\mu M} + \frac{8\eta^2\tau^2H^2K^2\phi_\star^2}{\mu^2}+ \frac{4\eta^2 H^4K^2\zeta_\star^2}{\mu^2} + \frac{4\eta^3H^2\tau^2K^2\sigma_2^2}{\mu^3} + \frac{4\eta^2H^2\sigma_2^2K\ln(K)}{\mu^2}\enspace.
    \end{align*}
    Now, let us upper bound the exponential term more carefully. Recall that due to the choice of our step-size,
    \begin{align*}
        \rho_1 &=  \rb{1-\eta\mu}^K + \rb{1-\rb{1-\eta\mu}^K}\frac{4\eta^2H^2\tau^2}{\mu^2}K^2(1-\eta\mu)^{K-2}\enspace,\\
        &\leq^{\text{(a)}} \rb{1-\eta\mu}^K + \eta\mu K\frac{4\eta^2H^2\tau^2}{\mu^2}K^2(1-\eta\mu)^{K-2}\enspace,\\
        &\leq \rb{1-\eta\mu}^K + \frac{4H^2\tau^2}{\mu^4R^3}\ln^3\rb{\frac{B^2}{\epsilon}}(1-\eta\mu)^{K-2}\enspace,\\
        &\leq \rb{1-\eta\mu}^{K-2} + \frac{4H^2\tau^2}{\mu^4R^3}\ln^3\rb{\frac{B^2}{\epsilon}}(1-\eta\mu)^{K-2}\enspace,\\
        &\leq \rb{1 +  \frac{4H^2\tau^2}{\mu^4R^3}\ln^3\rb{\frac{B^2}{\epsilon}}}\rb{1-\eta\mu}^{K-2}\enspace,\\
        &\leq e^{-\eta\mu(K-2) + \frac{4H^2\tau^2}{\mu^4R^3}\ln^3\rb{\frac{B^2}{\epsilon}}}. 
    \end{align*}
    where in (a) we use Bernoulli's inequality, and the choice of the step-size which implies that $\eta\mu < 1$. 
    Assuming $K\geq 4$ which allows us to upper bound $K/2$ by $K-2$, and raising both sides to the power $R$ gives,
    \begin{align*}
        \rho_1^R \leq e^{-\frac{\eta\mu KR}{2} + \frac{4H^2\tau^2}{\mu^4R^2}\ln^3\rb{\frac{B^2}{\epsilon}}} \leq^{\text{(a)}} e^{-\frac{\eta\mu KR}{2} + 1}\enspace,
    \end{align*}
    where in (a) we assumed that \textcolor{red}{$R\geq \frac{2H\tau}{\mu^2}\ln^{3/2}\rb{\frac{B^2}{\epsilon}}$}. Finally, we will pick the $\epsilon$ as follows,
    \begin{align*}
        \epsilon := \max\cb{\frac{2\sigma_2^2}{\mu^2MKR}, \frac{8\tau^2H^2\phi_\star^2}{\mu^4R^2}, \frac{4H^4\zeta_\star^2}{\mu^4R^2}, \frac{4H^2\tau^2\sigma_2^2}{\mu^6KR^3}, \frac{4H^2\sigma_2^2\ln(K)}{\mu^4KR^2}, \epsilon_{target}}\enspace,
    \end{align*}
    where $\epsilon_{target}$ is a target, which is greater than or equal to the machine precision (say, floating point precision), thus implying that $\ln\rb{\frac{B^2}{\epsilon}}$ is a numerical constant. We note two things that justify this step-size,
    \begin{itemize}
        \item The largest $\epsilon$ will lead to the step size we end up using, and in particular govern which term dominates the convergence rate. For instance, let us assume that $\epsilon = \frac{2\sigma_2^2}{\mu^2 MKR}$. Furthermore, let $\frac{1}{2H}\geq \frac{1}{\mu KR}\ln\rb{\frac{B^2}{\epsilon}}$ which implies that $e^{-\frac{\mu KR}{2H}} \leq \frac{2\sigma_2^2}{\mu^2 MKR}$. With $\eta = \frac{1}{\mu KR}\ln\rb{\frac{B^2}{\epsilon}}$, this makes the convergence rate,
        \begin{align*}
            A(KR) \leq \frac{2\sigma_2^2}{\mu^2 MKR} + \frac{2\sigma_2^2}{\mu^2 MKR}\ln\rb{\frac{B^2}{\frac{2\sigma_2^2}{\mu^2 MKR}}} = \tilde\ooo\rb{e^{-\frac{\mu KR}{2H}} + \frac{\sigma_2^2}{\mu^2 MKR}}\enspace.
        \end{align*}
        \item On the other hand if $\frac{1}{2H}\leq \frac{1}{\mu KR}\ln\rb{\frac{B^2}{\epsilon}}$, then it implies that, $e^{-\frac{\mu KR}{2H}} \geq \frac{2\sigma_2^2}{\mu^2 MKR}$, which makes the convergence rate,
        \begin{align*}
            A(KR) \leq e^{-\frac{\mu KR}{2H}} + \frac{\sigma_2^2}{\mu HM} =  \tilde\ooo\rb{e^{-\frac{\mu KR}{2H}} + \frac{\sigma_2^2}{\mu^2 MKR}}\enspace.
        \end{align*}
    \end{itemize}
    Using the above logic for all possible choices of $\epsilon$ (and thus $\eta$) allows us to give the following convergence rate,
    \begin{align*}
        A(KR) &= \tilde\ooo\rb{e^{-\frac{\mu KR}{2H}}B^2 + \frac{\sigma_2^2}{\mu^2MKR} + \frac{\tau^2H^2\phi_\star^2}{\mu^4R^2} + \frac{H^4\zeta_\star^2}{\mu^4R^2} + \frac{H^2\tau^2\sigma_2^2}{\mu^6KR^3} + \frac{H^2\sigma_2^2\ln(K)}{\mu^4KR^2}}\enspace.
    \end{align*}
\end{proof}

When we assume the functions are quadratic we can replace some of the smoothness constants in the above convergence rate with $\tau$, by relying on the better \textcolor{red}{red} upper bound of \Cref{lem:iterate_error_second_recursion}, as with $Q=0$ we do not need to bound the fourth moment of consensus error. The proof follows the above lemma and results in the following rate for quadratics.
\begin{lemma}\label{lem:UB_Sconvex_quad}
    Assume we have a quadratic problem instance satisfying \Cref{ass:strongly_convex,ass:smooth_second,ass:stoch_bounded_second_moment,ass:bounded_optima,ass:zeta_star,ass:phi_star,ass:tau}. Then using $\eta<1/H$ and such that $\rho_2=\rb{1-\eta\mu}^K + \rb{1-\rb{1-\eta\mu}^K}\frac{4\eta^2\tau^4}{\mu^2}K^2(1-\eta\mu)^{K-2}<1$ we can get the following convergence guarantee with initialization $x_0=0$,
    \begin{align*}
        A(KR) &\leq \rho_2^RB^2 + \frac{1-\rb{1-\eta\mu}^K}{1-\rho_2}\cdot\frac{\eta\sigma_2^2}{\mu M} + \frac{1-\rb{1-\eta\mu}^K}{1-\rho_2}\cdot\frac{4\eta^2\tau^4K^2(1-\eta\mu)^{K-2} \phi_\star^2}{\mu^2}\\
        &\quad + \frac{1-(1-\eta\mu)^K}{1-\rho_2}\rb{\frac{2\eta^2 H^2\tau^2K^2\zeta_\star^2}{\mu^2} + \frac{2\eta^3\tau^4K^2\sigma_2^2}{\mu^3} + \frac{2\eta^2\tau^2\sigma_2^2K\ln(K)}{\mu^2}}\enspace.
    \end{align*}
\end{lemma}
One notable aspect is that for quadratics, when $\tau=0$, we can obtain a fast convergence guarantee that matches dense mini-batch SGD, i.e., with $KR$ communication rounds, or in other words, we can demonstrate the extreme communication efficiency of Local SGD. We do not get this for non-quadratics, which highlights the need to understand the effect of third-order smoothness. This is not surprising, as third-order smoothness is known to play a vital role in the convergence of local SGD, even in a homogeneous setting (cf. \Cref{ch:baselines_and_lower_bounds}). Just as in the strongly convex case, we can tune the step size to achieve the following convergence rate for quadratics.

\begin{lemma}[Quadratics Iterate Convergence with $\tau, \zeta_\star, \phi_\star$]
    Assume we have a quadratic problem instance satisfying \Cref{ass:strongly_convex,ass:smooth_second,ass:stoch_bounded_second_moment,ass:bounded_optima,ass:zeta_star,ass:phi_star,ass:tau}, $R\geq \max\cb{\frac{3\tau^2}{\mu^2}\ln\rb{\frac{B^2}{\epsilon}}, \frac{2\tau^2}{\mu^2}\ln^{3/2}\rb{\frac{B^2}{\epsilon}}}$ and $K\geq 4$. Then we can get the following convergence guarantee for local SGD, initialized at $x_0=0$ 
    \begin{align*}
        A(KR) &= \tilde\ooo\rb{e^{-\frac{\mu KR}{2H}}B^2 + \frac{\sigma_2^2}{\mu^2MKR} + \frac{\tau^4\phi_\star^2}{\mu^4R^2} + \frac{\tau^2H^2\zeta_\star^2}{\mu^4R^2} + \frac{\tau^4\sigma_2^2}{\mu^6KR^3} + \frac{\tau^2\sigma_2^2\ln(K)}{\mu^4KR^2}}\enspace,
    \end{align*}
    where we pick the step-size,
    \begin{align*}
        \eta = \min\cb{\frac{1}{2H}, \frac{1}{\mu KR}\ln\rb{\frac{B^2}{\epsilon}}}\enspace,
    \end{align*}
    for the choice of $\epsilon$,
    \begin{align*}
        \epsilon := \max\cb{\frac{2\sigma_2^2}{\mu^2MKR}, \frac{8\tau^4\phi_\star^2}{\mu^4R^2}, \frac{4\tau^2H^2\zeta_\star^2}{\mu^4R^2}, \frac{4\tau^4\sigma_2^2}{\mu^6KR^3}, \frac{4\tau^2\sigma_2^2\ln(K)}{\mu^4KR^2}, \epsilon_{target}}\enspace,
    \end{align*}
    where $\epsilon_{target}$ is a target, which is greater than or equal to the machine precision.
\end{lemma}
It can be noted in the above convergence rate than when $\tau=0$ we recover the fast convergence rate of dense mini-batch SGD.

\subsection{Convergence in Function Value without Third-order Smoothness}\label{app:together_2}
\begin{lemma}[Strongly Convex Function Convergence with $\tau,\ \zeta_\star,\ \phi_\star$]\label{lem:UB_Sconvex_function_wo_Q}
     Assume we have a problem instance satisfying \Cref{ass:strongly_convex,ass:smooth_second,ass:stoch_bounded_second_moment,ass:bounded_optima,ass:zeta_star,ass:phi_star,ass:tau}, $R\geq \frac{4\tau\sqrt{\kappa}}{\mu}\max\cb{\ln\rb{\frac{\mu B^2}{\epsilon}}, \sqrt{\frac{2}{3}\ln^3\rb{\frac{\mu B^2}{\epsilon}}}}$ and $KR\geq 4\kappa$. Then we can get the following convergence guarantee for local SGD, initialized at $x_0=0$,
    \begin{align*}
        \ee\sb{F(\hat x)} - F(x^\star) &= \tilde\ooo\rb{e^{-\frac{\mu KR}{2H}}\mu B^2 + \frac{H^3\zeta_\star^2}{\mu^2R^2}+ \frac{H \tau^2\sigma_2^2}{\mu^4KR^3}+ \frac{H \sigma_2^2\ln(K)}{\mu^2 KR^2}+ \frac{H \tau^2 \phi_\star^2}{\mu^2R^2}+ \frac{\sigma_2^2}{\mu MKR}}\enspace,
    \end{align*}
    where we define $\hat x = \sum_{t=0}^{T-1}w_tx_t$ for the choice of weights $$w_t:= \frac{\rho_4^{R-1-\delta(t)/K}(1-\eta\mu)^{\delta(t) + K -1-t}}{W}$$ for $W=\frac{1-\rho_4^R}{1-\rho_4}\cdot \frac{1-(1-\eta\mu)^K}{\eta\mu}$ and $\rho_4 = \rb{1 - \eta \mu}^K + \rb{1-(1-\eta\mu)^K}\frac{8\eta^2 H \tau^2K^2}{\mu}(1-\eta\mu)^{K-2}$. And we pick the step-size as,
    \begin{align*}
        \eta = \min\cb{\frac{1}{2H}, \frac{1}{\mu KR}\ln\rb{\frac{\mu B^2}{\epsilon}}}\enspace,
    \end{align*}
    for the choice of $\epsilon$,
    \begin{align*}
        \epsilon = \min\cb{\max\cb{\frac{4 H^3\zeta_\star^2}{\mu^2R^2}, \frac{4H \tau^2\sigma_2^2}{\mu^4KR^3}, \frac{4H \sigma_2^2\ln(K)}{\mu^2 KR^2}, \frac{8 H \tau^2 \phi_\star^2}{\mu^2R^2}, \frac{3\sigma_2^2}{\mu MKR}, \epsilon_{target}}, \frac{\mu B^2}{6}}\enspace,
    \end{align*}
    where $\epsilon_{target}$ is a target, which is greater than or equal to the machine precision.
\end{lemma}
\begin{proof}
The main task in this subsection is to combine \Cref{lem:function_error_easy,lem:consensus_error_second_recursion}. Recall \Cref{lem:function_error_easy} implies for all $t\in[0,T-1]$,
\begin{align*}
    A(t+1) \leq \rb{1 - \eta \mu} A(t) - \eta E(t) + 2\eta HC(t) + \frac{3\eta^2 \sigma_2^2}{M} \enspace . \tag{$\star$}
\end{align*}
Also recall the upper bound on the consensus error from \Cref{lem:consensus_error_second_recursion} for all $t\in[0,T]$,
\begin{align*}
    C(t) &\leq 2\eta^2 H^2K^2\zeta_\star^2 + \frac{2\eta^3\tau^2K^2\sigma_2^2}{\mu} + 2\eta^2\sigma_2^2K\ln(K)+ 4\eta^2\tau^2\rb{t-\delta(t)}^2(1-\eta \mu)^{2(t-1-\delta(t))}\rb{A(\delta(t)) + \phi_\star^2}\enspace.
\end{align*}
Plugging this upper bound into $(\star)$ gives us,
\begin{align*}
    A(t+1) &\leq \rb{1 - \eta \mu} A(t) - \eta E(t) + 4\eta^3 H^3 K^2\zeta_\star^2 + \frac{4\eta^4 H \tau^2K^2\sigma_2^2}{\mu}+ 4\eta^3 H \sigma_2^2K\ln(K) \\
    &\quad  + 8\eta^3 H \tau^2K^2(1-\eta\mu)^{2(t-1-\delta(t))}\rb{A(\delta(t)) + \phi_\star^2} + \frac{3\eta^2 \sigma_2^2}{M} \enspace.
\end{align*}
Unrolling the above recursion for over an arbitrary round $r\in[0,R-1]$ gives us (c.f., the calculations in \Cref{lem:consensus_error_second_recursion}),
\begin{align*}
    A(K(r+1)) &\leq \rb{1 - \eta \mu}^K A(Kr) - \eta \sum_{t=Kr}^{Kr+K-1} (1-\eta \mu)^{Kr+K-1-t} E(t)\\
    &\quad + \rb{1-(1-\eta\mu)^K}\frac{8\eta^2 H \tau^2K^2}{\mu} (1-\eta\mu)^{K-2}A(Kr) + \frac{1-(1-\eta\mu)^K}{\eta\mu}C_1 \enspace.
\end{align*}
Where $C_1$ is the sum of constant terms in the upper bound which do not depend on $t$ and is defined as,
\begin{align*}
    C_1 := 4\eta^3 H^3 K^2\zeta_\star^2 + \frac{4\eta^4 H \tau^2K^2\sigma_2^2}{\mu}+ 4\eta^3 H \sigma_2^2K\ln(K) + 8\eta^3 H \tau^2K^2 \phi_\star^2 + \frac{3\eta^2 \sigma_2^2}{M} \enspace.
\end{align*}
We also define the following constant,
\begin{align*}
    \rho_4 := \rb{1 - \eta \mu}^K + \rb{1-(1-\eta\mu)^K}\frac{8\eta^2 H \tau^2K^2}{\mu}(1-\eta\mu)^{K-2} \enspace.
\end{align*}
These notations allows us to re-write the above recursion as follows for $r\in[0,R-1]$,
\begin{align*}
    A(K(r+1)) \leq \rho_4 A(Kr)- \eta \sum_{t=Kr}^{Kr+K-1} (1-\eta \mu)^{Kr+K-1-t} E(t) + \frac{1-(1-\eta\mu)^K}{\eta\mu}C_1\enspace.
\end{align*}
Now unrolling the recursion over $R$ rounds gives us,
\begin{align*}
    A(KR) &\leq \rho_4^R A(0) - \eta \sum_{r=0}^{R-1} \rho_4^{R-1-r} \sum_{t=Kr}^{Kr+K-1} (1-\eta \mu)^{Kr+K-1-t} E(t) + \frac{1-(1-\eta\mu)^K}{\eta\mu}\sum_{r=0}^{R-1} \rho_4^{R-1-r} C_1\enspace,\\
    &\leq \rho_4^R A(0) - \eta \sum_{r=0}^{R-1} \rho_4^{R-1-r} \sum_{t=Kr}^{Kr+K-1} (1-\eta \mu)^{Kr+K-1-t} E(t) + \frac{1-(1-\eta\mu)^K}{\eta\mu}\cdot\frac{1-\rho_4^R}{1-\rho_4}\cdot C_1\enspace.
\end{align*}
We will now define the following sum of weights,
\begin{align*}
    W &:= \sum_{r=0}^{R-1} \rho_4^{R-1-r} \sum_{t=Kr}^{Kr+K-1} (1-\eta \mu)^{Kr+K-1-t}\enspace,\\
    &= \sum_{r=0}^{R-1} \rho_4^{R-1-r} \cdot \frac{1-(1-\eta\mu)^K}{\eta\mu}\enspace,\\
    &= \frac{1-\rho_4^R}{1-\rho_4}\cdot \frac{1-(1-\eta\mu)^K}{\eta\mu}\enspace.
\end{align*}
Dividing by $\eta W$ in the above recursion and re-arranging gives us the following,
\begin{align*}
    &\textcolor{blue}{\frac{1}{W}\sum_{r=0}^{R-1} \rho_4^{R-1-r} \sum_{t=Kr}^{Kr+K-1} (1-\eta \mu)^{Kr+K-1-t} E(t)} \\ 
    &\qquad\leq \frac{\rho_4^R}{\eta W} A(0) - \frac{A(KR)}{\eta W} + \frac{1}{\eta W}\cdot\frac{1-(1-\eta\mu)^K}{\eta\mu}\cdot\frac{1-\rho_4^R}{1-\rho_4}\cdot C_1\enspace,\\
    &\qquad\leq \frac{\rho_4^R}{1-\rho_4^R}\cdot\frac{1-\rho_4}{1-(1-\eta\mu)^K}\mu B^2 + \frac{C_1}{\eta}\enspace,\\
    &\qquad= \frac{\rho_4^R}{1-\rho_4^R}\rb{1 - \frac{8\eta^2H\tau^2K^2}{\mu}(1-\eta\mu)^{K-2}}\mu B^2 + 4\eta^2 H^3 K^2\zeta_\star^2 + \frac{4\eta^3 H \tau^2K^2\sigma_2^2}{\mu}+ 4\eta^2 H \sigma_2^2K\ln(K)\\
    &\qquad\quad  + 8\eta^2 H \tau^2K^2 \phi_\star^2 + \frac{3\eta \sigma_2^2}{M}\enspace.
\end{align*}
Now similar to the proof in the previous section we will pick the step-size as follows,
\begin{align*}
    \eta := \min\cb{\frac{1}{2H}, \frac{1}{\mu KR}\ln\rb{\frac{\mu B^2}{\epsilon}}}\enspace,
\end{align*}
where we will define $\epsilon$ later in the proof. Our goal now is to bound the term $\frac{\rho_4^R}{1-\rho_4^R}\rb{1 - \frac{8\eta^2H\tau^2K^2}{\mu}(1-\eta\mu)^{K-2}}$ so that it looks more like the exponential decay in usual convergence analyses. We first note the following,
\begin{align*}
    \frac{8H\tau^2}{\mu^3R^2}\ln^2\rb{\frac{\mu B^2}{\epsilon}} \leq \frac{1}{2}\enspace,
\end{align*}
by assuming $\textcolor{red}{R \geq \frac{4\tau}{\mu}\sqrt{\kappa}\ln\rb{\frac{\mu B^2}{\epsilon}}}$. This allows us to upper bound $\rb{1 - \frac{8\eta^2H\tau^2K^2}{\mu}(1-\eta\mu)^{K-2}}$ by $1$. Now we will upper bound $\frac{\rho_4^R}{1-\rho_4^R}$. To do this we first note the following, 
\begin{align*}
    \rho_4^R &= (1-\eta\mu)^{KR}\rb{1 + \rb{1-(1-\eta\mu)^K}\frac{8\eta^2 H \tau^2K^2}{\mu(1-\eta\mu)}}^R\enspace,\\
    &\leq^{\text{(a)}} e^{-\eta\mu KR}\rb{1 + \eta\mu K\frac{8\eta^2 H \tau^2K^2}{\mu(1-\eta\mu)^2}}^R\enspace,\\
    &\leq e^{-\eta\mu KR}\rb{1 + \frac{1}{R^3}\ln^3\rb{\frac{\mu B^2}{\epsilon}}\frac{8 H \tau^2}{\mu^3(1-\mu/2H)^2}}^R\enspace,\\
    &\leq e^{-\eta\mu KR + \frac{1}{R^2}\ln^3\rb{\frac{\mu B^2}{\epsilon}}\frac{8 H \tau^2}{\mu^3(1-1/(2\kappa)^2)}}\enspace,\\
    &\leq^{\text{($\kappa \geq 1$)}} e^{-\eta\mu KR + \frac{1}{R^2}\ln^3\rb{\frac{\mu B^2}{\epsilon}}\frac{32 H \tau^2}{3\mu^3}}\enspace,\\
    &\leq^{\text{(b)}} e^{-\eta\mu KR + 1}\enspace,
\end{align*}
where in (a) we use the Bernoulli's inequality after noting that $\eta\mu < 1$ for our choice of step-size; and in (b) we used $\textcolor{red}{R\geq \frac{\tau}{\mu}\sqrt{\ln^3\rb{\frac{\mu B^2}{\epsilon}}\frac{32 \kappa}{3}}}$. Now using this upper bound we get,
\begin{align*}
    \frac{\rho_4^R}{1-\rho_4^R} &\leq \frac{e^{-\eta\mu KR + 1}}{1 - e^{-\eta\mu KR + 1}}\enspace,\\
    &\leq^{\text{(a)}} 2e^{-\eta\mu KR + 1} \leq 6e^{-\eta\mu KR}\enspace,
\end{align*}
where in (a) we assume that $e^{-\eta\mu KR + 1}\leq \frac{1}{2}$ which can be verified to be true for both choices of step-sizes as follows,
\begin{align*}
    &(i)\ e^{-\frac{\mu KR}{2H} + 1} \leq \frac{1}{2} \Leftarrow 2e \leq e^{\frac{\mu KR}{2H}} \Leftarrow \textcolor{red}{4\kappa\leq KR}\enspace;\\
    &(ii)\ e^{-\ln\rb{\mu B^2/\epsilon} + 1}\leq \frac{1}{2} \Leftarrow \frac{e\epsilon}{\mu B^2} \leq \frac{1}{2} \Leftarrow \textcolor{red}{\epsilon \leq \frac{\mu B^2}{6}}\enspace.
\end{align*}
We are almost done, but we still need to choose an $\epsilon$. We do this the same way as in the previous section: we pick $\epsilon$ as the maximum of the target accuracy $\epsilon_{target}$ and the value of the convergence rate terms which are an increasing function of $\eta$, at $\eta' = \frac{1}{\mu KR}$. In particular we pick $\epsilon$ as,
\begin{align*}
    \epsilon = \min\cb{\max\cb{\frac{4 H^3\zeta_\star^2}{\mu^2R^2}, \frac{4H \tau^2\sigma_2^2}{\mu^4KR^3}, \frac{4H \sigma_2^2\ln(K)}{\mu^2 KR^2}, \frac{8 H \tau^2 \phi_\star^2}{\mu^2R^2}, \frac{3\sigma_2^2}{\mu MKR}, \epsilon_{target}}, \frac{\mu B^2}{6}}\enspace.
\end{align*}
Finally, note that the we have essentially used the weights on the models $\{x_0, \dots, x_{KR-1}\}$ defined by the \textcolor{blue}{blue term}. Rigorously for time step $t\in[0,T-1]$ we use the following weight,
\begin{align*}
    w_t &= \frac{\rho_4^{R-1-\delta(t)/K}(1-\eta\mu)^{\delta(t) + K -1-t}}{W}\enspace,
\end{align*}
and we bound the function sub-optimality of the point $\sum_{t=0}^{T-1}w_tx_t$ by using Jensen's inequality as follows,
\begin{align*}
    \ee\sb{F\rb{\sum_{t=0}^{T-1}w_tx_t}}- F(x^\star) \leq \sum_{t=0}^{T-1}w_t\rb{\ee\sb{F\rb{x_t}}- F(x^\star)}\enspace.
\end{align*}
Thus, our choice of $\epsilon$, $\eta$, and averaging weights proves the lemma statement, assuming the \textcolor{red}{highlighted required conditions}.
\end{proof}

In the following lemma, we state the result for strongly convex quadratics, by noting that in the proof of the above lemma, we simply replace the usage of \Cref{lem:function_error_easy} by \Cref{lem:function_error_hard} and note that $Q=0$ for quadratics, which allows us to replace several smoothness constants $H$ in the convergence rate by $\tau$. 
\begin{lemma}[Strongly Convex Function Convergence with $\tau,\ \zeta_\star,\ \phi_\star$ for Quadratics]\label{lem:UB_Sconvex_quad_function_wo_Q}
    Assume we have a quadratic problem instance satisfying \Cref{ass:strongly_convex,ass:smooth_second,ass:stoch_bounded_second_moment,ass:bounded_optima,ass:zeta_star,ass:phi_star,ass:tau} with, $$R\geq \frac{4\tau^2}{\mu^2}\max\cb{\ln\rb{\frac{\mu B^2}{\epsilon}}, \sqrt{\frac{2}{3}\ln^3\rb{\frac{\mu B^2}{\epsilon}}}} \quad \text{and}\quad KR\geq 4\kappa\enspace.$$ Then we can get the following convergence guarantee for local SGD, initialized at $x_0=0$,
    \begin{align*}
        \ee\sb{F(\hat x)} - F(x^\star) &= \tilde\ooo\rb{e^{-\frac{\mu KR}{2H}}\mu B^2 + \frac{\tau^2H^2\zeta_\star^2}{\mu^3R^2}+ \frac{\tau^4\sigma_2^2}{\mu^5KR^3}+ \frac{\tau^2 \sigma_2^2\ln(K)}{\mu^3 KR^2}+ \frac{\tau^4 \phi_\star^2}{\mu^3R^2}+ \frac{\sigma_2^2}{\mu MKR}}\enspace,
    \end{align*}
    where we define $\hat x = \sum_{t=0}^{T-1}w_tx_t$ for the choice of weights $$w_t:= \frac{\rho_4^{R-1-\delta(t)/K}(1-\eta\mu)^{\delta(t) + K -1-t}}{W}$$ for $W=\frac{1-\rho_4^R}{1-\rho_4}\cdot \frac{1-(1-\eta\mu)^K}{\eta\mu}$ and $\rho_4 = \rb{1 - \eta \mu}^K + \rb{1-(1-\eta\mu)^K}\frac{8\eta^2 \tau^4K^2}{\mu^2}(1-\eta\mu)^{K-2}$. And we pick the step-size as,
    \begin{align*}
        \eta = \min\cb{\frac{1}{2H}, \frac{1}{\mu KR}\ln\rb{\frac{\mu B^2}{\epsilon}}}\enspace,
    \end{align*}
    for the choice of $\epsilon$,
    \begin{align*}
        \epsilon = \min\cb{\max\cb{\frac{4 \tau^2H^2\zeta_\star^2}{\mu^3R^2}, \frac{4 \tau^4\sigma_2^2}{\mu^5KR^3}, \frac{4\tau^2 \sigma_2^2\ln(K)}{\mu^3 KR^2}, \frac{8 \tau^4 \phi_\star^2}{\mu^3R^2}, \frac{3\sigma_2^2}{\mu MKR}, \epsilon_{target}}, \frac{\mu B^2}{6}}\enspace,
    \end{align*}
    where $\epsilon_{target}$ is a target accuracy, which is greater than or equal to the machine precision.
\end{lemma}

\subsection{Convergence in Iterates with Third-order Smoothness}\label{app:together_3}
The main technical challenge in incorporating third-order smoothness (c.f., \Cref{ass:smooth_third}) in our upper bounds lies in bounding the sequence $D(\cdot)$ while working with the \textcolor{red}{upper bound} in \Cref{lem:iterate_error_second_recursion}. One natural approach is to mirror the analysis in the previous section: unroll the consensus error recursion back to the previous communication round, substitute that into the upper bound for $A(\cdot)$, and then iterate across rounds. However, this strategy quickly encounters difficulties. We need to control the fourth moment of the iterate error, $B(\cdot)$, and we lack a uniform upper bound for it.

To overcome this, we adopt a different strategy. As the following lemma shows, we analyze the pair $(A(\cdot), B(\cdot))$ together in terms of the pair $(C(\cdot), D(\cdot))$, treating them as components of a two-dimensional recursion. Once we do this, we can more or less use ideas similar to those before.

\begin{lemma}
Assume we have a problem instance satisfying \Cref{ass:strongly_convex,ass:smooth_second,ass:smooth_third,ass:stoch_bounded_second_moment,ass:bounded_optima,ass:zeta_star,ass:phi_star,ass:tau}. Then using $\eta<1/H$ and defining 
 \begin{align*}
     \rho_3 &:= (1-\eta\mu)^K + \Bigg(\rb{1-\rb{1-\eta\mu}^K}\\
     &\quad \times \rb{\frac{2\tau^2}{\mu^2} + \frac{2Q^2B^2}{\mu^2} + \frac{16\eta^2\sigma_2^2\tau^2}{\mu^2 M B^2} + \frac{H^4}{\mu^4} + \frac{16\eta^2\sigma_2^2Q^2}{\mu^2 M}}\rb{4\eta^2\tau^2K^2 + 64\eta^4\tau^4K^4}\Bigg)\enspace,\\
     \Psi &:= 4\eta^2\tau^2K^2\phi_\star^2 + \frac{128\eta^5\tau^4K^4\sigma_4^2}{\mu B^2}\phi_\star^2 + \frac{320\eta^4\sigma_2^2\tau^2K^3}{B^2}\phi_\star^2 + \frac{64\eta^4\tau^4K^4}{B^2}\phi_\star^4\\
    &\quad + 2\eta^2 H^2K^2\zeta_\star^2 + \frac{2\eta^3\tau^2K^2\sigma_2^2}{\mu} + 2\eta^2\sigma_2^2K\ln(K)+
         \frac{8\eta^3K^3H^4\zeta_\star^4}{\mu B^2} + \frac{88\eta^5K^3\tau^4\sigma_4^4}{\mu^3B^2}\\
         &\quad + \frac{160\eta^4K^4\sigma_2^2H^2\zeta_\star^2}{B^2} + \frac{160\eta^5\tau^2K^4\sigma_2^4}{\mu B^2} + \frac{112\eta^4K^3\sigma_4^4\ln(K)}{B^2}\enspace.
 \end{align*}
 we can get the following convergence guarantee with initialization $x_0=0$,
 \begin{align*}
     \max&\cb{A(KR), \frac{B(KR)}{B^2}} \leq 2\rho_3^RB^2 + \frac{1-\rb{1-\eta\mu}^K}{1-\rho_3}\rb{\frac{\eta\sigma_2^2}{\mu M} + \frac{9\eta^3\sigma_4^4}{\mu M^2B^2}}\\
    &\quad + \frac{1-\rb{1-\eta\mu}^K}{1-\rho_3}\rb{\frac{2\tau^2}{\mu^2} + \frac{2Q^2B^2}{\mu^2} + \frac{16\eta^2\sigma_2^2\tau^2}{\mu^2 M B^2} + \frac{H^4}{\mu^4} + \frac{16\eta^2\sigma_2^2Q^2}{\mu^2 M}}\Psi\enspace.
 \end{align*}
\end{lemma} 
\begin{proof}
We will denote the following vectors for all $t\in[0,T]$,
\begin{align*}
    \aa(t) := \begin{bmatrix}
        A(t)\\
        B(t)/B^2
    \end{bmatrix} \qquad \text{and} \qquad \cc(t) := \begin{bmatrix}
        C(t)\\
        D(t)/B^2
    \end{bmatrix}\enspace,
\end{align*}
where note that $B$ comes from \Cref{ass:bounded_optima} and we divide the sequences $B(t),\ C(t)$ by $B^2$ to make them ``dimensionally consistent" or similarly scale-variant as the sequences $A(t),\ C(t)$. Based on the recursions we have developed in \Cref{lem:iterate_error_second_recursion,lem:iterate_error_fourth_recursion} we get the following vector recursion,
\begin{align*}
    \aa(t+1) &\leq (1-\eta\mu)\begin{bmatrix}
        1 & 0\\
        \frac{8\eta^2\sigma_2^2}{MB^2} & 1
    \end{bmatrix}\aa(t) + \begin{bmatrix}
        \frac{2\eta\tau^2}{\mu} & \frac{2\eta Q^2B^2}{\mu}\\
        \frac{16\eta^3\sigma_2^2\tau^2}{\mu MB^2} & \frac{\eta H^4}{\mu^3} + \frac{16\eta^3\sigma_2^2Q^2}{\mu M}
    \end{bmatrix}\cc(t) + \begin{bmatrix}
        \frac{\eta^2\sigma_2^2}{M}\\
        \frac{9\eta^4\sigma_4^4}{M^2B^2}
    \end{bmatrix}\enspace,\\
    &=: P\aa(t) + Q\cc(t) + N\enspace,\\
    &\leq P^{t+1-\delta(t)}\aa(\delta(t)) + \sum_{j=\delta(t)}^{t}P^{t-j}\rb{Q\cc(j) + N}\enspace,
\end{align*}
where we define $P, Q\in\rr^{2\times 2}$ and $N\in\rr^2$ to simplify the calculations. Let us also recall the recursion we get for the consensus error terms based on \Cref{lem:consensus_error_second_recursion,lem:consensus_error_fourth_recursion},
\begin{align*}
    \cc(t) &\leq \begin{bmatrix}
        4\eta^2\tau^2K^2 & 0\\
        \frac{128\eta^5\tau^4K^4\sigma_4^2}{\mu B^2} + \frac{320\eta^4\sigma_2^2\tau^2K^3}{B^2} & 64\eta^4\tau^4K^4
    \end{bmatrix}\aa(\delta(t))\\
    &\quad + \begin{bmatrix}
        4\eta^2\tau^2K^2\phi_\star^2\\
        \frac{128\eta^5\tau^4K^4\sigma_4^2}{\mu B^2}\phi_\star^2 + \frac{320\eta^4\sigma_2^2\tau^2K^3}{B^2}\phi_\star^2 + \frac{64\eta^4\tau^4K^4}{B^2}\phi_\star^4
    \end{bmatrix}\\
    &\quad + \begin{bmatrix}
        2\eta^2 H^2K^2\zeta_\star^2 + \frac{2\eta^3\tau^2K^2\sigma_2^2}{\mu} + 2\eta^2\sigma_2^2K\ln(K)\\
         \frac{8\eta^3K^3H^4\zeta_\star^4}{\mu B^2} + \frac{88\eta^5K^3\tau^4\sigma_4^4}{\mu^3B^2} + \frac{160\eta^4K^4\sigma_2^2H^2\zeta_\star^2}{B^2} + \frac{160\eta^5\tau^2K^4\sigma_2^4}{\mu B^2} + \frac{112\eta^4K^3\sigma_4^4\ln(K)}{B^2}
    \end{bmatrix}\enspace,\\
    &=: U\aa(\delta(t)) + V\enspace,
\end{align*}
where we define $U\in\rr^{2\times 2}$ and $V\in\rr^2$. Now we can plug in this upper bound in the inequality above, which gives us,
\begin{align*}
    \aa(t+1) &\leq P^{t+1-\delta(t)}\aa(\delta(t)) + \sum_{j=\delta(t)}^{t}P^{t-j}\rb{QU\aa(\delta(t)) + QV + N}\enspace.
\end{align*}
Now, let us denote $t = KR-1$ and unroll across communication rounds to get the following,
\begin{align*}
    \aa(KR) &\leq P^{K}\aa(K(R-1)) + \sum_{j=K(R-1)}^{KR-1}P^{KR-1-j}\rb{QU\aa(K(R-1)) + QV + N}\enspace,\\
    &=: P^{K}\aa(K(R-1)) + \bar P\rb{QU\aa(K(R-1)) + QV + N}\enspace,\\
    &= \rb{P^{K} + \bar P QU}\aa(K(R-1)) + \bar P\rb{QV + N}\enspace,
\end{align*}
where we define $\bar P = \sum_{j=K(R-1)}^{KR-1}P^{KR-1-j} \in \rr^{2\times 2}$. Taking the norm on both sides and using the triangle inequality, we get,
\begin{align*}
    \norm{\aa(KR)} &\leq \norm{\rb{P^{K} + \bar P QU}}\norm{\aa(K(R-1))} + \norm{\bar P Q}\norm{V} + \norm{\bar P}\norm{N}\enspace,\\
    &\leq \textcolor{red}{\rb{\norm{P^{K}} + \norm{\bar P}\norm{Q}\norm{U}}\norm{\aa(K(R-1))} + \norm{\bar P}\norm{Q}\norm{V} + \norm{\bar P}\norm{N}}\enspace.
\end{align*}
We will not individually upper bound these spectral norms.
 First note that due to $P$ being a lower triangular matrix,
\begin{align*}
    P^K = (1-\eta\mu)^K\begin{bmatrix}
        1 &0\\
        \frac{8\eta^2\sigma_2^2 K}{M B^2} & 1
    \end{bmatrix}\enspace.
\end{align*}
Since $P^K$ is a lower triangular matrix, its eigenvalues can be read off its diagonal. In particular, we note that $\norm{P^K} = (1-\eta\mu)^K$. We can use a similar idea to upper bound $\norm{\bar{P}}$ as follows,
\begin{align*}
    \bar P &= \begin{bmatrix}
        \sum_{i=0}^{K-1}(1-\eta\mu)^i & 0\\
        \frac{8\eta^2\sigma_2^2}{MB^2}\sum_{i=0}^{K-1}i(1-\eta\mu)^i & \sum_{i=0}^{K-1}(1-\eta\mu)^i
    \end{bmatrix}\enspace.
\end{align*}
This implies $\norm{\bar P} = \frac{1 - (1-\eta\mu)^K}{\eta\mu}$. We also note the following about $Q$, noting that the spectral norm is upper-bounded by the Frobenius norm,
\begin{align*}
    \norm{Q} &\leq \frac{2\eta\tau^2}{\mu} + \frac{2\eta Q^2B^2}{\mu} + \frac{16\eta^3\sigma_2^2\tau^2}{\mu M B^2} + \frac{\eta H^4}{\mu^3} + \frac{16\eta^3\sigma_2^2Q^2}{\mu M}\enspace.
\end{align*}
Finally, noting that $U$ is also lower diagonal, we note that,
\begin{align*}
    \norm{U} &= \max\cb{4\eta^2\tau^2K^2, 64\eta^4\tau^4K^4}\enspace,\\
    &\leq 4\eta^2\tau^2K^2 + 64\eta^4\tau^4K^4\enspace.
\end{align*}
Combining the upper bounds for $P^K$, $\bar P$, $Q$, $U$ we get,
\begin{align*}
    &\norm{P^{K}} + \norm{\bar P}\norm{Q}\norm{U}\\ 
    &\leq (1-\eta\mu)^K\\
    &\quad + \rb{1-\rb{1-\eta\mu}^K}\rb{\frac{2\tau^2}{\mu^2} + \frac{2Q^2B^2}{\mu^2} + \frac{16\eta^2\sigma_2^2\tau^2}{\mu^2 M B^2} + \frac{H^4}{\mu^4} + \frac{16\eta^2\sigma_2^2Q^2}{\mu^2 M}}\rb{4\eta^2\tau^2K^2 + 64\eta^4\tau^4K^4}\enspace,\\
    &=: \rho_3\enspace.
\end{align*}
Note that when $\tau=0$, then $\rho_3 = (1-\eta\mu)^K$, which will lead to the fast exponential decay we do get in the homogeneous setting. Using the above calculation, we can also conclude that,
\begin{align*}
    \norm{\bar P}\norm{Q}\norm{V} &\leq \rb{1-\rb{1-\eta\mu}^K}\rb{\frac{2\tau^2}{\mu^2} + \frac{2Q^2B^2}{\mu^2} + \frac{16\eta^2\sigma_2^2\tau^2}{\mu^2 M B^2} + \frac{H^4}{\mu^4} + \frac{16\eta^2\sigma_2^2Q^2}{\mu^2 M}}\norm{V}\enspace,\\
    \norm{\bar P}\norm{N} &\leq \rb{1-\rb{1-\eta\mu}^K}\rb{\frac{\eta\sigma_2^2}{\mu M} + \frac{9\eta^3\sigma_4^4}{\mu M^2B^2}}\enspace. 
\end{align*}
Plugging this back into the \textcolor{red}{red} inequality and then unrolling the recursion, we get,
\begin{align*}
    \norm{\aa(KR)} &\leq \rho_3^R \norm{\aa(K(R-1))}\\
    &\quad + \frac{1-\rb{1-\eta\mu}^K}{1-\rho_3}\rb{\frac{2\tau^2}{\mu^2} + \frac{2Q^2B^2}{\mu^2} + \frac{16\eta^2\sigma_2^2\tau^2}{\mu^2 M B^2} + \frac{H^4}{\mu^4} + \frac{16\eta^2\sigma_2^2Q^2}{\mu^2 M}}\norm{V}\\
    &\quad + \frac{1-\rb{1-\eta\mu}^K}{1-\rho_3}\rb{\frac{\eta\sigma_2^2}{\mu M} + \frac{9\eta^3\sigma_4^4}{\mu M^2B^2}},
\end{align*}
which proves our convergence rate upon applying the triangle inequality to note that,
\begin{align*}
    \norm{V} &\leq 4\eta^2\tau^2K^2\phi_\star^2 + \frac{128\eta^5\tau^4K^4\sigma_4^2}{\mu B^2}\phi_\star^2 + \frac{320\eta^4\sigma_2^2\tau^2K^3}{B^2}\phi_\star^2 + \frac{64\eta^4\tau^4K^4}{B^2}\phi_\star^4\\
    &\quad + 2\eta^2 H^2K^2\zeta_\star^2 + \frac{2\eta^3\tau^2K^2\sigma_2^2}{\mu} + 2\eta^2\sigma_2^2K\ln(K)+
         \frac{8\eta^3K^3H^4\zeta_\star^4}{\mu B^2} + \frac{88\eta^5K^3\tau^4\sigma_4^4}{\mu^3B^2}\\
         &\quad + \frac{160\eta^4K^4\sigma_2^2H^2\zeta_\star^2}{B^2} + \frac{160\eta^5\tau^2K^4\sigma_2^4}{\mu B^2} + \frac{112\eta^4K^3\sigma_4^4\ln(K)}{B^2}\enspace.
\end{align*}
This proves the lemma.
\end{proof}

We will now tune the step size using a similar approach to the one in the previous section to achieve the desired convergence rate. 

\begin{lemma}\label{lem:UB_Sconvex_w_Q}
    Assuming sufficiently many communication rounds,$$R\geq \frac{8\tau}{\mu}\cdot\max\cb{\ln^{2}\rb{\frac{B^2}{\epsilon}}\cdot\rb{\frac{4QB}{\mu} + \frac{5H^2}{\mu^2}}, \ln^{3/2}\rb{\frac{B^2}{\epsilon}}\rb{1 + \sqrt{\frac{QB}{\mu}} + \frac{H}{\mu}}, \frac{\ln(B^2/\epsilon)}{\ln(\ln(B^2/\epsilon))}},$$ $B^2>e\epsilon$, and $KR\geq \frac{8\sigma_2}{\mu^2\sqrt{M}}\ln\rb{\frac{B^2}{\epsilon}}\cdot \max\cb{\frac{\tau}{B},Q}$ we can get the following convergence guarantee for local SGD, initializing at $x_0=0$ and optimizing functions satisfying \Cref{ass:strongly_convex,ass:smooth_second,ass:smooth_third,ass:stoch_bounded_second_moment,ass:bounded_optima,ass:zeta_star,ass:phi_star,ass:tau},
    \begin{align*}
        \norm{\aa(KR)} &= \tilde\ooo\Bigg(e^{-\eta\mu KR}B^2 + \frac{\sigma_2^2}{\mu^2 MKR} + \frac{\sigma_4^4}{\mu^4K^3R^3 M^2B^2} +\kappa'\rb{\frac{\tau^2\phi_\star^2}{\mu^2R^2}+ \frac{\tau^4\sigma_4^2}{\mu^6 KR^5B^2}\phi_\star^2}\\
        &\quad +\kappa'\rb{\frac{\sigma_2^2\tau^2}{\mu^4KR^4B^2}\phi_\star^2 + \frac{\tau^4}{\mu^4B^2R^4}\phi_\star^4 + \frac{H^2\zeta_\star^2}{\mu^2 R^2} + \frac{\tau^2\sigma_2^2}{\mu^4KR^3} + \frac{\sigma_2^2\ln(K)}{\mu^2KR^2}}\\
        &\quad + \kappa'\rb{\frac{H^4\zeta_\star^4}{\mu^4R^3 B^2} +\frac{\tau^4\sigma_4^4}{\mu^8K^2R^5B^2} + \frac{\sigma_2^2H^2\zeta_\star^2}{\mu^4B^2R^4} + \frac{\tau^2\sigma_2^4}{\mu^6KR^5 B^2} + \frac{\sigma_4^4\ln(K)}{\mu^4KB^2R^4}}\Bigg)\enspace,
    \end{align*}
    where we define $\kappa' := 2 + \frac{4Q^2B^2}{\mu^2} + \frac{6H^4}{\mu^4}$ and we pick the step-size,
    \begin{align*}
        \eta = \min\cb{\frac{1}{2H}, \frac{1}{\mu KR}\ln\rb{\frac{B^2}{\epsilon}}}\enspace,
    \end{align*}
    with the choice of $\epsilon$ is given by 
    \begin{align*}
        \epsilon := \max\Bigg\{&\frac{\sigma^2}{\mu^2 MKR}, \frac{\sigma^4}{\mu^4K^3R^3 M^2B^2}, \frac{\tau^2\phi_\star^2\kappa'}{\mu^2R^2}, \frac{\tau^4\sigma_4^2\kappa'\phi_\star^2}{\mu^6 KR^5B^2}, \frac{\sigma_2^2\tau^2\kappa'\phi_\star^2}{\mu^4KR^4B^2}, \frac{\tau^4\kappa'\phi_\star^4}{\mu^4B^2R^4},\\
        &\quad  \frac{H^2\zeta_\star^2\kappa'}{\mu^2 R^2} + \frac{\tau^2\sigma_2^2}{\mu^4KR^3} + \frac{\sigma_2^2\ln(K)}{\mu^2KR^2},\frac{H^4\zeta_\star^4\kappa'}{\mu^4R^3 B^2}, \frac{\tau^4\sigma_4^4\kappa'}{\mu^8K^2R^5B^2}, \frac{\sigma_2^2H^2\zeta_\star^2\kappa'}{\mu^4B^2R^4},\\
        &\quad \frac{\tau^2\sigma_2^4\kappa'}{\mu^6KR^5 B^2}, \frac{\sigma_4^4\ln(K)\kappa'}{\mu^4KB^2R^4}, \epsilon_{target}\Bigg\}
    \end{align*}
    where $\epsilon_{target}$ is the target accuracy, greater than or equal to the machine precision.
\end{lemma}
\begin{proof}
    We will pick the following step-size,
    \begin{align*}
        \eta = \min \cb{\frac{1}{2H}, \frac{1}{\mu KR}\ln \rb{\frac{B^2}{\epsilon}}}\enspace,
    \end{align*}
    where the choice of $\epsilon>0$ will be made explicit later. We will first identify the requirements on problem parameters to guarantee that,
    \begin{align*}
        &\frac{1-(1-\eta\mu)^K}{1-\rho_3} \leq 2\enspace,\\
        &\Leftrightarrow \frac{1-(1-\eta\mu)^K}{2} \leq (1-\rho_3)\enspace,\\
        &\Leftrightarrow \frac{1}{2} \leq 1- \rb{\frac{2\tau^2}{\mu^2} + \frac{2Q^2B^2}{\mu^2} + \frac{16\eta^2\sigma_2^2\tau^2}{\mu^2 M B^2} + \frac{H^4}{\mu^4} + \frac{16\eta^2\sigma_2^2Q^2}{\mu^2 M}}\rb{4\eta^2\tau^2K^2 + 64\eta^4\tau^4K^4}\enspace,\\
        &\Leftrightarrow \rb{\frac{2\tau^2}{\mu^2} + \frac{2Q^2B^2}{\mu^2} + \frac{16\eta^2\sigma_2^2\tau^2}{\mu^2 M B^2} + \frac{H^4}{\mu^4} + \frac{16\eta^2\sigma_2^2Q^2}{\mu^2 M}}\rb{4\eta^2\tau^2K^2 + 64\eta^4\tau^4K^4} \leq \frac{1}{2}\enspace,\\
        &\Leftarrow^{\text{(a)}} \rb{\frac{2Q^2B^2}{\mu^2} + \frac{16\sigma_2^2\tau^2}{\mu^4K^2R^2 M B^2}\ln^2\rb{\frac{B^2}{\epsilon}} + \frac{3H^4}{\mu^4} + \frac{16\sigma_2^2Q^2}{\mu^4K^2R^2 M}\ln^2\rb{\frac{B^2}{\epsilon}}}\\
        &\qquad \times \rb{\frac{4\tau^2}{\mu^2R^2}\ln^2\rb{\frac{B^2}{\epsilon}} + \frac{64\tau^4}{\mu^4 R^4}\ln^4\rb{\frac{B^2}{\epsilon}}} \leq \frac{1}{2}\enspace,\\
        &\Leftarrow (i)\ \rb{\frac{2Q^2B^2}{\mu^2} + \frac{3H^4}{\mu^4}}\rb{\frac{4\tau^2}{\mu^2R^2}\ln^2\rb{\frac{B^2}{\epsilon}} + \frac{64\tau^4}{\mu^4 R^4}\ln^4\rb{\frac{B^2}{\epsilon}}}\leq \frac{1}{4}\enspace; \quad \text{ and}\\
        &\quad  (ii)\ \rb{\frac{16\sigma_2^2\tau^2}{\mu^4K^2R^2 M B^2}\ln^2\rb{\frac{B^2}{\epsilon}}+ \frac{16\sigma_2^2Q^2}{\mu^4K^2R^2 M}\ln^2\rb{\frac{B^2}{\epsilon}}}\\
        &\qquad \times \rb{\frac{4\tau^2}{\mu^2R^2}\ln^2\rb{\frac{B^2}{\epsilon}} + \frac{64\tau^4}{\mu^4 R^4}\ln^4\rb{\frac{B^2}{\epsilon}}}\leq \frac{1}{4}\enspace,\\
        &\Leftarrow (i)\ \sqrt{\frac{16Q^2B^2}{\mu^2} + \frac{24H^4}{\mu^4}}\cdot\frac{2\tau}{\mu}\ln \rb{\frac{B^2}{\epsilon}} \leq R\enspace;\\
        &\quad (ii)\ \sqrt[4]{\frac{16Q^2B^2}{\mu^2} + \frac{24H^4}{\mu^4}}\cdot\frac{\sqrt[4]{64}\tau}{\mu}\ln \rb{\frac{B^2}{\epsilon}}\leq R\enspace;\\
        &\quad (iii)\ \frac{8\sigma_2\tau}{\mu^2 \sqrt{M} B}\ln\rb{\frac{B^2}{\epsilon}} \leq KR\enspace;\\
        &\quad (iv)\ \frac{8\sigma_2 Q}{\mu^2 \sqrt{M}}\ln\rb{\frac{B^2}{\epsilon}}\leq KR\enspace;\quad\text{ and}\\
        &\quad (v)\ \frac{4\tau}{\mu}\ln\rb{\frac{B^2}{\epsilon}}\leq R \enspace,\\
        &\Leftarrow (i)\ \textcolor{red}{KR\geq \frac{8\sigma_2}{\mu^2\sqrt{M}}\ln\rb{\frac{B^2}{\epsilon}}\cdot \max\cb{\frac{\tau}{B},Q}}\enspace; \quad \text{and}\\ 
        &\quad (ii)\ \textcolor{red}{R\geq \frac{3\tau}{\mu}\ln\rb{\frac{B^2}{\epsilon}}\cdot\rb{\frac{4QB}{\mu} + \frac{5H^2}{\mu^2}}}\enspace.
    \end{align*}
    where in (a) we used that $\tau^2/\mu^2 \leq H^2/\mu^2$. Now we will upper bound $\rho_3$ as follows,
    \begin{align*}
        \rho_3 &\leq^{\text{(a)}} (1-\eta\mu)^K + \frac{1}{R}\ln\rb{\frac{B^2}{\epsilon}}\rb{ \frac{2Q^2B^2}{\mu^2} + \frac{16\sigma_2^2}{\mu^4K^2R^2 M}\rb{\frac{\tau^2}{B^2}+1}\ln^2\rb{\frac{B^2}{\epsilon}} + \frac{3H^4}{\mu^4}}\\
        &\qquad \times \rb{\frac{4\tau^2}{\mu^2R^2}\ln^2\rb{\frac{B^2}{\epsilon}} + \frac{64\tau^4}{\mu^4 R^4}\ln^4\rb{\frac{B^2}{\epsilon}}}\enspace,\\
        &\leq^{\text{(b)}} e^{-\eta\mu K} + \frac{1}{R}\ln\rb{\frac{B^2}{\epsilon}}\rb{ \frac{2Q^2B^2}{\mu^2} + 1 + \frac{3H^4}{\mu^4}}\rb{\frac{4\tau^2}{\mu^2R^2}\ln^2\rb{\frac{B^2}{\epsilon}} + \frac{64\tau^4}{\mu^4 R^4}\ln^4\rb{\frac{B^2}{\epsilon}}}\enspace,\\
        &\leq e^{-\eta\mu K}\Bigg(1+\frac{e^{\eta\mu K}}{R}\ln\rb{\frac{B^2}{\epsilon}}\rb{ \frac{2Q^2B^2}{\mu^2} + 1 + \frac{3H^4}{\mu^4}}\rb{\frac{4\tau^2}{\mu^2R^2}\ln^2\rb{\frac{B^2}{\epsilon}} + \frac{64\tau^4}{\mu^4 R^4}\ln^4\rb{\frac{B^2}{\epsilon}}}\Bigg)\enspace,\\
        &\leq e^{-\eta\mu K}\exp\rb{\frac{e^{\eta\mu K}}{R}\ln\rb{\frac{B^2}{\epsilon}}\rb{ \frac{2Q^2B^2}{\mu^2} + 1 + \frac{3H^4}{\mu^4}}\rb{\frac{4\tau^2}{\mu^2R^2}\ln^2\rb{\frac{B^2}{\epsilon}} + \frac{64\tau^4}{\mu^4 R^4}\ln^4\rb{\frac{B^2}{\epsilon}}}}\enspace,\\
        &\leq e^{-\eta\mu K}\exp\rb{\frac{e^{1/R\ln(B^2/\epsilon)}}{R}\ln\rb{\frac{B^2}{\epsilon}}\rb{ \frac{2Q^2B^2}{\mu^2} + 1 + \frac{3H^4}{\mu^4}}\rb{\frac{4\tau^2}{\mu^2R^2}\ln^2\rb{\frac{B^2}{\epsilon}} + \frac{64\tau^4}{\mu^4 R^4}\ln^4\rb{\frac{B^2}{\epsilon}}}}\enspace,\\
         &\leq e^{-\eta\mu K}\exp\rb{\frac{1}{R}\rb{\frac{B^2}{\epsilon}}^{1/R}\ln\rb{\frac{B^2}{\epsilon}}\rb{ \frac{2Q^2B^2}{\mu^2} + 1 + \frac{3H^4}{\mu^4}}\rb{\frac{4\tau^2}{\mu^2R^2}\ln^2\rb{\frac{B^2}{\epsilon}} + \frac{64\tau^4}{\mu^4 R^4}\ln^4\rb{\frac{B^2}{\epsilon}}}}\enspace,
    \end{align*}
    where in (a) we use Bernoulli's Inequality and the choice of step-size, which implies $\eta\mu<1$ as well as the fact that $\tau^2/\mu^2 \leq H^4/\mu^4$; and in (b) we assumed that \textcolor{red}{the conditions derived} above to ensure $\frac{1-(1-\eta\mu)^K}{2} \leq 1-\rho_3$ are true, which allows us to conclude $\frac{16\sigma^2}{\mu^4K^2R^2 M}\rb{\frac{\tau^2}{B^2}+1}\ln^2\rb{\frac{B^2}{\epsilon}}\leq 1$. Raising both sides to the power $R$ gives us,
    \begin{align*}
        \rho_3^R &\leq e^{-\eta\mu KR}\exp\rb{\rb{\frac{B^2}{\epsilon}}^{1/R}\ln\rb{\frac{B^2}{\epsilon}}\rb{ \frac{2Q^2B^2}{\mu^2} + 1 + \frac{3H^4}{\mu^4}}\rb{\frac{4\tau^2}{\mu^2R^2}\ln^2\rb{\frac{B^2}{\epsilon}} + \frac{64\tau^4}{\mu^4 R^4}\ln^4\rb{\frac{B^2}{\epsilon}}}}\enspace,\\
        &\leq^{\text{(a)}} e^{-\eta\mu KR}\exp\rb{\ln^2\rb{\frac{B^2}{\epsilon}}\rb{ \frac{2Q^2B^2}{\mu^2} + 1 + \frac{3H^4}{\mu^4}}\rb{\frac{4\tau^2}{\mu^2R^2}\ln^2\rb{\frac{B^2}{\epsilon}} + \frac{64\tau^4}{\mu^4 R^4}\ln^4\rb{\frac{B^2}{\epsilon}}}}\enspace,\\
        &\leq^{\text{(b)}} e^{-\eta\mu KR+1}\enspace,
    \end{align*}
    where in (a) we assume that $\textcolor{red}{R\geq \frac{\ln(B^2/\epsilon)}{\ln(\ln(B^2/\epsilon))}}$; and in (b)
    we assume $\textcolor{red}{R\geq \frac{8\tau}{\mu}\ln^2\rb{\frac{B^2}{\epsilon}}\rb{1 + \frac{QB}{\mu} + \frac{H^2}{\mu^2}}}$ as well as $\textcolor{red}{R\geq \frac{8\tau}{\mu}\ln^{3/2}\rb{\frac{B^2}{\epsilon}}\rb{1 + \sqrt{\frac{QB}{\mu}} + \frac{H}{\mu}}}$. These observations, along with the \textcolor{red}{conditions} derived so far allow us to simplify the convergence rate as follows,
    \begin{align*}
        \norm{\aa(KR)} s&\leq e^{-\eta\mu KR + 1}\sqrt{2}B^2 + \frac{2\eta\sigma_2^2}{\mu M} + \frac{18\eta^3\sigma_4^4}{\mu M^2B^2} + \rb{2 + \frac{4Q^2B^2}{\mu^2} + \frac{6H^4}{\mu^4}}\norm{V}\enspace,\\
        &\leq 4e^{-\eta\mu KR}B^2 + \frac{2\sigma_2^2}{\mu^2 MKR}\ln\rb{\frac{B^2}{\epsilon}} + \frac{18\sigma_4^4}{\mu^4K^3R^3 M^2B^2}\ln^3\rb{\frac{B^2}{\epsilon}} +\kappa'\rb{\frac{4\tau^2\phi_\star^2}{\mu^2R^2}\ln^2\rb{\frac{B^2}{\epsilon}}}\\
        &\quad +\kappa'\rb{\frac{128\tau^4\sigma_4^2}{\mu^6 KR^5B^2}\phi_\star^2\ln^5\rb{\frac{B^2}{\epsilon}} +\frac{320\sigma_2^2\tau^2}{\mu^4KR^4B^2}\phi_\star^2\ln^4\rb{\frac{B^2}{\epsilon}} + \frac{64\tau^4}{\mu^4B^2R^4}\phi_\star^4\ln^4\rb{\frac{B^2}{\epsilon}}}\\
        &\quad +\kappa'\rb{\frac{2H^2\zeta_\star^2}{\mu^2 R^2}\ln^2\rb{\frac{B^2}{\epsilon}} + \frac{2\tau^2\sigma_2^2}{\mu^4KR^3}\ln^3\rb{\frac{B^2}{\epsilon}} + \frac{2\sigma_2^2\ln(K)}{\mu^2KR^2}\ln^2\rb{\frac{B^2}{\epsilon}}}\\
        &\quad + \kappa'\rb{\frac{8H^4\zeta_\star^4}{\mu^4R^3 B^2}\ln^3\rb{\frac{B^2}{\epsilon}} +\frac{88\tau^4\sigma_4^4}{\mu^8K^2R^5B^2}\ln^5\rb{\frac{B^2}{\epsilon}} + \frac{160\sigma^2H^2\zeta_\star^2}{\mu^4B^2R^4}\ln^4\rb{\frac{B^2}{\epsilon}}}\\
        &\quad +\kappa'\rb{\frac{160\tau^2\sigma_2^4}{\mu^6KR^5 B^2}\ln^5\rb{\frac{B^2}{\epsilon}} + \frac{112\sigma_4^4\ln(K)}{\mu^4KB^2R^4}\ln^4\rb{\frac{B^2}{\epsilon}}}\enspace,
    \end{align*}
    where we define $\kappa' := \rb{2 + \frac{4Q^2B^2}{\mu^2} + \frac{6H^4}{\mu^4}}$. We are almost done, but we need to define $\epsilon$. Our choice of $\epsilon$ is simply the maximum of all the terms (after removing the logarithmic factors) in the above convergence bound, except for the first exponential term and the target accuracy $\epsilon_{target}$, which is an input to the algorithm. Like in the previous lemmas' proofs, we recall that the term dominating in $\epsilon$ also dominates the final convergence rate.  This choice of $\epsilon$ and $\eta$, proves the lemma statement. 
\end{proof}

\section{More Details on the Experiments}\label{app:experiments}
In this appendix, we describe in full detail how we generated the synthetic data for each client and how we controlled first- and second-order heterogeneity without altering the inherent difficulty of the individual optimization problems (e.g.\ their condition numbers or solution norms) for the experiments in the main body.

\subsection{Data generation for each client}

We consider a linear regression problem with parameter dimension \(d\).  There are \(M\) clients, indexed by \(m=1,\dots,M\).  For each client \(m\), we generate i.i.d.\ data \((\beta_m,y_m)\sim\ddd_m\) with
\[
  \beta_m \;\sim\;\nnn(\mu_m,\,I_d), 
  \qquad
  y_m = \langle x_m^\star,\beta_m\rangle + \varepsilon,
  \;\;\varepsilon\sim\nnn(0,\sigma_{noise}^2)\enspace.
\]
The corresponding per-sample squared loss is
\[
  f(x;(\beta_m,y_m))
  = \tfrac12\bigl(y_m - \langle x,\beta_m\rangle\bigr)^2\enspace,
\]
and the population objective on client \(m\) is
\[
  F_m(x)
  = \ee_{(\beta,y)\sim\ddd_m}\bigl[f(x;(\beta,y))\bigr]
  = \tfrac12\,(x - x_m^\star)^\top\bigl(\mu_m\mu_m^\top + I_d\bigr)\,(x - x_m^\star)
    + \tfrac12\,\sigma_{noise}^2\enspace.
\]
Under suitable bounds on
\(\|\mu_m\|\), \(\sigma_{noise}\), and \(\|x_m^\star\|\), these objectives satisfy ~\Cref{ass:strongly_convex,ass:smooth_second,ass:stoch_bounded_second_moment,ass:bounded_optima} for all \(x\) in a bounded region.

\subsection{Controlling first-order (concept) heterogeneity}

We fix the norm of each true optimizer to \(\|x_m^\star\|=R_\star\).  To vary the maximum pairwise distance
\(\max_{m,n}\!\|x_m^\star - x_n^\star\| = \zeta_\star\), we sample each
\[
  x_m^\star 
  = R_\star\,v_m
  \quad\text{with}\quad
  v_m\in\rr^d,\;\|v_m\|=1\enspace,
\]
where \(v_m\) is drawn uniformly from the spherical cap of half-angle
\[
  \phi(\zeta_\star)
  = \arcsin\!\rb{\frac{\zeta_\star}{2R_\star}}\enspace,
\]
around a fixed ``central'' random unit vector \(v_0\).  This ensures
\(\|x_m^\star\|=R_\star\) for all \(m\), and
\(\max_{m,n}\|x_m^\star - x_n^\star\|=\zeta_\star\), so that larger \(\zeta_\star\)
increases concept heterogeneity purely by angular dispersion, without
changing the optimizer norms. This process is illustrated in \Cref{fig:spherical_cap}. In our experiments, we fix $R_\star=1$.

\begin{figure}[ht]
  \centering
  \includegraphics[width=0.6\textwidth]{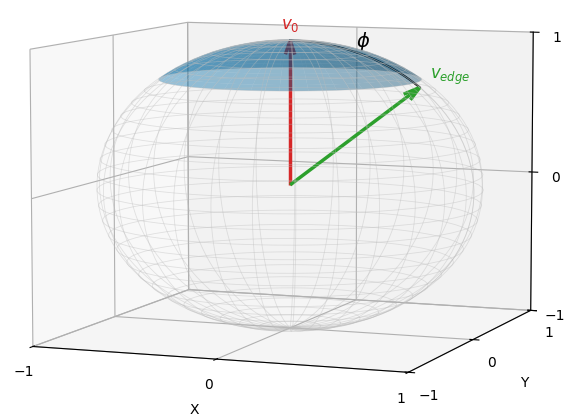}
  \caption{%
    Illustration of sampling unit vectors from a spherical cap.  We draw a cross-section of the unit sphere, mark the central axis $v_0$, and show the cap of half-angle $\phi(\zeta_\star)$ (shaded blue).
  }
  \label{fig:spherical_cap}
\end{figure}

\subsection{Controlling second-order (covariate) heterogeneity}

Likewise, we fix each covariance matrix to \(I_d\) and fix the norm of the
feature mean to \(\|\mu_m\|=\mu_0\).  To vary the maximum pairwise mean
distance \(\max_{m,n}\|\mu_m - \mu_n\| = \tau\), we sample
\[
  \mu_m = \mu_0\,u_m,
  \quad
  u_m\in\rr^d,\;\|u_m\|=1\enspace,
\]
with \(u_m\) drawn uniformly from the spherical cap of half-angle
\[
  \theta(\tau)
  = \arcsin\!\bigl(\tau/(2\,\mu_0)\bigr)\enspace,
\]
around the same central direction \(v_0\).  Again, this rotates the means
without altering \(\|\mu_m\|\) or the eigenvalues of the Hessians
\(\nabla^2 F_m = \mu_m\mu_m^\top + I_d\), whose condition number remains
\(1+\mu_0^2\). In our experiments, we fix $\mu_0=5$.



\subsection{Hyper-parameter tuning and metrics}

For every experimental setting \((\tau,\zeta_\star)\) (or every \(\tau\) in the
communication-complexity study) we first sample $v_0$ and sample $\{x_m^\star\}$, then we perform $20$ independent
trials with fresh draws of \(\{\mu_m\}\).  In each trial, we
search over a logarithmic grid of step-sizes
\(\eta\in[10^{-3},10^{-1}]\) and record either:
\begin{itemize}
  \item The final \(\ell_2\) error \(\|x^R - \bar x^\star\|\) after \(R\)
        rounds (for the heatmap in Figure~\ref{fig:heatmap}), or
  \item The minimum number of rounds \(r\le R_{\max}\) needed to reach
        \(\|x^r - \bar x^\star\|\le \epsilon\) (for the communication plot in
        Figure~\ref{fig:comm_rounds}).
\end{itemize}
We then average these quantities over the \(n_{runs}\) trials to obtain the plotted heatmaps and curves. 

\subsection{Ensuring fixed problem difficulty}

By sampling \(\{x_m^\star\}\) and \(\{\mu_m\}\) on fixed-radius spheres
and using identity covariances, we keep every client’s Hessian condition
number and solution norm constant, so that any change in convergence or
communication cost is attributable purely to the angular dispersion
(i.e.\ heterogeneity) parameters \(\tau\) and \(\zeta\), not to changes in
problem conditioning or scale.

\chapter{Additional Details for Chapter 6}\label{app:chap6}
\section{Proof of Non-convex Lower Bounds}\label{sec:non_convex_lb}
In this section, we prove Theorems \ref{thm:lb_cent} and \ref{thm:lb}. Both these results share the communication complexity terms $\min\{\Delta \tau/\epsilon, H^2\zeta^2/\epsilon\}$. We'd show that any algorithm in $\aaa_{ZR}$, no matter whether it uses an exact or stochastic oracle, and for any number of oracle queries $K$ between communication rounds, must incur these many communication rounds. To do so, we'd use the non-convex hard instance proposed by \citet{carmon2020lower} and split it across different machines similar to \citet{arjevani2015communication, woodworth2020minibatch}. Specifically, we consider the following functions (where we assume for simplicity that $d$ is even):
\begin{align}
    F(x) &:= \frac{F_1(x) + F_2(x)}{2}\enspace,\\
    F_1(x) &:= - \psi(x)\phi(x_1) + \sum_{i=1}^{d/2-1}\sb{\psi(-x_{2i})\phi(-x_{2_i+1}) - \psi(x_{2i})\phi(x_{2i+1})}\enspace,\\
    F_2(x) &:= \sum_{i=1}^{d/2}\sb{\psi(-x_{2i-1})\phi(-x_{2i}) - \psi(x_{2i-1})\phi(x_{2i})}\enspace,
\end{align}
where the component functions $\psi(\cdot)$ and $\phi(\cdot)$ are defined as follows,
\begin{align}
    \psi(t) = \begin{cases} 0, & t\leq 1/2,\\ \exp\rb{1-\frac{1}{(2t-1)^2}}, & t>1/2.\end{cases} \quad\text{and}\quad \phi(t) = \sqrt{e}\int_{-\infty}^{t}e^{-\frac{1}{2}\tau^2}d\tau\enspace. 
\end{align}

The functions $F_1, F_2$ have the following interesting property: Let $E_k$ be the (co-ordinate) span of first $k$ basis vectors, i.e., $\text{span}(e_1, \dots, e_k)$. Note that when $x_k\in E_k$ and $k$ is odd, we have $$\nabla F_1(x_k)\in E_k \text{ and } \nabla F_2(x_k)\in E_{k+1}\enspace,$$ while when $k$ is even, $$\nabla F_1(x_k)\in E_{k+1} \text{ and } \nabla F_2(x_k)\in E_{k}\enspace.$$
In our construction, half of the machines will have the function $F_1$, and the other half will have the function $F_2$ (assuming $M$ is even; we will see later that it only changes the lower bound by a factor of $M-1/M$). First, we initialize all the $M$ machines at $0$ and optimize using any distributed zero-respecting algorithm (see Definition \ref{def:zero_respecting}). Then, the only way to access the next coordinate is to query the gradient of one of two functions---$F_1$ if the next coordinate is odd and $F_2$ if the next coordinate is even. This means that, between two rounds of communication, at least one set of machines can't make any progress, and the other set of machines only learns about at most one new coordinate. Thus, the machines are forced to communicate at least $d-1$ times to be able to span $\rr^d$. More formally, we can prove the following lemma:

\begin{lemma}\label{lem:output_in_ER}
For any vector $v\in \rr^d$, define $\supp{v} = \{i\in[d]: v_i\neq 0\}$. Let $x_R$ be the output of any algorithm $A\in \aaa_{ZR}$ equipped with oracles $\{\ooo_{F_m}\}_{m\in[M]}$ on each machine, initialized at $0$ and optimizing the problem with $F_1$ on the first half machines and $F_2$ on the secocnd half. Then after $R$ rounds of communication, $$\supp{x_R} \in E_R\enspace.$$
\end{lemma}
The proof of this lemma is identical to Lemma 9 in \cite{woodworth2020minibatch}. We'd use this observation along with some properties of the hard instance to show our lower bound. In particular, we note the following properties for the function $F(\cdot)$. 
\begin{lemma}[Lemma 3 in \citet{carmon2020lower}]\label{lem:hard_function_facts}
The function $F$ satisfies the following:
\begin{enumerate}
    \item[i.] We have $F(0) - \inf_{x}F(x) \leq \Delta_0d$, where $\Delta_0=12$.
    \item[ii.] For all $x\in \rr^d$, $\norm{\nabla F(x)}\leq 23\sqrt{d}$.
    \item[iii.] For every $p\geq 1$, the $p$-th order derivatives of $F$ are $l_p$-Lipschitz continuous, where $$l_p\leq \exp\rb{\frac{5}{2}p\log p + cp}$$ for a numerical constant $c<\infty$. In particular, $l_1=152$ (c.f., Lemma 2.2 in \citet{arjevani2019lower}).
\end{enumerate}
\end{lemma}

Note that these properties imply the following for $F$ (c.f., Lemma 2 in \citet{carmon2020lower}.). 
\begin{lemma}\label{lem:lower_bound_gradient_norm}
For all $x\in E_k$, where $k<d$, $\norm{\nabla F(x)} \geq 1$.
\end{lemma}
In other words, if the model vector $x$ doesn't span $\rr^d$, it will be forced to have a large gradient. And our distributed problem structure forces the iterates to lie in $E_R$ after $R$ communication rounds, as highlighted in Lemma \ref{lem:output_in_ER}. Formalizing this idea results in the following communication complexity lower bound: 

\begin{theorem}[Communication complexity second-order]\label{thm:lb_second} Any algorithm $A \in \aaa_{zr}$ optimizing a problem instance satisfying \Cref{ass:smooth_second,ass:bounded_func_subopt,ass:tau} and with $K>0$ intermittent accesses to deterministic $n$-point oracles (cf. \Cref{def:oracle_first_multi_point}) on all the clients needs communication rounds, $$R\geq b_1\cdot\frac{\Delta \tau}{\epsilon}\enspace,$$ to output $x^{A}_R$ such that $\ee[\norm{\nabla F(x^A_R)}^2] \leq \epsilon$ where $\epsilon< b_2\tau\Delta$ and $b_1, b_2$ are numerical constants.
\end{theorem}

\begin{proof}
Let $\Delta_0, l_1$ be the numerical constants as in Lemma \ref{lem:hard_function_facts}. Given accuracy parameter $0<\epsilon < \frac{\tau\Delta}{4\Delta_0l_1}$ we define the following functions defined on $\mathbb{R}^{d+1}\to \mathbb{R}$, $$F_1^\star(x):= \frac{\tau\lambda^2}{4l_1}F_1\rb{\frac{x_{1:d}}{\lambda}} + \frac{H}{4}x_{d+1}^2,\ F_2^\star(x):= \frac{\tau\lambda^2}{4l_1}F_2\rb{\frac{x_{1:d}}{\lambda}}+ \frac{L}{4}x_{d+1}^2\enspace,$$ where $\lambda:= \frac{4l_1}{\tau}\cdot \sqrt{\epsilon}$, and $x_{1:d}\in \mathbb{R}^{d}$ denotes $x\in\mathbb{R}^{d+1}$ restricted to the first $d$ dimensions. For $M > 2$, we place $F_1^\star$ on the first $\lfloor M/2\rfloor$ machines, $F_2^\star$ on the next $\lfloor M/2\rfloor$ machines, and if $M$ is odd, we place the zero function on the last machine. This only worsens the result by a factor of $\rb{\frac{M-1}{M}}^2$ as we'd see below, so we can assume without loss of generality that $M$ is even. We define $$F^\star(x) := \frac{F_1^\star(x) + F_2^\star(x)}{2} = \frac{\tau \lambda^2}{4l_1}F\left(\frac{x_{1:d}}{\lambda}\right) + \frac{H}{4}x_{d+1}^2\enspace,$$ as the average objective of $M$ machines. Further choosing $d = \left\lfloor \frac{\tau\Delta}{4\Delta_0l_1\epsilon}\right\rfloor \geq 1$ guarantees that (due to Lemma \ref{lem:hard_function_facts}), $$F^\star(0) - \inf_{x}F^\star(x)  = F(0) - \inf_{x\in \mathbb{R}^d}F(x)\leq \frac{\tau\lambda^2\Delta_0}{l_1}\cdot d = \frac{4l_1\epsilon\Delta_0}{\tau}\left\lfloor \frac{\tau\Delta}{4\Delta_0l_1\epsilon}\right\rfloor \leq \Delta\enspace.$$
Additionally, each of our objectives is $H$-smooth as $\tau \leq H$. The second order heterogeneity of our problem is bounded by $\tau$ as for all $x$, $$\frac{1}{2}\norm{\nabla^2F_1^{\star}(x) - \nabla^2F_2^{\star}(x)} = \frac{\tau}{8l_1}\norm{\nabla^2 F_1\rb{\frac{x}{\lambda}} - \nabla^2 F_2\rb{\frac{x}{\lambda}}} \leq \tau\enspace.$$

Thus, $F_1^\star, F_2^\star$ characterize a distributed optimization problem that satisfies all our assumptions. Now, we initialize our algorithm at $0$. Then, using Lemma \ref{lem:output_in_ER}, we know that for all $r\in[R]$, the output of the algorithm after $r$ communication rounds, i.e., $x_r\in E_r$. In particular for $r\in[d-1]$ using Lemma \ref{lem:lower_bound_gradient_norm} this implies that $$\ee\sb{\norm{\nabla F^\star(x_r)}^2} \geq \rb{\frac{\tau \lambda}{4l_1}}^2 \geq \epsilon.$$ Thus, if we want to achieve $\epsilon$-stationarity, we need to communicate at least $d-1$ times. In other words, $$R\geq d-1 \geq \frac{1}{8\Delta_0l_1}\cdot\frac{\tau \Delta}{\epsilon}\enspace.$$ This concludes the proof of the theorem with $b_1 = \frac{1}{8\Delta_0 l_1}$ and $b_2 =\frac{1}{4\Delta_0 l_1}$. 
\end{proof}

Similarly while optimizing problems that satisfy \Cref{ass:zeta} we can get the following communication lower bound.

\begin{theorem}[Communication complexity first-order]\label{thm:lb_first} Any algorithm $A \in \aaa_{zr}$ optimizing a problem instance satisying \Cref{ass:smooth_second,ass:bounded_func_subopt,ass:zeta} and with $K>0$ intermittent accesses to deterministic $n$-point oracles (cf. \Cref{def:oracle_first_multi_point}) on all the clients needs communication rounds, $$R\geq b_3\cdot\frac{H^2\zeta^2 }{\epsilon}\enspace,$$ to output $x^{A}_R$ such that $\ee[\norm{\nabla F(x^A_R)}^2] \leq \epsilon$ where $\epsilon< b_4H^2\zeta^2$ and $b_3, b_4$ are numerical constants.
\end{theorem}

\begin{proof}
Let $\Delta_0, l_1$ be the numerical constants as in Lemma \ref{lem:hard_function_facts}. Given accuracy parameter $0< \epsilon < \frac{H^2\zeta^2}{\Delta_0 l_1}$ we define the following functions, $$F_1^\star(x):= \frac{H^2\zeta^2\lambda^2}{\Delta l_1}F_1\rb{\frac{x}{\lambda}},\ F_2^\star(x):= \frac{H^2\zeta^2\lambda^2}{\Delta l_1}F_2\rb{\frac{x}{\lambda}}\enspace,$$ where $\lambda:= \frac{\Delta l_1}{H^2\zeta^2}\cdot \sqrt{\epsilon}$. For $M > 2$, we place $F_1^\star$ on the first $\lfloor M/2\rfloor$ machines, $F_2^\star$ on the next $\lfloor M/2\rfloor$ machines, and if $M$ is odd, we place the zero function on the last machine. This only worsens the result by a factor of $\rb{\frac{M-1}{M}}^2$ as we'd see below, so we can assume without loss of generality that $M$ is even. We define $$F^\star(x) := \frac{F_1^\star(x) + F_2^\star(x)}{2}\enspace,$$ as the average objective of $M$ machines. Further choosing $d = \left\lfloor \frac{e^cH^2\zeta^2}{\Delta_0l_1\epsilon}\right\rfloor \geq 1$ guarantees that (due to Lemma \ref{lem:hard_function_facts}), $$F^\star(0) - \inf_{x}F^\star(x) \leq \frac{H^2\zeta^2\lambda^2\Delta_0}{\Delta l_1}\cdot d = \frac{\Delta l_1\epsilon\Delta_0}{H^2\zeta^2}\left\lfloor \frac{H^2\zeta^2}{\Delta_0l_1\epsilon}\right\rfloor \leq \Delta\enspace.$$
Additionally, each of our objectives is $ L$-smooth, as $H^2\zeta^2/\Delta \leq H$. The first order heterogeneity of our problem is bounded by $H^2\zeta^2$ as for all $x$ (upto numerical constants),
\begin{align*}
\frac{1}{M}\sum_{m\in[M]}\norm{\nabla F_m(x) - F(x)}^2 &= \frac{1}{2}\norm{\nabla F_1^{\star}(x) - \nabla F_2^{\star}(x)}^2,\\
   &= \frac{\epsilon}{2}\norm{\nabla F_1\rb{\frac{x}{\lambda}} - \nabla  F_2\rb{\frac{x}{\lambda}}}^2\enspace,\\
   &\leq (23)^2\epsilon d\enspace,\\
   &= (23)^2\epsilon\left\lfloor \frac{H^2\zeta^2}{\Delta_0l_1\epsilon}\right\rfloor\enspace,\\
   &\leq \frac{(23)^2}{\Delta_0l_1} \cdot H^2\zeta^2 \leq H^2\zeta^2\enspace,  
\end{align*}
where the last step follows from noting that $\Delta_0 = 12, l_1=152$.

Thus, $F_1^\star, F_2^\star$ characterize a distributed optimization problem that satisfies all our assumptions. Now, we initialize our algorithm at $0$. Then, using Lemma \ref{lem:output_in_ER}, we know that for all $r\in[R]$, the output of the algorithm after $r$ communication rounds, i.e., $x_r\in E_r$. In particular for $r\in[d-1]$ using Lemma \ref{lem:lower_bound_gradient_norm} this implies that $$\ee\sb{\norm{\nabla F^\star(x_r)}^2} \geq \rb{\frac{H^2\zeta^2 \lambda}{\Delta l_1}}^2 \geq \epsilon\enspace.$$ Thus if we want to achieve, $\epsilon$-stationarity we need to communicate at least $d-1$ times. In other words, $$R\geq d-1 \geq \frac{1}{2\Delta_0l_1}\cdot\frac{H^2\zeta^2}{\epsilon}\enspace.$$ 

This concludes the proof of the theorem with $b_3 = \frac{1}{2\Delta_0 l_1}$ and $b_4 =\frac{1}{\Delta_0 l_1}$.
\end{proof}

Note that Theorems \ref{thm:lb_first} and \ref{thm:lb_second} imply a non-trivial lower bound even if the clients are allowed infinite oracle accesses between two communication rounds, i.e., $K\to \infty$ in the intermittent communication setting. Next, we combine these results with known first-order oracle complexity lower bounds to get the stated theorem statements. We begin by first re-stating theorem \ref{thm:lb}.

\begin{theorem}[General Lower Bound]\label{thm:general_lb} Any algorithm $A \in \aaa_{zr}$ optimizing a problem instance satisying \Cref{ass:smooth_second,ass:bounded_func_subopt,ass:tau,ass:zeta} and with $K>0$ intermittent accesses to stochastic $2$-point oracles (cf. \Cref{def:oracle_first_multi_point}) satisfying \Cref{ass:stoch_bounded_second_moment}, outputs $x^A_R$ after $R\geq c_2$ rounds such that,
$$\ee\left[\norm{\nabla F(x^A_R)}^2\right]\geq c_1\cdot\rb{  \min\left\{\frac{H^2\zeta^2}{R}, \frac{\Delta \tau}{R}\right\} + \frac{\Delta H}{KR} + \frac{\sigma_2^2}{MKR} + \left(\frac{\sigma_2 \Delta H}{MKR}\right)^{2/3}}\enspace,$$
where $c_1, c_2$ are numerical constants.
\end{theorem}
\begin{proof}
Note that using Theorems \ref{thm:lb_first} and \ref{thm:lb_second} we've proven that, the communication complexity is lower bounded by $\min\left\{\frac{\Delta \tau}{\epsilon}, \frac{H^2\zeta^2}{\epsilon}\right\}$ when $\tau/2, H62\zeta^2/\Delta \leq H$ and $c_2\cdot \epsilon \leq \cdot\min\{\tau \Delta, H^2\zeta^2\}$ (where $1/c_2$ is the maximum of the numerical constants appearing in \ref{thm:lb_first} and \ref{thm:lb_second}). This implies the first two terms in the lower bound for $R\geq c_1$.

To obtain the second term, we apply the function $F$ to all the machines and equip them with exact oracles, i.e., $\sigma = 0$. Since the oracle is queried at the same input on all the machines, as well as returns the same fixed output, the $M$ machines can be simulated by a single machine. Furthermore, a single query to a two-point oracle at two different points $v, w\in \rr^d$ is equivalent to querying the single point oracle two times at $v,w$. Thus, we can implement any algorithm $A\in \aaa_{ZR}^{cent}$ which requires $K$ total intermittent accesses to a two-point oracle for all $m\in[M]$, by instead considering a single machine with $2K$ intermittent accesses to a single-point oracle (cf. \Cref{def:oracle_first_multi_point}). According to Carmon et al., the latter problem requires at least $\Delta H/\epsilon$ oracle calls, which implies that our parallel problem requires at least $\Delta H/(K\epsilon)$ communication rounds. This gives the second term.

Finally, due to \citet{arjevani2019lower}, any zero respecting algorithm optimizing $F$ requires at least $\sigma_2^2/\epsilon + \sigma_2 \Delta H/\epsilon^{3/2}$ stochastic oracle calls to an active oracle (i.e., an oracle which takes as input both the query point and the random seed, c.f., Section 5.2 in \citet{arjevani2019lower}) which is strictly more powerful than the two-point oracle in \Cref{def:oracle_first_multi_point}. Thus, if we put $F_m=F$ on all machines, and give each machine active oracles, then the oracle queries must be lower bounded by $2MKR \geq \sigma_2^2/\epsilon + \sigma_2 \Delta H/\epsilon^{3/2}$. This, in turn, proves a lower bound on the queries to the weaker two-point oracles and proves the final two terms. 

We choose $c_1$ as the minimum of the numerical constants coming from Theorems \ref{thm:lb_first}, \ref{thm:lb_second}, \citet{carmon2020lower} and  \citet{arjevani2019lower}.
\end{proof}

Similarly, we can prove Theorem \ref{thm:lb_cent}.

\begin{theorem}[Centralized Lower Bound]
Any algorithm $A \in \aaa_{zr}^{cent}$ optimizing a problem instance satisying \Cref{ass:smooth_second,ass:bounded_func_subopt,ass:tau,ass:zeta} and with $K>0$ intermittent accesses to stochastic $2$-point oracles (cf. \Cref{def:oracle_first_multi_point}) satisfying \Cref{ass:stoch_bounded_second_moment}, over $R\geq c_1$ communication rounds must output $x^{A}_R$ such that $$\ee\left[\norm{\nabla F(x^A_R)}^2\right] \geq c_2\cdot \left(\frac{\Delta H}{R} + \frac{\sigma_2^2}{MKR} + \left(\frac{\sigma_2 \Delta H}{MKR}\right)^{2/3}\right)\enspace,$$ where $c_1, c_2$ are numerical constants. 
\end{theorem}
\begin{proof}
The last two oracle complexity terms follow the same way as in Theorem \ref{thm:lb} due to \citet{arjevani2019lower}. We only need to show how to get the higher first term. For this, we use the argument in  \citet{carmon2020lower}. We apply the function $F$ to all the machines and equip them with exact oracles, i.e., $\sigma_2=0$; moreover, since this is a homogeneous problem, $\tau, \zeta=0$ for this distributed problem. Furthermore, since the oracle is queried with the same input on all the machines and returns the same fixed output, the $M$ machines can be simulated by a single machine with only one intermittent access. A single query to the two-point oracle at two different points $v, w\in \rr^d$ is equivalent to querying the single-point oracle two times at $v,w$. According to \citet{carmon2020lower}, the latter problem requires at least $\Delta H/\epsilon$ oracle calls, implying that our parallel problem requires at least $\Delta H/\epsilon$ communication rounds. This gives the first term of the lower bound.    
\end{proof}

\section{Proof of Theorem \ref{thm:alg1}}\label{sec:app_proof_alg1}
In this section, we provide the full statement of Theorem \ref{thm:alg1} and its corresponding proofs.  More specifically, we choose the input $T=K$ in Algorithm \ref{alg:fed_vr} and present the results accordingly. We first present the full statement of Theorem \ref{thm:alg1}.
\begin{theorem}\label{thm:alg1_full}
Suppose we have a problem instance satisying \Cref{ass:smooth_second,ass:bounded_func_subopt,ass:tau} then, 
\begin{itemize}[leftmargin=1em]
    \item[(a)] if each client $m\in[M]$ has a stochastic two-point oracle (cf. \Cref{def:oracle_first_multi_point}), and assuming $\frac{\Delta H}{R} \leq \frac{\sigma_2^2}{\sqrt{MKb}}$, then the output $\tilde x$ of Algorithm \ref{alg:fed_vr} using
    \begin{align*}
        \beta=\max\left\{\frac{1}{R}, \frac{(\Delta H)^{2/3}(MKb)^{1/3}}{\sigma^{4/3}R^{2/3}}\right\} ,b_0=KR,
        \eta=c_1\cdot\min\left\{\frac{1}{H}, \frac{1}{K\tau},\frac{1}{\sqrt{K}H}, \frac{(\beta MK)^{1/2}}{HK}\right\}\enspace,
    \end{align*}
    satisfies the following
$$\ee\|\nabla F(\tilde x)\|^2\leq c_2\cdot\bigg( \frac{\Delta \tau}{R} +\frac{\Delta H}{KR}+\frac{\Delta H}{R\sqrt{Kb}}+\left(\frac{\sigma_2\Delta H}{MKbR}\right)^{2/3}+\frac{\sigma_2^2}{MKbR}\bigg)\enspace.$$
    \item[(b)] if each client  $m\in[M]$ has a deterministic two-point oracle ($\sigma_2=0$ in \Cref{def:oracle_first_multi_point}), then the output $\tilde x$ of Algorithm \ref{alg:fed_vr} using $\beta=1$ and $\eta=\min\left\{\frac{1}{H},\frac{1}{K\tau}\right\}$ satisfies,
$$\ee\|\nabla F(\tilde x)\|^2\leq c_3\cdot\left(\frac{\Delta\tau}{R} + \frac{\Delta H }{KR}\right)\enspace,$$
\end{itemize}
where $c_1,c_2,c_3$ are numerical constants.

In addition, if we have $\epsilon^{1/2}\leq \sigma_2\tau/(HM)$, $\epsilon\sigma_2^2\leq (\Delta H)^2$, and $M\epsilon^{1/2}\leq \min\{\sigma_2,\sigma_2^3/(H\Delta)\}$, then Algorithm \ref{alg:fed_vr} using $K=\sigma_2 H/(M\tau\epsilon^{1/2})$, $b_0=\sigma_2^3/(H\Delta M\epsilon^{1/2})$, $\beta=H\epsilon^{1/2}/(\sigma_2\tau)$ can achieve the $\epsilon$-approximate stationary point with the following communication and gradient complexities
\begin{align*}
    R\leq c_4\frac{\Delta \tau}{\epsilon}~\text{and}~N\leq c_5\frac{\Delta H\sigma_2}{\epsilon^{3/2}}\enspace,
\end{align*}
where $c_4,c_5$ are numerical constants.
\end{theorem}
\begin{proof}[Proof of Theorem \ref{thm:alg1_full} and Three Regimes in Figure \ref{fig:3region}]
In the following proof, we assume that each client can use a mini-batch gradient with a batch size of $b$, which allows us to obtain a more general result.
First, we will bound the term $\|w_{r+1,k}^j-x_r\|^2$ for each client at local updates. Let's consider the local updates for client $j$. For $k>1$, we have
\begin{align*}
   \|w_{r+1,k}^j-x_r\|^2&=\|w_{r+1,k-1}^j-\eta v_{r,k-1}^j-x_r\|^2\\
   &\leq \bigg(1+\frac{1}{K}\bigg)\|w_{r+1,k-1}^j-x_r\|^2+(1+K)\eta^2\|v_{r,k-1}^j\|^2\enspace,\\
   &\leq \bigg(1+\frac{1}{K}\bigg)\|w_{r+1,k-1}^j-x_r\|^2+2(1+K)\eta^2\|v_{r,k-1}^j-\nabla F(w_{r+1,k-1}^j)\|^2\\
   &\qquad+2(1+K)\eta^2\|\nabla F(w_{r+1,k-1}^j)\|^2\enspace.
\end{align*}
Therefore, recursively using the above inequality and the fact that $w_{r+1,1}^j=x_r$, we can obtain
\begin{align}\label{fp_eq3}
    \|w_{r+1,k}^j-x_r\|^2&\leq 2(1+K)\eta^2\sum_{l=2}^k\bigg(1+\frac{1}{K}\bigg)^{k-l}\|v_{r,l-1}^j-\nabla F(w_{r+1,l-1}^j)\|^2\nonumber\\
    &\qquad+2(1+K)\eta^2\sum_{l=2}^k\bigg(1+\frac{1}{K}\bigg)^{k-l}\|\nabla F(w_{r+1,l-1}^j)\|^2\nonumber\enspace,\\
    &\leq 2e(1+K)\eta^2\sum_{k=2}^K\|v_{r,k-1}^j-\nabla F(w_{r+1,k-1}^j)\|^2+2e(1+K)\eta^2\sum_{k=2}^K\|\nabla F(w_{r+1,k-1}^j)\|^2\nonumber\enspace,\\
    &=2e(1+K)\eta^2\sum_{k=1}^{K-1}\|v_{r,k}^j-\nabla F(w_{r+1,k}^j)\|^2+2e(1+K)\eta^2\sum_{k=1}^{K-1}\|\nabla F(w_{r+1,k}^j)\|^2\enspace.
\end{align}
Next, we will bound the estimation error between the local gradient estimator and the full gradient $\ee\|v_{r,k}^j-\nabla F(w_{r+1,k}^j)\|^2$. According to the definition $v_{r,k}^j=\nabla F_{j,\bbb_{r,k}^j}(w^j_{r+1,k})+v_{r,k-1}^j-\nabla F_{j,\bbb_{r,k}^j}(w^j_{r+1,k-1})$, we have
\begin{align*}
    \ee\|v_{r,k}^j-&\nabla F(w_{r+1,k}^j)\|^2\\ 
     &=\ee\big\|\big(v_{r,k-1}^j-\nabla F(w_{r+1,k-1}^j)\big)\\
    &\qquad+\big(\nabla F_{j,\bbb_{r,k}^j}(w^j_{r+1,k})-\nabla F_{j,\bbb_{r,k}^j}(w^j_{r+1,k-1})-\nabla F_{j}(w^j_{r+1,k})+\nabla F_{j}(w^j_{r+1,k-1})\big)\\
    &\qquad+\big(\nabla F_{j}(w^j_{r+1,k})-\nabla F_{j}(w^j_{r+1,k-1})+\nabla F(w_{r+1,k-1}^j)-\nabla F(w_{r+1,k}^j)\big)\big\|^2\enspace,\\
    &= \ee\big\|\nabla F_{j,\bbb_{r,k}^j}(w^j_{r+1,k})-\nabla F_{j,\bbb_{r,k}^j}(w^j_{r+1,k-1})-\nabla F_{j}(w^j_{r+1,k})+\nabla F_{j}(w^j_{r+1,k-1})\big\|^2\\
    &\qquad+\ee\big\|\big(v_{r,k-1}^j-\nabla F(w_{r+1,k-1}^j)\big)\\
    &\qquad\qquad+\big(\nabla F_{j}(w^j_{r+1,k})-\nabla F_{j}(w^j_{r+1,k-1})+\nabla F(w_{r+1,k-1}^j)-\nabla F(w_{r+1,k}^j)\big)\big\|^2\enspace,\\
    &\leq \frac{L^2}{b}\ee\|w^j_{r+1,k}-w^j_{r+1,k-1}\|^2+\bigg(1+\frac{1}{K}\bigg)\ee\|v_{r,k-1}^j-\nabla F(w_{r+1,k-1}^j)\|^2\\
    &\qquad+(1+K)\ee\big\|\nabla F_{j}(w^j_{r+1,k})-\nabla F_{j}(w^j_{r+1,k-1})+\nabla F(w_{r+1,k-1}^j)-\nabla F(w_{r+1,k}^j)\big\|^2\enspace,
\end{align*}
where the second equality is due to the independence of the random variables, the inequality comes from the fact that the mini-batch gradients consist of $b$ i.i.d. samples, and each client $m\in[M]$ has the two-point stochastic oracle from \Cref{def:oracle_first_multi_point}. Therefore, using the above inequality recursively, we can get 
\begin{align}\label{eq:fp_2nd_hetero}
    &\ee\|v_{r,k}^j-\nabla F(w_{r+1,k}^j)\|^2\nonumber\\
    &\leq e\ee\|v_{r,0}^j-\nabla F(w_{r+1,0}^j)\|^2+\frac{eH^2}{b}\sum_{k=1}^K\ee\|w_{r+1,k}^j-w_{r+1,k-1}^j\|^2\nonumber\\
    &\qquad+e(1+K)\sum_{k=1}^K\ee\big\|\nabla F_{j}(w^j_{r+1,k})-\nabla F_{j}(w^j_{r+1,k-1})+\nabla F(w_{r+1,k-1}^j)-\nabla F(w_{r+1,k}^j)\big\|^2\enspace.
\end{align}
Due to \Cref{ass:tau} and Lemma 3 in \citet{karimireddy2020mime}, \eqref{eq:fp_2nd_hetero} implies that
\begin{align*}
    &\ee\|v_{r,k}^j-\nabla F(w_{r+1,k}^j)\|^2\\
    &\leq e\ee\|v_{r,0}^j-\nabla F(w_{r+1,0}^j)\|^2+\bigg(\frac{eH^2}{b}+8eK\tau^2\bigg)\sum_{k=1}^K\ee\|w_{r+1,k}^j-w_{r+1,k-1}^j\|^2\enspace,\\
    &\leq e\ee\|v_{r,0}^j-\nabla F(w_{r+1,0}^j)\|^2+2\eta^2\bigg(\frac{eKH^2}{b}+8eK^2\tau^2\bigg)\frac{1}{K}\sum_{k=1}^K\ee\|v_{r,k-1}^j-\nabla F(w_{r+1,k-1}^j)\|^2\\
    &\qquad+2\eta^2\bigg(\frac{eKH^2}{b}+8eK^2\tau^2\bigg)\frac{1}{K}\sum_{k=1}^K\ee\|\nabla F(w_{r+1,k-1}^j)\|^2\enspace,
\end{align*}
where the second inequality is due to the updating rule as well as adding and subtracting $\nabla F(w_{r+1,k-1}^j)$. As a result, if we choose $\eta \leq 1/(CK\tau)$ and $\eta \leq \sqrt{b}/(C^\prime\sqrt{K}H)$, and the fact that $w_{r+1,0}^j=w_{r+1,1}^j=x_r$, we can obtain
\begin{align}\label{eq:fp_var_b}
    \frac{1}{K}\sum_{k=1}^K\ee\|v_{r,k}^j-\nabla F(w_{r+1,k}^j)\|^2&\leq 2e\ee\|v_{r}-\nabla F(x_{r})\|^2+\frac{1}{6K}\sum_{k=1}^K\ee\|\nabla F(w_{r+1,k}^j)\|^2\enspace.
\end{align}

Given the above results, we are ready to establish the convergence guarantee of Algorithm \ref{alg:fed_vr}. For client $\tilde m$ sampled at $t$-th iteration for the local update, we have
\begin{align*}
        F(w_{r+1,k+1}^{\tilde m})&\leq F(w_{r+1,k}^{\tilde m})+\la\nabla F(w_{r+1,k}^{\tilde m}),w_{r+1,k+1}^{\tilde m}- w_{r+1,k}^{\tilde m}\rangle+\frac{H}{2}\|w_{r+1,k+1}^{\tilde m}-w_{r+1,k}^{\tilde m}\|^2\enspace,\nonumber\\
    &=F(w_{r+1,k}^{\tilde m})-\eta\la\nabla F(w_{r+1,k}^{\tilde m}),v_{r,k}^{\tilde m}\rangle+\frac{\eta^2H}{2}\|v_{r,k}^{\tilde m}\|^2\enspace,\nonumber\\
    &=F(w_{r+1,k}^{\tilde m})-\eta\la\nabla F(w_{r+1,k}^{\tilde m}),v_{r,k}^{\tilde m}-\nabla F(w_{r+1,k}^{\tilde m})+\nabla F(w_{r+1,k}^{\tilde m})\rangle+\frac{\eta^2H}{2}\|v_{r,k}^{\tilde m}\|^2\enspace,\nonumber\\
    &\leq F(w_{r+1,k}^{\tilde m})-\eta\|\nabla F(w_{r+1,k}^{\tilde m})\|^2-\eta\la\nabla F(w_{r+1,k}^{\tilde m}),v_{r,k}^{\tilde m}-\nabla F(w_{r+1,k}^{\tilde m})\rangle\nonumber\\
    &\qquad+\eta^2H\|v_{r,k}^{\tilde m}-\nabla F(w_{r+1,k}^{\tilde m})\|^2+\eta^2H\|\nabla F(w_{r+1,k}^{\tilde m})\|^2\enspace,\nonumber\\
    &\leq F(w_{r+1,k}^{\tilde m})-\eta\bigg(\frac{3}{4}-\eta H\bigg)\|\nabla F(w_{r+1,k}^{\tilde m})\|^2+\eta(1+\eta H)\|v_{r,k}^{\tilde m}-\nabla F(w_{r+1,k}^{\tilde m})\|^2\enspace,\nonumber\\
    &\leq F(w_{r+1,k}^{\tilde m})-\frac{\eta}{2}\|\nabla F(w_{r+1,k}^{\tilde m})\|^2+\frac{5}{4}\eta\|v_{r,k}^{\tilde m}-\nabla F(w_{r+1,k}^{\tilde m})\|^2\enspace,
\end{align*}
where the last inequality is due to the fact that $\eta\leq 1/(4H)$. Therefore, we can obtain that 
\begin{align*}
    \|\nabla F(w_{r+1,k}^{\tilde m})\|^2\leq \frac{2}{\eta}\big(F(w_{r+1,k}^{\tilde m})-F(w_{r+1,k+1}^{\tilde m})\big)+3\|v_{r,k}^{\tilde m}-\nabla F(w_{r+1,k}^{\tilde m})\|^2\enspace.
\end{align*}
Recall that $w_{r+1,1}^{\tilde m}=x_r$ and $w_{r+1,k+1}^{\tilde m}=x_{r+1}$, averaging from $k=1,\ldots K$, and taking expectation, we can get
\begin{align}\label{eq:fp_grad_b}
    \frac{1}{K}\sum_{k=1}^K\ee\|\nabla F(w_{r+1,k}^{\tilde m})\|^2\leq \frac{2}{K\eta}\big(\ee F(x_r)-\ee F(x_{r+1})\big)+\frac{3}{K}\sum_{k=1}^K\ee\|v_{r,k}^{\tilde m}-\nabla F(w_{r+1,k}^{\tilde m})\|^2\enspace.
\end{align}
Combining \eqref{eq:fp_var_b} and \eqref{eq:fp_grad_b}, we can obtain
\begin{align*}
    \frac{1}{K}\sum_{k=1}^K\ee\|\nabla F(w_{r+1,k}^{\tilde m})\|^2&\leq \frac{2}{K\eta}\big(\ee F(x_r)-\ee F(x_{r+1})\big)+6e\ee\|v_{r}-\nabla F(x_{r})\|^2+\frac{1}{2K}\sum_{k=1}^K\ee\|\nabla F(w_{r+1,k}^{\tilde m})\|^2\enspace,
\end{align*}
which implies that 
\begin{align}\label{eq:fp_grad_b2}
    \frac{1}{K}\sum_{k=1}^K\ee\|\nabla F(w_{r+1,k}^{\tilde m})\|^2&\leq\frac{4}{K\eta}\big(\ee F(x_r)-\ee F(x_{r+1})\big)+12e\ee\|v_{r}-\nabla F(x_{r})\|^2\enspace.
\end{align}

Averaging \eqref{eq:fp_grad_b2} from $t=0,\ldots,R-1$, we can obtain
\begin{align*}
    \frac{1}{RK}\sum_{r=0}^{R-1}\sum_{k=1}^K\ee\|\nabla F(w_{r+1,k}^{\tilde m})\|^2&\leq\frac{4}{RK\eta}\big(\ee F(x_0)-\ee F(x_{r})\big)
    +\frac{12e}{R}\sum_{r=0}^{R-1}\ee\|v_{r}-\nabla F(x_{r})\|^2\enspace,
\end{align*}
by the definition of $\tilde x$, we have 
\begin{align}\label{eq:fp_fb1}
    \ee\|\nabla F(\tilde x)\|^2&\leq\frac{4}{RK\eta}\big(\ee F(x_0)-\ee F(x_{R})\big)
    +\frac{12e}{R}\sum_{r=0}^{R-1}\ee\|v_{r}-\nabla F(x_{r})\|^2\enspace.
\end{align}
Next, we consider the estimation error between $v_{r}$ and $\nabla F(x_{r})$.
Recall that we have
\begin{align*}
  v_{r}=\frac{1}{M}\sum_{j=1}^M\nabla F_{j,\bbb_{r}^j}(x_{r})+(1-\beta)\bigg(v_{r-1}-\frac{1}{M}\sum_{j=1}^M\nabla F_{j,\bbb_{r}^j}(x_{r-1})\bigg)\enspace.
\end{align*}
Thus, we obtain that 
\begin{align*}
   v_{r}-\nabla F(x_{r})&=(1-\beta)\big(v_{r-1}-\nabla F(x_{r-1})\big)+\beta\bigg(\frac{1}{M}\sum_{j=1}^M\nabla F_{j,\bbb_{r}^j}(x_{r})-\nabla F(x_{r})\bigg)\\
   &\qquad+(1-\beta)\bigg(\frac{1}{M}\sum_{j=1}^M\nabla F_{j,\bbb_{r}^j}(x_{r})-\frac{1}{M}\sum_{j=1}^M\nabla F_{j,\bbb_{r}^j}(x_{r-1})+\nabla F(x_{r-1})-\nabla F(x_{r})\bigg)\enspace.
\end{align*}
Therefore, using the conditional expectation up to the $r$-th iteration, we have
\begin{align}\label{eq:fp_eq1}
    \ee_{r}\big\|v_{r}-\nabla F(x_{r})\big\|^2
    &\leq (1-\beta)^2\ee_{r}\big\|v_{r-1}-\nabla F(x_{r-1})\big\|^2+2\beta^2\ee_{r}\bigg\|\frac{1}{M}\sum_{j=1}^M\nabla F_{j,\bbb_{r}^j}(x_{r})-\frac{1}{M}\sum_{j=1}^M\nabla F_{j}(x_{r})\bigg\|^2
    \nonumber\\
    &\qquad+2(1-\beta)^2\frac{H^2}{MKb}\ee_{r}\big\|x_{r}-x_{r-1}\big\|^2\enspace,\nonumber\\
    &\leq (1-\beta)^2\ee_{r}\big\|v_{r-1}-\nabla F(x_{r-1})\big\|^2+2\beta^2\frac{\sigma_2^2}{MKb}+2(1-\beta)^2\frac{H^2}{MKb}\ee_{r}\big\|x_{r}-x_{r-1}\big\|^2\enspace,
\end{align}
where the first inequality is due to the fact that the mini-batch gradients consist of $b$ i.i.d. samples and each client has a two-point stochastic oracle, and the last inequality is due to the bounded variance assumption in \Cref{def:oracle_first_multi_point}. Therefore, taking expectations over all iterations for \eqref{eq:fp_eq1}, we can get
\begin{align}\label{eq:fp_contract_lemma9}
  \ee\big\|v_{r}-\nabla F(x_{r})\big\|^2 &\leq (1-\beta)^2\ee\big\|v_{r-1}-\nabla F(x_{r-1})\big\|^2+2\beta^2\frac{\sigma_2^2}{MKb}+2(1-\beta)^2\frac{H^2}{MKb}\ee\big\|x_{r}-x_{r-1}\big\|^2\enspace. 
\end{align}
Furthermore, we have
\begin{align*}
    \beta \sum_{r=0}^{R-1}\ee\big\|v_{r}-\nabla F(x_{r})\big\|^2 &=\sum_{r=0}^{R-1}\ee\big\|v_{r}-\nabla F(x_{r})\big\|^2-(1-\beta)\sum_{r=0}^{R-1}\ee\big\|v_{r}-\nabla F(x_{r})\big\|^2\enspace,\\
    &=\sum_{r=1}^{R}\ee\big\|v_{r}-\nabla F(x_{r})\big\|^2-(1-\beta)\sum_{r=0}^{R-1}\ee\big\|v_{r}-\nabla F(x_{r})\big\|^2-\ee\big\|v_{R}-\nabla F(x_{R})\big\|^2\\
    &\qquad+\ee\big\|v_{0}-\nabla F(x_{0})\big\|^2\enspace,\\
    &\leq \sum_{r=1}^{R}\ee\big\|v_{r}-\nabla F(x_{r})\big\|^2-(1-\beta)^2\sum_{r=0}^{R-1}\ee\big\|v_{r}-\nabla F(x_{r})\big\|^2-\ee\big\|v_{R}-\nabla F(x_{R})\big\|^2\\
    &\qquad+\ee\big\|v_{0}-\nabla F(x_{0})\big\|^2\enspace,\\
    &\leq 2(1-\beta)^2\frac{H^2}{MKb}\sum_{r=0}^{R-1}\ee\big\|x_{r+1}-x_{r}\big\|^2+2\beta^2R\frac{\sigma_2^2}{MKb}+\ee\big\|v_{0}-\nabla F(x_{0})\big\|^2\enspace,
\end{align*}
where the last inequality is due to \eqref{eq:fp_contract_lemma9}. Since we have 
\begin{align*}
    \ee\big\|v_{0}-\nabla F(x_{0})\big\|^2=\ee\bigg\|\frac{1}{M}\sum_{j=1}^M\nabla F_{j,\bbb_{0}^j}(x_{0})-\nabla F(x_0)\bigg\|^2\leq \frac{\sigma_2^2}{Mb_0}\enspace.
\end{align*}
Therefore, we have
\begin{align*}
    \beta \sum_{r=0}^{R-1}\ee\big\|v_{r}-\nabla F(x_{r})\big\|^2&\leq \frac{2(1-\beta)^2H^2}{MKb}\sum_{r=0}^{R-1}\ee\big\|x_{r+1}-x_{r}\big\|^2+2\beta^2R\frac{\sigma_2^2}{MKb}+\frac{\sigma_2^2}{Mb_0}\enspace.
\end{align*}
This implies that
\begin{align}\label{eq:fp_new_vr11}
    \frac{1}{R}\sum_{r=0}^{R-1}\ee\big\|v_{r}-\nabla F(x_{r})\big\|^2\leq \frac{2(1-\beta)^2H^2}{\beta MKb R}\sum_{r=0}^{R-1}\ee\big\|x_{r+1}-x_{r}\big\|^2+2\beta\frac{\sigma_2^2}{MKb}+\frac{\sigma_2^2}{\beta RMb_0}\enspace.
\end{align}
In addition, combining \eqref{fp_eq3} and \eqref{eq:fp_var_b}, we can get 
\begin{align}\label{eq:fp_para_dis1}
    \ee\|w_{r+1,k}^j-x_r\|^2&\leq 8e^2K^2\eta^2\ee\|v_{r}-\nabla F(x_{r})\|^2\nonumber\\
    &\qquad+\frac{2e(1+K)\eta^2}{6}\sum_{k=1}^K\ee\|\nabla F(w_{r+1,k}^j)\|^2+2e(1+K)\eta^2\sum_{k=1}^{K-1}\|\nabla F(w_{r+1,k}^j)\|^2\enspace\nonumber\\
    &\leq 8e^2K^2\eta^2\ee\|v_{r}-\nabla F(x_{r})\|^2+10eK^2\eta^2\frac{1}{K}\sum_{k=1}^{K-1}\ee\|\nabla F(w_{r+1,k}^j)\|^2\enspace.
\end{align}
Therefore, we have
\begin{align}\label{eq:fp_para_dis2}
   \ee\|x_{r+1}-x_r\|^2&= \ee\|w_{r+1,k+1}^{\tilde m}-x_r\|^2\nonumber\\
   &\leq 8e^2K^2\eta^2\ee\|v_{r}-\nabla F(x_{r})\|^2+10eK^2\eta^2\frac{1}{K}\sum_{k=1}^{K-1}\ee\|\nabla F(w_{r+1,k}^{\tilde m})\|^2\enspace.
\end{align}

Thus, plugging \eqref{eq:fp_para_dis2} into \eqref{eq:fp_new_vr11}, we can get
\begin{align*}
   \frac{1}{R}\sum_{r=0}^{R-1}\ee\big\|v_{r}-\nabla F(x_{r})\big\|^2&\leq \frac{160H^2K^2\eta^2}{\beta MKb }\frac{1}{R}\sum_{r=0}^{R-1}\bigg(\ee\|v_{r}-\nabla F(x_{r})\|^2+\frac{1}{K}\sum_{k=1}^{K-1}\ee\|\nabla F(w_{r+1,k}^{\tilde m})\|^2\bigg)\\
   &\qquad+2\beta\frac{\sigma_2^2}{MKb}+\frac{\sigma_2^2}{\beta RMb_0}\enspace,\\
   &\leq \frac{1}{24e+1}\frac{1}{R}\sum_{r=0}^{R-1}\bigg(\ee\|v_{r}-\nabla F(x_{r})\|^2+\frac{1}{K}\sum_{k=1}^{K-1}\ee\|\nabla F(w_{r+1,k}^{\tilde m})\|^2\bigg)\\
   &\qquad+2\beta\frac{\sigma_2^2}{MKb}+\frac{\sigma_2^2}{\beta RMb_0}\enspace,
\end{align*}
where the last inequality is due to the fact that $\eta\leq \sqrt{\beta MKb}/(C^{\prime \prime}LK)$. Thus, we have
\begin{align}\label{eq:fp_var_fb}
    \frac{1}{R}\sum_{r=0}^{R-1}\ee\big\|v_{r}-\nabla F(x_{r})\big\|^2\leq \frac{1}{24e}\frac{1}{R}\sum_{r=0}^{R-1}\frac{1}{K}\sum_{k=1}^{K-1}\ee\|\nabla F(w_{r+1,k}^{\tilde m})\|^2+4\beta\frac{\sigma_2^2}{MKb}+2\frac{\sigma_2^2}{\beta RMb_0}\enspace.
\end{align}
Combining \eqref{eq:fp_fb1} and \eqref{eq:fp_var_fb}, we can obtain
\begin{align*}
\ee\|\nabla F(\tilde x)\|^2&\leq\frac{4}{RK\eta}\big(\ee F(x_0)-\ee F(x_{R})\big)
+\frac{1}{2}\ee\|\nabla F(\tilde x)\|^2+48e\beta\frac{\sigma_2^2}{MKb}+24e\frac{\sigma_2^2}{\beta RMb_0}\enspace,
\end{align*}
which implies
\begin{align}\label{eq:fp_fb2}
\ee\|\nabla F(\tilde x)\|^2&\leq\frac{8}{RK\eta}\big( F(x_0)- F(x^*)\big)
+96e\beta\frac{\sigma_2^2}{MKb}+48e\frac{\sigma_2^2}{\beta RMb_0}\enspace.
\end{align}
Note that we have the following requirements for the stepsize $\eta$: $\eta \leq 1/(4H)$, $\eta \leq 1/(CK\tau)$, $\eta \leq \sqrt{b}/(C^\prime\sqrt{K}H)$, $\eta\leq \sqrt{\beta MKb}/(C^{\prime \prime}HK)$. Plugging these requirements, we can get
\begin{align}\label{eq:rebutal_ulb}
   \ee\|\nabla F(\tilde x)\|^2&\leq C_1\bigg( \frac{\Delta \tau}{R} +\frac{\Delta H}{KR}+\frac{\Delta H}{R\sqrt{Kb}}+\frac{\Delta H}{R\sqrt{\beta MKb}}+\beta\frac{\sigma_2^2}{MKb}+\frac{\sigma_2^2}{\beta RMb_0}\bigg)\enspace.
\end{align}
Therefore, if we choose $b_0=KR$ and
\begin{align*}
    \beta= \max\left\{\frac{1}{R}, \frac{(\Delta H)^{2/3}(MKb)^{1/3}}{\sigma_2^{4/3}R^{2/3}}\right\} =: \max\{\beta_1, \beta_2\}\enspace,
\end{align*}
we can obtain,
\begin{align*}
   \ee\|\nabla F(\tilde x)\|^2&\leq C_1\bigg( \frac{\Delta \tau}{R} +\frac{\Delta H}{KR}+\frac{\Delta H}{R\sqrt{Kb}}+\frac{\Delta H}{R\sqrt{\beta_2 MKb}}+(\beta_1 + \beta_2)\frac{\sigma_2^2}{MKb}+\frac{\sigma_2^2}{\beta_1 MKR^2}\bigg)\enspace.
\end{align*}
which simplifies to,
\begin{align}\label{eq:new_ub_thm3.1}
   \ee\|\nabla F(\tilde x)\|^2&\leq C_1\bigg( \frac{\Delta \tau}{R} +\frac{\Delta H}{KR}+\frac{\Delta H}{R\sqrt{Kb}}+\left(\frac{\sigma_2\Delta  H}{MKbR}\right)^{2/3}+\frac{\sigma_2^2}{MKbR}\bigg)\enspace.
\end{align}
Since we need to ensure that $\beta\leq 1$, we require the following assumption for $\beta_2\leq 1$ ($R\geq 1$ w.l.o.g.),
\begin{align*}
    \frac{\Delta H}{R} \leq \frac{\sigma_2^2}{\sqrt{MKb}}\enspace.
\end{align*}
This concludes the proof of Theorem \ref{thm:alg1_full} (a).

\noindent\textbf{Deterministic case:} Note that if each client $m\in[M]$ has a deterministic two-point oracle, we can choose $\beta=1$, and according to \eqref{eq:fp_fb1}, we can obtain
\begin{align}\label{eq:fp_fb1_deter}
    \ee\|\nabla F(\tilde x)\|^2&\leq\frac{4}{RK\eta}\big(\ee F(x_0)-\ee F(x_{R})\big)
    +\frac{12e}{R}\sum_{r=0}^{R-1}\ee\|v_{r}-\nabla F(x_{r})\|^2\enspace,
\end{align}
where we have the following requirements of stepsize $\eta$: $\eta \leq 1/(4H)$, $\eta \leq 1/(CK\tau)$. Furthermore, we have $\vbb_t=\nabla F(x_r)$, which implies that
\begin{align*}
   \ee\|\nabla F(\tilde x)\|^2&\leq C_4\bigg( \frac{\Delta \tau}{R} +\frac{\Delta H}{KR}\bigg)\enspace.
\end{align*}
This concludes the proof of Theorem \ref{thm:alg1_full} (b).

\noindent\textbf{Three regimes:} In the following, we discuss how to obtain the result in Figure \ref{fig:3region}. We always assume that $\tau\leq L$ and, without loss of generality, we assume $b=1$ and ignore all the dependence on constants. According to \eqref{eq:rebutal_ulb}, 
if we choose $\beta,b_0$ such that
\begin{align}\label{eq:fp_beta_b0}
    \beta\frac{\sigma_2^2}{MKb}\leq \epsilon\quad ~\text{and}\quad~\frac{\sigma_2^2}{\beta RM\epsilon}\leq b_0\enspace,
\end{align}
we can obtain
\begin{align}\label{eq:fp_final_rate}
   \ee\|\nabla F(\tilde x)\|^2&\leq C_5\bigg( \frac{\Delta \tau}{R} +\frac{\Delta H}{KR}+\frac{\Delta H}{R\sqrt{Kb}}+\frac{\Delta Hs}{R\sqrt{\beta MKb}}+\epsilon\bigg)\enspace.
\end{align}
Therefore, to achieve $\ee\|\nabla F(\tilde x)\|^2\leq \epsilon$, we need the following communication complexity
\begin{align*}
    R=C_3\bigg(\frac{\Delta \tau}{\epsilon} +\frac{\Delta H}{K\epsilon}+\frac{\Delta H}{\epsilon\sqrt{Kb}}+\frac{\Delta H}{\epsilon\sqrt{\beta MKb}}\bigg)\enspace.
\end{align*}
Furthermore,the gradient complexity of Algorithm \ref{alg:fed_vr} is $N=MbKR+bK+Mb_0$. If we have
\begin{align}\label{eq:beta_b0_1}
    Mb_0\leq N\enspace,
\end{align}
then we have the following gradient complexity:
\begin{align*}
    N=C_4MbKR=C_4\bigg(\frac{MbK\Delta \tau}{\epsilon} +\frac{Mb\Delta H}{\epsilon}+\frac{M\Delta H\sqrt{Kb}}{\epsilon}+\frac{\Delta H\sqrt{ MKb}}{\epsilon\sqrt{\beta}}\bigg)\enspace.
\end{align*}

Note that we want to keep the $R=\Delta \tau/\epsilon$ while minimizing $N$, i.e., to obtain $N$ close to $\Delta H \sigma/\epsilon^{3/2}$. Recall that we have
\begin{align*}
    R=\frac{\Delta \tau}{\epsilon} +\frac{\Delta H}{\epsilon\sqrt{K}}+\frac{\Delta H}{\epsilon\sqrt{\beta MK}}\quad~\text{and}\quad~N=\frac{MK\Delta \tau}{\epsilon} +\frac{M\Delta H\sqrt{K}}{\epsilon}+\frac{\Delta H\sqrt{ MK}}{\epsilon\sqrt{\beta}}\enspace.
\end{align*}
To achieve $R=\Delta \tau/\epsilon$, we need
\begin{align}\label{eq:k_req}
    K\geq\max\bigg\{\frac{H^2}{\tau^2},\frac{ H^2}{\beta M\tau^2}\bigg\}\enspace.
\end{align}
\\
\textbf{Green regime}: We want to achieve the best of both worlds, i.e., $R=\Delta \tau/\epsilon$ and $N=\Delta H \sigma_2/\epsilon^{3/2}$. According to $N$, we need to have
\begin{align}\label{eq:k_req_green}
    K\leq\max\bigg\{\frac{H}{\tau}\cdot\frac{\sigma_2}{M\epsilon^{1/2}},\frac{\sigma_2^2}{M^2\epsilon},\frac{\sigma_2^2\beta}{M\epsilon}\bigg\}\enspace.
\end{align}
Therefore, combining \eqref{eq:k_req} and \eqref{eq:k_req_green}, we can obtain
\begin{align*}
    \epsilon^{1/2}\leq \frac{\sigma_2\tau}{HM}\quad~\text{and}\quad~\beta\geq \frac{H\epsilon^{1/2}}{\sigma_2 \tau}\enspace.
\end{align*}
In addition, according to \eqref{eq:fp_beta_b0}, we have
\begin{align*}
    \beta\leq\frac{\epsilon MK}{\sigma_2^2 }\leq \frac{\epsilon N}{R\sigma_2^2}=\frac{H\epsilon^{1/2}}{\sigma_2 \tau}\enspace.
\end{align*}
Therefore, we can choose $\beta=H\epsilon^{1/2}/(\sigma_2\tau)$, and this will lead to
\begin{align*}
    K=\frac{\sigma_2 H}{M\tau\sqrt{\epsilon}}\enspace.
\end{align*}
In addition, according to \eqref{eq:fp_beta_b0} and \eqref{eq:beta_b0_1}, we have
\begin{align*}
    b_0=\frac{\sigma_2^3}{\Delta H M\epsilon^{1/2}}\enspace,
\end{align*}
and we need
\begin{align*}
    \frac{\sigma_2^3}{\Delta H \epsilon^{1/2}}\leq\frac{\sigma_2^2}{\epsilon}\leq \frac{\Delta H\sigma_2}{\epsilon^{3/2}}\enspace,
\end{align*}
which will hold if we have $\epsilon\sigma_2^2\leq(\Delta H)^2$.\\
To summarize, if we have $\epsilon^{1/2}\leq \sigma_2\tau/(HM)$, $H\epsilon^{1/2}\leq\sigma_2 \tau$ ($\epsilon\leq\sigma_2^2$), and $\epsilon\sigma_2^2\leq (\Delta H)^2$, we have 
\begin{align*}
    R=\frac{\Delta \tau}{\epsilon}\quad~\text{and}\quad~N=\frac{\Delta H\sigma_2}{\epsilon^{3/2}}\enspace.
\end{align*}
If we choose $K=\sigma_2 H/(M\tau\epsilon^{1/2})\geq 1$ (as $M\epsilon^{1/2}\leq \sigma_2$), $b_0=\sigma_2^3/(H\Delta M\epsilon^{1/2})$, $\beta=H\epsilon^{1/2}/(\sigma_2\tau)$ (always less than 1 in this regime). This gives us the green regime in Figure \ref{fig:3region}.
\\
\textbf{Orange regime}: In this regime, we still want to keep the $R=\Delta \tau/\epsilon$ while minimizing $N$. Since we have $\epsilon^{1/2}\geq \sigma_2\tau/(HM)$, we cannot make $N=\Delta \sigma_2 H/\epsilon^{3/2}$. Thus, according to \eqref{eq:k_req}, we have
\begin{align*}
    N=\frac{MH\Delta}{\epsilon}\cdot\frac{H}{\tau}+\frac{\sqrt{M}H\Delta}{\sqrt{\beta}\epsilon}\cdot\frac{H}{\tau}+\frac{MH\Delta}{\epsilon}\cdot\frac{H}{\tau\beta M}\enspace.
\end{align*}
By choosing $\beta=1/M$, we can get
\begin{align*}
    N=\frac{MH\Delta}{\epsilon}\cdot\frac{H}{\tau}\enspace.
\end{align*}
And we have $K=H^2/\tau^2$.
Furthermore, according to \eqref{eq:fp_beta_b0} and \eqref{eq:beta_b0_1}, we have
\begin{align*}
    \frac{\sigma_2^2\tau^2}{M^2H^2}\leq \epsilon,\quad b_0=\frac{\sigma_2^2}{\Delta \tau},\quad \frac{M\sigma_2^2}{\Delta \tau}\leq \frac{MH\Delta}{\epsilon}\cdot\frac{H}{\tau}\enspace.
\end{align*}
where the first inequality holds due to $\epsilon^{1/2}\geq \sigma\tau/(HM)$ and the last one holds if we have $\epsilon\sigma_2^2\leq(H\Delta)^2$.\\
To summarize, if we have $\epsilon^{1/2}\geq \sigma_2\tau/(HM)$ and $\epsilon\sigma_2^2\leq (\Delta H)^2$, we have
\begin{align*}
     R=\frac{\Delta \tau}{\epsilon}~\text{and}~N=\frac{MH\Delta}{\epsilon}\cdot\frac{H}{\tau}\enspace,
\end{align*}
if we choose $K=H^2/\tau^2$, $b_0=\sigma_2^2/(\Delta \tau)$.\\
\textbf{Red region}: If we have $\epsilon\geq \Delta \tau$, then we only need $R=1$, and thus we have $N\geq MH^2\Delta^2/\epsilon^2$. 
\end{proof}

\section{Mini-batch STORM}\label{sec:minibatch_storm}
In this section, we present the convergence guarantee of mini-batch STORM for completeness. More specifically, if we choose the number of local updates to be one in Algorithm \ref{alg:fed_vr}, our method will reduce to mini-batch STORM. 
As a result, we have the following convergence guarantee.
\begin{theorem}\label{thm:minibatch_storm}
Suppose we have a problem instance satisying \Cref{ass:smooth_second,ass:bounded_func_subopt,ass:tau} where each client $m\in[M]$ has a stochastic $2$-point oracle (cf., \Cref{def:oracle_first_multi_point}), then the output $\tilde x$ of mini-batch STORM using $\beta=\frac{(\Delta H)^{2/3}(MK)^{1/3}}{\sigma_2^{4/3}R^{2/3}} \leq 1$, $b_0~=~\min\left\{\frac{\sigma_2^{4/3}(RK)^{2/3}}{(\Delta H)^{2/3}M^{1/3}}, \frac{\sigma_2^{8/3}(KR)^{1/3}}{(\Delta H)^{4/3}M^{2/3}}\right\},$ and $\eta=\min\left\{\frac{1}{H}, \frac{(\beta M)^{1/2}}{HK^{1/2}}\right\}$ satisfies
$$\ee\|\nabla F(\tilde x)\|^2\leq c_1\cdot\left(\frac{\Delta H}{R}+\frac{\sigma_2^2}{MKR}+ \left(\frac{\Delta \sigma_2 H}{RMK}\right)^{2/3}\right)\enspace,$$
\end{theorem}
where $c_1$ is a numerical constant.
\begin{proof}[Proof of Theorem \ref{thm:minibatch_storm}]
The proof of this result directly follows the proof of Theorem \ref{thm:alg1_full}. We can just set $K=1$, let $\tau=H$, and ignoring the $\Delta H/(R\sqrt{Kb})$ term (which appears when local updates $K>1$) in \eqref{eq:new_ub_thm3.1} to get
\begin{align*}
   \ee\|\nabla F(\tilde x)\|^2&\leq C_1\bigg( \frac{\Delta H}{R}+\frac{\sigma_2^2}{MbR} + \left(\frac{\sigma_2 \Delta H}{MbR}\right)^{2/3}\bigg)\enspace,
\end{align*}
provided that
\begin{align*}
    \beta=\frac{(\Delta H)^{2/3}(Mb)^{1/3}}{\sigma_2^{4/3}R^{2/3}} \leq 1\enspace.
\end{align*}
Finally, if we choose the batch size to be the number of updates in the local update algorithms, i.e., $b=K$, we obtain that
\begin{align*}
   \ee\|\nabla F(\tilde x)\|^2&\leq C_1\bigg( \frac{\Delta H}{R}+\frac{\sigma_2^2}{MKR}+ \left(\frac{\Delta \sigma_2 H}{RMK}\right)^{2/3}\bigg)\enspace,
\end{align*}
and we have 
\begin{align*}
    \beta=\frac{(\Delta H)^{2/3}(MK)^{1/3}}{\sigma_2^{4/3}R^{2/3}} \leq 1,~b_0~=~\min\left\{\frac{\sigma_2^{4/3}(RK)^{2/3}}{(\Delta H)^{2/3}M^{1/3}}, \frac{\sigma_2^{8/3}(KR)^{1/3}}{(\Delta H)^{4/3}M^{2/3}}\right\}\enspace.
\end{align*}
Note that $C_1,C_2$ are numerical constants.
\end{proof}
\chapter{Additional Details for Chapter 7}\label{app:chap7}
We introduce some additional notation that we will use in this appendix. First, we will denote the function classes satisfying the assumptions introduced in the main body as follows:
\begin{enumerate}
    \item $\boldsymbol{\fff^{G}}$, the class of convex, differentiable, and $G$-Lipschitz functions, i.e., the class to which the cost functions belong when they satisfy \Cref{ass:online_convex,ass:online_bounded_gradients}.
    \item $\boldsymbol{\fff^{H}}$, the class of convex, differentiable, and $H$-smooth functions, i.e., the class to which the cost functions belong when they satisfy \Cref{ass:online_convex,ass:online_smooth_second}
    \item $\boldsymbol{\fff_{lin}^{G}} \subset \boldsymbol{\fff^{G, H}}$, which includes linear cost functions with gradients bounded by $G$, i.e., the class to which the cost functions belong when they satisfy \Cref{ass:online_linear}. Note that linear functions are the ``smoothest" functions in the class $\boldsymbol{\fff^{G, H}}$, i.e., the class to which the cost functions belong when they satisfy both \Cref{ass:online_bounded_gradients,ass:online_smooth_second}.
\end{enumerate}
Now recall that we defined a problem class $\ppp$ as taking in the history at any particular time, as well as an algorithm (and not its randomness) to output a distribution over $M$ different functions. In this appendix, we will often make the restrictions on the problem class explicit by using superscripts. For instance, if the cost functions satisfy \Cref{ass:online_bounded_gradients} we denote the problem class as $\ppp^{\fff^G}$. Furthermore, if the cost functions satisfy, \Cref{ass:online_zeta} we will use $\ppp^{\fff^G,\hat\zeta}$. This usage will be clear in the discussion and will enable us to present our analysis concisely. 

Finally, to make it explicit that we are hoping to characterize the min-max complexity of several problems in the intermittent communication setting with $M$ machines, $K$ local updates, and $R$ communication rounds, we will denote the min-max regret by $\rrr_{M,K,R}(\ppp,\aaa)$ for some problem class $\ppp$ and algorithm class $\aaa$. This also allows us to refer to the serial setting, i.e., a single machine's min-max regret, by using $\rrr_{1,K,R}(\ppp,\aaa)$, and noting that $\ppp,\ \aaa$ essentially reduce to adversaries and algorithms on a single machine when used in this notation.
\section{Related Problem Settings and Reductions}\label{sec:related}
\begin{figure}[!tbh]
    \centering
    \includegraphics[width=0.9\textwidth]{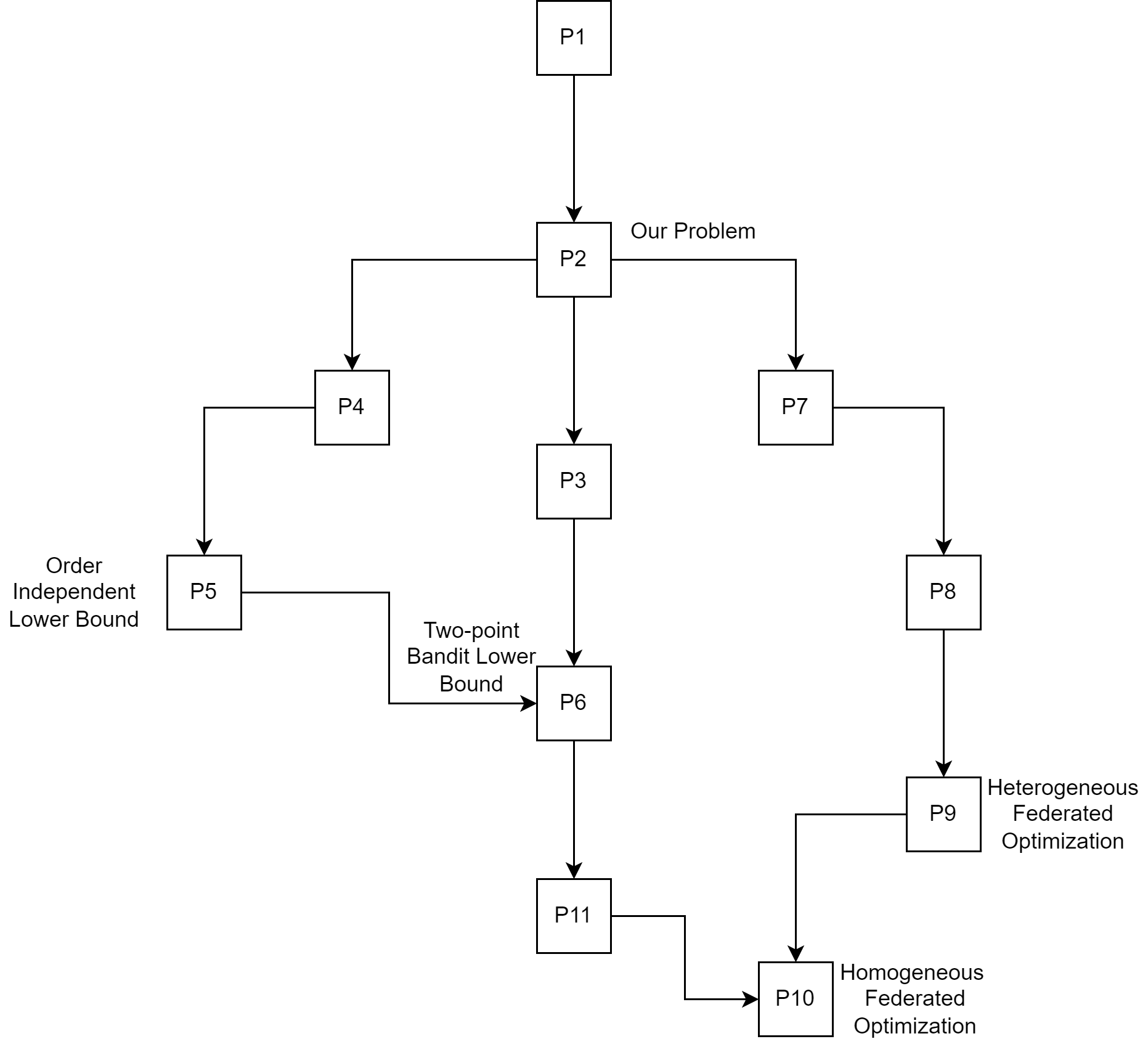}
    \caption{Summary of the problem space of federated online optimization. An arrow from the parent to the child denotes that the child's min-max problem is easier or has a lower min-max value than the parent's problem. Note that to demonstrate the absence of benefit from collaboration for first-order algorithms on the problem (P2), we utilize the lower bound construction for the problem (P5). The figure clarifies why this does not contradict the benefit of collaboration for problems (P9) and (P10), as they lie on a different path from the parent (P2).}
    \label{fig:hierarchy}
\end{figure}
In this section, we will characterize different federated learning problems and objectives using the min-max value, defined in \eqref{eq:P2}. As before, we look at an algorithm class $\aaa$ and adversary class $\ppp$. The hardest problem we can hope to solve is when the adversary, besides knowing $\hhh_t$ and $A$, also knows the randomization of the algorithm at any given time $t$:
\begin{align}
    \min_{A\in\aaa} \ee_{A}\sb{\max_{P\in \ppp}  \ee_P\sb{\frac{1}{MT}\sum_{t\in [T], m\in[M]}f_t^m(x_t^m) - \min_{x^\star\in \bb_2(B)}\frac{1}{MT}\sum_{t\in [T], m\in[M]}f_t^m(x^\star)}}\enspace.\tag{\textbf{P1}}
\end{align}
For this problem, note that both the min and max-player do not gain from randomization, as on both players, respectively, we can just put all the distributions' mass on the best deterministic strategies. As a result, we can instead look at the sub-classes $\aaa_{det}\subset \aaa$, $\ppp_{det}\subset \ppp$ denoting deterministic strategies. This simplifies problem (P1) to the following:
\begin{align}
    \min_{A\in\aaa_{det}} \max_{P\in \ppp_{det}}  \frac{1}{MT}\sum_{t\in [T], m\in[M]}f_t^m(x_t^m) - \min_{x^\star\in \bb_2(B)}\frac{1}{MT}\sum_{t\in [T], m\in[M]}f_t^m(x^\star)\enspace.\tag{\textbf{P1}}
\end{align}
Recall that this is not the problem we study in the paper, but instead, as defined in \eqref{eq:P2}, we look at the following simpler problem where the max-player/adversary does not know the random seeds of the min-player:
\begin{align}
    \min_{A\in\aaa}\max_{P\in \ppp}  \ee_{A, P}\sb{\frac{1}{MT}\sum_{t\in [T], m\in[M]}f_t^m(x_t^m) - \min_{x^\star\in \bb_2(B)}\frac{1}{MT}\sum_{t\in [T], m\in[M]}f_t^m(x^\star)}\enspace.\tag{\textbf{P2}}
\end{align}
Note that the adversary does not benefit from randomization so that we can replace the $\ppp$ above by $\ppp_{det}$. However, making the randomization on the max-player explicit makes it easier to state the following easier version of the problem (P2) with a weaker comparator $x^\star$ that does not depend on the randomness of the adversary and is thus worse in general:
\begin{align}
    \min_{A\in\aaa}\max_{P\in \ppp}  \ee_{A}\rb{\ee_P\sb{\frac{1}{MT}\sum_{t\in [T], m\in[M]}f_t^m(x_t^m)} - \min_{x^\star\in \bb_2(B)}\ee_P\sb{\frac{1}{MT}\sum_{t\in [T], m\in[M]}f_t^m(x^\star)}}\enspace. \tag{\textbf{P3}}
\end{align}
This form of regret is common in multi-armed bandit literature and is often referred to as \textit{``pseudo-regret''}. Intuitively, one wants to disregard the random perturbations an adversary might add in multi-armed bandits while comparing to a hindsight optimal model. We have only discussed the \textit{``fully adaptive''} setting so far. We can also relax the problem (P2) by weakening the adversary. One way to do this is by requiring the functions to be the same across the machines, which leads to the following problem,
\begin{align*}
    \min_{A\in\aaa}\max_{P\in \ppp_{cood}}  \ee_{A}\sb{\frac{1}{MT}\sum_{t\in [T], m\in[M]}f_t(x_t^m) - \min_{x^\star\in \bb_2(B)}\frac{1}{T}\sum_{t\in [T]}f_t(x^\star)}\enspace.\tag{\textbf{P4}}
\end{align*}
By $\ppp_{cood}$, we denote the class of adversaries that make a coordinated attack. Such an attack will be useful when we show the lower bounds for algorithms in class $\aaa_{online-IC}^1$. Note that if the functions across the machines are the same, then $\hat\zeta=0$ in Assumption \ref{ass:online_zeta}. Depending on the algorithm class, this problem may be equivalent to, or may not be equivalent to, a fully serial problem, as demonstrated in this thesis. We can further simplify problem (P4) by making the adversary stochastic,
\begin{align*}
    \min_{A\in\aaa}\max_{P\in \ppp_{cood, stoc}}  \ee_{A}\sb{\frac{1}{MT}\sum_{t\in [T], m\in[M]}f_t(x_t^m) - \min_{x^\star\in \bb_2(B)}\frac{1}{T}\sum_{t\in [T]}f_t(x^\star)}\enspace.\tag{\textbf{P5}}
\end{align*}
Since this amounts to just picking a fixed distribution on the cost functions that stays constant over time, we can alternatively restate problem (P5) in terms of choosing a distribution $\ddd\in \Delta(\fff)$ that can only depend on the description of the algorithm $A$,
\begin{align*}
    \min_{A\in\aaa}\max_{\ddd\in \Delta(\fff)}  \ee_{A, \{f_t\sim \ddd\}_{t\in[T]}}\sb{\frac{1}{MT}\sum_{t\in [T], m\in[M]}f_t(x_t^m) - \min_{x^\star\in \bb_2(B)}\frac{1}{T}\sum_{t\in [T]}f_t(x^\star)}\enspace. \tag{\textbf{P5}}
\end{align*}
We can further simplify problem (P5) by having a weaker comparator that does not depend on the randomness of sampling from $\ddd$, and by noting that the randomness of the adversary is independent of any randomness in the algorithm, 
\begin{align*}
    \min_{A\in\aaa}\max_{\ddd\in \Delta(\fff)}  \ee_{A, f_t\sim \ddd}\sb{\frac{1}{MT}\sum_{t\in [T], m\in[M]}\ee_{f_t\sim \ddd}\sb{f_t(x_t^m)}}- \min_{x^\star\in \bb_2(B)}\ee_{f\sim \ddd}\sb{f(x^\star)}\enspace. \tag{\textbf{P6}}
\end{align*}
Above, we have not removed the randomness for sampling the function in the expectation because the choice of the models $x_t^m$'s will also depend on this randomness. This can also be seen as a relaxation of the problem (P3) to have a stochastic adversary. Recalling the definition of $\{F_m := \ee_{f\sim \ddd_m}[f]\}_{m\in[M]}$ and $F:= \frac{1}{M}\sum_{m\in[M]}F_m$, and applying tower rule problem (P6) can be re-written as,
\begin{align*}
    \min_{A\in\aaa}\max_{\ddd\in \Delta(\fff)}  \ee_{A, f_t\sim \ddd}\sb{\frac{1}{MT}\sum_{t\in [T], m\in[M]}F(x_t^m)}- \min_{x^\star\in \bb_2(B)}F(x^\star)\enspace. \tag{\textbf{P6}}
\end{align*}
Now let's relax (P2) directly to have stochastic adversaries that sample independently on each machine. In particular, this means there are fixed distributions $\ddd_m$ on each machine and at each time step $\{f_t^m\}^{m\in[M]}\sim \eee = \ddd_1\times\dots\times\ddd_m$.  To simplify the discussion, we assume that the problem class has no additional assumption and is just a selection of $MKR$ functions from some class $\fff$. This simplification allows us to relax (P2) by selecting the functions at machine $m$ at time $t$ from the distribution $\ddd_m$. This leads to the following problem,
\begin{align}
    \min_{A\in\aaa}\max_{\{\ddd_m\sim \Delta(\fff)\}_{m\in [M]}}  \ee_{A, \{f_t^m\sim \ddd_m\}_{t\in[T]}^{m\in [M]}}\sb{\frac{1}{MT}\sum_{t\in [T], m\in[M]}f_t^m(x_t^m) - \min_{x^\star\in \bb_2(B)}\frac{1}{MT}\sum_{t\in [T], m\in[M]}f_t^m(x^\star)}\enspace. \tag{\textbf{P7}}
\end{align}
If we weaken the comparator for this problem by not allowing it to depend on the randomness of sampling functions and recall the definitions for $F_m$ and $F$, we get the following problem,
\begin{align}
    \min_{A\in\aaa}\max_{\{\ddd_m\sim \Delta(\fff)\}_{m\in [M]}}  \ee_{A, \{f_t^m\sim \ddd_m\}_{t\in[T]}^{m\in [M]}}\sb{\frac{1}{MT}\sum_{t\in [T], m\in[M]}F_m(x_t^m)} - \min_{x^\star\in \bb_2(B)}F(x^\star)\enspace. \tag{\textbf{P8}}
\end{align}
We note that problem (P8) is the regret minimization version of the usual heterogeneous federated optimization problem \citep{mcmahan2016communication, woodworth2020minibatch}. To make the final connection to the usual federated optimization literature, we note that the problem becomes easier if the algorithm can look at all the functions before deciding which model to choose. In other words, minimizing regret online is harder than obtaining one final retrospective model. This means we can simplify the problem (P8) to the following problem, where $A$ outputs $\hat{x}$ after looking at all the functions. More specifically, $A(\{\ggg_t^m\}_{t\in[T]}^{m\in[M]})=\hat{(X)}\in \Delta(\rr^d)$, and $\hat{x}\sim \hat X$. We denote the class of such algorithms by $\aaa_{opt}$, i.e., (stochastic) optimization algorithms. This allows us to relax to the following problem, 
\begin{align}
    \min_{A\in\aaa_{opt}}\max_{\{\ddd_m\sim \Delta(\fff)\}_{m\in [M]}}  \ee_{A, \{f_t^m\sim \ddd_m\}_{t\in[T]}^{m\in [M]}}\sb{\frac{1}{MT}\sum_{t\in [T], m\in[M]}F_m(\hat{x})} - \min_{x^\star\in \bb_2(B)}F(x^\star)\enspace. \tag{\textbf{P9}}
\end{align}
With some re-writing of the notation, this reduces to the usual heterogeneous federated optimization problem \citep{mcmahan2016communication, woodworth2020minibatch, patel2022towards},
\begin{align}
    \min_{A\in\aaa}\max_{\{\ddd_m\sim \Delta(\fff)\}_{m\in [M]}}  \ee_{A, \{f_t^m\sim \ddd_m\}_{t\in[T]}^{m\in [M]}}\sb{F(\hat{x})} - \min_{x^\star\in \bb_2(B)}F(x^\star)\enspace. \tag{\textbf{P9}}
\end{align}
Note that we don't remove the expectation with respect to sampling functions as $\hat x$ depends on that randomness, along with any randomness in $A$. Further assuming $\ddd_m= \ddd$ for all $m\in[M]$ (P9) to the usual homogeneous federated optimization problem \citep{woodworth2020local},
\begin{align*}
    \min_{A\in\aaa}\max_{\ddd\in \Delta(\fff)}  \ee_{A, \{f_t^m\sim \ddd\}^{m\in[M]}_{t\in[T]}}\sb{F(\hat{x})}- \min_{x^\star\in \bb_2(B)}F(x^\star)\enspace. \tag{\textbf{P10}}
\end{align*}
Note that we can achieve a similar relaxation of the problem (P6) by converting regret minimization into finding a final good solution. The problem will look as follows
\begin{align*}
    \min_{A\in\aaa}\max_{\ddd\in \Delta(\fff)}  \ee_{A, \{f_t\sim \ddd\}_{t\in[T]}}\sb{F(\hat{x})}- \min_{x^\star\in \bb_2(B)}F(x^\star)\enspace. \tag{\textbf{P11}}
\end{align*}
The key difference between (P10) and (P11) is that $\hat{x}$ depends on $MT$ v/s $T$ random functions, respectively, in each case. This means (P10) is simpler than (P11) as it gets to see more information about the distribution $\ddd$. This concludes the discussion, and we summarize the comparisons between different min-max problems in Figure \ref{fig:hierarchy}. With this discussion, we are ready to understand the min-max complexities for the problem (P2) for different function and algorithm classes. 

\section{Proof of First-order Lower Bounds}\label{app:first_order_online_lb}
\subsection{Proof of \texorpdfstring{\Cref{thm:first_lip}}{TEXT}}
\begin{proof}
We first prove the upper bound on the average regret of non-collaborative OGD and then show that it is optimal, i.e., equals 
$\rrr\rb{\ppp^{\fff^{G}, \hat\zeta}, \aaa_{online-IC}^1}$. Note that the following bound is always true for any stream of functions and sequence of models; we are just changing the comparator:
\begin{align*}
    \frac{1}{M}\sum_{m\in[M]}\rb{\sum_{t\in [T]}f_t^m(x_t^m) - \min_{x^{m, \star}\in \bb_2(B)}\sum_{t\in [T]}f_t^m(x^m)} \geq \frac{1}{M}\sum_{t\in [T], m\in[M]}f_t^m(x_t^m) - \min_{x^\star\in \bb_2(B)}\sum_{t\in [T]}f_t(x)\enspace.
\end{align*}
This means we can upper bound $\rrr\rb{\ppp^{\fff^{G}, \hat\zeta}, \aaa_{online-IC}^1}$ by running online gradient descent (OGD) independently on each machine without collaboration, i.e.,
\begin{align}\label{eq:GB1UB}
    \rrr_{M,K,R}\rb{\ppp^{\fff^{G}, \hat\zeta}, \aaa_{online-IC}^1} = \ooo\rb{\rrr_{1,K, R}\rb{\ppp^{\fff^{G}}, \aaa_{online-IC}^1}} = \theta\rb{\frac{GB}{\sqrt{T}}}\enspace.
\end{align} 
The min-max rate for a single machine follows classical results using vanilla OGD \citep{zinkevich2010parallelized} (c.f., Theorem 3.1 by \citet{hazan2016introduction}). Now we prove that this average regret is optimal. Recall that we want to understand the problem (P2)'s lower bounds. Note that to lower bound (P2), we can lower bound any children problems in Figure \ref{fig:hierarchy}. 

In particular, from figure \ref{fig:hierarchy}, we can see that $(P2)\gtrsim (P4) \gtrsim (P5)$\footnote{We use $\lesssim, \gtrsim$  to compare the problems by referring to their min-max regrets.} and then note that for the adversary in $(P5)$, $\zeta=0$ by design as all the machines see the same function. Furthermore, to lower bound problem (P5), we can lower bound the following quantity, as $\fff^G_{lin}\subset \fff^G$,
\begin{align*}
    \min_{A\in\aaa_{online-IC}^1}\max_{\ddd\in \Delta(\fff^{G}_{lin})}  \ee_{A, \{f_t\sim \ddd\}_{t\in[T]}}\sb{\frac{1}{MT}\sum_{t\in [T], m\in[M]}f_t(x_t^m) - \min_{x^\star\in \bb_2(B)}\frac{1}{T}\sum_{t\in [T]}f_t(x^\star)}\enspace.
\end{align*}
In other words it is sufficient to specify a distribution $\ddd\in \Delta(\fff^{G}_{lin})\subset \Delta(\fff^{G})$ such that for \textbf{any} sequence of models $\{x_t^m\}_{t\in[T]}^{m\in[M]}$, 
$$\ee_{\{f_t\sim \ddd\}_{t\in[T]}}\sb{\frac{1}{MT}\sum_{m,t}f_t(x_t^m) - \min_{x^\star\in \bb_2(B)}\frac{1}{T}\sum_{t}f_t(x^\star)}\gtrsim \frac{GB}{\sqrt{T}}\enspace.$$
Such lower bounds are folklore in serial online convex optimization (c.f., Theorem 3.2 \citep{hazan2016introduction}). One such easy construction is choosing $f_t(x) = \inner{\beta_t}{x}$ where $\beta_t\sim \frac{G}{\sqrt{d}}\cdot Unif(\{+1, -1\}^d)$. This ensures the following,
\begin{align}
    &\ee_{\{f_t\sim \ddd\}_{t\in[T]}}\sb{\frac{1}{MT}\sum_{m,t}f_t(x_t^m) - \min_{x^\star\in \bb_2(B)}\frac{1}{T}\sum_{t}f_t(x^\star)}\nonumber\\
    &= \ee_{\cb{\beta_t\sim \frac{G}{\sqrt{d}}\cdot Unif(\{+1, -1\}^d)}_{t\in[T]}}\sb{\frac{1}{MT}\sum_{m,t}\inner{\beta_t}{x_t^m} - \min_{x^\star\in \bb_2(B)}\frac{1}{T}\sum_{t}\inner{\beta_t}{x^\star}}\enspace,\nonumber\\
    &= \ee_{\cb{\beta_t\sim \frac{G}{\sqrt{d}}\cdot Unif(\{+1, -1\}^d)}_{t\in[T]}}\sb{\frac{1}{MT}\sum_{m,t}\ee_{\beta_t\sim \frac{G}{\sqrt{d}}\cdot Unif(\{+1, -1\}^d)}\sb{\inner{\beta_t}{x_t^m}}}\nonumber\\ 
    &\qquad- \ee_{\{\beta_t\sim \frac{G}{\sqrt{d}}\cdot Unif(\{+1, -1\}^d)\}_{t\in[T]}}\sb{\min_{x^\star\in \bb_2(B)}\frac{1}{T}\sum_{t}\inner{\beta_t}{x^\star}}\enspace,\nonumber\\
    &= 0 - \frac{GB}{Td}\ee_{\{\beta_t\sim \cdot Unif(\{+1, -1\}^d)\}_{t\in[T]}}\sb{\min_{x^\star\in \bb_2(\sqrt{d})}\inner{\sum_{t}\beta_t}{x^\star}},\nonumber\\
    &\geq \frac{GB}{Td}\sum_{i\in[d]}\ee_{\{\beta_{t,i}\sim \cdot Unif(\{+1, -1\})\}_{t\in[T]}}\sb{-\min_{|x^\star_i|\leq 1}\inner{\sum_{t}\beta_{t,i}}{x^\star_i}}\enspace,\nonumber\\
    &= \frac{GB}{T}\ee_{\{u_t\sim \cdot Unif(\{+1, -1\})\}_{t\in[T]}}\sb{-\min_{|y^\star|\leq 1}\inner{\sum_{t}u_{t}}{y^\star}}\enspace,\nonumber\\
    &= \frac{GB}{T}\ee_{\{u_t\sim \cdot Unif(\{+1, -1\})\}_{t\in[T]}}\sb{\lvert\sum_{t}u_{t}\rvert}\enspace,\nonumber\\
    &\geq \frac{GB}{2\sqrt{T}} \label{eq:GB1LB}\enspace,
\end{align}
where the first inequality uses the fact that splitting across the dimensions can only hurt the minimization, thus making the overall quantity smaller, and the last inequality uses a standard result about the absolute sum of Rademacher random variables \footnote{For instance, see \href{https://en.wikipedia.org/wiki/Random_walk\#One-dimensional_random_walk}{this standard result} on single dimensional random walks.}. This finishes the lower bound proof. We have thus shown that the regret of the non-collaborative baseline is optimal, and combining bounds \eqref{eq:GB1UB} and \eqref{eq:GB1LB}, we can conclude that $$\rrr(\ppp(\fff^{G}, \zeta), \aaa_{online-IC}^1) \cong \frac{GB}{\sqrt{T}}\enspace,$$ 
where we use $\cong$ to denote equality up to numerical constants, i.e., both $\lesssim$ and $\gtrsim$ at the same time. This completes the proof.
\end{proof}

\subsection{Proof of \texorpdfstring{\Cref{thm:first_smth}}{TEXT}}
\begin{proof}
    For the upper bound, we can use the upper bound for online gradient descent in the serial setting following from a classical work on optimistic rates (c.f, Theorem 3 \cite{srebro2010optimistic}). Then we use the same lower-bounding strategy as in theorem \ref{thm:first_lip} but instead lower bound (P11), and note that $(P2)\gtrsim (P4) \gtrsim (P5) \gtrsim (P6) \gtrsim (P11)$. Focusing on (P11) ensures that $\hat\zeta=0$, as our attack is coordinated. Then to lower bound (P11), we use the construction and distribution as used in the proof of Theorem 4 by \citet{woodworth2021even}, which is a sample complexity lower bound that only depends on $T$, i.e., the number of samples observed from $\ddd$. This finishes the proof.
\end{proof}

\subsection{Implications of Theorems \ref{thm:first_lip} and \ref{thm:first_smth}: }
The above theorems imply that there is no benefit of collaboration in the worst case if the machines already have access to gradient information! This is counter-intuitive at first because several works have shown in the stochastic setting that collaboration indeed helps \citep{woodworth2020minibatch, koloskova2020unified}. 

\begin{center}
\textbf{How do we reconcile these results?} 
\end{center}

Note that while proving theorem \ref{thm:first_lip}, we crucially rely on the chain of reductions $(P2)\gtrsim (P4) \gtrsim (P5)$. Similarly, while proving theorem \ref{thm:first_smth}, we rely on the chain of reductions $(P2)\gtrsim (P4) \gtrsim (P5)\gtrsim (P6) \gtrsim (P11)$. These reductions allow us to lower-bound the min-max regret through an adversary that can use the same function on each machine. This is the main difference with respect to usual federated optimization literature, where the problems of interest are (P9) and (P10), and such coordinated attacks (making $\hat\zeta=0$) are not possible for non-degenerate distributions. This becomes clear by looking at figure \ref{fig:hierarchy}, where (P5) and (P11) are both at least as hard as (P10), and (P9) is on a different chain of reductions. This means the lower bounds in theorems \ref{thm:first_lip} and \ref{thm:first_smth} do not apply to the usual stochastic federated optimization and that there is no contradiction. Another way to view the tree is that any lower problem in the tree does not necessarily suffer from the lower bounds that apply to its parents. Thus, (P10) is not limited by the lower bound applicable to (P11).

\begin{remark}
Note from the above theorems that having a first-order heterogeneity bound $\hat \zeta$ does not help. In fact, as evident in the proof of these theorems, $\hat\zeta=0$ for problem (P4). This is unsurprising as we used a coordinated attack to give the lower bounds. However, a small $\hat\zeta$ should intuitively help in the stochastic federated settings, i.e., for problems (P9) and (P10), as it restricts the clients' distributions. Having said that, as discussed before \Cref{ass:zeta} and \Cref{ass:online_zeta} are quite different.     
\end{remark}

\section{Proofs of Zeroth-order Results}
\subsection{Proof of Theorem \ref{thm:favlb}}
In this section, we provide the proofs of Theorem \ref{thm:favlb}. We first introduce several notations, which will be used in our analysis.
Let $d(x,y)=\|x\|_2^2/2-\|\hat y\|_2^2/2-\langle y,x-\hat y\rangle$, where $\|x\|_2\leq B$ and $\hat y$ is the projected point of $y$ in to the $\ell_2$-norm ball with radius $B$. We have the following holds
\begin{align}\label{eq:dis_ineq}
    d(x,y)\geq \frac{1}{2}\|x-\hat y\|_2^2\enspace.
\end{align}
This is due to the following: if $\|y\|_2\leq B$, \eqref{eq:dis_ineq} clearly holds. If $\|y\|_2>B$, we have 
\begin{align*}
    d(x,y)-\frac{1}{2}\|x-\hat y\|_2^2=\langle x-\hat y,\hat y-y\rangle\geq \langle \hat y-\hat y,\hat y-y\rangle=0\enspace,
\end{align*}
where the inequality is due to the fact that $\hat y-y=(1-\|y\|_2/B)\hat y$ lies in the opposite direction of $\hat y$, and $x=\hat y$ will minimize the inner product. Now, we are ready to prove the regret of Algorithm \ref{alg:fed_pogd}. 
\begin{proof}
Define the following notations
\begin{align*}
    \bar x_{t}=\frac{1}{M}\sum_{m=1}^Mx_t^m,~~\bar w_t=\text{Proj}(\bar x_{t}),~~w_t^m=\text{Proj}(x_t^m)\enspace.
\end{align*}
We have
\begin{align}\label{dis_lya}
    d(x^\star,\bar x_{t+1})&=\frac{1}{2}\|x^\star\|_2^2-\frac{1}{2}\|\bar w_{t+1}\|_2^2-\langle \bar x_{t+1},x^\star-\bar w_{t+1}\rangle\nonumber\\
    &=\frac{1}{2}\|x^\star\|_2^2-\frac{1}{2}\|\bar w_{t+1}\|_2^2-\langle \bar x_{t}-\eta\frac{1}{M}\sum_{m=1}^Mg_t^m,x^\star-\bar w_{t+1}\rangle\enspace,\nonumber\\
    &=\underbrace{\frac{1}{2}\|x^\star\|_2^2-\frac{1}{2}\|\bar w_{t+1}\|_2^2-\langle \bar x_{t},x^\star-\bar w_{t+1}\rangle}_{I_1} \ \underbrace{-\eta\frac{1}{M}\sum_{m=1}^M\langle g_t^m,\bar w_{t+1}-x^\star\rangle}_{I_2}\enspace,
\end{align}
where the second equality comes from the updating rule of Algorithm \ref{alg:fed_pogd}. For the term $I_1$, we have
\begin{align*}
    I_1&=\frac{1}{2}\|x^\star\|_2^2-\frac{1}{2}\|\bar w_{t+1}\|_2^2-\langle \bar x_{t},x^\star-\bar w_{t+1}\rangle\enspace,\\
    &=\frac{1}{2}\|x^\star\|_2^2-\frac{1}{2}\|\bar w_{t}\|_2^2-\langle \bar x_{t},x^\star-\bar w_{t}\rangle-\langle \bar x_{t},\bar w_{t}-\bar w_{t+1}\rangle-\frac{1}{2}\|\bar w_{t+1}\|_2^2+\frac{1}{2}\|\bar w_{t}\|_2^2\enspace,\\
    &=d(x^\star,\bar x_t)-d(\bar w_{t+1},\bar x_t)\enspace,\\
    &\leq d(x^\star,\bar x_t)-\frac{1}{2}\|\bar w_{t+1}-\bar w_{t}\|_2^2\enspace,
\end{align*}
where the last inequality is due to \eqref{eq:dis_ineq}. For the term $I_2$, we have
\begin{align*}
    I_2&=-\eta\frac{1}{M}\sum_{m=1}^M\langle g_t^m,\bar w_{t+1}-x^\star\rangle\enspace,\\
    &=\underbrace{-\eta\frac{1}{M}\sum_{m=1}^M\langle g_t^m-\nabla f_t^m(w_t^m),\bar w_{t+1}-x^\star\rangle}_{I_{21}}\ \underbrace{-\eta\frac{1}{M}\sum_{m=1}^M\langle \nabla f_t^m(w_t^m),\bar w_{t+1}-\bar x^\star\rangle}_{I_{22}}\enspace.
\end{align*}
For the term $I_{21}$, we have
\begin{align*}
    \ee[I_{21}]&=\eta \ee\frac{1}{M}\sum_{m=1}^M\langle \nabla f_t^m(w_t^m)-g_t^m,\bar w_{t+1}-x^\star\rangle\enspace,\\
    &=\eta \ee\frac{1}{M}\sum_{m=1}^M\langle \nabla f_t^m(w_t^m)-g_t^m,\bar w_{t+1}-\bar w_t\rangle\enspace,\\
    &\leq \eta \ee\bigg\|\frac{1}{M}\sum_{m=1}^M(\nabla f_t^m(w_t^m)-g_t^m)\bigg\|_2\cdot\|\bar w_{t+1}-\bar w_t\|_2\enspace,\\
    &\leq \eta\frac{\sigma}{\sqrt{M}}\ee_t\|\bar w_{t+1}-\bar w_t\|_2\enspace,
\end{align*}
where in the last inequality, we use an arbitrary uniform upper bound on the stochastic gradient using $\sigma$, which we will eventually bound using \Cref{eq:single_feedback_variance}.
Using this we have
\begin{align*}
    \ee[I_{21}]\leq \eta\frac{\sigma}{\sqrt{M}}\ee\|\bar w_{t+1}-\bar w_t\|_2\enspace.
\end{align*}
For the term $I_{22}$, we have 
\begin{align*}
    I_{22}&=-\eta\frac{1}{M}\sum_{m=1}^M\langle \nabla f_t^m(w_t^m),w_{t}^m-\bar x^\star\rangle-\eta\frac{1}{M}\sum_{m=1}^M\langle \nabla f_t^m(w_t^m),\bar w_{t}-w_{t}^m\rangle -\eta\frac{1}{M}\sum_{m=1}^M\langle \nabla f_t^m(w_t^m),\bar w_{t+1}-\bar w_{t}\rangle\enspace,\\
    &\leq -\eta\frac{1}{M}\sum_{m=1}^M(f_t^m(w_t^m)-f_t^m(x^\star)+\eta\frac{1}{M}\sum_{m=1}^M\|\nabla f_t^m(w_t^m)\|_2\cdot\|\bar w_t-w_t^m\|_2+\eta^2
    \bigg\|\frac{1}{M}\sum_{m=1}^M\nabla f_t^m(w_t^m)\bigg\|_2^2\\
    &\qquad+\frac{1}{4}\|\bar w_{t+1}-\bar w_{t}\|_2^2\enspace.
\end{align*}
Therefore, combining \eqref{dis_lya} and the upper bound of $I_1$ and $I_2$, we have
\begin{align*}
\ee  d(x^\star,\bar x_{t+1})&\leq \ee  d(x^\star,\bar x_{t})-\frac{1}{4}\ee \|\bar w_{t+1}-\bar w_{t}\|_2^2+\eta\frac{\sigma}{\sqrt{M}}\ee\|\bar w_{t+1}-\bar w_t\|_2-\eta\frac{1}{M}\sum_{m=1}^M\ee(f_t^m(w_t^m)-f_t^m(x^\star)\\
&\qquad+\eta\frac{1}{M}\sum_{m=1}^M\ee\|\nabla f_t^m(w_t^m)\|_2\cdot\|\bar w_t-w_t^m\|_2+\eta^2
    \ee\bigg\|\frac{1}{M}\sum_{m=1}^M\nabla f_t^m(w_t^m)\bigg\|_2^2\enspace.
\end{align*}
Therefore, we have (using the same $\sigma$ as above)
\begin{align*}
    \eta\frac{1}{M}\sum_{m=1}^M\ee(f_t^m(w_t^m)-f_t^m(x^\star)&\leq \ee  d(x^\star,\bar x_{t})-\ee  d(x^\star,\bar x_{t+1})+\eta^2\frac{\sigma^2}{M}+\eta\frac{1}{M}\sum_{m=1}^M\ee\|\nabla f_t^m(w_t^m)\|_2\cdot\|\bar w_t-w_t^m\|_2\\
&\qquad+\eta^2
    \ee\bigg\|\frac{1}{M}\sum_{m=1}^M\nabla f_t^m(w_t^m)\bigg\|_2^2\enspace.
\end{align*}
In addition, we have 
\begin{align*}
    \frac{1}{M}\sum_{m=1}^M\ee\|\bar w_t-w_t^m\|_2\leq\frac{1}{M}\sum_{m=1}^M \ee\|\bar x_t-x_t^m\|_2\leq 2\eta (\sigma\sqrt{K} + \hat\zeta K)\enspace,
\end{align*}
where the last inequality is due to the linear function and follows almost the same proof as in \Cref{lem:consensus_error_second_zeta}. Thus, we can obtain (the indicator function comes from the fact that if $K=1$, there would be no consensus error)
    \begin{align*}
        \frac{1}{M}\sum_{m\in[M]}\ee\sb{f_t^m(w_t^m) - f_t^m(x^\star}
        &\leq \frac{1}{\eta}\rb{\ee  d(x^\star,\bar x_{t})-\ee  d(x^\star,\bar x_{t+1})} + \eta\rb{G^2 + \frac{\sigma^2}{M}} + \ii_{K>1}\cdot 2G(\sigma\sqrt{K} + \hat\zeta K)\eta\enspace.
    \end{align*}
Since $\ee \sb{d(x^\star,\bar x_{T})} \geq \ee \sb{\|x^\star-\bar w_T\|_2^2/2}\geq 0$ and $\ee\sb{d(x^\star,\bar x_{0})} =\|x^\star\|_2^2/2$, summing the above inequality over $t$, we can get
\begin{align*}
    \frac{1}{M}\sum_{t\in [KR], m\in[M]}\ee\sb{f_t^m(w_t^m)-f_t^m(x^\star)} \lesssim\frac{B^2}{\eta} +\eta\rb{G^2 + \frac{\sigma^2}{M} + \ii_{K>1}\cdot G(\sigma\sqrt{K} + \hat\zeta K)}T\enspace.
\end{align*}
If we choose $\eta$ such that
$$\eta = \frac{B}{G\sqrt{T}}\cdot\min\bigg\{1,\frac{G\sqrt{M}}{\sigma}, \frac{\sqrt{G}}{\ii_{K>1}\sqrt{\sigma }K^{1/4}}, \frac{\sqrt{G}}{\ii_{K>1}\sqrt{\hat\zeta K}}\bigg\}\enspace,$$
we can get
\begin{align*}
   \frac{1}{MKR}\sum_{t\in [KR], m\in[M]}\ee\sb{f_t^m(w_t^m)-f_t^m(x^\star)} \lesssim  
        \frac{GB}{\sqrt{KR}} + \frac{\sigma B}{\sqrt{MKR}} + \ii_{K>1}\cdot \rb{\frac{\sqrt{G\sigma}B}{K^{1/4}\sqrt{R}} + \frac{\sqrt{G\hat\zeta}B}{\sqrt{R}}}\enspace.
\end{align*}
To get the regret, we need to notice that we have the linear function, and thus we have: the smoothed function $\hat f=f$ and $\ee \sb{f_t^m(w_t^m)}=\ee \sb{f_t^m(w_t^m+\delta u_t^m)}$, where the expectation is over $u_t^m$. Furthermore, $$\|g_t^m\|_2^2=d^2\big(f_t^m(w_t^m+\delta u_t^m)\big)^2\leq \frac{d^2G^2(B+\delta)^2}{\delta^2}\leq 4d^2G^2\enspace,$$ where the last inequality is due the choice of $\delta=B$. Since $$\ee\sb{g_t^m}=\nabla \hat f_t^m(w_t^m) \quad\text{and}\quad \ee\sb{\|g_t^m-\nabla \hat f_t^m(w_t^m)\|_2^2}\leq \ee\sb{\|g_t^m\|_2^2}\enspace,$$ we can plug in $\sigma^2=4d^2G^2$ to get our regret, thus proving the theorem.
\end{proof}

\subsection{Proof of Theorem \ref{thm:bd_grad_first_stoch}}\label{app:bdg}
In this sub-section and the next, we consider access to a first-order stochastic oracle as an intermediate step before examining the zeroth-order oracle. Specifically, as we saw in the previous subsection, it is useful to view the zeroth-order algorithms as stochastic gradient algorithms with some bounded stochastic gradient variance $\sigma$, and then select the problem parameters to obtain an appropriate $\sigma$. We will do the same thing again; formally each machine has access to a \textbf{stochastic gradient} $g_t^m$ of $f_t^m$ at point $x_t^m$, such that it is unbiased and has bounded variance (cf. \Cref{ass:stoch_bounded_second_moment}), i.e., for all $x\in\rr^d$, $$\ee[g_t^m(x_t^m)|x_t^m]=\nabla f_t^m(x_t^m) \quad\text{ and }\quad \ee\sb{\norm{g_t^m(x_t^m) - \nabla f_t^m(x_t^m)}^2|x_t^m} \leq \sigma^2\enspace.$$ 
In Algorithm \ref{alg:fed_ogd}, we constructed a particular stochastic gradient estimator at $x_t^m$ with $\sigma^2=G^2d$. We can define the corresponding problem class $\ppp^{\fff^{G}, \hat\zeta, \sigma}$, i.e., cost functions satisfying \Cref{ass:online_convex,ass:online_bounded_gradients,ass:online_zeta} where agents have access to a stochastic first-order oracle. We have the following lemma about this problem class:
\begin{lemma}\label{lemma:bd_grad_first_stoch}
Consider the problem class $\ppp^{\fff^{G}, \hat\zeta, \sigma}$. If we choose $$\eta = \frac{B}{G\sqrt{T}}\cdot\min\bigg\{1,\frac{G\sqrt{M}}{\sigma}\frac{\sqrt{G}}{\ii_{K>1}\sqrt{\sigma K}}, \frac{1}{\ii_{K>1}\sqrt{K}}\bigg\}\enspace,$$ then the models $\{x_t^m\}_{t,m=1}^{T,M}$ of Algorithm \ref{alg:fed_ogd} satisfy the following guarantee:
    $$\frac{1}{MKR}\sum_{t\in [KR], m\in[M]}\ee\sb{f_t^m(x_t^m)-f_t^m(x^\star)} \lesssim  
        \frac{GB}{\sqrt{KR}} + \frac{\sigma B}{\sqrt{MKR}} + \ii_{K>1}\cdot \rb{\frac{\sqrt{\sigma G}B}{\sqrt{R}} + \frac{GB}{\sqrt{R}}}\enspace,$$
    where $x^\star \in \arg\min_{x\in \rr^d}\sum_{t\in [KR]}f_t(x)$, and the expectation is w.r.t. the stochastic gradients.    
\end{lemma}

\begin{remark}
Note that when $K=1$, the upper bound in Lemma \ref{lemma:bd_grad_first_stoch} reduces to the first two terms, both of which are known to be optimal due to lower bounds in the stochastic setting, i.e., against a stochastic online adversary \cite{nemirovski1994efficient, hazan2016introduction}. We now use this lemma to guarantee bandit two-point feedback oracles for the same function class. We recall that one can obtain a stochastic gradient for a \textit{``smoothed-version"} $\hat{f}$ of a Lipschitz function $f$ at any point $x\in\xxx$, using two function value calls to $f$ around the point $x$ \cite{shamir2017optimal, duchi2015optimal}.     
\end{remark}
With this lemma, we can prove Theorem \ref{thm:bd_grad_first_stoch}.
\begin{proof}[Proof of Theorem \ref{thm:bd_grad_first_stoch}]
    First, we consider smoothed functions $$\hat{f}_t^m(x) := \ee_{u\sim Unif(S_{d-1})}[f_t^m(x + \delta u)],$$ for some $\delta>0$ and $S_{d-1}$ denoting the euclidean unit sphere. Based on the gradient estimator proposed by \citet{shamir2017optimal} (which can be implemented with two-point bandit feedback) and Lemma \ref{lemma:bd_grad_first_stoch}, we can get the following regret guarantee (noting that $\sigma \leq  c_1\sqrt{d}G$ for a numerical constant $c_1$, c.f., \cite{shamir2017optimal}): 
    $$\ee\sb{\frac{1}{MKR}\sum_{t\in[KR], m\in[M]}\hat{f}_t^m(\hat{x}_t^m)} - \frac{1}{MKR}\sum_{t\in[KR], m\in[M]}\hat{f}_t^m(x^\star) \lesssim \frac{GB}{\sqrt{KR}} + \frac{GB\sqrt{d}}{\sqrt{MKR}} + \ii_{K>1}\cdot\frac{GBd^{1/4}}{\sqrt{R}}\enspace,$$
    where the expectation is with respect to the stochasticity in the stochastic gradient estimator. To transform this into a regret guarantee for $f$ we need to account for two things: 
    \begin{enumerate}
        \item The difference between the smoothed function $\hat{f}$ and the original function $f$. This is easy to handle because both these functions are pointwise close, i.e., $\sup_{x \in \xxx}|f(x)-\hat{f}(x)|\leq G\delta$.
        \item The difference between the points $\hat{x}_t^m$ at which the stochastic gradient is computed for $\hat{f}_t^m$ and the actual points $x_t^{m,1}$ and $x_t^{m,2}$ on which we incur regret while making zeroth-order queries to $f_t^m$. This is also easy to handle because due to the definition of the estimator, $x_t^{m,1}, x_t^{m,1}\in B_\delta(\hat{x}_t^m)$, where $B_\delta(x)$ is the $L_2$ ball of radius $\delta$ around $x$.
    \end{enumerate}
    In light of the last two observations, the average regret between the smoothed and original functions only differs by a factor of $2G\delta$, i.e.,
    \begin{align*}
      \ee\sb{\frac{1}{2MKR}\sum_{t\in[KR], m\in[M], j\in[2]}f_t^m(x_t^{m,j})} - &\frac{1}{MKR}\sum_{t\in[KR], m\in[M]}f_t^m(x^\star)\\
      &\lesssim G\delta + \frac{GB}{\sqrt{KR}} + \frac{GB\sqrt{d}}{\sqrt{MKR}} + \ii_{K>1}\cdot\frac{GBd^{1/4}}{\sqrt{R}}\enspace,\\
      &\lesssim \frac{GB}{\sqrt{KR}} + \frac{GB\sqrt{d}}{\sqrt{MKR}} + \ii_{K>1}\cdot\frac{GBd^{1/4}}{\sqrt{R}}\enspace,
    \end{align*}
    where the last inequality is due to the choice of $\delta$ such that $\delta \lesssim \frac{Bd^{1/4}}{\sqrt{R}}\rb{1+ \frac{d^{1/4}}{\sqrt{MK}}}$.
\end{proof}

\subsection{Proof of Theorem \ref{thm:smooth_first_stoch}}\label{app:smth}
Similar to before, we start by looking at $\ppp^{\fff^{G,H}, \hat\zeta, \sigma, F_\star}$, i.e., cost functions satisfying \Cref{ass:online_convex,ass:online_bounded_gradients,ass:online_smooth_second,ass:online_zeta,ass:online_bounded_optimal}. We have the following lemma.
\begin{lemma}\label{lem:smooth_first_stoch}
Consider the problem class $\ppp^{\fff^{G,H}, \hat\zeta, \sigma, F_\star}$.  The models $\{x_t^{m}\}_{t,m=1}^{T,M}$ of Algorithm \ref{alg:fed_ogd} with appropriate $\eta$ (specified in the proof) satisfy the following regret guarantee (for a numerical constant $c$):
    \begin{align*}
      \frac{1}{MKR}&\sum_{t\in [KR], m\in[M]}\ee\sb{f_t^m(x_t^m)-f_t^m(x^\star)} \leq c\cdot \Bigg(\frac{HB^2}{KR} + \frac{\sigma B}{\sqrt{MKR}} + \min\cb{\frac{GB}{\sqrt{KR}}, \frac{\sqrt{HF_\star}B}{\sqrt{KR}}},\\
        & +\ii_{K>1}\cdot\min\Bigg\{\frac{H^{1/3}B^{4/3}\sigma^{2/3}}{K^{1/3}R^{2/3}} + \frac{H^{1/3}B^{4/3}\hat\zeta^{2/3}}{R^{2/3}} + \frac{\sqrt{\hat\zeta\sigma}B}{K^{1/4}\sqrt{R}} + \frac{\hat\zeta B}{\sqrt{R}},\frac{\sqrt{G\sigma}B}{K^{1/4}\sqrt{R}} + \frac{\sqrt{G\hat\zeta}B}{\sqrt{R}}\Bigg\}\bigg)\enspace,  
    \end{align*}
    where $x^\star \in \arg\min_{x\in \bb_2(B)}\sum_{t\in [KR]}f_t(x)$, and the expectation is w.r.t. the stochastic gradients. The models also satisfy the guarantee of Lemma \ref{lemma:bd_grad_first_stoch} with the same step-size.    
\end{lemma}

\begin{proof}[Proof of Theorem \ref{thm:smooth_first_stoch}]
Given Lemma \ref{lem:smooth_first_stoch}, it is now straightforward to prove Theorem \ref{thm:smooth_first_stoch} similar to the proof for Theorem \ref{thm:bd_grad_first_stoch} by replacing $\sigma^2$ with $G^2d$ ad choosing small enough $\delta$ such that $G\delta \cong$ the r.h.s. of the theorem statement.
\end{proof}
Our main job in the remainder of this appendix is to prove the two \Cref{lemma:bd_grad_first_stoch,lem:smooth_first_stoch} which abstract away the zeroth-order access using stochastic gradients. 

\subsection{Proof of Lemma \ref{lemma:bd_grad_first_stoch}}
In this section, we prove Lemma \ref{lemma:bd_grad_first_stoch}.
\begin{proof}[Proof of Lemma \ref{lemma:bd_grad_first_stoch}]
    Consider any time step $t\in[KR]$ and define ghost iterate $\bar{x}_t = \frac{1}{M}\sum_{m\in[M]}x_t^m$ (which not might actually get computed). If $K=1$, the machines calculate the stochastic gradient at the same point, $\bar{x}_t$. Then using the update rule of Algorithm \ref{alg:fed_ogd}, we can get the following:
    \begin{align*}
        \ee_t\sb{\norm{\bar{x}_{t+1}-x^\star}^2} &= \ee_t\sb{\norm{\bar{x}_t - \frac{\eta_t}{M}\sum_{m\in[M]}\nabla f_t^m(x_t^m) - x^\star + \frac{\eta_t}{M}\sum_{m=1}^{M}\rb{\nabla f_t^m(x_t^m) - g_t^m(x_t^m)}}^2}\enspace,\\
        &= \norm{\bar{x}_{t}-x^\star}^2 + \frac{\eta_t^2}{M^2}\norm{\sum_{m\in[M]}\nabla f_t^m(x_t^m)}^2 - \frac{2\eta_t}{M}\sum_{m\in[M]}\inner{\bar{x}_t-x^\star}{\nabla f_t^m(x_t^m)} + \frac{\eta_t^2\sigma^2}{M}\enspace,\\
        &= \norm{\bar{x}_{t}-x^\star}^2 + \frac{\eta_t^2}{M^2}\norm{\sum_{m\in[M]}\nabla f_t^m(x_t^m)}^2 - \frac{2\eta_t}{M}\sum_{m\in[M]}\inner{x^m_t-x^\star}{\nabla f_t^m(x_t^m)}\\
        &\quad + \ii_{K>1}\cdot\frac{2\eta_t}{M}\sum_{m\in[M]}\inner{x^m_t-\bar{x}_t}{\nabla f_t^m(x_t^m)} + \frac{\eta_t^2\sigma^2}{M}\enspace,\nonumber\\
        &\leq \norm{\bar{x}_{t}-x^\star}^2 + \frac{\eta_t^2}{M^2}\norm{\sum_{m\in[M]}\nabla f_t^m(x_t^m)}^2 - \frac{2\eta_t}{M}\sum_{m\in[M]}\rb{f_t^m(x_t^m) - f_t^m(x^\star)} \\
        &\quad + \ii_{K>1}\cdot\frac{2\eta_t}{M}\sum_{m\in[M]}\inner{x^m_t-\bar{x}_t}{\nabla f_t^m(x_t^m)} + \frac{\eta_t^2\sigma^2}{M}\enspace,\nonumber
    \end{align*}
    where $\ee_t$ is the expectation conditioned on the filtration at time $t$ under which $x_t^m$'s are measurable, and the last inequality is due to the convexity of each function. Re-arranging this leads to 
    \begin{align}
        \frac{1}{M}\sum_{m\in[M]}\rb{f_t^m(x_t^m) - f_t^m(x^\star)} &\leq \frac{1}{2\eta_t}\rb{\norm{\bar{x}_{t}-x^\star}^2 - \ee_t\sb{\norm{\bar{x}_{t+1}-x^\star}^2}} +  \frac{\eta_t}{2 M^2}\norm{\sum_{m\in[M]}\nabla f_t^m(x_t^m)}^2 \nonumber\\
        &\quad + \ii_{K>1}\cdot\frac{1}{M}\sum_{m\in[M]}\ee_t\inner{x^m_t-\bar{x}_t}{\nabla f_t^m(x_t^m)} + \frac{\eta_t\sigma^2}{2M}\enspace,\nonumber\\
        &\leq \frac{1}{2\eta_t}\rb{\norm{\bar{x}_{t}-x^\star}^2 - \ee_t\sb{\norm{\bar{x}_{t+1}-x^\star}^2}} + \frac{\eta_t}{2}\rb{G^2 + \frac{\sigma^2}{M}}\nonumber\\
        &\quad + \ii_{K>1}\cdot\frac{G}{M}\sum_{m\in[M]}\ee\sb{\norm{x_t^m-\bar{x}_t}}\label{eq:inter1}\enspace.
    \end{align}
    The last inequality comes from each function's $G$-Lipschitzness.
    For the last term in \eqref{eq:inter1}, we can upper bound it almost identically in the same way as in \Cref{lem:consensus_error_second_zeta} (noting that $\hat\zeta\leq 2G$) to get that
    \begin{align}\label{eq:concensus}
        \frac{1}{M}\sum_{m\in[M]}\ee\sb{\norm{x_t^m-\bar{x}_t}} & \leq 2(\sigma+G)K\eta\enspace.
    \end{align}
     Plugging \eqref{eq:concensus} into \eqref{eq:inter1} and choosing a constant step-size $\eta$, and taking full expectation we get
    \begin{align*}
        \frac{1}{M}\sum_{m\in[M]}\ee\sb{f_t^m(x_t^m) - f_t^m(x^\star)}
        &\leq \frac{1}{2\eta}\rb{\norm{\ee\sb{\bar{x}_{t}-x^\star}^2} - \ee\sb{\norm{\bar{x}_{t+1}-x^\star}^2}}+ \frac{\eta}{2}\rb{G^2 + \frac{\sigma^2}{M}}\\
        &\qquad + \ii_{K>1}\cdot 2G(\sigma + G)K\eta\enspace.
    \end{align*}
    Summing this over time $t\in[KR]$ we get,
    \begin{align*}
        \frac{1}{M}\sum_{m\in[M], t\in[T]}\ee\sb{f_t^m(x_t^m) - f_t^m(x^\star)} &\lesssim \frac{\norm{\bar{x}_{0}-x^\star}^2}{\eta} + \eta\rb{G^2 + \frac{\sigma^2}{M} + \ii_{K>1}\cdot\sigma G K + \ii_{K>1}\cdot\zeta G K}T\enspace,\\
        &\lesssim \frac{B^2}{\eta} + \eta\rb{G^2 + \frac{\sigma^2}{M} + \ii_{K>1}\cdot\sigma G K + \ii_{K>1}\cdot G^2K}T\enspace.
    \end{align*}
    Finally choosing, $$\eta = \frac{B}{G\sqrt{T}}\cdot\min\bigg\{1,\frac{G\sqrt{M}}{\sigma}, \frac{\sqrt{G}}{\ii_{K>1}\sqrt{\sigma K}}, \frac{1}{\ii_{K>1}\sqrt{K}}\bigg\}\enspace,$$
    we can obtain,
    \begin{align}
        \frac{1}{M}\sum_{m\in[M], t\in[T]}\ee\sb{f_t^m(x_t^m) - f_t^m(x^\star)} &\lesssim GB\sqrt{T} + \ii_{K>1}\cdot\sqrt{\sigma G}B\sqrt{KT} + \ii_{K>1}\cdot GB\sqrt{KT} + \frac{\sigma B\sqrt{T}}{\sqrt{M}}\enspace.
    \end{align}
    Dividing by $KR$ finishes the proof.
\end{proof}

\subsection{Proof of Lemma \ref{lem:smooth_first_stoch}}
In this section, we provide the proof for Lemma \ref{lem:smooth_first_stoch} following a very similar analysis as the one due to \citet{woodworth2020minibatch} for the stochastic setting.
\begin{proof}[Proof of Lemma \ref{lem:smooth_first_stoch}]
    Consider any time step $t\in[KR]$ and define ghost iterate $\bar{x}_t = \frac{1}{M}\sum_{m\in[M]}x_t^m$ (which not might actually get computed). Then using the update rule of Algorithm \ref{alg:fed_ogd}, we can get:
    \begin{align*}
        \ee_t\sb{\norm{\bar{x}_{t+1}-x^\star}^2} &= \ee_t\sb{\norm{\bar{x}_t - \frac{\eta_t}{M}\sum_{m\in[M]}\nabla f_t^m(x_t^m) - x^\star + \frac{\eta_t}{M}\sum_{m=1}^{M}\rb{\nabla f_t^m(x_t^m) - g_t^m(x_t^m)}}^2},\\
        &= \norm{\bar{x}_{t}-x^\star}^2 + \frac{\eta_t^2}{M^2}\norm{\sum_{m\in[M]}\nabla f_t^m(x_t^m)}^2 - \frac{2\eta_t}{M}\sum_{m\in[M]}\inner{\bar{x}_t-x^\star}{\nabla f_t^m(x_t^m)} + \frac{\eta_t^2\sigma^2}{M}\\
        &= \norm{\bar{x}_{t}-x^\star}^2 + \frac{\eta_t^2}{M^2}\norm{\sum_{m\in[M]}\nabla f_t^m(x_t^m)}^2 - \frac{2\eta_t}{M}\sum_{m\in[M]}\inner{x^m_t-x^\star}{\nabla f_t^m(x_t^m)}\\
        &\quad + \ii_{K>1}\cdot\frac{2\eta_t}{M}\sum_{m\in[M]}\inner{x^m_t-\bar{x}_t}{\nabla f_t^m(x_t^m)} + \frac{\eta_t^2\sigma^2}{M}\nonumber\\
        &\leq \norm{\bar{x}_{t}-x^\star}^2 + \frac{\eta_t^2}{M^2}\norm{\sum_{m\in[M]}\nabla f_t^m(x_t^m)}^2 - \frac{2\eta_t}{M}\sum_{m\in[M]}\rb{f_t^m(x_t^m) - f_t^m(x^\star)} \\
        &\quad + \ii_{K>1}\cdot\frac{2\eta_t}{M}\sum_{m\in[M]}\inner{x^m_t-\bar{x}_t}{\nabla f_t^m(x_t^m)} + \frac{\eta_t^2\sigma^2}{M}\nonumber,
    \end{align*}
    where $\ee_t$ is the expectation taken with respect to the filtration at time $t$, and the last line comes from the convexity of each function. Re-arranging this and taking expectation gives 
    leads to
    \begin{align}
        \frac{1}{M}\sum_{m\in[M]}\ee\rb{f_t^m(x_t^m) - f_t^m(x^\star)} &\leq \frac{1}{2\eta_t}\rb{\ee\norm{\bar{x}_{t}-x^\star}^2 - \ee\sb{\norm{\bar{x}_{t+1}-x^\star}^2}} +  \textcolor{blue}{\frac{\eta_t}{2 M^2}\ee\norm{\sum_{m\in[M]}\nabla f_t^m(x_t^m)}^2}\nonumber\\
        &\quad + \ii_{K>1}\cdot\textcolor{red}{\frac{1}{M}\sum_{m\in[M]}\ee\inner{x^m_t-\bar{x}_t}{\nabla f_t^m(x_t^m)}} + \frac{\eta_t\sigma^2}{2M}\enspace.\label{eq:smth_gen_recursion}
    \end{align}
    \paragraph{Bounding the blue term.} We consider two different ways to bound the term. First note that similar to \Cref{lem:consensus_error_second_zeta} we can just use the following bound,
    \begin{align}
        \textcolor{blue}{\frac{\eta_t}{2 M^2}\ee\norm{\sum_{m\in[M]}\nabla f_t^m(x_t^m)}^2} &\leq \frac{\eta_t G^2}{2}\enspace. \label{eq:blue_bd_one}
    \end{align}
    However, since we also have smoothness, we can use the self-bounding property (c.f., Lemma 4.1 \cite{srebro2010optimistic}) to get,
    \begin{align}
        \textcolor{blue}{\frac{\eta_t}{2 M^2}\ee\norm{\sum_{m\in[M]}\nabla f_t^m(x_t^m)}^2} &\leq \frac{\eta_t H}{2M}\sum_{m\in[M]}\rb{f_t^m(x_t^m) - f_t^m(x_t^\star)} + \frac{\eta_t H}{2M} \sum_{m\in[M]}f_t^m(x_t^\star)\enspace,\nonumber\\
        &\leq \frac{\eta_t H}{2M} \sum_{m\in[M]}f_t^m(x^\star)\enspace,\label{eq:blue_bd_two}
    \end{align}
    where $x_t^\star$ is the optimizer of $\frac{1}{M}\sum_{m\in[M]}f_t^m(x)$.

    \paragraph{Bounding the red term.} We will bound the term in three different ways. Similar to Lemma \ref{lemma:bd_grad_first_stoch}, we can bound the term after taking expectation and then bounding the consensus term similar to \Cref{lem:consensus_error_second_zeta} (recalling that $\hat\zeta\leq G$) as follows,
    \begin{align}\label{eq:red_bd_one}
        \textcolor{red}{\frac{1}{M}\sum_{m\in[M]}\ee\sb{\inner{x^m_t-\bar{x}_t}{\nabla f_t^m(x_t^m)}}} &\leq \frac{G}{M}\sum_{m\in[M]}\ee\sb{\norm{x_t^m-\bar{x}_t}}\enspace,\nonumber\\
        &\leq 2G(\sigma + G)\sum_{t'=\delta(t)}^{\delta(t)+K - 1}\eta_{t'}\enspace,
    \end{align}
    where $\delta(t)$ maps $t$ to the last time on or before time $t$ when communication happened. Alternatively, we can use smoothness as follows after assuming $\eta_t\leq 1/2H$,
    \begin{align}\label{eq:red_bd_two}
        &\textcolor{red}{\frac{1}{M}\sum_{m\in[M]}\ee\sb{\inner{x^m_t-\bar{x}_t}{\nabla f_t^m(x_t^m)}}}\\
        &\quad= \frac{1}{M}\sum_{m\in[M]}\ee\sb{\inner{x^m_t-\bar{x}_t}{\nabla f_t^m(x_t^m) - \nabla f_t(\bar{x}_t)}}\enspace,\nonumber\\
        &\quad\leq \sqrt{\frac{1}{M}\sum_{m\in[M]}\ee\norm{x^m_t-\bar{x}_t}^2}\sqrt{\frac{1}{M}\sum_{m\in[M]}\ee\norm{\nabla f_t^m(x_t^m) - \nabla f_t(\bar{x}_t)}^2}\enspace,\nonumber\\
        &\quad\leq^{\text{(\Cref{lem:mod_am_gm,ass:online_zeta})}} \sqrt{\frac{1}{M}\sum_{m\in[M]}\ee\norm{x^m_t-\bar{x}_t}^2}\sqrt{\frac{2}{M}\sum_{m\in[M]}H^2\ee\norm{x^m_t-\bar{x}_t}^2 + 2\hat\zeta^2}\enspace,\nonumber\\
        &\quad\leq^{(\text{a})} \frac{2H}{M}\sum_{m\in[M]}\ee\norm{x^m_t-\bar{x}_t}^2 + 2\hat\zeta\sqrt{\frac{1}{M}\sum_{m\in[M]}\ee\norm{x^m_t-\bar{x}_t}^2},\nonumber\\
        &\quad\lesssim^{(\text{b})} 2\eta_t^2H(\sigma^2 K + \zeta^2K^2) + 2\eta_t\hat\zeta(\sigma\sqrt{K} + \zeta K)\enspace,
    \end{align}
    where in (a) we used triangle inequality, and in (b) we used a similar upper bound as in \Cref{lem:consensus_error_second_zeta}. We can also use the lipschitzness and smoothness assumption together with a constant step size $\eta<1/2H$ to obtain,
    \begin{align}\label{eq:red_bd_three}
        \textcolor{red}{\frac{1}{M}\sum_{m\in[M]}\ee\sb{\inner{x^m_t-\bar{x}_t}{\nabla f_t^m(x_t^m)}}} &\leq \frac{G}{M}\sum_{m\in[M]}\ee\sb{\norm{x_t^m-\bar{x}_t}}\enspace,\nonumber\\
        &\lesssim \eta G(\sigma\sqrt{K} + \hat\zeta K)\enspace.
    \end{align}
    \paragraph{Combining everything.} After using a constant step-size $\eta$, summing \eqref{eq:smth_gen_recursion} over time, we can use the upper bound of the red and blue terms in different ways. If we plug in \eqref{eq:blue_bd_one} and \eqref{eq:red_bd_one} we recover the guarantee of lemma \ref{lemma:bd_grad_first_stoch}. This is not surprising because $\fff^{G,H,B}\subseteq \fff^{G,B}$. Combining the upper bounds in all other combinations, assuming $\eta<\frac{1}{2H}$, we can show the following upper bound
    \begin{align*}
        \frac{Reg(M,K,R)}{KR} &\lesssim \frac{HB^2}{KR} + \frac{\sigma B}{\sqrt{MKR}} + \min\cb{\frac{GB}{\sqrt{KR}}, \frac{\sqrt{HF_\star}B}{\sqrt{KR}}},\\
        &\quad +\ii_{K>1}\min\cb{\frac{H^{1/3}B^{4/3}\sigma^{2/3}}{K^{1/3}R^{2/3}} + \frac{H^{1/3}B^{4/3}\hat\zeta^{2/3}}{R^{2/3}} + \frac{\sqrt{\hat\zeta\sigma}B}{K^{1/4}\sqrt{R}} + \frac{\hat\zeta B}{\sqrt{R}}, \frac{\sqrt{G\sigma}B}{K^{1/4}\sqrt{R}} + \frac{\sqrt{G\hat\zeta}B}{\sqrt{R}}}\enspace,
    \end{align*}
    where we used the step size,
    \begin{align*}
        \eta &= \min\Bigg\{\frac{1}{2H}, \frac{B\sqrt{M}}{\sigma \sqrt{KR}}, \max\cb{\frac{B}{G\sqrt{KR}}, \frac{B}{\sqrt{HF_\star KR}}},\\ 
        &\qquad\frac{1}{\ii_{K>1}}\cdot\max\Bigg\{\min\cb{\frac{B^{2/3}}{H^{1/3}\sigma^{2/3}K^{2/3}R^{1/3}}, \frac{B^{2/3}}{H^{1/3}\hat\zeta^{2/3}KR^{1/3}}, \frac{B}{K^{3/4}\sqrt{\hat\zeta \sigma R}}, \frac{B}{\hat\zeta K\sqrt{R}}},\\ 
        &\qquad\min\cb{\frac{B}{K^{3/4}\sqrt{G \sigma R}}, \frac{B}{K\sqrt{\hat\zeta GR}}}\Bigg\}\Bigg\}\enspace.
    \end{align*}
    This finishes the proof.
\end{proof}

\subsubsection{Modifying the Proof for Federated Adversarial Linear Bandits}\label{sec:modify}
To prove the guarantee for the adversarial linear bandits, we first note that the self-bounding property can't be used anymore as the functions are not non-negative. Thus, we proceed with the lemma's proof with the following changes:
\begin{itemize}
    \item We don't prove the additional upper bound in \eqref{eq:blue_bd_two} for blue term.
    \item While upper bounding the red term in \eqref{eq:red_bd_two}, we set $H=0$ and use this single bound for the red term.
\end{itemize}
After making these changes, combining all the terms, and tuning the learning rate, we recover the correct lemma for federated adversarial linear bandits.

\section{Lower bound for Two-point Feedback}\label{app:two_pt_lb}
We want to prove a lower bound when the problem instance $\ppp$ satisfies \Cref{ass:online_convex,ass:online_bounded_gradients,ass:online_zeta} and we have an algorithm with two-point bandit feedback, i.e., in the class $\aaa_{online-IC}^{0,2}$. In particular, we want to show that
\begin{align}
\rrr(\ppp, \aaa^{0,2}_{online-IC}) = \Omega\rb{ \frac{GB}{\sqrt{KR}} + \frac{GB\sqrt{d}}{\sqrt{MKR}}}\enspace.
\end{align}
To prove this, we'd use the reduction $(P2)\gtrsim (P4) \gtrsim (P5) \gtrsim (P6)$ (cf. \Cref{sec:related}). Then we note for the problem (P6), $\hat\zeta=0$, and using $2$-point feedback, we get in total $2MKR$ function value accesses to $\ddd$. We can then use the lower bound in Proposition 2 by \citet{duchi2015optimal} for the problem (P6) for $2M$ points of feedback and $KR$ iterations. Combined with the order-independent lower bound, which we prove using problem (P5) in Theorem \ref{thm:first_lip}, this proves the required result.

\bibliographystyle{plainnat}
\bibliography{references}
\end{document}